\documentclass{article} 

\usepackage{iclr2026_conference,times}

\usepackage[utf8]{inputenc} 
\usepackage[T1]{fontenc}    
\usepackage{hyperref}       
\usepackage{url}            
\usepackage{booktabs}       
\usepackage{amsfonts}       
\usepackage{nicefrac}       
\usepackage{microtype}      

\usepackage{amsmath}
\usepackage{amssymb}
\usepackage{mathtools}
\usepackage{amsfonts}
\usepackage{amsthm}
\usepackage{dsfont}

\usepackage{multirow}
\usepackage{array}
\usepackage{makecell}
\usepackage[table,dvipsnames]{xcolor}
\usepackage{wrapfig}
\usepackage{colortbl}

\usepackage{graphicx}
\usepackage{graphbox}

\usepackage[inline]{enumitem}
\newlist{todolist}{itemize}{2}
\setlist[todolist]{label=$\square$}
\usepackage{pifont}
\newcommand{\cmark}{\textcolor{Green}{\ding{51}}}%
\newcommand{\xmark}{\textcolor{red}{\ding{55}}}%
\newcommand{\xmarktimes}{\textcolor{red}{\ding{53}}}%

\usepackage{algorithm}
\usepackage{algpseudocode}

\usepackage{xcolor}

\usepackage{minitoc}

\DeclareMathOperator*{\E}{\mathbb{E}}

\newtheorem{theorem}{Theorem}[section]
\newtheorem{lemma}[theorem]{Lemma}
\theoremstyle{definition}

\newcommand\numberthis{\addtocounter{equation}{1}\tag{\theequation}}

\title{NatADiff: Adversarial Boundary Guidance for Natural Adversarial Diffusion}

\author{Max Collins\\
    School of Physics, Maths and Computing\\
    The University of Western Australia\\
    Perth, WA 6009\\
    \texttt{max.collins@research.uwa.edu.au}\\
    \And
    Jordan Vice\\
    School of Physics, Maths and Computing\\
    The University of Western Australia\\
    Perth, WA 6009\\
    \texttt{jordan.vice@uwa.edu.au}\\
    \And
    Tim French\\
    School of Physics, Maths and Computing\\
    The University of Western Australia\\
    Perth, WA 6009\\
    \texttt{tim.french@uwa.edu.au}\\
    \And
    Ajmal Mian\\
    School of Physics, Maths and Computing\\
    The University of Western Australia\\
    Perth, WA 6009\\
    \texttt{ajmal.mian@uwa.edu.au}\\
}

\iclrfinalcopy 
\begin{document}
\doparttoc 
\faketableofcontents 

\part{} 

\maketitle

\begin{abstract}
    Adversarial samples exploit irregularities in the manifold ``learned'' by deep learning models to cause misclassifications. The study of these adversarial samples provides insight into the features a model uses to classify inputs, which can be leveraged to improve robustness against future attacks. However, much of the existing literature focuses on constrained adversarial samples, which do not accurately reflect test-time errors encountered in real-world settings. To address this, we propose `NatADiff', an adversarial sampling scheme that leverages denoising diffusion to generate natural adversarial samples. Our approach is based on the observation that natural adversarial samples frequently contain structural elements from the adversarial class. Deep learning models can exploit these structural elements to shortcut the classification process, rather than learning to genuinely distinguish between classes. To leverage this behavior, we guide the diffusion trajectory towards the intersection of the true and adversarial classes, combining time-travel sampling with augmented classifier guidance to enhance attack transferability while preserving image quality. Our method achieves comparable white-box attack success rates to current state-of-the-art techniques, while exhibiting significantly higher transferability across model architectures and improved alignment with natural test-time errors as measured by FID. These results demonstrate that NatADiff produces adversarial samples that not only transfer more effectively across models, but more faithfully resemble naturally occurring test-time errors when compared with other generative adversarial sampling schemes.
\end{abstract}

\section{Introduction}
Deep learning models can react unpredictably when there is domain difference between training and test data \citep{Szegedy2014,Goodfellow2015}. \textit{Constrained} adversarial attacks exploit this vulnerability, adding visually imperceptible pixel-level perturbations to deliberately \textit{fool} models into misclassification \citep{Szegedy2014, Goodfellow2015, Madry2019, Croce2020}. More recently, \textit{unconstrained} adversarial attacks have been proposed which allow for unrestricted perturbation magnitudes, provided the resulting adversarial image lies sufficiently close to the natural image manifold \citep{Song2018, AdvDiffuser_Chen2023, ContentDiffusionAttack_Chen2023}. 

Defences to these attacks have been proposed \citep{Szegedy2014, Madry2019, Gu2015, Xu2018, Samangouei2018, Nie2022}; however, they largely target attacks formed by adding perturbations to natural images--overlooking the existence of \textit{natural} adversarial samples. Natural adversarial samples are more commonly known as test-time errors, and they represent the strongest class of unconstrained adversarial attack, as they are valid (perturbation-free and naturally occurring) model inputs that are erroneously classified \citep{Hendrycks2021}. The absence of an adversarial perturbation renders many defensive measures ineffective \citep{DetectingNaturalvsArtificialAdvSamples_Agarwal2022}. Furthermore, natural adversarial samples have been widely studied in the literature, and they have been found to exhibit high transferability--where multiple classifiers incorrectly classify the same sample \citep{Hendrycks2021}. It is hypothesized that this is caused by classifiers independently learning to rely on the same erroneous contextual cues to shortcut classification and reduce training losses without generalising to the underlying task \citep{Hendrycks2021, ShortcutLearning_Geirhos2020, RiskMinimisation_Arjovsky2020}.

Generating natural adversarial samples offers an opportunity to better understand the mechanisms underpinning test-time errors. Prior work has sought to achieve this by using classifier gradients to perturb the sampling process of generative adversarial networks (GANs) and denoising diffusion models \citep{Song2018, ContentDiffusionAttack_Chen2023, Dai2024, DiffAttack_Chen2025}. However, GAN-based approaches lack theoretical justification for perturbing the sample path, and doing so often degrades image quality \citep{StyleGAN_Karras2019, Abdal2019}. Alternatively, directly injecting classifier gradients into the diffusion sampling trajectory can result in generating constrained adversarial samples \citep{Vaeth2024, Shen2024} (see Figure~\ref{fig:CurrentAdversarial}~(c)). Moreover, existing methods do not account for the link between learned erroneous contextual cues and test-time errors.

We propose \textit{NatADiff}, a highly transferable, diffusion-based \citep{DenoisingDiffusion_Ho2020, Song2022} adversarial sample generation method. NatADiff leverages the link between contextual cues and test-time errors by guiding the diffusion sampling trajectory towards the intersection of the adversarial and true classes, a technique we define as ``adversarial boundary guidance''. Additionally, we incorporate classifier augmentations to reduce the strength of the constrained adversarial perturbation and to further guide the sampling trajectory towards regions of the image manifold that incorporate features from the adversarial class. We find that NatADiff-generated samples achieve comparable white-box (same target and victim classifier) attack success rates to current state-of-the-art adversarial attacks, while exhibiting significantly higher transferability (different target and victim classifier) across models. Furthermore, samples generated using NatADiff align more closely with known test-time errors (with respect to their Fr\'echet inception distance (FID) \citep{Frechet1957}) than those generated through adversarial classifier guidance alone \citep{Dai2024}. These results demonstrate that NatADiff produces adversarial samples that not only transfer more effectively across models, but more faithfully resemble naturally occurring test-time errors. To summarize our contributions: 
\begin{enumerate*}[label=(\roman*)]
    \item We propose NatADiff, incorporating classifier transformations, gradient normalization, and time-travel sampling \citep{Shen2024, Lugmayr2022, FreeDoM_Yu2023} to improve adversarial classifier guidance and image quality;
    \item We design an adversarial boundary guidance algorithm to reliably navigate the complex, learned manifold, allowing us to generate natural adversarial samples with significantly higher transferability than existing approaches.
    \item We explore how convolution and transformer based classifiers perceive natural adversarial samples, exposing interesting properties of the feature representations learned by deep learning models.
\end{enumerate*}

\section{Definitions and Preliminaries} \label{sec:Definitions and Preliminaries}
\noindent\textbf{Constrained, unconstrained, and natural adversarial samples.} Broadly speaking there are three categories of adversarial sample: unconstrained, constrained, and natural. Let $f: \mathcal{I}_\mathcal{U} \rightarrow \mathcal{Y}$ be a trained image classifier, $\mathcal{I}_{\mathcal{U}}$ be the set of allowable image inputs, $\mathcal{I}_{\mathcal{N}} \subseteq \mathcal{I}_{\mathcal{U}}$ be the set of natural images, $\mathcal{Y}$ be the set of image classification labels, and $\mathcal{O}: \mathcal{I}_\mathcal{U} \rightarrow \mathcal{Y}$ be an oracle (``perfect'' human) classifier. Unconstrained adversarial samples require only that the image is misclassified: $\mathcal{A}_U \triangleq \{x \in \mathcal{I}_\mathcal{U} : f(x) \neq \mathcal{O}(x) \}$ \citep{Song2018}. Constrained adversarial samples are restricted to an $\epsilon$-neighbourhood about some natural image: $\mathcal{A}_C \triangleq \{x + \delta \in \mathcal{I}_\mathcal{U} : x \in \mathcal{I}_\mathcal{N}, \lVert \delta \rVert_p \leq \epsilon, f(x + \delta) \neq \mathcal{O}(x + \delta) \}$ \citep{Szegedy2014}. Natural adversarial samples are natural images that are misclassified: $\mathcal{A}_N \triangleq \{x \in \mathcal{I}_\mathcal{N} : f(x) \neq \mathcal{O}(x) \}$ \citep{Hendrycks2021}. Finally, it follows from the above definitions that $\mathcal{A}_{\mathcal{N}} \subseteq \mathcal{A}_{\mathcal{C}} \subseteq \mathcal{A}_{\mathcal{U}}$. 

Natural adversarial samples are a well-documented phenomenon in deep learning \citep{Hendrycks2021}. Literature suggests that they typically occur when deep learning models learn to rely on erroneous contextual cues to shortcut classification, as opposed to truly learning to distinguish between classes \citep{Hendrycks2021, ShortcutLearning_Geirhos2020, RiskMinimisation_Arjovsky2020} (see Appendix~\ref{apdx:NaturalAdversarial} for examples). These cues are typically features within an image that are highly correlated with a target class but not indicative of the class. For instance, a model that uses oceanic environments as a cue for predicting ``shark'' may misclassify an image of a shark lying on sand. By exploiting these easy-to-learn cues, models can reduce training loss without correctly learning to generalize to the underlying classification task. Additionally, it has been observed that natural adversarial samples are able to bypass common adversarial defences \citep{DetectingNaturalvsArtificialAdvSamples_Agarwal2022}, and they exhibit high transferability i.e., the same image is misclassified by multiple classifiers \citep{Hendrycks2021}. The significant transferability of natural adversarial samples can be attributed to classifiers independently learning to rely on the same contextual cues, which is likely a consequence of shared correlations between cues and class labels across independent datasets \citep{Hendrycks2021}.

\arrayrulecolor{black!80}
\begin{figure}
    \centering
    \resizebox{1\textwidth}{!}{
    {\Huge
    \begin{tabular}{@{}|@{}c@{}c@{}|c|@{}c@{}c@{}|c|@{}c@{}c@{}|c|@{}c@{}c@{}|@{}}
        \cline{1-2}\cline{4-5}\cline{7-8}\cline{10-11}
        \includegraphics[width=0.31\linewidth, trim={-14 0 -7 -30}]{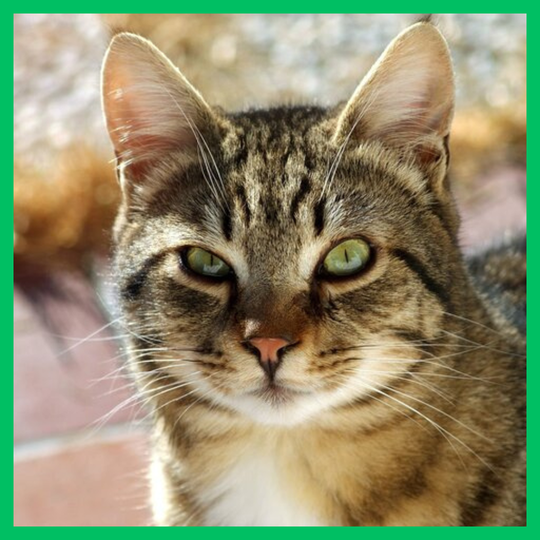} &
        \includegraphics[width=0.31\linewidth, trim={-7 0 -14 -30}]{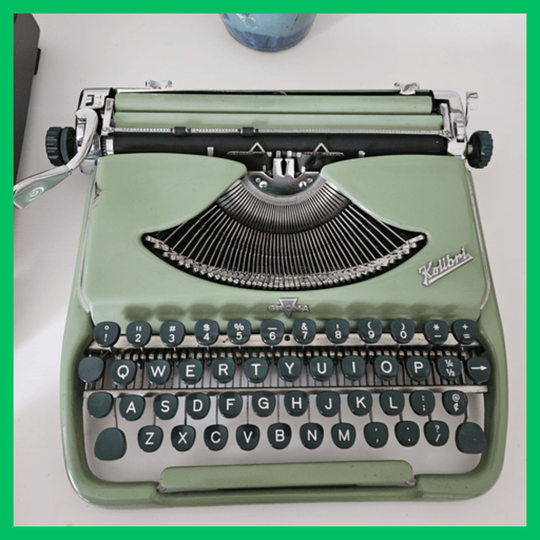} &&
        \includegraphics[width=0.31\linewidth, trim={-14 0 -7 -30}]{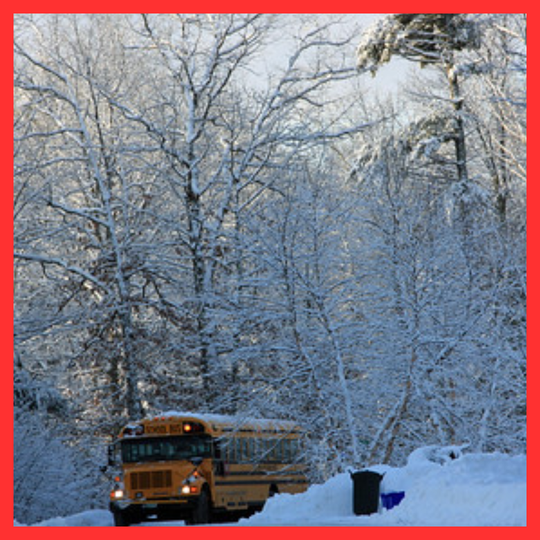} &
        \includegraphics[width=0.31\linewidth, trim={-7 0 -14 -30}]{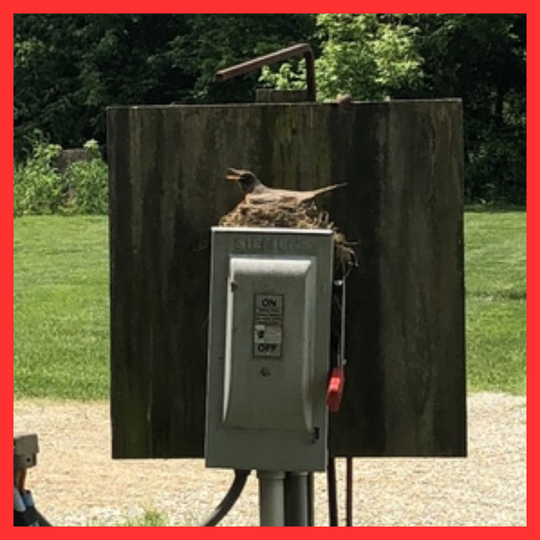} &&
        \includegraphics[width=0.31\linewidth, trim={-14 0 -7 -30}]{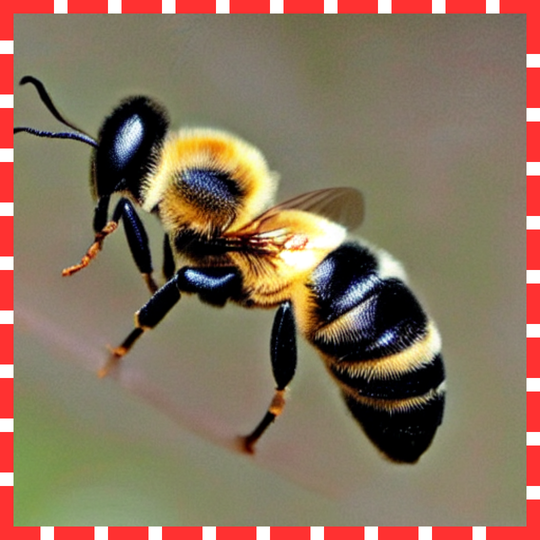} &
        \includegraphics[width=0.31\linewidth, trim={-7 0 -14 -30}]{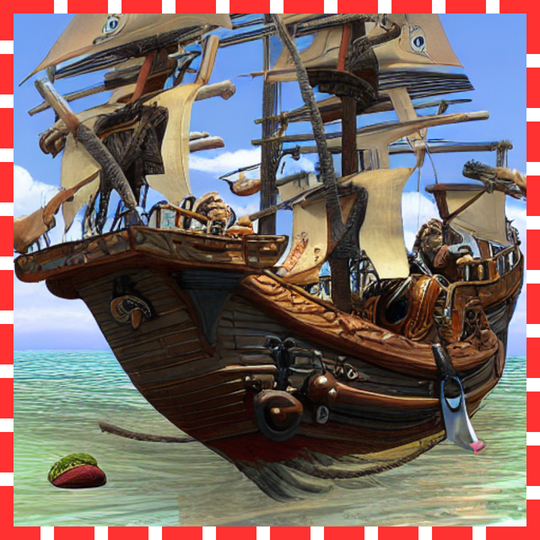} &&
        \includegraphics[width=0.31\linewidth, trim={-14 0 -7 -30}]{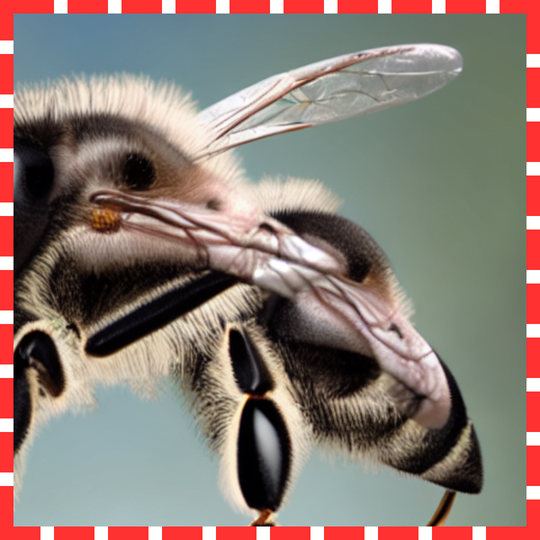} &
        \includegraphics[width=0.31\linewidth, trim={-7 0 -14 -30}]{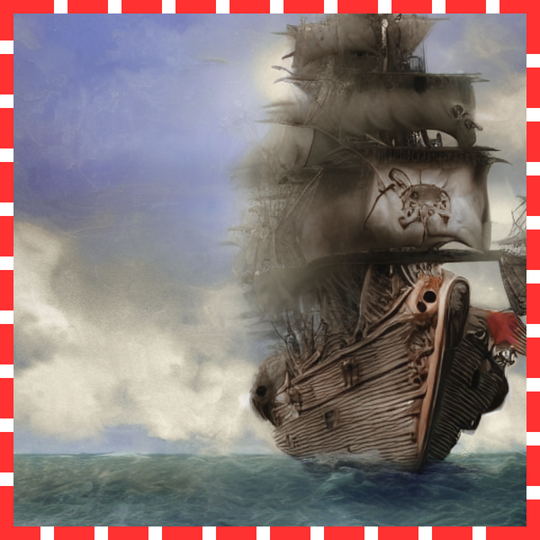} \\

        \multicolumn{2}{|c|}{\textbf{\Huge +}} && & && \multicolumn{2}{c|}{\textbf{\textbf{\Huge --}}} && & \\

        \includegraphics[width=0.31\linewidth, trim={-14 0 -7 0}]{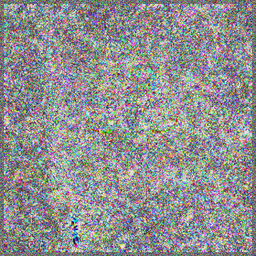} &
        \includegraphics[width=0.31\linewidth, trim={-7 0 -14 0}]{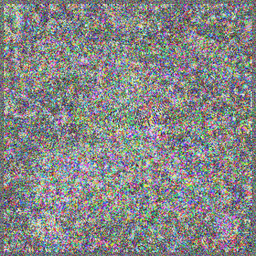} &&
        \multicolumn{2}{c|}{\raisebox{4em}{\multirow{4}{*}{\scalebox{12}{\xmarktimes}}}} &&
        \includegraphics[width=0.31\linewidth, trim={-14 0 -7 0}]{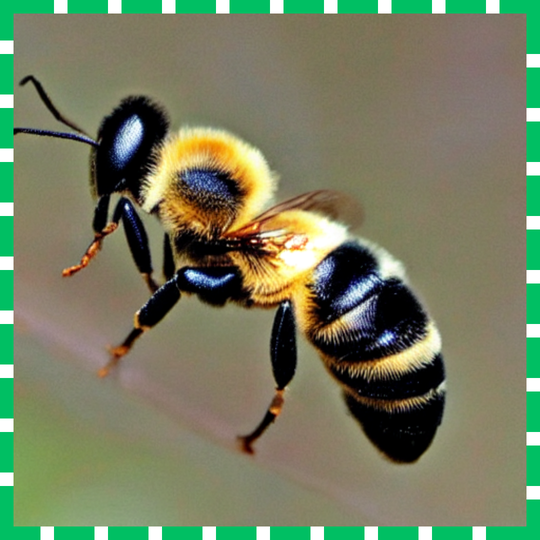} &
        \includegraphics[width=0.31\linewidth, trim={-7 0 -14 0}]{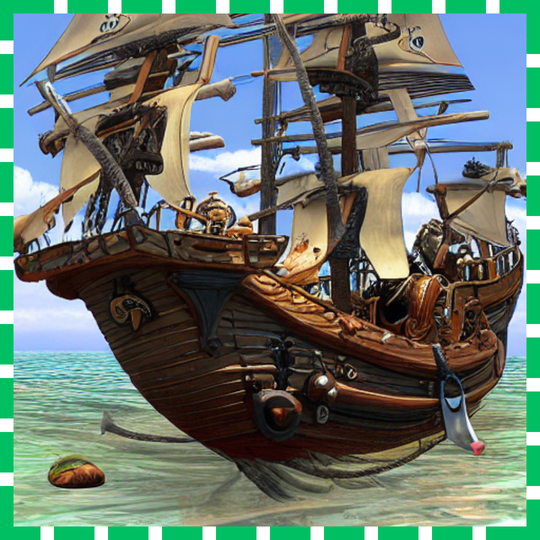} &&
        \multicolumn{2}{c|}{\includegraphics[width=0.58\linewidth, trim={0 0 0 0}]{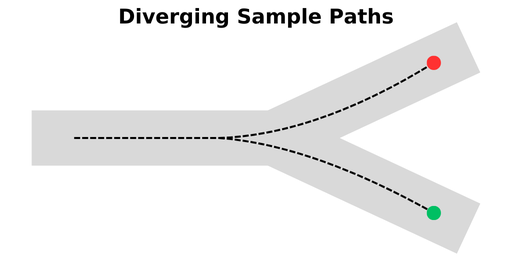}} \\

        \multicolumn{2}{|c|}{\textbf{\Huge =}} && & && \multicolumn{2}{c|}{\textbf{\textbf{\Huge =}}} && & \\

        \includegraphics[width=0.31\linewidth, trim={-14 0 -7 0}]{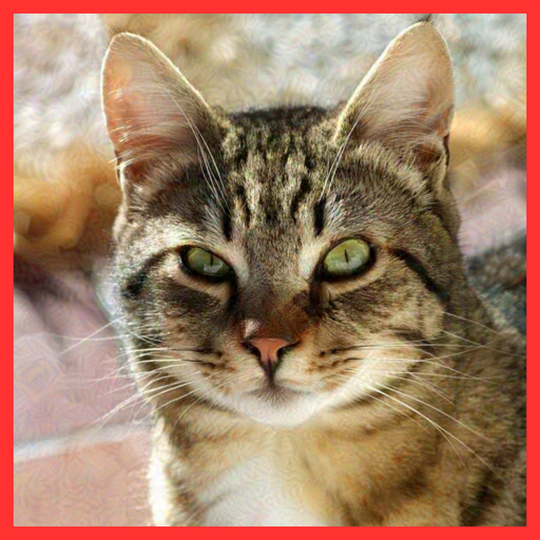} &
        \includegraphics[width=0.31\linewidth, trim={-7 0 -14 0}]{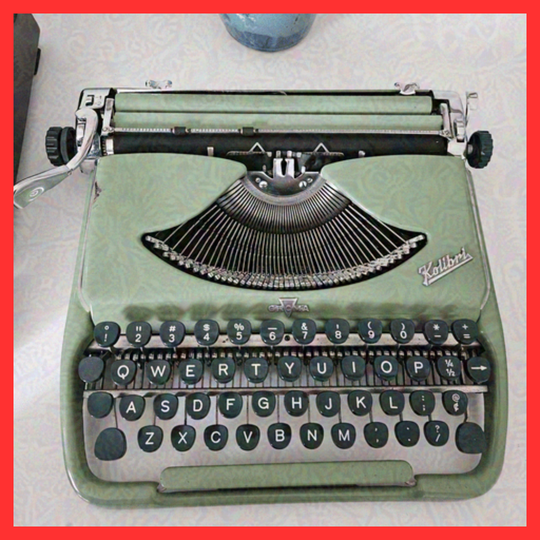} &&
        & &&
        \includegraphics[width=0.31\linewidth, trim={-14 0 -7 0}]{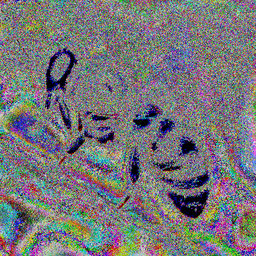} &
        \includegraphics[width=0.31\linewidth, trim={-7 0 -14 0}]{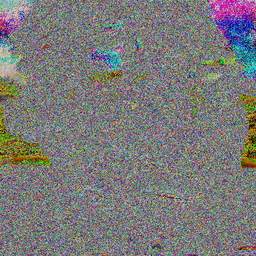} &&
        \includegraphics[width=0.31\linewidth, trim={-14 0 -7 0}]{./images/MethodComparison/Bee.png} &
        \includegraphics[width=0.31\linewidth, trim={-7 0 -14 0}]{./images/MethodComparison/PirateShip.png} \\
        \cline{1-2}\cline{4-5}\cline{7-8}\cline{10-11}
        \noalign{\vskip 0.5ex}
        \multicolumn{2}{c}{\textit{(a) PGD}} & \multicolumn{1}{c}{} & \multicolumn{2}{c}{\textit{(b) Nat. Adv. Samples}} & \multicolumn{1}{c}{} & \multicolumn{2}{c}{\textit{(c) Adv. Class. Guid.}} & \multicolumn{1}{c}{} & \multicolumn{2}{c}{\textit{(d) NatADiff}} \\
    \end{tabular}
    }
    }
    \caption{A comparison of different types of adversarial samples. \textcolor{Green}{Green} and \textcolor{red}{red} borders indicate non-adversarial and adversarial samples, respectively. A dotted border denotes artificially generated images, while a solid border indicates real-world photographs. (a) Constrained adversarial attacks (PGD \citep{Madry2019} used here) add perturbations to clean images. (b) Natural adversarial samples are test-time errors that do not contain perturbations. (c) Adversarial classifier guidance \citep{Dai2024} produces constrained adversarial samples, as the difference between images generated with and without the guidance is minimal--their difference amounts to a constrained perturbation. (d) Adversarial samples generated with NatADiff diverge from those generated without NatADiff.}
    \label{fig:CurrentAdversarial}
\end{figure}
\arrayrulecolor{black}

\noindent\textbf{Denoising diffusion generative models} \citep{DenoisingDiffusion_Ho2020} leverage a stochastic differential equation (SDE) to ``learn'' the space of natural images, allowing for the generation of natural, within-distribution images. The SDE is characterized by the forward process:
\begin{equation}
    \label{eq:diffusion forward}
    d \boldsymbol{x}_t = f(t) \boldsymbol{x}_t dt + g(t) \cdot d\boldsymbol{B}_t \quad \forall \ t \in [0, T].
\end{equation} 
where $\boldsymbol{x}_t \in \mathbb{R}^m$, $f(t) : \mathbb{R} \rightarrow \mathbb{R}$ and $g(t) : \mathbb{R} \rightarrow \mathbb{R}$ are continuous functions of $t$, and $\mathop{\cdot} d\boldsymbol{B}_t$ denotes an It\^{o} integral with respect to the standard multi-dimensional Brownian motion process $\boldsymbol{B}_t \in \mathbb{R}^m$ \citep{Pavliotis2014_DiffusionConditions}. Functions $f$ and $g$ are chosen such that the forward process progressively ``destroys'' structure in the image, $\boldsymbol{x}_0$, adding noise until it is approximately marginally Gaussian at termination time, T, i.e., $p(\boldsymbol{x}_T) \approx \mathcal{N}(0, \sigma^2_T I)$. To generate a natural image, the forward process can be reversed, and structure recovered from Gaussian noise using either \citeauthor{Anderson1982}'s reverse-time diffusion \citep{Anderson1982} or the flow ODE \citep{Song2021} derived from (\ref{eq:diffusion forward}). Both formulations require an estimate of the score function, $\nabla_{\boldsymbol{x}_t} \log(p(\boldsymbol{x}_t))$, which is approximated by a neural network, $\boldsymbol{\epsilon}_{\theta}(\boldsymbol{x}_t, t)$. This network is trained to predict the original image from a noisy version using the objective: 
\begin{equation}
    \label{eq:Diffusion Objective}
    \underset{\theta}{\min} \ {\E}_{\boldsymbol{x}_0, \boldsymbol{x}_t, t \sim p(\boldsymbol{x}_0, \boldsymbol{x}_t, t)} \left[ \left \lVert \boldsymbol{x}_0 - (\boldsymbol{x}_t - \beta(t) \boldsymbol{\epsilon}_{\theta}(\boldsymbol{x}_t, t)) / \alpha(t) \right \rVert_2^2 \right],
\end{equation} 
where $\alpha(t) = e^{\int_{0}^t f(u) du}$ and $\beta(t)^2 = \alpha(t)^2 \int_{0}^{t}\frac{g(u)^2}{\alpha(u)^2}du$. Given an optimal solution, $\boldsymbol{\epsilon}_{\theta^\star}(\boldsymbol{x}_t, t)$, to the above, the score function is given by (see Theorem~\ref{Thm:Score-Model Link} in Appendix~\ref{apdx:DiffusionMath}): 
\begin{equation}
    \label{eq:Score-model link}
    \nabla_{\boldsymbol{x}_t} \log(p(\boldsymbol{x}_t)) = - \boldsymbol{\epsilon}_{\theta^\star}(\boldsymbol{x}_t, t) / \beta(t) \quad \forall \ t \in (0, T], \ \boldsymbol{x}_t \in \mathbb{R}^m.
\end{equation}
Additionally, while $\boldsymbol{\epsilon}_{\theta^\star}(\boldsymbol{x}_t, t)$ can be used to directly estimate $\boldsymbol{x}_0$, it is typically of a lower quality than samples generated iteratively from the reverse-time diffusion or flow ODE \citep{DenoisingDiffusion_Ho2020, Song2021, Nichol2021}.

\textbf{Denoising diffusion class guidance} provides finer control over the diffusion process by sampling from $\boldsymbol{x}_0 \sim p(\boldsymbol{x}_0 | y)$ instead of $\boldsymbol{x}_0 \sim p(\boldsymbol{x}_0)$, where $y$ represents some conditioning information. To control the strength of class-guided diffusion, the marginal distribution is treated as $\bar{p}(\boldsymbol{x}_t | y) \propto p(y | \boldsymbol{x}_t)^{\omega} p(\boldsymbol{x}_t) / p(y)$, where $\omega \in \mathbb{R}_{>0}$ governs how strictly the diffusion adheres to the class constraint. For $\omega > 1$, the probability mass of $p(y | \boldsymbol{x}_t)^{\omega}$ is ``tightened'' around the regions of most probable $y$, while for $\omega < 1$, the probability mass is more diffuse. This results in stronger and weaker class adherence, respectively.

Class conditioning is incorporated directly into the diffusion score function in (\ref{eq:Score-model link}) by replacing $p(\boldsymbol{x}_t | y)$ with $\bar{p}(\boldsymbol{x}_t | y)$ as follows:  
\begin{align}  
    \nabla_{\boldsymbol{x}_t} \log(\bar{p}(\boldsymbol{x}_t | y)) &= \omega \nabla_{\boldsymbol{x}_t} \log (p(y | \boldsymbol{x}_t)) - \frac{1}{\beta(t)} \boldsymbol{\epsilon}_{\theta^\star}(\boldsymbol{x}_t, t) \label{eq:classifier guidance} \\
    &= -\frac{1}{\beta(t)} \bigl[ \omega \boldsymbol{\epsilon}_{\theta^\star}(\boldsymbol{x}_t, t, y) + (1 - \omega) \boldsymbol{\epsilon}_{\theta^\star}(\boldsymbol{x}_t, t) \bigr], \label{eq:classifier-free guidance}  
\end{align}  
where $\boldsymbol{\epsilon}_{\theta^\star}(\boldsymbol{x}_t, t)$ and $\boldsymbol{\epsilon}_{\theta^\star}(\boldsymbol{x}_t, t, y)$ are neural networks trained to estimate the noise added in the forward diffusion process when $\boldsymbol{x}_0 \sim p(\boldsymbol{x}_0)$ and $\boldsymbol{x}_0 \sim p(\boldsymbol{x}_0 | y)$, respectively. The distinction between (\ref{eq:classifier guidance}) and (\ref{eq:classifier-free guidance}) illustrates the difference between the two forms of class-guided diffusion: \textit{classifier} and \textit {classifier-free guidance}. In classifier guidance, a separate model, $p_{\theta}(y | \boldsymbol{x}_t)$, is trained to predict the probability of $y$ given $\boldsymbol{x}_t$ \citep{ClassifierGuidance_Dhariwal2021}. In contrast, classifier-free guidance requires training a diffusion model to directly estimate $\nabla_{\boldsymbol{x}_t} \log (p(\boldsymbol{x}_t | y))$ \citep{ClassifierFreeGuidance_Ho2022}.

It is important to note that $\bar{p}(\boldsymbol{x}_t | y)$ does not represent the marginal distribution that arises from applying the diffusion in (\ref{eq:diffusion forward}) to $\boldsymbol{x}_0 \sim p(\boldsymbol{x}_0 | y)$ \citep{Karras2024}. Instead, it is a mechanism that forces the sampling trajectory of $\boldsymbol{x}_t$ into regions with a higher probability of $p(y | \boldsymbol{x}_t)$, and in doing so, deviates from reverse-time diffusion and flow ODE dynamics. However, despite foregoing theoretical guarantees of sampling convergence, class-guided diffusion often exhibits superior sampling quality \citep{ClassifierGuidance_Dhariwal2021, ClassifierFreeGuidance_Ho2022}.

\section{Related Work}
\noindent\textbf{Generating unconstrained adversarial samples.}
Previous work has shown that modern generative models are capable of creating artificial unconstrained adversarial samples \citep{Song2018, GeneratingNaturalAdversaries_Zhao2018, ContentDiffusionAttack_Chen2023, Dai2024}. Initial approaches used GANs as the generative backbone for these attack; however, GANs are sensitive to perturbations to their sampling path, and they lack theoretical justification for such perturbations \citep{StyleGAN_Karras2019, Abdal2019}. Recent methods have leveraged denoising diffusion models \citep{DenoisingDiffusion_Ho2020}. Diffusion models possess superior generation quality to GANs, and provide theoretical justification for perturbing the sampling path \citep{ClassifierGuidance_Dhariwal2021}. \cite{Dai2024} leveraged these properties to develop \textit{AdvDiff}, which treats the true image class, $y$, and adversarial target, $\tilde{y}$, as random variables. The joint distribution can be decomposed as $p(\boldsymbol{x}_t, y, \tilde{y}) = p(y | \boldsymbol{x}_t) p(\tilde{y} | \boldsymbol{x}_t) p(\boldsymbol{x}_t)$, where it is assumed that $y$ and $\tilde{y}$ are conditionally independent given the noisy image, $\boldsymbol{x}_t$. Thus, given the forward diffusion in (\ref{eq:diffusion forward}), the corresponding diffusion score function (see (\ref{eq:classifier guidance}) and (\ref{eq:classifier-free guidance})) becomes 
\begin{equation} 
    \nabla_{\boldsymbol{x}_t} \log(\bar{p}(\boldsymbol{x}_t | y , \tilde{y})) = -\frac{1}{\beta(t)} \bigl[ \omega \boldsymbol{\epsilon}_{\theta^\star}(\boldsymbol{x}_t, t, y) + (1 - \omega) \boldsymbol{\epsilon}_{\theta^\star}(\boldsymbol{x}_t, t) \bigr] + s \nabla_{\boldsymbol{x}_t} \log (p(\tilde{y} | \boldsymbol{x}_t)), \label{eq:AdvDiff Guidance}
\end{equation} 
where $\omega$ and $s$ control the strength of the guidance, $\boldsymbol{\epsilon}_{\theta^\star}$ is a network trained to remove noise from $\boldsymbol{x}_t$, and the adversarial gradient, $\nabla_{\boldsymbol{x}_t} \log (p(\tilde{y} | \boldsymbol{x}_t))$, is derived from a victim classifier that provides class probabilities. Since AdvDiff directly uses the victim classifier gradient, it can be considered a form of classifier guidance, and is therefore susceptible to the same issues as classifier-guided diffusion.  

Classifier-guided diffusion requires training a model, $p_{\theta}(y | \boldsymbol{x}_t)$, to predict the class of an image that has been corrupted with Gaussian noise \citep{ClassifierGuidance_Dhariwal2021}, which is a form of adversarial training \citep{Rakin2018, Li2019}. When a non-adversarially robust classifier is used instead, the diffusion process typically generates visually coherent samples that do not adhere to the desired class conditioning, but are erroneously classified as the desired class \citep{Vaeth2024, Shen2024}. We hypothesize that this phenomenon arises due to constrained adversarial samples frequently lying within an $\epsilon$-neighborhood of natural samples \citep{Goodfellow2015, Madry2019}. Under this hypothesis, the diffusion model acts as a constraint that pushes the sample towards the natural image manifold, while the non-adversarially robust classifier introduces a perturbation that directs the sample towards the nearest region containing samples of the desired class. The resulting trade-off between the diffusion model and classifier guidance incentivizes the diffusion trajectory to converge towards the adversarial regions that frequently lie imperceptibly close to the natural image manifold--that is, pockets of constrained adversarial samples \citep{Shen2024}.

\section{Methodology} \label{sec:NatADiff}
Natural adversarial samples frequently occur when classifiers over-rely on contextual cues to shortcut classification \citep{Hendrycks2021}. We incorporate this key observation into our proposed, diffusion-based natural adversarial sampling scheme--NatADiff (see Algorithm~\ref{alg:NatADiff}). NatADiff leverages adversarial boundary guidance to incorporate features from the adversarial class. In addition, we use augmented classifier guidance and time-travel sampling to enhance attack transferability while preserving image quality.

\noindent\textbf{Accounting for sample noise.} Classifier-guided diffusion specifically trains a classifier to predict the class label of a noisy sample $\boldsymbol{x}_t$. However, in \textit{adversarial} diffusion guidance, the victim model is typically an ``off-the-shelf'' classifier that was never trained on noisy samples. Directly passing $\boldsymbol{x}_t$ to this classifier will likely degrade classification accuracy, leading to inferior diffusion guidance. To address this, we take the same approach as \citep{FreeDoM_Yu2023, UniversalGuidance_Bansal2023, Shen2024} and use Tweedie's formula \citep{TweediesFormula_Efron2011} to pass the classifier the current estimate of $\boldsymbol{x}_0$ at time $t$: 
\begin{equation} 
    \hat{\boldsymbol{x}}_0(\boldsymbol{x}_t) = (\boldsymbol{x}_t - \beta(t) \boldsymbol{\epsilon}_{\theta^*}(\boldsymbol{x}_t, t, y)) / \alpha(t). \label{eq:Pred x0}
\end{equation}

\noindent\textbf{Reducing adversarial gradient.} Constrained adversarial attacks are sensitive to image transformations, with rotations, crops, and translations reducing the success rates of common attack algorithms \citep{Guo2018}. We leverage this by applying differentiable image transforms to reduce the effect of the adversarial gradient that points in the direction of constrained adversarial perturbations. We find that this increases the prevalence of visible adversarial features (see Appendix~\ref{apdx:Effect of classifier augmentations} for ablation study). These transformations are similar to the ones used by \cite{Shen2024} to perform training-free classifier-guided diffusion. The local adversarial signal is ``averaged out'', reducing the likelihood of generating constrained adversarial samples, and forcing the manifestation of features from the--in our case--adversarial class conditioning \citep{Shen2024}.

Given a collection of differentiable image transforms: $\boldsymbol{\mathcal{T}} = \{ \mathcal{T}_1, \mathcal{T}_2, \dots \}$, we compute the adversarial classifier gradient as
\begin{equation}
    \label{eq:Transform Adv Gradient}
    \nabla_{\boldsymbol{x}_t} \log(p(\tilde{y} | \boldsymbol{x}_t)) = \boldsymbol{g}(\boldsymbol{x}_t) / \lVert \boldsymbol{g}(\boldsymbol{x}_t) \rVert_2,
\end{equation} 
where $\boldsymbol{g}(\boldsymbol{x}_t) = \nabla_{\boldsymbol{x}_t} \log \left( \sigma_{\tilde{y}} \left( \frac{1}{\lvert \boldsymbol{\mathcal{T}} \rvert} \sum_{i=1}^{\lvert \boldsymbol{\mathcal{T}} \rvert} h(\mathcal{T}_i(\hat{\boldsymbol{x}}_0(\boldsymbol{x}_t))) \right) \right)$, $h : \mathbb{R}^m \rightarrow \mathbb{R}^{|\mathcal{Y}|}$ is a function that returns the victim classifier's logit predictions, and $\sigma_{\tilde{y}} : \mathbb{R}^{|\mathcal{Y}|} \rightarrow \mathbb{R}$ is a sigmoid function that returns the probability of the target adversarial class.

\noindent\textbf{Adversarial boundary guidance.} Initial experiments showed that substituting the improved adversarial gradient from (\ref{eq:Transform Adv Gradient}) into (\ref{eq:AdvDiff Guidance}) did not steer the diffusion trajectory towards natural adversarial samples (see Appendix~\ref{apdx:Selection of mu} for ablation study). This may occur because classifier augmentations eliminate many of the constrained adversarial samples that lie close to the image manifold, but not those further away. Consequently, if the initial sampling point of the diffusion trajectory is too distant from a region of natural adversarial samples, adversarial guidance will push the sample off the image manifold.

To address this, we leverage the connection between natural adversarial samples and the use of contextual cues as a classification shortcut \citep{Hendrycks2021, ShortcutLearning_Geirhos2020, RiskMinimisation_Arjovsky2020}. We propose adversarial boundary guidance as a method of directing the diffusion trajectory towards samples that incorporate erroneous contextual cues, i.e., features from the adversarial class. We define adversarial boundary guidance as
\begin{equation} 
    \nabla_{\boldsymbol{x}_t} \log(\bar{p}(\boldsymbol{x}_t | y, \tilde{y})) = -\frac{1}{\beta(t)} \bigl[ \boldsymbol{\epsilon}_{\theta^\star}(\boldsymbol{x}_t, t) + (\omega - \mu\omega) \boldsymbol{v}_y + \mu \rho \boldsymbol{v}_{y \cap \tilde{y}} \bigr] + s \nabla_{\boldsymbol{x}_t} \log (p(\tilde{y} | \boldsymbol{x}_t)), \label{eq:Adversarial Boundary Guidance} \\
\end{equation}
where $\omega, \rho, s \in \mathbb{R}_{\geq 0}$, $\mu \in [0, 1]$, $\boldsymbol{v}_y = \boldsymbol{\epsilon}_{\theta^\star}(\boldsymbol{x}_t, t, y) - \boldsymbol{\epsilon}_{\theta^\star}(\boldsymbol{x}_t, t)$, and $\boldsymbol{v}_{y \cap \tilde{y}} = \boldsymbol{\epsilon}_{\theta^\star}(\boldsymbol{x}_t, t, y \cap \tilde{y}) - \boldsymbol{\epsilon}_{\theta^\star}(\boldsymbol{x}_t, t)$. $\omega$ and $\rho$ govern the strength of classifier-free guidance, $s$ controls adversarial classifier guidance strength, and $\mu$ regulates how strongly the sample tends towards the intersection of the true and adversarial classes. For sufficiently large $\mu$, the sampling trajectory should approach the class intersection, incorporating enough elements from the adversarial class to cause a misclassification, while remaining within the bounds of the true class from a human's perspective. Note when $\mu = 0$, adversarial boundary guidance is equivalent to adversarial classifier guidance \citep{Dai2024}.

To justify (\ref{eq:Adversarial Boundary Guidance}), we note that the classifier-free score function can be rewritten as $\nabla_{\boldsymbol{x}_t} \log(\bar{p}(\boldsymbol{x}_t | y)) = -\frac{1}{\beta(t)} \left[ \boldsymbol{\epsilon}_{\theta^\star}(\boldsymbol{x}_t, t) + \omega \boldsymbol{v}_y \right]$ where $\boldsymbol{v}_y = \boldsymbol{\epsilon}_{\theta^\star}(\boldsymbol{x}_t, t, y) - \boldsymbol{\epsilon}_{\theta^\star}(\boldsymbol{x}_t, t)$ is a vector that points towards regions of the manifold containing images of class $y$. Additionally, recall that $\bar{p}(\boldsymbol{x}_t | y)$ is not the marginal density arising from a valid diffusion, rather it is a magnification of the guidance provided by a network that has ``learned'' the image manifold. To further exploit the information contained in this network, we introduce $\boldsymbol{v}_{y \cap \tilde{y}}$, which directs the sampling trajectory towards the class intersection.

\noindent\textbf{Time-travel sampling.} Significant disruption to the diffusion sampling path risks degradation in sample quality, or falling off the image manifold \citep{Lugmayr2022,FreeDoM_Yu2023}. To mitigate these issues, we incorporate time-travel sampling into our diffusion scheme, which has been shown to increase image quality in cases where standard diffusion sampling would otherwise fail \citep{Lugmayr2022, FreeDoM_Yu2023, Shen2024}. 

By injecting additional sampling steps, time-travel sampling allows the diffusion model to explore a wider region of the sample space and recover from suboptimal trajectories. This helps maintain sample quality and prevents the generation process from diverging away from the image manifold. More concretely, given a sequence of sampling times $\{ t_{i} \}_{i=1}^N$ with $t_{i+1} > t_i$ for all $i$, time-travel sampling resets the diffusion state at time $t_i$ by running the forward process, $\boldsymbol{x}_{t_{i+k}} \sim p(x_{t_{i+k}} | x_{t_i})$, and then resampling $\boldsymbol{x}_{t_i}$ using the reverse process \citep{Anderson1982, Song2021}. This procedure is repeated $R$ times before $\boldsymbol{x}_{t_i}$ is accepted, after which sampling proceeds to $\boldsymbol{x}_{t_{i-1}}$. To improve efficiency, time-travel sampling can be applied to a subset of diffusion steps \citep{FreeDoM_Yu2023}.

\begin{wrapfigure}{R}{0.5\textwidth}
    \vspace{-8mm}
    \begin{minipage}{0.5\textwidth}
    \begin{algorithm}[H]
        \scriptsize
        \caption{NatADiff} \label{alg:NatADiff}
        \begin{algorithmic}
            \Require adversarial guidance parameters: $\omega$, $\rho$, $\mu$, $s$; true and adversarial classes: $y$, $\tilde{y}$; victim classifier: $h$; forward diffusion functions: $\alpha(t)$, $\beta(t)$; stable diffusion model: $\boldsymbol{\epsilon}_{\theta^\star}$; VAE decoder: $V_{\text{dec}}$; collection of differentiable image transforms: $\{ \mathcal{T}_1, \mathcal{T}_2, \dots \}$; sequence of sampling steps with $t_1 = 0$, $t_N = T$, and $t_{i+1} > t_i$: $\{ t_{i} \}_{i=1}^N$; time-travel parameters: $R$, $r_l$, $r_u$; adversarial classifier bounds: $c_l$, $c_u$; number sampling attempts: $S$; guidance scalers: $\delta_\mu$, $\delta_s$\\
        
            \State $\boldsymbol{z}_T \sim \mathcal{N}(0, I)$
            \For{$s = 1, \dots, S$}
                \For{$i = N, \dots, 1$}
                    \If{$r_l \leq t_i \leq r_u$}
                        \State $\tilde{R} = R$
                    \Else
                        \State $\tilde{R} = 1$
                    \EndIf
                    \For{$r = \tilde{R}, \dots, 1$} \Comment{Time-travel loop}
                        \State $\boldsymbol{v}_y = \boldsymbol{\epsilon}_{\theta^\star}(\boldsymbol{z}_{t_i}, t_i, y) - \boldsymbol{\epsilon}_{\theta^\star}(\boldsymbol{z}_{t_i}, t_i)$
                        \State $\boldsymbol{v}_{y \cap \tilde{y}} = \boldsymbol{\epsilon}_{\theta^\star}(\boldsymbol{z}_{t_i}, t_i, y \cap \tilde{y}) - \boldsymbol{\epsilon}_{\theta^\star}(\boldsymbol{z}_{t_i}, t_i)$
                        \State $\hat{\boldsymbol{\epsilon}} = \boldsymbol{\epsilon}_{\theta^\star}(\boldsymbol{x}_{t_i}, t_i) + (\omega - \mu\omega) \boldsymbol{v}_y + \mu \rho \boldsymbol{v}_{y \cap \tilde{y}}$
                        \If{$c_l \leq t \leq c_u$}
                            \State $\hat{\boldsymbol{x}}_0 = V_{\text{dec}} \left( \frac{\boldsymbol{z}_{t_i} - \beta(t_i) \hat{\boldsymbol{\epsilon}}}{\alpha(t_i)} \right)$
                            \State $\boldsymbol{g} = \nabla_{\boldsymbol{z}_{t_i}} \log \left( \sigma_{\tilde{y}} \left( \frac{1}{\lvert \boldsymbol{\mathcal{T}} \rvert} \sum_{j=1}^{\lvert \boldsymbol{\mathcal{T}} \rvert} h(\mathcal{T}_j(\hat{\boldsymbol{x}}_0)) \right) \right)$
                            \State $\boldsymbol{g} = \frac{\boldsymbol{g}}{\lVert \boldsymbol{g} \rVert_2}$
                            \State $\hat{\boldsymbol{\epsilon}} = \hat{\boldsymbol{\epsilon}} - s \beta(t) \boldsymbol{g}$
                        \EndIf
                        \State $\boldsymbol{z}_{t_{i-1}}$ $\gets$ \text{reverse diffusion step using $\hat{\boldsymbol{\epsilon}}$}
                        \If{$r > 1$} \Comment{Sampling $\boldsymbol{z}_{t_i} \sim p(\boldsymbol{z}_{t_i} | \boldsymbol{z}_{t_{i-1}})$}
                            \State $a = \frac{\alpha(t_i)}{\alpha(t_{i-1})}$
                            \State $b^2 = \beta(t_i)^2 - \left( a\beta(t_{i-1}) \right)^2$
                            \State $\boldsymbol{z}_{t_i} \sim \mathcal{N} \left(a \boldsymbol{z}_{t_{i-1}}, b^2 \cdot I \right)$
                        \EndIf
                    \EndFor
                \EndFor
                \If{$\text{argmax}(h(V_{\text{dec}}(\boldsymbol{z}_{0}))) \neq \tilde{y}$}
                    \State $\mu = \mu + \delta_{\mu}$
                    \State $s = s + \delta_s$
                \Else
                    \State \textbf{break} \Comment{End the search early if sample is found}
                \EndIf
            \EndFor
            \State \textbf{return} $V_{\text{dec}}(\boldsymbol{z}_{0})$
        \end{algorithmic}
    \end{algorithm}
    \end{minipage}
    \vspace{-3mm}
\end{wrapfigure}

\noindent\textbf{Similarity targeting.}
Many popular adversarial attacks operate in an untargeted setting \citep{Szegedy2014, Goodfellow2015, Madry2019, Croce2020}, where the only requirement is that the predicted class differs from the true class, i.e., $\tilde{y} \neq y$. These attacks often update the adversarial target dynamically during optimization, selecting the most probable incorrect class at each step, and they frequently outperform targeted variants \citep{Scaling_Croce2020}. To extend NatADiff to untargeted settings, we propose \textit{similarity targeting} (see Algorithm~\ref{alg:NatADiff Similarity} in Appendix~\ref{apdx:NatADiff Algorithm}).

Similarity targeting is based on the assumption that it is easier to incorporate adversarial features from classes that are semantically similar to the true class. To heuristically measure this similarity, we leverage the CLIP \citep{CLIP_Radford2021} text encoder, which maps class labels into a shared image-text embedding space. We then select the adversarial target as the class most similar to the true class in this embedding space, as measured by cosine similarity. Concretely, given the CLIP text encoder $C_{\text{enc}} : \mathcal{Y} \rightarrow \mathbb{R}^m$, the true class label, $y_i$, and the set of candidate adversarial labels $\mathcal{Y}_{\text{cand}} = \{ y_1, \dots, y_n \} \setminus y_i$, we define the adversarial target as
\begin{equation}
    \label{eq:Similarity Targetting}
    \tilde{y} = \underset{y \in \mathcal{Y}{\text{cand}}}{\textup{arg max}} \frac{C_{\text{enc}}(y_i) \cdot C_{\text{enc}}(y)}{\lVert C_{\text{enc}}(y_i) \rVert_2 \lVert C_{\text{enc}}(y) \rVert_2}.
\end{equation}

\section{Experiments} \label{sec:Experiments}
\subsection{Experiment details} \label{sec:Experiment details}
We evaluate the effectiveness of NatADiff on the ImageNet \citep{Deng2009} classification task, which requires a model to classify an image into one of 1,000 distinct object categories. We target a range of off-the-shelf ImageNet classifiers and assess the attack success rates and visual quality of the generated samples. All experiments are conducted on an NVIDIA RTX 4090 GPU, and each sample takes approximately 103 seconds to generate (see Appendix~\ref{apdx:Runtime comparison} for runtime comparisons).

\noindent\textbf{Surrogate and victim models.} NatADiff and other comparable attack methods require access to classifier gradients when generating adversarial samples. The model whose gradients are used in this way is referred to as the \textit{surrogate model}, and we test ResNet-50 (RN-50) \citep{He2015}, Inception-v3 (Inc-v3) \citep{Inceptionv3_Szegedy2016}, and Vision Transformer (ViT-H) \citep{Dosovitskiy2021} surrogates. We examine the performance of these adversarial samples across RN-50, Inc-v3, ViT-H, adversarially trained ResNet (AdvRes) and Inception (AdvInc) \citep{Kurakin2018}, ResNet-152 (RN-152) \citep{He2015}, Max-ViT \citep{Max-ViT_Tu2022}, Swin-B \citep{Swin-B_Liu2021}, and DeIT \citep{DeIT_Touvron2021} \textit{victim} models.

\noindent\textbf{Diffusion model.} We use Stable Diffusion 1.5 \citep{StableDiffusion_Rombach2022} (SD1.5) as our base diffusion model for NatADiff and adversarial classifier guidance \citep{Dai2024}. SD1.5 is a pretrained latent text-to-image diffusion model. The diffusion process is performed in a latent space, and a variational autoencoder (VAE), \citep{Kingma2014}, $V_{\text{dec}}$, is used to decode latent samples into the image space. To facilitate the use of adversarial classifier guidance the VAE must be incorporated into the gradient calculation. Specifically, given a sample, $\boldsymbol{z}_t$, from the latent diffusion process, we introduce the VAE, $V_{\text{dec}}$, into (\ref{eq:Pred x0}) as $\hat{\boldsymbol{x}}_0(\boldsymbol{z}_t) = V_{\text{dec}} \left( (\boldsymbol{z}_t - \beta(t) \boldsymbol{\epsilon}_{\theta^*}(\boldsymbol{z}_t, t, y)) / \alpha(t) \right)$, and take the gradient with respect to $\boldsymbol{z}_t$ instead of $\boldsymbol{x}_t$ in (\ref{eq:AdvDiff Guidance}) and (\ref{eq:Transform Adv Gradient}). Finally, we use 200 sampling steps under the DDIM \citep{Song2022} parameterization, which defines the drift and diffusion coefficients in (\ref{eq:diffusion forward}) as $f(t) = \frac{1}{2} \frac{d}{dt} \log( \hat{\alpha}_t)$ and $g(t)^2 = - \frac{d}{dt} \log( \hat{\alpha}_t)$, respectively \citep{Han2024}.

\noindent\textbf{NatADiff settings.}
We run NatADiff under both \textit{targeted} and \textit{untargeted} attack settings. For targeted attacks, we assign a random adversarial target to each sample. For untargeted attacks, we use similarity targeting from Section~\ref{sec:NatADiff}. During adversarial boundary guidance we use the text prompt \textit{``$<$class name of $y$$>$''} as the conditioning for the true class guidance, $y$. For intersection guidance, $y \cap \tilde{y}$, we use the prompt \textit{``$<$class name of $\tilde{y}$$>$ and $<$class name of $y$$>$''}. We delay adversarial classifier guidance until timestep $t \leq 700$, i.e., we set $s = 0$ for all $t > 700$. Finally, we choose a conservative value of $\mu = 0.2$ for all experiments, and select $s$ based on the target classifier (see Appendix~\ref{apdx:Selection of mu} for ablation study and additional experiment details). We generate 2,000 adversarial samples in each experiment run.

\noindent\textbf{Comparison methods.}
We compare NatADiff to state-of-the-art constrained and unconstrained adversarial attacks: PGD \cite{Madry2019}, AutoAttack (AA) \cite{Croce2020}, NCF \citep{NCF_Yuan2022}, DiffAttack \citep{DiffAttack_Chen2025}, ACA \citep{ContentDiffusionAttack_Chen2023}, and adversarial classifier guidance (AdvClass) \citep{Dai2024}. All methods use their default parameter settings and for comparison methods that alter a pre-existing ``clean'' image, we use their suggested ImageNet-compatible dataset as our clean baseline \citep{AdvCompDataset_Kurakin2017}. We apply AdvClass under both \textit{targeted} and \textit{untargeted} attack settings (using the same similarity targeting as NatADiff). Finally, we acknowledge that although NCF, ACA, and AdvClass are classified as unconstrained attacks, only NatADiff and AdvClass have a fully unrestricted attack domain, as they are not bound to any initial clean image and are capable of synthesising entirely new images.

\noindent\textbf{Metrics.}
To assess attack performance, we follow \cite{ContentDiffusionAttack_Chen2023}, and report attack success rate (ASR) as the percentage of misclassified samples. To evaluate image quality, we use the Inception Score (IS) \citep{Salimans2016} and Fréchet Inception Distance (FID) \citep{Frechet1957}. IS provides a direct measure of image quality, while FID estimates the similarity between the distributions of generated and real images. We compute FID with respect to both the ImageNet-Val \citep{Deng2009} and ImageNet-A \citep{Hendrycks2021} datasets to assess how closely NatADiff samples resemble natural images and known natural adversarial examples, respectively.

\subsection{Results} \label{sec:Results}
\begin{figure*}
    \centering
    {
    \resizebox{\linewidth}{!}{%
    \scriptsize
    \begin{tabular}{@{}l@{}l@{\hspace{0.5em}}l@{}l@{\hspace{0.5em}}l@{}l@{}}
        \multicolumn{2}{c}{\normalsize\textbf{ResNet-50}} & \multicolumn{2}{c}{\normalsize\textbf{Inception-v3}} & \multicolumn{2}{c}{\normalsize\textbf{ViT-H}} \\
        \includegraphics[width=0.2\textwidth]{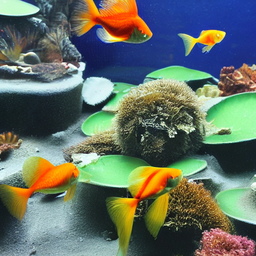} & \includegraphics[width=0.2\textwidth]{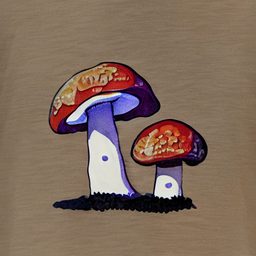} & \includegraphics[width=0.2\textwidth]{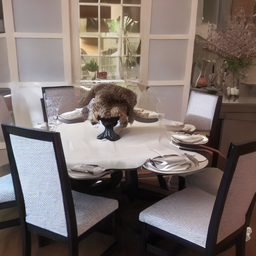} & \includegraphics[width=0.2\textwidth]{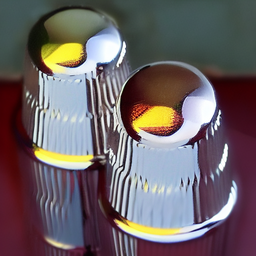} & \includegraphics[width=0.2\textwidth]{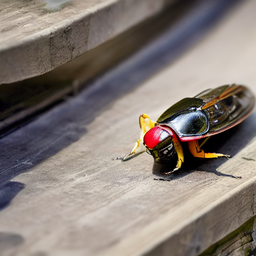} & \includegraphics[width=0.2\textwidth]{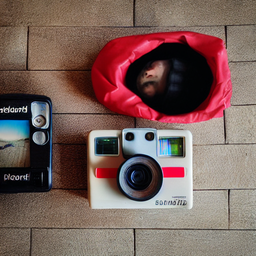} \\

        \textcolor{Green}{\textbf{True:}} Goldfish & \textcolor{Green}{\textbf{True:}} Mushroom & \textcolor{Green}{\textbf{True:}} Dining Table & \textcolor{Green}{\textbf{True:}} Thimble & \textcolor{Green}{\textbf{True:}} Cicada & \textcolor{Green}{\textbf{True:}} Polaroid Camera \\
        \textcolor{red}{\textbf{Adv\textsuperscript{T}:}} Titi Monkey & \textcolor{red}{\textbf{Adv\textsuperscript{T}:}} Packet & \textcolor{red}{\textbf{Adv\textsuperscript{T}:}} Platypus & \textcolor{red}{\textbf{Adv\textsuperscript{T}:}} Traffic Light & \textcolor{red}{\textbf{Adv\textsuperscript{T}:}} Bobsled & \textcolor{red}{\textbf{Adv\textsuperscript{T}:}} Sleeping Bag \\
        
        \includegraphics[width=0.2\textwidth]{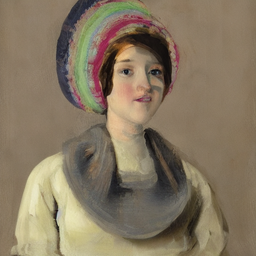} & \includegraphics[width=0.2\textwidth]{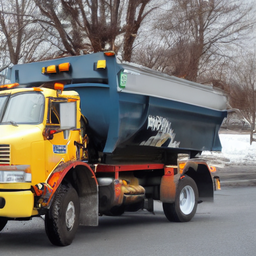} & \includegraphics[width=0.2\textwidth]{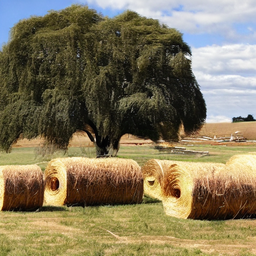} & \includegraphics[width=0.2\textwidth]{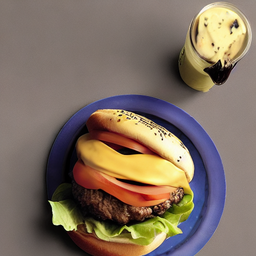} & \includegraphics[width=0.2\textwidth]{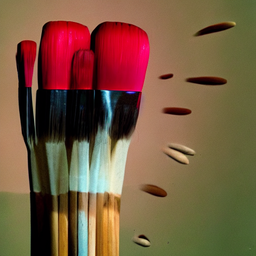} & \includegraphics[width=0.2\textwidth]{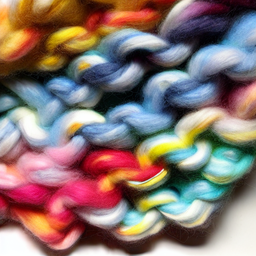} \\
        
        \textcolor{Green}{\textbf{True:}} Bonnet & \textcolor{Green}{\textbf{True:}} Garbage Truck & \textcolor{Green}{\textbf{True:}} Hay & \textcolor{Green}{\textbf{True:}} Cheeseburger & \textcolor{Green}{\textbf{True:}} Paintbrush & \textcolor{Green}{\textbf{True:}} Wool \\
        \textcolor{red}{\textbf{Adv\textsuperscript{U}:}} Sombrero & \textcolor{red}{\textbf{Adv\textsuperscript{U}:}} Snowplow & \textcolor{red}{\textbf{Adv\textsuperscript{U}:}} Ox & \textcolor{red}{\textbf{Adv\textsuperscript{U}:}} Banana & \textcolor{red}{\textbf{Adv\textsuperscript{U}:}} Matchstick & \textcolor{red}{\textbf{Adv\textsuperscript{U}:}} Dishrag \\
    \end{tabular}}
    }
    \vspace{-2mm}
    \caption{Adversarial samples generated by NatADiff with ResNet-50 \citep{He2015}, Inception-v3 \citep{Inceptionv3_Szegedy2016}, and ViT-H \citep{Dosovitskiy2021} surrogate models (see column labels). We report the true class and adversarial target for each image. Superscripts T and U denote random and similarity targeted attacks, respectively.}
    \label{fig:NatADiff samples}
\end{figure*}

\begin{figure*}
    \centering
    {
    \resizebox{\linewidth}{!}{%
    \small
    \begin{tabular}{@{}c@{}c@{}c@{}c@{}c@{}c@{}c@{}c@{}}
        \normalsize\textbf{Source} & \normalsize\textbf{PGD} & \normalsize\textbf{AA} & \normalsize\textbf{NCF} & \normalsize\textbf{DiffAttack} & \normalsize\textbf{ACA} & \normalsize\textbf{AdvClass} & \normalsize\textbf{NatADiff} \\
        \includegraphics[width=0.2\textwidth]{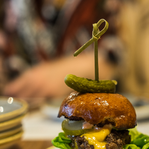} & \includegraphics[width=0.2\textwidth]{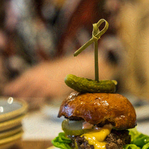} & \includegraphics[width=0.2\textwidth]{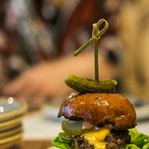} & \includegraphics[width=0.2\textwidth]{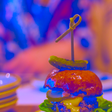} & \includegraphics[width=0.2\textwidth]{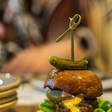} & \includegraphics[width=0.2\textwidth]{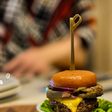} & \includegraphics[width=0.2\textwidth]{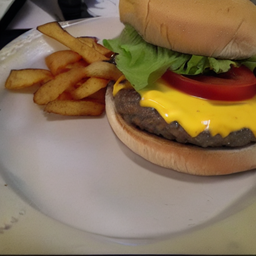} & \includegraphics[width=0.2\textwidth]{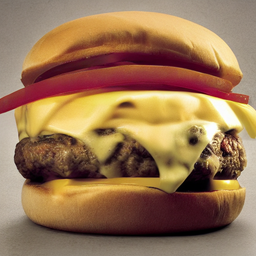}
    \end{tabular}}
    }
    \vspace{-2mm}
    \caption{Source and adversarial samples generated by PGD \cite{Madry2019}, AutoAttack (AA) \cite{Croce2020}, NCF \citep{NCF_Yuan2022}, DiffAttack \citep{DiffAttack_Chen2025}, ACA \citep{ContentDiffusionAttack_Chen2023}, AdvClass \citep{Dai2024}, and NatADiff with a ResNet-50 \citep{He2015} surrogate model. Note true class is ``burger'' and adversarial target for AdvClass and NatADiff is ``banana''.}
    \label{fig:Attack visual comparison}
\end{figure*}

\begin{table}
    \centering
    \caption{\textbf{Attack success rate} (\%) and \textbf{image quality} of adversarial samples. Superscripts T and U denote random and similarity targeted attacks, respectively. \textcolor{red}{\textbf{Bold}} and \textcolor{BrickRed}{\underline{underlined}} values highlight the best and second best scores for each surrogate model. White-box ASR (same surrogate and victim model) is denoted with an $^*$. Note we do not report image quality for constrained perturbation-based attacks (these attacks make imperceptible image alterations).}
    \label{tab:Classifier ASR and image quality}
    \resizebox{1\textwidth}{!}{
        \begin{tabular}{@{}cc@{}c@{\hspace{1em}}ccccc@{}c@{\hspace{1em}}cccc@{}ccccc@{}}
            \specialrule{1.2pt}{0pt}{0pt}
            \noalign{\vskip 0.5ex}
            \multicolumn{1}{@{}c}{\multirow{3}{*}{\shortstack{Surrogate\\Model}}} & \multicolumn{1}{c@{}}{\multirow{3}{*}{Attack}} && \multicolumn{10}{c}{Victim Model ASR (\%)} && \multicolumn{1}{c}{\multirow{3}{*}{\shortstack{Average\\ASR}}} & \multicolumn{1}{c}{\multirow{3}{*}{\shortstack{IS\\($\boldsymbol{\uparrow}$)}}} & \multicolumn{1}{c}{\multirow{3}{*}{\shortstack{FID-\\Val ($\boldsymbol{\downarrow}$)}}} & \multicolumn{1}{c}{\multirow{3}{*}{\shortstack{FID-\\A ($\boldsymbol{\downarrow}$)}}} \\
            &&& \multicolumn{5}{c}{CNNs} && \multicolumn{4}{c}{Transformers} &&&& \\
            \cline{4-8} \cline{10-13} 
            \noalign{\vskip 0.5ex}
            &&& RN-50 & Inc-v3 & RN-152 & AdvRes & AdvInc && ViT-H & Max-ViT & Swin-B & DeIT &&&& \\
            \specialrule{1.2pt}{0pt}{0pt}
            \noalign{\vskip 0.5ex}
            & Clean & & $5.3$ & $7.6$ & $2.9$ & $3.0$ & $5.8$ && $10.9$ & $3.8$ & $4.5$ & $7.4$ && $5.7$ & $55.0$ & $58.0$ & $94.7$ \\
            \specialrule{1.2pt}{0pt}{0pt}
            \noalign{\vskip 0.2ex}
            \multicolumn{1}{@{}c}{\multirow{9}{*}{RN-50}} & PGD & & $99.4^*$ & $11.8$ & $5.2$ & $4.9$ & $8.1$ && $10.5$ & $4.4$ & $5.5$ & $8.2$ && $17.6$ & - & - & - \\
            & AA & & $\textcolor{red}{\mathbf{100^*}}$ & $13.3$ & $10.0$ & $3.9$ & $8.8$ && $10.5$ & $5.4$ & $5.6$ & $8.0$ && $18.4$ & - & - & - \\
            & NCF & & $74.8^*$ & $33.4$ & $37.3$ & $28.2$ & $31.2$ && $17.2$ & $24.0$ & $31.7$ & $37.2$ && $35.0$ & $30.4$ & $69.7$ & $85.5$ \\
            & DiffAttack & & $92.5^*$ & $47.1$ & $52.5$ & $35.3$ & $43.3$ && $28.4$ & $44.6$ & $42.4$ & $38.9$ && $47.2$ & $26.8$ & $64.1$ & $\textcolor{red}{\mathbf{76.8}}$ \\
            & ACA & & $78.8^*$ & $53.3$ & $52.7$ & $49.8$ & $53.1$ && $\textcolor{BrickRed}{\underline{41.8}}$ & $\textcolor{BrickRed}{\underline{46.4}}$ & $\textcolor{BrickRed}{\underline{49.3}}$ & $50.6$ && $52.9$ & $23.9$ & $65.0$ & $77.9$ \\
            & AdvClass\textsuperscript{T} & & $99.6^*$ & $35.0$ & $32.1$ & $31.4$ & $33.5$ && $25.8$ & $30.0$ & $30.8$ & $32.8$ && $39.0$ & $38.3$ & $\textcolor{red}{\mathbf{48.9}}$ & $92.4$ \\
            & AdvClass\textsuperscript{U} & & $\textcolor{BrickRed}{\underline{99.9^*}}$ & $42.5$ & $44.3$ & $38.7$ & $41.1$ && $29.7$ & $37.6$ & $38.4$ & $39.1$ && $45.7$ & $\textcolor{BrickRed}{\underline{38.5}}$ & $\textcolor{BrickRed}{\underline{50.2}}$ & $92.7$ \\
            \cline{2-18}
            \noalign{\vskip 0.5ex}
            & NatADiff\textsuperscript{T} & & $96.9^*$ & $\textcolor{BrickRed}{\underline{60.1}}$ & $\textcolor{BrickRed}{\underline{56.5}}$ & $\textcolor{BrickRed}{\underline{55.3}}$ & $\textcolor{BrickRed}{\underline{58.9}}$ && $36.8$ & $45.3$ & $49.0$ & $\textcolor{BrickRed}{\underline{52.3}}$ && $\textcolor{BrickRed}{\underline{56.8}}$ & $26.0$ & $66.5$ & $\textcolor{BrickRed}{\underline{77.3}}$ \\
            & NatADiff\textsuperscript{U} & & $99.3^*$ & $\textcolor{red}{\mathbf{68.3}}$ & $\textcolor{red}{\mathbf{72.1}}$ & $\textcolor{red}{\mathbf{65.3}}$ & $\textcolor{red}{\mathbf{66.8}}$ && $\textcolor{red}{\mathbf{45.3}}$ & $\textcolor{red}{\mathbf{64.1}}$ & $\textcolor{red}{\mathbf{65.2}}$ & $\textcolor{red}{\mathbf{67.0}}$ && $\textcolor{red}{\mathbf{68.2}}$ & $\textcolor{red}{\mathbf{43.2}}$ & $51.4$ & $95.9$ \\
            \specialrule{1.2pt}{0pt}{0pt}
            \noalign{\vskip 0.2ex}
            \multicolumn{1}{c}{\multirow{9}{*}{Inc-v3}} & PGD & & $6.0$ & $99.7^*$ & $4.0$ & $5.1$ & $10.4$ && $10.2$ & $4.1$ & $5.6$ & $7.4$ && $16.9$ & - & - & - \\
            & AA & & $7.3$ & $\textcolor{red}{\mathbf{100^*}}$ & $4.9$ & $4.8$ & $12.8$ && $10.6$ & $5.7$ & $6.1$ & $8.0$ && $17.8$ & - & - & - \\
            & NCF & & $31.0$ & $66.7^*$ & $23.1$ & $29.0$ & $36.3$ && $15.8$ & $18.3$ & $20.4$ & $30.5$ && $30.1$ & $31.7$ & $69.1$ & $83.0$ \\
            & DiffAttack & & $29.0$ & $74.6^*$ & $23.7$ & $30.0$ & $39.9$ && $18.9$ & $22.9$ & $26.5$ & $25.8$ && $32.4$ & $33.2$ & $63.7$ & $\textcolor{red}{\mathbf{78.2}}$ \\
            & ACA & & $50.9$ & $67.8^*$ & $48.2$ & $54.2$ & $60.1$ && $\textcolor{BrickRed}{\underline{43.6}}$ & $\textcolor{BrickRed}{\underline{45.1}}$ & $\textcolor{BrickRed}{\underline{48.8}}$ & $\textcolor{BrickRed}{\underline{51.3}}$ && $52.2$ & $23.1$ & $68.0$ & $\textcolor{BrickRed}{\underline{78.8}}$ \\
            & AdvClass\textsuperscript{T} & & $35.1$ & $99.6^*$ & $34.5$ & $35.6$ & $39.5$ && $28.8$ & $32.4$ & $34.0$ & $35.7$ && $41.7$ & $33.7$ & $51.0$ & $89.2$ \\
            & AdvClass\textsuperscript{U} & & $38.0$ & $\textcolor{BrickRed}{\underline{99.9^*}}$ & $38.7$ & $40.4$ & $44.2$ && $30.0$ & $36.0$ & $36.6$ & $38.9$ && $44.8$ & $\textcolor{BrickRed}{\underline{39.7}}$ & $\textcolor{red}{\mathbf{49.4}}$ & $93.3$ \\
            \cline{2-18}
            \noalign{\vskip 0.5ex}
            & NatADiff\textsuperscript{T} & & $\textcolor{BrickRed}{\underline{53.4}}$ & $97.9^*$ & $\textcolor{BrickRed}{\underline{49.4}}$ & $\textcolor{BrickRed}{\underline{57.3}}$ & $\textcolor{BrickRed}{\underline{62.6}}$ && $35.4$ & $44.4$ & $45.1$ & $50.8$ && $\textcolor{BrickRed}{\underline{55.2}}$ & $27.7$ & $66.6$ & $\textcolor{red}{\mathbf{78.2}}$ \\
            & NatADiff\textsuperscript{U} & & $\textcolor{red}{\mathbf{67.4}}$ & $99.4^*$ & $\textcolor{red}{\mathbf{65.7}}$ & $\textcolor{red}{\mathbf{70.1}}$ & $\textcolor{red}{\mathbf{75.7}}$ && $\textcolor{red}{\mathbf{44.4}}$ & $\textcolor{red}{\mathbf{60.3}}$ & $\textcolor{red}{\mathbf{60.2}}$ & $\textcolor{red}{\mathbf{63.1}}$ && $\textcolor{red}{\mathbf{67.4}}$ & $\textcolor{red}{\mathbf{47.0}}$ & $\textcolor{BrickRed}{\underline{50.5}}$ & $98.9$ \\
            \specialrule{1.2pt}{0pt}{0pt}
            \noalign{\vskip 0.2ex}
            \multicolumn{1}{c}{\multirow{9}{*}{ViT-H}} & PGD & & $5.8$ & $11.0$ & $3.6$ & $4.0$ & $7.8$ && $96.2^*$ & $4.5$ & $5.4$ & $9.2$ && $16.4$ & - & - & - \\
            & AA & & $6.5$ & $9.8$ & $3.9$ & $4.3$ & $8.6$ && $\textcolor{red}{\mathbf{100^*}}$ & $4.5$ & $5.9$ & $9.9$ && $17.0$ & - & - & - \\
            & NCF & & $20.0$ & $19.4$ & $14.8$ & $15.4$ & $18.5$ && $50.6^*$ & $11.9$ & $15.6$ & $21.2$ && $20.8$ & $\textcolor{red}{\mathbf{39.8}}$ & $63.1$ & $86.4$ \\
            & DiffAttack & & $20.5$ & $25.0$ & $17.2$ & $18.9$ & $22.4$ && $73.2^*$ & $18.1$ & $22.3$ & $20.6$ && $26.5$ & $35.2$ & $63.4$ & $\textcolor{red}{\mathbf{80.0}}$ \\
            & ACA & & $50.5$ & $54.5$ & $48.1$ & $49.1$ & $52.8$ && $75.8^*$ & $47.5$ & $49.7$ & $50.5$ && $53.2$ & $25.5$ & $64.2$ & $\textcolor{BrickRed}{\underline{80.9}}$ \\
            & AdvClass\textsuperscript{T} & & $33.9$ & $35.9$ & $33.4$ & $34.4$ & $34.4$ && $92.6^*$ & $31.9$ & $33.4$ & $36.0$ && $40.7$ & $38.9$ & $\textcolor{red}{\mathbf{48.5}}$ & $95.2$ \\
            & AdvClass\textsuperscript{U} & & $35.2$ & $37.5$ & $35.8$ & $35.2$ & $36.0$ && $98.7^*$ & $33.9$ & $34.9$ & $37.7$ && $42.8$ & $\textcolor{BrickRed}{\underline{39.2}}$ & $\textcolor{red}{\mathbf{48.5}}$ & $98.8$ \\
            \cline{2-18}
            \noalign{\vskip 0.5ex}
            & NatADiff\textsuperscript{T} & & $\textcolor{red}{\mathbf{70.7}}$ & $\textcolor{red}{\mathbf{73.5}}$ & $\textcolor{red}{\mathbf{68.4}}$ & $\textcolor{red}{\mathbf{71.3}}$ & $\textcolor{red}{\mathbf{72.1}}$ && $98.5^*$ & $\textcolor{red}{\mathbf{65.7}}$ & $\textcolor{red}{\mathbf{66.9}}$ & $\textcolor{red}{\mathbf{71.7}}$ && $\textcolor{red}{\mathbf{73.2}}$ & $15.3$ & $88.0$ & $93.5$ \\
            & NatADiff\textsuperscript{U} & & $\textcolor{BrickRed}{\underline{66.8}}$ & $\textcolor{BrickRed}{\underline{67.0}}$ & $\textcolor{BrickRed}{\underline{65.3}}$ & $\textcolor{BrickRed}{\underline{64.9}}$ & $\textcolor{BrickRed}{\underline{65.8}}$ && $\textcolor{BrickRed}{\underline{99.6^*}}$ & $\textcolor{BrickRed}{\underline{63.9}}$ & $\textcolor{BrickRed}{\underline{65.4}}$ & $\textcolor{BrickRed}{\underline{68.6}}$ && $\textcolor{BrickRed}{\underline{69.7}}$ & $31.9$ & $\textcolor{BrickRed}{\underline{53.9}}$ & $96.2$ \\
            \specialrule{1.2pt}{0pt}{0pt}
            \noalign{\vskip 0.2ex}
        \end{tabular}
    }
    \vspace{-6mm}
\end{table}

\noindent\textbf{Attack success.}
NatADiff had comparable white-box ASR to current state-of-the-art attacks, but vastly superior transferability across all experiments (see Table~\ref{tab:Classifier ASR and image quality}). This suggests that NatADiff is able to more effectively leverage the mechanisms underpinning natural adversarial samples. Additionally, adversarial training did not offer any meaningful robustness to NatADiff, with both targeted and untargeted attacks transferring to adversarially trained ResNet and Inception models.

PGD \citep{Madry2019} and AA \citep{Croce2020} had the lowest transferability, likely because both are perturbation-based attacks, i.e., they rely on the small pockets of adversarial samples that neighbor natural images. These adversarial regions are not guaranteed to align across classifier architectures, especially architectures that are dissimilar, e.g., convolutional vs. transformer. Similarly, the lower transferability of NCF \citep{NCF_Yuan2022} and DiffAttack \citep{DiffAttack_Chen2025} can be explained by their limited attack surface. NCF is restricted to attacking the color distribution of a “clean” source image (see Figure~\ref{fig:Attack visual comparison}). In contrast, DiffAttack crafts adversarial perturbations for a source image by perturbing the diffusion latent space subject to the constraint that the reconstructed adversarial image must remain sufficiently close to the original (see Figure~\ref{fig:Attack visual comparison}).

Adversarial classifier guidance \citep{Dai2024} is outperformed by ACA \citep{ContentDiffusionAttack_Chen2023} and NatADiff in all experiments. This can be attributed to the limited guidance provided by injecting non-robust classifier gradients into the diffusion sampling trajectory \citep{Shen2024}. ACA is the most comparable to NatADiff performance-wise; however, ACA alters the semantic structure of a source image and is thus constrained by the semantics of the initial image. In contrast, NatADiff has a wider attack surface as it is free to generate any image that fools a surrogate classifier. Furthermore, NatADiff uses a diffusion model to incorporate adversarial features that are classifier-agnostic (as seen in Figure~\ref{fig:NatADiff samples}), and as such, is the only method that does not solely rely on the gradient of a surrogate classifier.

ViT-H \citep{Dosovitskiy2021} is the current state-of-the-art in image classification and is the most resistant to transfer attacks. This is unsurprising, as it uses the modern transformer architecture and is the largest model examined. ViT-H learns a more robust feature representation than convolutional and smaller transformer models, which makes it less susceptible to both constrained and natural adversarial samples. However, despite the strengths of the ViT-H architecture, NatADiff is able to reliably generate samples that transfer to ViT-H—albeit at a lower ASR than equivalent attacks against all other models.

When comparing NatADiff's targeted attacks with their untargeted counterparts, we see that untargeted attacks outperform targeted attacks both in terms of victim classifier performance and transferability for all classifier \textit{except} ViT-H. We believe that the discrepancy between targeted and untargeted ViT-H samples can be attributed to the strength of the decision boundary learnt by the ViT-H model. The diffusion model struggles to generate feasible \textit{targeted} adversarial samples for ViT-H, leading to the introduction of image artifacts (as seen in Figure~\ref{fig:Low-Quality Targeted ViT Samples} from Appendix~\ref{apdx:Low-Quality Targeted ViT Samples}, and supported by the IS and FID-Val scores in Table~\ref{tab:Classifier ASR and image quality}). These artifacts substantially degrade image quality and artificially inflate the attack success rate of the targeted ViT-H samples. Importantly, this degradation in image quality was limited to targeted ViT-H attacks. This indicates that some adversarial targets are easier to achieve than others, which further motivates the use of similarity targeting as a method for identifying classifier “weak spots.”

\noindent\textbf{Image quality.}
We observe a clear disparity in the image quality of targeted and untargeted NatADiff variants (see Table~\ref{tab:Classifier ASR and image quality}). Targeted NatADiff samples exhibit lower FID-A but worse IS and FID-VAL, indicating that they are closer in distribution to known natural adversarial examples, albeit with lower image quality and less alignment to the ImageNet validation dataset. In contrast, untargeted NatADiff achieves IS and FID-VAL comparable to other generative methods, but with a higher FID-A, suggesting that overall image quality improves at the expense of alignment with natural adversarial samples. This follows from the known characteristics of natural adversarial samples, which often blend features from disparate classes \citep{Hendrycks2021, ShortcutLearning_Geirhos2020, RiskMinimisation_Arjovsky2020}. Replicating such blending places greater demands on the underlying diffusion model to locate plausible points on the image manifold, which can introduce artifacts and degrade image quality. In contrast, similarity targeting blends more related classes, yielding samples with higher visual fidelity but less alignment with natural adversarial distributions. Additionally, NCF \citep{NCF_Yuan2022}, DiffAttack \citep{DiffAttack_Chen2025}, and ACA \citep{ContentDiffusionAttack_Chen2023} all achieve superior FID-A than AdvClass \citep{Dai2024} and untargeted NatADiff. This can be attributed to the low FID-A of the clean baseline dataset that NCF, DiffAttack, and ACA use as their source, which causes them to inherit the same distributional properties. Conversely, AdvClass and NatADiff generate artificial samples and are thus constrained both by the distributional tendencies of the underlying diffusion model and the effect of similarity targeting.

\section{Conclusion} \label{sec:Conclusion}
We introduce NatADiff, an adversarial sampling scheme that leverages diffusion models to generate highly transferable adversarial samples. Our method is motivated by the observation that natural adversarial samples frequently contain features from the adversarial class, which deep learning models exploit to shortcut the classification processes. To leverage this behavior, we guide the diffusion trajectory towards the intersection of the true and adversarial classes. Our method achieves comparable white-box attack success rates to current state-of-the-art techniques, while exhibiting significantly higher transferability across models. Furthermore, samples generated using NatADiff align more closely with known natural adversarial samples than those generated via adversarial classifier guidance alone. These results demonstrate that NatADiff produces adversarial samples that transfer more effectively than existing attacks, and more faithfully resemble naturally occurring test-time errors than those generated from vanilla adversarial diffusion guidance.

\subsubsection*{Acknowledgments}
Max Collins is the recipient of the Research Training Program scholarship funded by the Australian Government. Professor Ajmal Mian is the recipient of an Australian Research Council Future Fellowship Award (project number FT210100268) funded by the Australian Government. This research was partially supported by National Intelligence and Security Discovery Research Grants (project number NS220100007), funded by the Department of Defence, Australia.

\bibliographystyle{iclr2026_conference}
\bibliography{./references}

\appendix
\addcontentsline{toc}{section}{Appendix} 
\part{Appendix} 
\parttoc 

\newpage
\section{Limitations}
While NatADiff is effective at generating highly transferable adversarial samples, it remains computationally expensive due to the iterative nature of diffusion and the overhead introduced by adversarial guidance. This is an inherent limitation of diffusion-based generative methods, and one unlikely to change without significant advances in generative modeling or architectural design. Additionally, the use of similarity targeting on datasets like ImageNet can lead to \textit{subtle} misclassifications--e.g., between similar dog breeds--which may diminish the perceived severity of the attack. A potential refinement would be to surface a ranked list of similar classes, allowing users to select more divergent adversarial targets while retaining the semantic grounding of similarity-based selection. We also note that we used a conservative setting for the adversarial boundary guidance term, $\mu$, as larger values caused generated samples to occasionally include the adversarial class, as discussed in Appendix~\ref{apdx:Selection of mu}. Finally, we restrict our evaluation to ImageNet classifiers, as ImageNet offers a diverse label space, which supports varied attack scenarios. Extending NatADiff to more specialized domains remains an avenue for future work.

\section{Ethics and broader impacts}
We adhere to the ICLR code of ethics. We acknowledge that this work explores the use of generative models as a means of creating highly transferable adversarial samples. While adversarial attacks raise legitimate concerns regarding misuse, our objective is to expose fundamental vulnerabilities of current classifiers and to better understand the structure of natural adversarial samples. By making our models and code publicly available, we aim to support transparency and reproducibility, and we believe that insight into generative adversarial mechanisms is a necessary step toward building more secure and interpretable classifiers. We do not use private or sensitive data, and all data and models used are publicly released and broadly studied. In future work, we plan to explore how NatADiff can be extended to detect or defend against naturally occurring adversarial samples.

\section{Reproducibility}
We ensure reproducibility by providing detailed descriptions of our algorithms (see Algorithms~\ref{alg:NatADiff} and \ref{alg:NatADiff Similarity}) and experiment parameter settings (see Tables~\ref{tab:ImageNet experiment parameters} and \ref{tab:Oxford Pets experiment parameters}). Our full codebase is included in the supplementary material, along with all configuration files required to replicate our experiments. Comparison methods are implemented using publicly available repositories, and we follow the authors' recommended hyperparameters unless otherwise stated.

\section{Perturbation-based attacks and defences}
In this section, we provide a brief overview of existing constrained perturbation-based attack and defense strategies for image classification models. We focus on the optimization-based formulation of adversarial attacks and highlight the theoretical underpinnings of common training-time defenses.

\subsection{Adversarial attacks}
\cite{Szegedy2014} were the first to demonstrate that imperceptible perturbations to an image's pixel values could cause deep learning models to misclassify the image with a high probability (see Figure~1). Mathematically, these constrained adversarial attacks can be considered a solution to the following constrained optimization problem: 
\begin{equation}
\label{eq:PGD formulation}
    \min_{\boldsymbol{\delta} \in \mathcal{S}} \ \mathcal{L}_h(\boldsymbol{x} + \boldsymbol{\delta}; \boldsymbol{\theta}, \tilde{y}),
\end{equation} 
where $\boldsymbol{\delta} \in \mathbb{R}^m$ is the computed perturbation, $\boldsymbol{x} \in \mathbb{R}^m$ is the vectorized ``clean'' image, $h : \mathbb{R}^m \xrightarrow{} \mathcal{Y}$ is a ``trained'' classifier model with parameters $\boldsymbol{\theta}$, $\tilde{y} \in \mathcal{Y}$ is the class targeted by the adversarial attack, $\mathcal{L}_h(\cdot; \boldsymbol{\theta}, \tilde{y})$ is the loss of the classifier with respect to the target adversarial class, and $\mathcal{S} \triangleq \{ \boldsymbol{\delta} \in \mathbb{R}^m : \lVert \boldsymbol{\delta} \rVert_p < L \}$ is a convex set of allowable perturbation sizes. Algorithms such as fast gradient sign method (FGSM) \citep{Goodfellow2015}, projected gradient descent (PGD) \citep{Madry2019}, and AutoAttack \citep{Croce2020}, have been proposed to efficiently solve the optimization problem in (\ref{eq:PGD formulation}).

Other attack methods have relaxed the constraint on the magnitude of the adversarial perturbation. These unconstrained adversarial attacks seek to alter the semantic information within an image, resulting in misclassification without visually altering the perceptible class. Additionally, techniques like selective cropping and rotation, texture remapping, color pallette transformations, and generative sampling have all been used to successfully ``fool'' modern deep learning models \citep{Brown2018, Bhattad2020, Song2018, ATGAN_Wang2020, Dai2024, Xie2025}. 

\subsection{Defences against adversarial attacks}
Several defensive measures have been proposed that aim to purify adversarial inputs \citep{Samangouei2018, Nie2022}, harden model architectures against attacks \citep{Gu2015, Xu2018}, or improve training procedures \citep{Szegedy2014, Madry2019}. A key challenge in designing adversarial defences is preventing attackers from crafting new attacks that exploit the adapted model. For this reason \textit{adversarial training} has become one of the most popular defences, as it both addresses the source of the adversarial attack, while providing theoretical guarantees of robustness against all possible perturbation-based adversaries.

Adversarial training can be formulated as the following saddle-point optimization problem: 
\begin{equation}
    \label{eq:adversarial training formulation}
    \min_{\boldsymbol{\theta}} \ {\E}_{(\boldsymbol{x}, y) \sim \mathcal{D}} \left[ \max_{\boldsymbol{\delta} \in \mathcal{S}} \ \mathcal{L}_h(\boldsymbol{x} + \boldsymbol{\delta}; \boldsymbol{\theta}, y) \right],
\end{equation} 
where $\mathcal{D}$ is the joint distribution of naturally occurring images and classes, and $y \in \mathcal{Y}$ is the true class label of $\boldsymbol{x}$ \citep{Szegedy2014,Madry2019}. The optimization problem in (\ref{eq:adversarial training formulation}) can be thought of as minimizing the loss caused by the strongest possible adversarial attack. Thus, any model that minimizes (\ref{eq:adversarial training formulation}) is theoretically guaranteed to be resistant to its strongest possible adversarial perturbations.

\section{Example natural adversarial samples} \label{apdx:NaturalAdversarial}
Figure~\ref{fig:NaturalAdversarial} shows natural adversarial samples from the ImageNet-A dataset \citep{Hendrycks2021}, each paired with a heatmap of the classifier-guidance gradient with respect to the adversarial class, $\nabla_{\boldsymbol{x}}\log(p(\tilde{y} \mid \boldsymbol{x}))$. These gradients highlight the image features that contribute to misclassification and that would be emphasized during adversarial classifier-guided diffusion \citep{Dai2024}. In the first image, the classifier gradient is concentrated around the school bus and the snowbanks running alongside the road; in the second, it is concentrated on the snail and its shadow; and in the third, it is concentrated on the power switch. This suggests the classifier has “learned” to associate vehicles beside snowbanks with snowplows, dark elliptical objects with cockroaches, and vertical rectangular boxes with pay phones.

\begin{figure}
    \centering
    \resizebox{0.7\textwidth}{!}{
    {
    \large
    \begin{tabular}{@{}c@{\hspace{0.5em}}c@{\hspace{0.5em}}c@{}}
        \multicolumn{1}{@{}p{0.4\linewidth}@{}}{\textcolor{Green}{School Bus} \hfill  \textcolor{red}{Snowplow ($34 \%$)}} &
        \multicolumn{1}{@{}p{0.4\linewidth}@{}}{\textcolor{Green}{Snail} \hfill  \textcolor{red}{Cockroach ($43 \%$)}} &
        \multicolumn{1}{@{}p{0.4\linewidth}@{}}{\textcolor{Green}{Robin} \hfill  \textcolor{red}{Pay-Phone ($22 \%$)}} \\

        \includegraphics[width=0.4\linewidth]{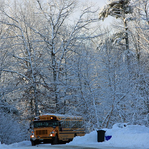} &
        \includegraphics[width=0.4\linewidth]{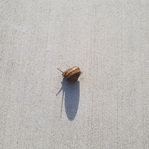} &
        \includegraphics[width=0.4\linewidth]{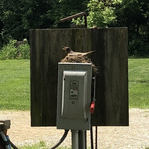} \\[0.5em]
        
        {\boldmath $\nabla_{\boldsymbol{x}}\log (p(\tilde{y} | \boldsymbol{x}))$ \unboldmath} & {\boldmath $\nabla_{\boldsymbol{x}}\log (p(\tilde{y} | \boldsymbol{x}))$ \unboldmath} & {\boldmath $\nabla_{\boldsymbol{x}}\log (p(\tilde{y} | \boldsymbol{x}))$ \unboldmath} \\

        \includegraphics[width=0.4\linewidth]{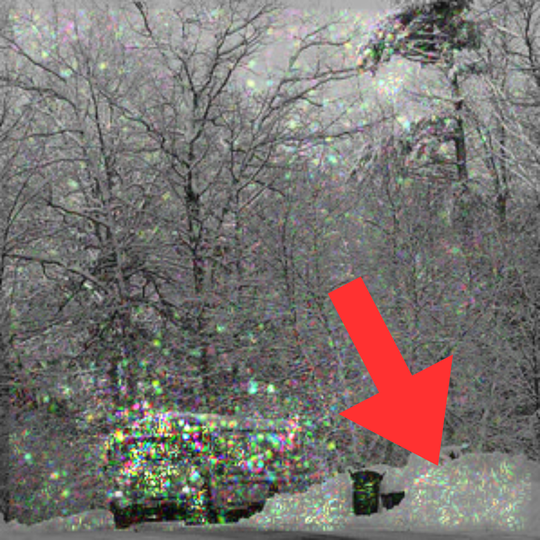} &
        \includegraphics[width=0.4\linewidth]{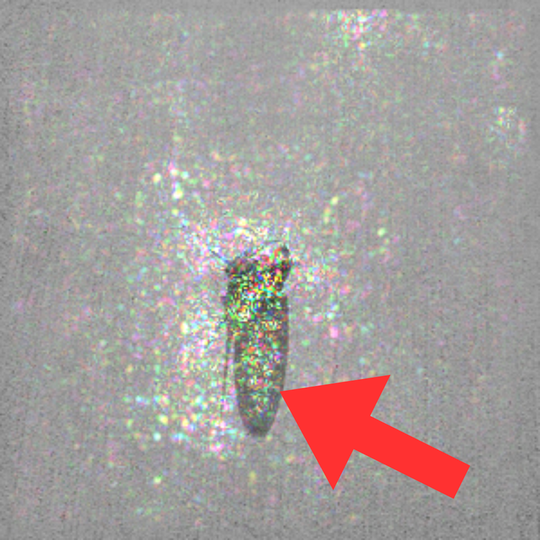} &
        \includegraphics[width=0.4\linewidth]{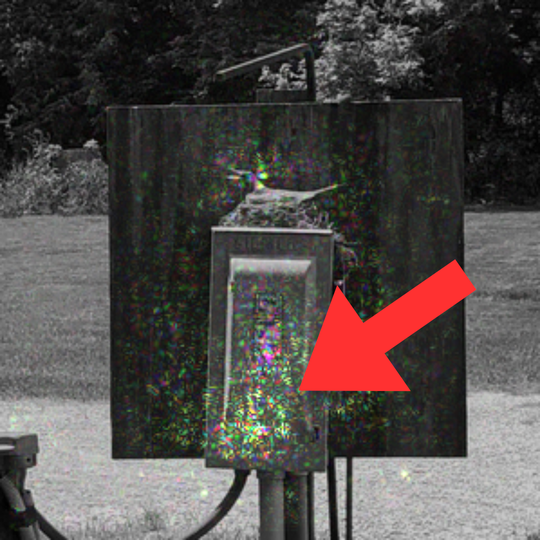} \\
    \end{tabular}
    }
    }
    \caption{\textbf{Top:} Natural adversarial samples compiled by \cite{Hendrycks2021} for ImageNet \citep{Deng2009} classifiers. The \textcolor{Green}{green} labels denote the ground-truth classes; the \textcolor{red}{red} labels are the classes assigned by a ResNet-50 classifier \citep{He2015}. \textbf{Bottom:} Heatmap of the ResNet-50 adversarial classifier-guidance \citep{Dai2024} gradient with respect to the adversarial classes. Arrows point to features from the adversarial class that affect the ResNet-50 classification.}
    \label{fig:NaturalAdversarial}
\end{figure}

\clearpage
\section{NatADiff similarity attack algorithm} \label{apdx:NatADiff Algorithm}
Algorithm~\ref{alg:NatADiff Similarity} provides the algorithm for the similarity targeted variant of NatADiff.

\begin{algorithm}[H]
    \scriptsize
    \caption{NatADiff--Similarity} \label{alg:NatADiff Similarity}
    \begin{algorithmic}
        \Require adversarial guidance parameters: $\omega$, $\rho$, $\mu$, $s$; true class: $y$; candidate adversarial classes: $\mathcal{Y}_{\text{cand}} = \{\tilde{y}_1, \tilde{y}_2, \dots \}$; victim classifier: $h$; forward diffusion functions: $\alpha(t)$, $\beta(t)$; stable diffusion model: $\boldsymbol{\epsilon}_{\theta^\star}$; VAE decoder: $V_{\text{dec}}$; CLIP text encoder: $C_{\text{enc}}$; collection of differentiable image transforms: $\{ \mathcal{T}_1, \mathcal{T}_2, \dots \}$; sequence of sampling steps with $t_1 = 0$, $t_N = T$, and $t_{i+1} > t_i$: $\{ t_{i} \}_{i=1}^N$; time-travel parameters: $R$, $r_l$, $r_u$; adversarial classifier bounds: $c_l$, $c_u$; number sampling attempts: $S$; guidance scalers: $\delta_\mu$, $\delta_s$\\

        \State $\tilde{y} = \underset{\gamma \in \mathcal{Y}{\text{cand}}}{\textup{arg} \min} \frac{C_{\text{enc}}(y) \cdot C_{\text{enc}}(\gamma)}{\lVert C_{\text{enc}}(y) \rVert_2 \lVert C_{\text{enc}}(\gamma) \rVert_2}$
        \State $\boldsymbol{z}_T \sim \mathcal{N}(0, I)$
        \For{$s = 1, \dots, S$}
            \For{$i = N, \dots, 1$}
                \If{$r_l \leq t_i \leq r_u$}
                    \State $\tilde{R} = R$
                \Else
                    \State $\tilde{R} = 1$
                \EndIf
                \For{$r = \tilde{R}, \dots, 1$} \Comment{Time-travel loop}
                    \State $\boldsymbol{v}_y = \boldsymbol{\epsilon}_{\theta^\star}(\boldsymbol{z}_{t_i}, t_i, y) - \boldsymbol{\epsilon}_{\theta^\star}(\boldsymbol{z}_{t_i}, t_i)$
                    \State $\boldsymbol{v}_{y \cap \tilde{y}} = \boldsymbol{\epsilon}_{\theta^\star}(\boldsymbol{z}_{t_i}, t_i, y \cap \tilde{y}) - \boldsymbol{\epsilon}_{\theta^\star}(\boldsymbol{z}_{t_i}, t_i)$
                    \State $\hat{\boldsymbol{\epsilon}} = \boldsymbol{\epsilon}_{\theta^\star}(\boldsymbol{x}_{t_i}, t_i) + (\omega - \mu\omega) \boldsymbol{v}_y + \mu \rho \boldsymbol{v}_{y \cap \tilde{y}}$
                    \If{$c_l \leq t \leq c_u$}
                        \State $\hat{\boldsymbol{x}}_0 = V_{\text{dec}} \left( \frac{\boldsymbol{z}_{t_i} - \beta(t_i) \hat{\boldsymbol{\epsilon}}}{\alpha(t_i)} \right)$
                        \State $\boldsymbol{g} = \nabla_{\boldsymbol{z}_{t_i}} \log \left( \sigma_{\tilde{y}} \left( \frac{1}{\lvert \boldsymbol{\mathcal{T}} \rvert} \sum_{j=1}^{\lvert \boldsymbol{\mathcal{T}} \rvert} h(\mathcal{T}_j(\hat{\boldsymbol{x}}_0)) \right) \right)$
                        \State $\boldsymbol{g} = \frac{\boldsymbol{g}}{\lVert \boldsymbol{g} \rVert_2}$
                        \State $\hat{\boldsymbol{\epsilon}} = \hat{\boldsymbol{\epsilon}} - s \beta(t) \boldsymbol{g}$
                    \EndIf
                    \State $\boldsymbol{z}_{t_{i-1}}$ $\gets$ \text{reverse diffusion step using $\hat{\boldsymbol{\epsilon}}$}
                    \If{$r > 1$} \Comment{Sampling $\boldsymbol{z}_{t_i} \sim p(\boldsymbol{z}_{t_i} | \boldsymbol{z}_{t_{i-1}})$}
                        \State $a = \frac{\alpha(t_i)}{\alpha(t_{i-1})}$
                        \State $b^2 = \beta(t_i)^2 - \left( a\beta(t_{i-1}) \right)^2$
                        \State $\boldsymbol{z}_{t_i} \sim \mathcal{N} \left(a \boldsymbol{z}_{t_{i-1}}, b^2 \cdot I \right)$
                    \EndIf
                \EndFor
            \EndFor
            \If{$\text{argmax}(h(V_{\text{dec}}(\boldsymbol{z}_{0}))) \neq \tilde{y}$}
                \State $\mu = \mu + \delta_{\mu}$
                \State $s = s + \delta_s$
            \Else
                \State \textbf{break} \Comment{End the search early if sample is found}
            \EndIf
        \EndFor
        \State \textbf{return} $V_{\text{dec}}(\boldsymbol{z}_{0})$
    \end{algorithmic}
\end{algorithm}

\section{NatADiff ablation studies}
Here we provide ablation studies examining the effect of classifier augmentations, adversarial boundary guidance, and adversarial classifier guidance. We also include a visualisation of how these components influence the generated image (see Figure~\ref{fig:NatADiff Grad Breakdown}). Recall from the main paper that augmented adversarial classifier guidance introduces visual features from the adversarial class (see Appendix~\ref{apdx:Effect of classifier augmentations}), adversarial boundary guidance further increases the amount of adversarial structure added and improves image quality (see Appendix~\ref{apdx:Selection of mu}), and time-travel sampling provides an additional improvement in image quality \citep{Lugmayr2022} (see Figure~\ref{fig:NatADiff Grad Breakdown}).

\begin{figure}
    \centering
    \resizebox{1\textwidth}{!}{
    {\footnotesize
    \begin{tabular}{@{}|@{\,}l@{\,}|@{}c@{}|@{}c@{}|@{}c@{}|@{}c@{}|@{}c@{}|@{}}
        \hline
        \textbf{Adv. Class. Guid.} & \xmark & \cmark & \cmark & \cmark & \cmark \\ 
        \textbf{Class. Aug.} & \xmark & \xmark & \cmark & \cmark & \cmark \\ 
        \textbf{Adv. Bound. Guid.} & \xmark & \xmark & \xmark & \cmark & \cmark \\ 
        \textbf{Time-Travel Samp.} & \xmark & \xmark & \xmark & \xmark & \cmark \\ 
        \hline
        \multicolumn{1}{|c|}{\textbf{Image}} & 
        \includegraphics[width=0.2\textwidth, align=c]{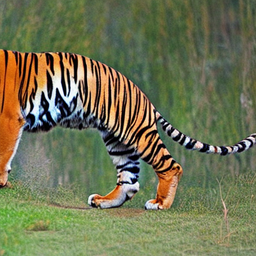} & 
        \includegraphics[width=0.2\textwidth, align=c]{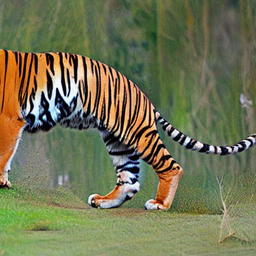} &
        \includegraphics[width=0.2\textwidth, align=c]{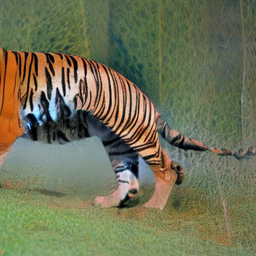} &
        \includegraphics[width=0.2\textwidth, align=c]{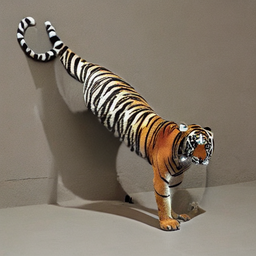} &
        \includegraphics[width=0.2\textwidth, align=c]{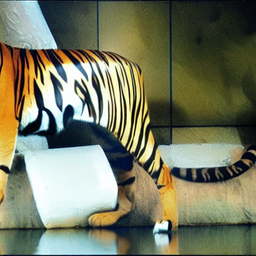} \\
        \hline
        \vspace{-2ex} & & & & & \\
        \multicolumn{1}{|c|}{\textbf{Classification}} &
        \makecell[ll]{
            \textbf{Res:} \textcolor{Green}{T $35 \%$} \\ 
            \textbf{Inc:} \textcolor{Green}{T $77 \%$} \\
            \textbf{ViT:} \textcolor{Green}{T $88 \%$} \\
            \textbf{ARes:} \textcolor{Green}{T $72 \%$} \\
            \textbf{AInc:} \textcolor{Green}{T $78 \%$} \\
        } &
        \makecell[ll]{
            \textbf{Res:} \textcolor{red}{TP $83 \%$} \\ 
            \textbf{Inc:} \textcolor{Green}{T $64 \%$} \\
            \textbf{ViT:} \textcolor{Green}{T $83 \%$} \\
            \textbf{ARes:} \textcolor{Green}{T $74 \%$} \\
            \textbf{AInc:} \textcolor{Green}{T $75 \%$} \\
        } &
        \makecell[ll]{
            \textbf{Res:} \textcolor{red}{TP $98 \%$} \\ 
            \textbf{Inc:} \textcolor{Green}{T $72 \%$} \\
            \textbf{ViT:} \textcolor{Green}{T $84 \%$} \\
            \textbf{ARes:} \textcolor{Green}{T $65 \%$} \\
            \textbf{AInc:} \textcolor{Green}{T $75 \%$} \\
        } &
        \makecell[ll]{
            \textbf{Res:} \textcolor{red}{TP $32 \%$} \\ 
            \textbf{Inc:} \textcolor{BurntOrange}{TC $20 \%$} \\
            \textbf{ViT:} \textcolor{Green}{T $62 \%$} \\
            \textbf{ARes:} \textcolor{red}{TP $32 \%$} \\
            \textbf{AInc:} \textcolor{Green}{T $52 \%$} \\
        } & 
        \makecell[ll]{
            \textbf{Res:} \textcolor{red}{TP $100 \%$} \\ 
            \textbf{Inc:} \textcolor{red}{TP $96 \%$} \\
            \textbf{ViT:} \textcolor{red}{TP $83 \%$} \\
            \textbf{ARes:} \textcolor{red}{TP $98 \%$} \\
            \textbf{AInc:} \textcolor{red}{TP $89 \%$} \\
        } \\
        \hline 
    \end{tabular}
    }
    }
    \caption{Effect of adversarial classifier guidance, classifier augmentations, adversarial boundary guidance, and time-travel sampling on samples generated by NatADiff. Prompt = ``tiger'', adversarial target = ``toilet paper'', surrogate model = ResNet-50 \citep{He2015}. Classification scores are given for ResNet-50, Inception-v3 \citep{Inceptionv3_Szegedy2016}, ViT-H \citep{Dosovitskiy2021}, and adversarially trained ResNet-50 and Inception victim models \citep{Kurakin2018}. Note: ``T'': ``Tiger'', ``TP'': ``Toilet Paper'', ``TC'': ``Tiger Cat''.}
    \label{fig:NatADiff Grad Breakdown}
\end{figure}

\subsection{Effect of classifier augmentations} \label{apdx:Effect of classifier augmentations}
In the main paper, we argue that classifier augmentations reduce the local adversarial signal of the surrogate classifier gradient and thereby encourage the diffusion model to introduce semantically meaningful adversarial features into the generated image. To verify this claim, we compare NatADiff samples produced with and without classifier augmentations. Using a ResNet-50 \citep{He2015} surrogate classifier, we report Attack Success Rate (ASR), Inception Score (IS) \citep{Salimans2016}, and Fréchet Inception Distance (FID) \citep{Frechet1957}. For FID, we evaluate with respect to both the ImageNet-Val \citep{Deng2009} and ImageNet-A \citep{Hendrycks2021} datasets to assess how closely NatADiff samples resemble natural images and known natural adversarial examples, respectively. The exact parameter settings for each experiment are provided in Table~\ref{tab:ImageNet experiment parameters}.

\begin{figure*}
    \centering
    {
    \resizebox{\linewidth}{!}{%
    \small
    \begin{tabular}{@{}l@{}l@{}l@{}l@{}l@{}l@{}}
        \multicolumn{6}{c}{\normalsize\textbf{NatADiff without Classifier Augmentations}} \\
        \includegraphics[width=0.2\textwidth]{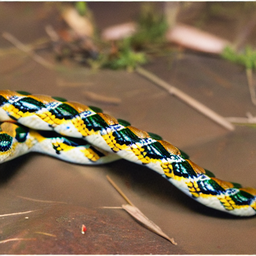} & \includegraphics[width=0.2\textwidth]{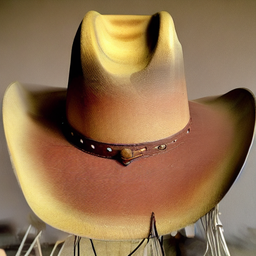} & \includegraphics[width=0.2\textwidth]{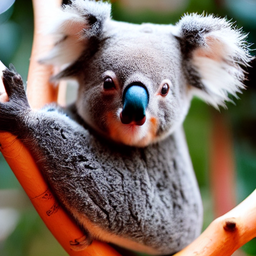} & \includegraphics[width=0.2\textwidth]{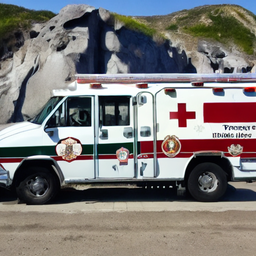} & \includegraphics[width=0.2\textwidth]{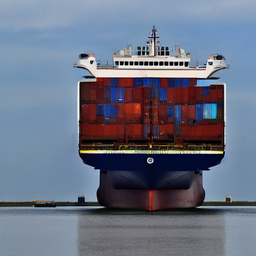} & \includegraphics[width=0.2\textwidth]{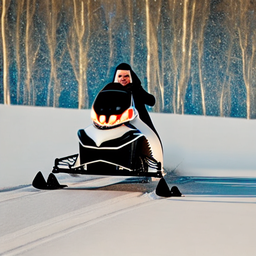} \\

        \multicolumn{6}{c}{\normalsize\textbf{NatADiff with Classifier Augmentations}} \\
        \includegraphics[width=0.2\textwidth]{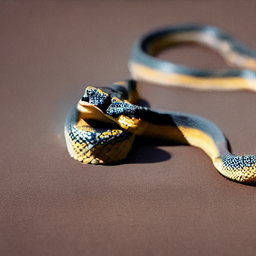} & \includegraphics[width=0.2\textwidth]{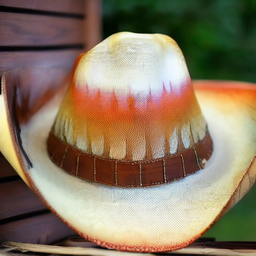} & \includegraphics[width=0.2\textwidth]{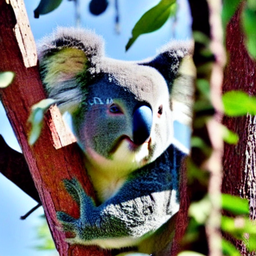} & \includegraphics[width=0.2\textwidth]{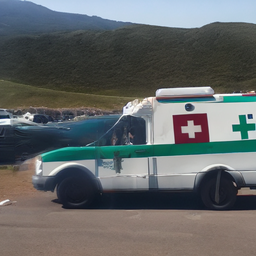} & \includegraphics[width=0.2\textwidth]{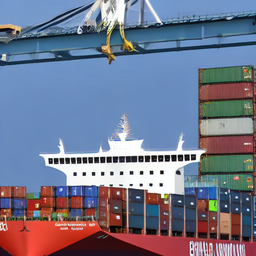} & \includegraphics[width=0.2\textwidth]{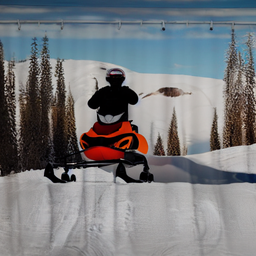} \\

        \textcolor{Green}{\textbf{True:}} Snake & \textcolor{Green}{\textbf{True:}} Cowboy Hat & \textcolor{Green}{\textbf{True:}} Koala & \textcolor{Green}{\textbf{True:}} Ambulance & \textcolor{Green}{\textbf{True:}} Container Ship & \textcolor{Green}{\textbf{True:}} Snowmobile \\
        \textcolor{red}{\textbf{Adv:}} Bow Tie & \textcolor{red}{\textbf{Adv:}} Mushroom & \textcolor{red}{\textbf{Adv:}} Birdhouse & \textcolor{red}{\textbf{Adv:}} Headland & \textcolor{red}{\textbf{Adv:}} Kite & \textcolor{red}{\textbf{Adv:}} Shower Curtain \\
    \end{tabular}}
    }
    \caption{Comparison of samples generated by NatADiff under targeted attack settings with and without classifier augmentations. We use a ResNet-50 \citep{He2015} surrogate model. We report the true class and adversarial target for each image.}
    \label{fig:NatADiff ablation - no augmentation}
\end{figure*}

\begin{table}
    \centering
    \caption{\textbf{Attack success rate} (\%) and \textbf{image quality} of adversarial samples generated by NatADiff under targeted attack settings with and without classifier augmentations. We use a ResNet-50 \citep{He2015} surrogate model. \textcolor{red}{\textbf{Bold}} values highlight the best score. White-box ASR (same surrogate and victim model) is denoted with an $^*$.}
    \label{tab:NatADiff ablation - no augmentation}
    \resizebox{1\textwidth}{!}{
        \begin{tabular}{@{}c@{}c@{\hspace{1em}}ccccc@{}c@{\hspace{1em}}cccc@{}ccccc@{}}
            \specialrule{1.2pt}{0pt}{0pt}
            \noalign{\vskip 0.5ex}
            \multicolumn{1}{c}{\multirow{3}{*}{Attack}} && \multicolumn{10}{c}{Victim Model ASR (\%)} && \multicolumn{1}{c}{\multirow{3}{*}{\shortstack{Average\\ASR}}} & \multicolumn{1}{c}{\multirow{3}{*}{\shortstack{IS\\($\boldsymbol{\uparrow}$)}}} & \multicolumn{1}{c}{\multirow{3}{*}{\shortstack{FID-\\Val ($\boldsymbol{\downarrow}$)}}} & \multicolumn{1}{c}{\multirow{3}{*}{\shortstack{FID-\\A ($\boldsymbol{\downarrow}$)}}} \\
            && \multicolumn{5}{c}{CNNs} && \multicolumn{4}{c}{Transformers} &&&& \\
            \cline{3-7} \cline{9-12} 
            \noalign{\vskip 0.5ex}
            && RN-50 & Inc-v3 & RN-152 & AdvRes & AdvInc && ViT-H & Max-ViT & Swin-B & DeIT &&&& \\
            \specialrule{1.2pt}{0pt}{0pt}
            \noalign{\vskip 0.5ex}
            NatADiff & & $96.9^*$ & $\textcolor{red}{\mathbf{60.1}}$ & $\textcolor{red}{\mathbf{56.5}}$ & $\textcolor{red}{\mathbf{55.3}}$ & $\textcolor{red}{\mathbf{58.9}}$ && $\textcolor{red}{\mathbf{36.8}}$ & $\textcolor{red}{\mathbf{45.3}}$ & $\textcolor{red}{\mathbf{49.0}}$ & $\textcolor{red}{\mathbf{52.3}}$ && $\textcolor{red}{\mathbf{56.8}}$ & $26.0$ & $66.5$ & $\textcolor{red}{\mathbf{77.3}}$ \\
            NatADiff (No-Aug) & & $\textcolor{red}{\mathbf{98.7^*}}$ & $48.5$ & $45.6$ & $44.1$ & $46.8$ && $31.7$ & $38.6$ & $40.6$ & $43.3$ && $48.7$ & $\textcolor{red}{\mathbf{30.5}}$ & $\textcolor{red}{\mathbf{56.7}}$ & $81.7$ \\
            \specialrule{1.2pt}{0pt}{0pt}
            \noalign{\vskip 0.2ex}
        \end{tabular}
    }
\end{table}

From Table~\ref{tab:NatADiff ablation - no augmentation} we observe that classifier augmentations significantly increase the transferability of adversarial samples, while retaining comparable white-box ASR (see Table~\ref{tab:NatADiff ablation - no augmentation}). Images generated with classifier augmentations have slightly reduced overall image quality (IS and FID-VAL), but improved FID-A. This suggests that classifier augmentations introduce slightly more generative artifacts in an image, but also incorporate more meaningful adversarial features, which produces images that align more closely with known natural adversarial samples (see Figures~\ref{fig:NatADiff Grad Breakdown} and \ref{fig:NatADiff ablation - no augmentation})

\subsection{Effect of boundary guidance strength} \label{apdx:Selection of mu}
The adversarial boundary guidance term, $\mu$, governs how strongly features from the adversarial class are incorporated into the generated sample. To evaluate the effect of this parameter, we conduct an ablation study across $\mu \in \{ 0.0, 0.1, 0.2, 0.3, 0.4, 0.5 \}$ using a ResNet-50 \citep{He2015} surrogate model and applying NatADiff in targeted mode, i.e., with adversarial classes selected at random. We report attack success rate, Inception Score (IS) \citep{Salimans2016} and Fréchet Inception Distance (FID) \citep{Frechet1957}. For FID, we evaluate with respect to both the ImageNet-Val \citep{Deng2009} and ImageNet-A \citep{Hendrycks2021} datasets to assess how closely NatADiff samples resemble natural images and known natural adversarial examples, respectively. The exact parameter settings for each experiment are provided in Table~\ref{tab:ImageNet experiment parameters}.

\begin{table}[h]
    \centering
    \caption{\textbf{Attack success rate} (ASR) and image quality of adversarial samples generated by NatADiff under targeted attack settings with varying adversarial boundary guidance strength, $\mu$. We use a ResNet-50 \citep{He2015} surrogate model. \textcolor{red}{\textbf{Bold}} values highlight the best score. White-box ASR (same surrogate and victim model) is denoted with an $^*$.}
    \label{tab:NatADiff ablation - boundary guidance}
    \resizebox{1\textwidth}{!}{
        \begin{tabular}{@{}c@{}c@{\hspace{1em}}ccccc@{}c@{\hspace{1em}}cccc@{}ccccc@{}}
            \specialrule{1.2pt}{0pt}{0pt}
            \noalign{\vskip 0.5ex}
            \multicolumn{1}{@{}c@{}}{\multirow{3}{*}{Attack}} && \multicolumn{10}{c}{Victim Model ASR (\%)} && \multicolumn{1}{c}{\multirow{3}{*}{\shortstack{Average\\ASR}}} & \multicolumn{1}{c}{\multirow{3}{*}{\shortstack{IS\\($\boldsymbol{\uparrow}$)}}} & \multicolumn{1}{c}{\multirow{3}{*}{\shortstack{FID-\\Val ($\boldsymbol{\downarrow}$)}}} & \multicolumn{1}{c}{\multirow{3}{*}{\shortstack{FID-\\A ($\boldsymbol{\downarrow}$)}}} \\
            && \multicolumn{5}{c}{CNNs} && \multicolumn{4}{c}{Transformers} &&&& \\
            \cline{3-7} \cline{9-12} 
            \noalign{\vskip 0.5ex}
            && RN-50 & Inc-v3 & RN-152 & AdvRes & AdvInc && ViT-H & Max-ViT & Swin-B & DeIT &&&& \\
            \specialrule{1.2pt}{0pt}{0pt}
            \noalign{\vskip 0.5ex}
            NatADiff ($\mu=0.0$) & & $95.2^*$ & $54.2$ & $49.6$ & $48.7$ & $53.7$ && $32.6$ & $42.4$ & $43.7$ & $48.4$ && $52.1$ & $26.1$ & $67.7$ & $78.9$ \\
            NatADiff ($\mu=0.1$) & & $95.4^*$ & $55.2$ & $52.5$ & $51.5$ & $53.9$ && $33.8$ & $44.2$ & $45.0$ & $49.3$ && $53.4$ & $26.6$ & $63.6$ & $78.1$ \\
            NatADiff ($\mu=0.2$) & & $96.9^*$ & $60.1$ & $56.5$ & $55.3$ & $58.9$ && $36.8$ & $45.3$ & $49.0$ & $52.3$ && $56.8$ & $26.0$ & $66.5$ & $\textcolor{red}{\mathbf{77.3}}$ \\
            NatADiff ($\mu=0.3$) & & $97.4^*$ & $62.4$ & $60.0$ & $57.8$ & $61.2$ && $42.6$ & $50.7$ & $53.4$ & $55.0$ && $60.1$ & $27.6$ & $63.8$ & $77.8$ \\
            NatADiff ($\mu=0.4$) & & $\textcolor{red}{\mathbf{98.5^*}}$ & $67.8$ & $65.2$ & $62.5$ & $65.5$ && $49.3$ & $57.1$ & $59.0$ & $60.7$ && $65.1$ & $28.9$ & $63.4$ & $80.2$ \\
            NatADiff ($\mu=0.5$) & & $\textcolor{red}{\mathbf{98.5^*}}$ & $\textcolor{red}{\mathbf{71.6}}$ & $\textcolor{red}{\mathbf{70.1}}$ & $\textcolor{red}{\mathbf{68.0}}$ & $\textcolor{red}{\mathbf{70.6}}$ && $\textcolor{red}{\mathbf{53.7}}$ & $\textcolor{red}{\mathbf{62.7}}$ & $\textcolor{red}{\mathbf{63.8}}$ & $\textcolor{red}{\mathbf{66.4}}$ && $\textcolor{red}{\mathbf{69.5}}$ & $\textcolor{red}{\mathbf{32.0}}$ & $\textcolor{red}{\mathbf{61.7}}$ & $80.1$ \\
            \specialrule{1.2pt}{0pt}{0pt}
            \noalign{\vskip 0.2ex}
        \end{tabular}
    }
\end{table}

\begin{figure*}[h]
    \centering
    {
    \resizebox{\linewidth}{!}{%
    \begin{tabular}{@{}c@{}c@{}c@{}c@{}c@{}c@{}}
        \multicolumn{3}{@{}l@{}}{\textcolor{Green}{Strawberry} \hfill} & \multicolumn{3}{@{}r@{}}{\hfill \textcolor{red}{Crossword}} \\
        \includegraphics[width=0.2\textwidth]{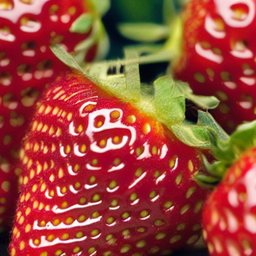} & \includegraphics[width=0.2\textwidth]{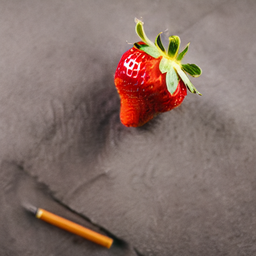} & \includegraphics[width=0.2\textwidth]{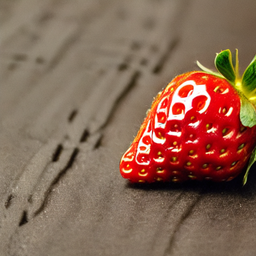} & \includegraphics[width=0.2\textwidth]{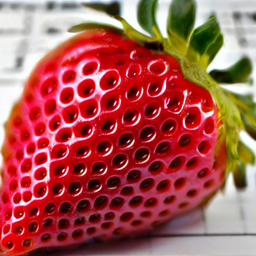} & \includegraphics[width=0.2\textwidth]{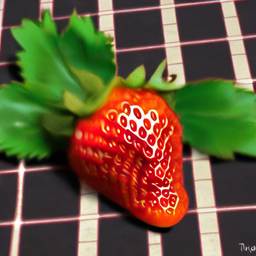} & \includegraphics[width=0.2\textwidth]{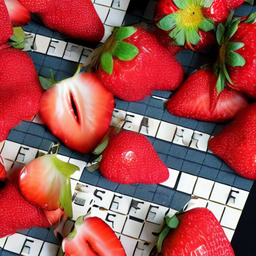} \\
        \multicolumn{6}{c}{(a)} \\

        \multicolumn{3}{@{}l@{}}{\textcolor{Green}{Stove} \hfill} & \multicolumn{3}{@{}r@{}}{\textcolor{red}{\hfill Ape}} \\
        \includegraphics[width=0.2\textwidth]{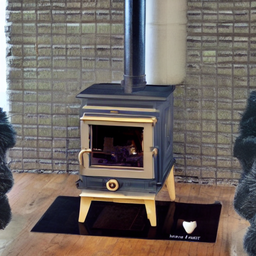} & \includegraphics[width=0.2\textwidth]{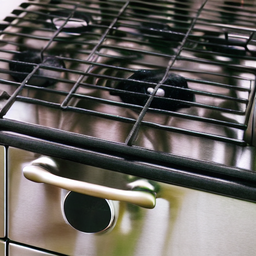} & \includegraphics[width=0.2\textwidth]{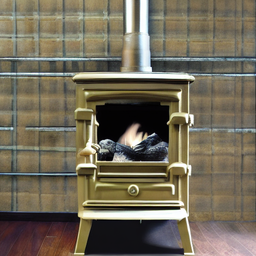} & \includegraphics[width=0.2\textwidth]{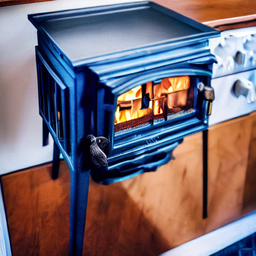} & \includegraphics[width=0.2\textwidth]{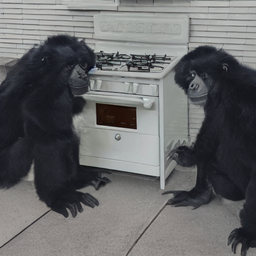} & \includegraphics[width=0.2\textwidth]{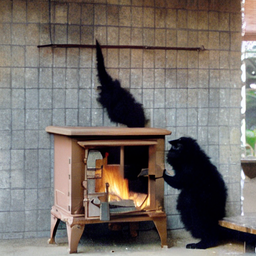} \\
        \multicolumn{6}{c}{(b)} \\

        \multicolumn{3}{@{}l@{}}{\textcolor{Green}{Lipstick} \hfill} & \multicolumn{3}{@{}r@{}}{\textcolor{red}{\hfill Flower}} \\
        \includegraphics[width=0.2\textwidth]{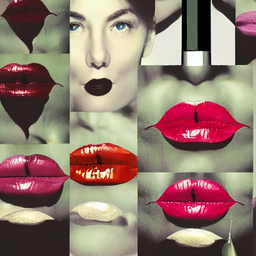} & \includegraphics[width=0.2\textwidth]{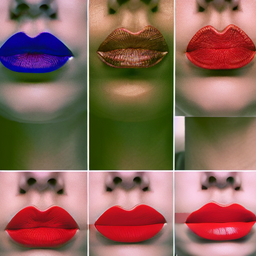} & \includegraphics[width=0.2\textwidth]{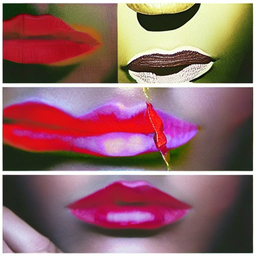} & \includegraphics[width=0.2\textwidth]{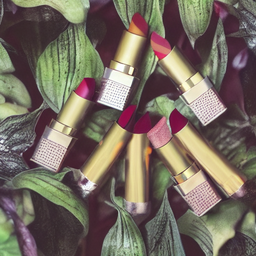} & \includegraphics[width=0.2\textwidth]{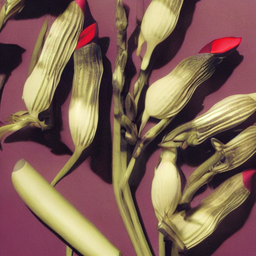} & \includegraphics[width=0.2\textwidth]{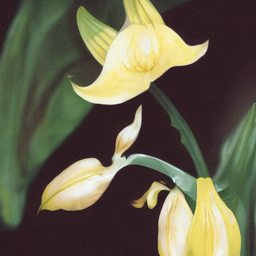} \\
        \multicolumn{6}{c}{(c)} \\

        \multicolumn{3}{@{}l@{}}{\textcolor{Green}{Bagel} \hfill} & \multicolumn{3}{@{}r@{}}{\textcolor{red}{\hfill Triumphal Arch}} \\
        \includegraphics[width=0.2\textwidth]{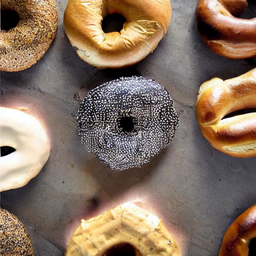} & \includegraphics[width=0.2\textwidth]{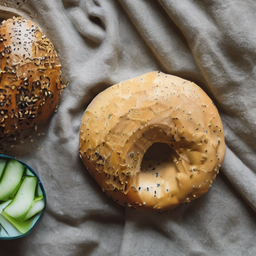} & \includegraphics[width=0.2\textwidth]{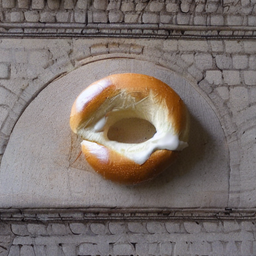} & \includegraphics[width=0.2\textwidth]{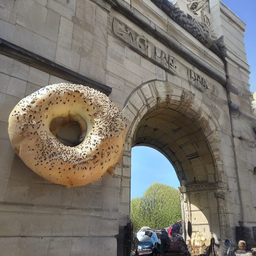} & \includegraphics[width=0.2\textwidth]{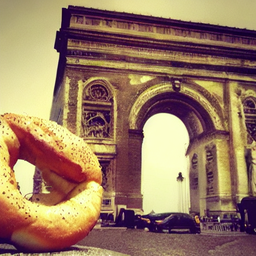} & \includegraphics[width=0.2\textwidth]{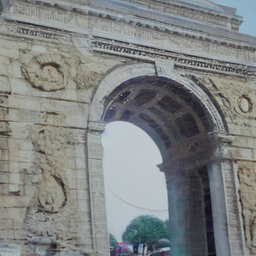} \\
        $\mu = 0.0$ & $\mu = 0.1$ & $\mu = 0.2$ & $\mu = 0.3$ & $\mu = 0.4$ & $\mu = 0.5$ \\
        \multicolumn{6}{c}{(d)}
    \end{tabular}}
    }
\caption{Samples generated by NatADiff using the same random seed while varying $\mu$ from $0.0$ to $0.5$, shown left to right. \textcolor{Green}{Green} and \textcolor{red}{red} labels denote the true and adversarial classes, respectively. Images in (a) and (b) exhibit the \textit{dual class} phenomenon, where large $\mu$ values cause objects from both the true and adversarial classes to appear. Images in (c) and (d) demonstrate the \textit{flipped class} phenomenon, where large $\mu$ values causes the sample to fully adopt the adversarial class, suppressing the original class features.}
\label{fig:NatADiff mu blending samples}
\end{figure*}

We observe that attack success rate, IS, and FID-Val increase alongside $\mu$ (see Table~\ref{tab:NatADiff ablation - boundary guidance}). Interestingly, the lowest FID-A was observed at $\mu=0.2$. These quantitative results suggest that larger values of $\mu$ tend to improve NatADiff performance; however, they do not capture the qualitative shift in sample structure. Large values of $\mu$ introduce two distinct phenomena: \textit{dual class} samples, in which both the true and adversarial classes are present in the image (see Figure~\ref{fig:NatADiff mu blending samples}~(a) and (b)), and \textit{flipped class} samples, in which the original class is entirely overwritten by the adversarial target (see Figure~\ref{fig:NatADiff mu blending samples}~(c) and (d)). Furthermore, as seen in Figure~\ref{fig:NatADiff mu blending samples}, the optimal value of $\mu$ appears to vary across true-adversarial class pairs. Thus, we select a conservative value of $\mu = 0.2$, as manual qualitative investigation found this did not lead to dual and flipped class samples, and experimental results indicate it best aligns with natural adversarial samples as measured by FID-A.

\subsection{Effect of adversarial boundary guidance prompt structure}
Adversarial boundary guidance requires that the true and adversarial classes be encoded into a textual prompt. To examine how this intersection prompt shapes attack performance, we compare NatADiff samples generated using different prompt templates. Concretely, we study four constructions--\textit{``$<$class name$>$ and $<$class name$>$''}, \textit{``A photo of $<$class name$>$ and $<$class name$>$''}, \textit{``$<$class name$>$ next to $<$class name$>$''}, and \textit{``$<$class name$>$ behind $<$class name$>$''}--and evaluate both class orderings, resulting in a total of 8 prompt formats. For each experiment we generate 150 targeted NatADiff samples using a ResNet-50 \citep{He2015} surrogate model, and we report attack success rate, Inception Score (IS) \citep{Salimans2016} and Fréchet Inception Distance (FID) \citep{Frechet1957}. For FID, we evaluate with respect to both the ImageNet-Val \citep{Deng2009} and ImageNet-A \citep{Hendrycks2021} datasets to assess how closely NatADiff samples resemble natural images and known natural adversarial examples, respectively. The exact parameter settings for each experiment are provided in Table~\ref{tab:ImageNet experiment parameters}.

Overall, the prompt structure exerted relatively minimal influence on attack performance or image quality. Table~\ref{tab:NatADiff ablation - boundary guidance prompt template} shows that average ASR varied by less than $7.7\%$ across all prompt formats, indicating that adversarial boundary guidance is robust to the particular phrasing used to express the class intersection. The format \textit{``A photo of $<$class name of $\tilde{y}$$>$ and $<$class name of $y$$>$''} yielded the highest average ASR, but this also coincided with the lowest IS, suggesting that the improved attack transferability may come at the cost of sample fidelity. For this reason, we adopt the prompt \textit{``$<$class name of $\tilde{y}$$>$ and $<$class name of $y$$>$''}, which attains the second-highest ASR while maintaining superior image fidelity, and therefore provides a more favourable balance between attack effectiveness and image quality.

\begin{table}[h]
    \centering
    \caption{\textbf{Attack success rate} (ASR) and image quality of adversarial samples generated by NatADiff under targeted attack settings with varying adversarial boundary guidance prompt templates. We use a ResNet-50 \citep{He2015} surrogate model. \textcolor{red}{\textbf{Bold}} values highlight the best score. White-box ASR (same surrogate and victim model) is denoted with an $^*$.}
    \label{tab:NatADiff ablation - boundary guidance prompt template}
    \resizebox{1\textwidth}{!}{
        \begin{tabular}{@{}l@{}c@{\hspace{1em}}ccccc@{}c@{\hspace{1em}}cccc@{}ccccc@{}}
            \specialrule{1.2pt}{0pt}{0pt}
            \noalign{\vskip 0.5ex}
            \multicolumn{1}{@{}c@{}}{\multirow{3}{*}{Prompt Structure}} && \multicolumn{10}{c}{Victim Model ASR (\%)} && \multicolumn{1}{c}{\multirow{3}{*}{\shortstack{Average\\ASR}}} & \multicolumn{1}{c}{\multirow{3}{*}{\shortstack{IS\\($\boldsymbol{\uparrow}$)}}} & \multicolumn{1}{c}{\multirow{3}{*}{\shortstack{FID-\\Val ($\boldsymbol{\downarrow}$)}}} & \multicolumn{1}{c}{\multirow{3}{*}{\shortstack{FID-\\A ($\boldsymbol{\downarrow}$)}}} \\
            && \multicolumn{5}{c}{CNNs} && \multicolumn{4}{c}{Transformers} &&&& \\
            \cline{3-7} \cline{9-12} 
            \noalign{\vskip 0.5ex}
            && RN-50 & Inc-v3 & RN-152 & AdvRes & AdvInc && ViT-H & Max-ViT & Swin-B & DeIT &&&& \\
            \specialrule{1.2pt}{0pt}{0pt}
            \noalign{\vskip 0.5ex}
            $y$ and $\tilde{y}$ & & $\textcolor{BrickRed}{\underline{99.3^*}}$ & $59.3$ & $\textcolor{BrickRed}{\underline{63.3}}$ & $\textcolor{BrickRed}{\underline{60.0}}$ & $56.7$ && $34.0$ & $47.3$ & $51.3$ & $54.0$ && $58.4$ & $\textcolor{BrickRed}{\underline{9.2}}$ & $205.0$ & $213.2$ \\
            $\tilde{y}$ and $y$ & & $98.7^*$ & $\textcolor{BrickRed}{\underline{62.0}}$ & $\textcolor{BrickRed}{\underline{63.3}}$ & $58.0$ & $\textcolor{BrickRed}{\underline{62.0}}$ && $\textcolor{BrickRed}{\underline{40.7}}$ & $48.7$ & $50.0$ & $54.0$ && $\textcolor{BrickRed}{\underline{59.7}}$ & $9.0$ & $201.1$ & $212.0$ \\
            A photo of $y$ and $\tilde{y}$ & & $98.7^*$ & $57.3$ & $62.0$ & $58.0$ & $54.7$ && $39.3$ & $\textcolor{BrickRed}{\underline{51.3}}$ & $\textcolor{red}{\mathbf{54.0}}$ & $\textcolor{BrickRed}{\underline{56.7}}$ && $59.1$ & $\textcolor{red}{\mathbf{9.4}}$ & $201.2$ & $212.6$ \\
            A photo of $\tilde{y}$ and $y$ & & $\textcolor{BrickRed}{\underline{99.3^*}}$ & $\textcolor{red}{\mathbf{67.3}}$ & $\textcolor{red}{\mathbf{68.7}}$ & $\textcolor{red}{\mathbf{63.3}}$ & $\textcolor{red}{\mathbf{62.7}}$ && $36.7$ & $\textcolor{red}{\mathbf{52.7}}$ & $\textcolor{BrickRed}{\underline{53.3}}$ & $\textcolor{red}{\mathbf{57.3}}$ && $\textcolor{red}{\mathbf{62.4}}$ & $8.4$ & $204.1$ & $211.6$ \\
            $y$ next to $\tilde{y}$ & & $97.3^*$ & $54.7$ & $52.7$ & $52.0$ & $56.7$ && $33.3$ & $46.0$ & $47.3$ & $52.0$ && $54.7$ & $8.4$ & $205.6$ & $214.6$ \\
            $\tilde{y}$ next to $y$ & & $\textcolor{red}{\mathbf{100.0^*}}$ & $59.3$ & $58.7$ & $57.3$ & $56.7$ && $32.7$ & $50.0$ & $48.7$ & $52.0$ && $57.3$ & $8.8$ & $\textcolor{BrickRed}{\underline{199.4}}$ & $\textcolor{BrickRed}{\underline{210.2}}$ \\
            $y$ behind $\tilde{y}$ & & $98.0^*$ & $59.3$ & $60.0$ & $\textcolor{BrickRed}{\underline{60.0}}$ & $55.3$ && $\textcolor{red}{\mathbf{41.3}}$ & $48.7$ & $52.0$ & $55.3$ && $58.9$ & $8.9$ & $205.4$ & $216.3$ \\
            $\tilde{y}$ behind $y$ & & $\textcolor{BrickRed}{\underline{99.3^*}}$ & $58.0$ & $57.3$ & $52.0$ & $58.7$ && $36.7$ & $44.0$ & $50.7$ & $50.0$ && $56.3$ & $8.4$ & $\textcolor{red}{\mathbf{197.9}}$ & $\textcolor{red}{\mathbf{208.7}}$ \\
            \specialrule{1.2pt}{0pt}{0pt}
            \noalign{\vskip 0.2ex}
        \end{tabular}
    }
\end{table}

\subsection{Selection of adversarial classifier guidance strength}
The adversarial classifier guidance term, $s$, governs the strength of the guidance provided by the surrogate classifier. This is analogous to the guidance term used in classifier-guided diffusion, where a trained classifier steers the sampling trajectory toward a desired class \citep{ClassifierGuidance_Dhariwal2021}. Classifier-guided diffusion is a well-understood process that trades sample diversity for class adherence, and the optimal guidance strength varies with the underlying classifier model \citep{ClassifierGuidance_Dhariwal2021, FreeDoM_Yu2023, Shen2024, Dai2024}.

To investigate whether our understanding of classifier-guided diffusion extends to \textit{adversarial} classifier guidance, we conduct an ablation study on the influence of the guidance strength term, $s$. For each ablation, we generate $50$ targeted NatADiff samples using using classifier guidance strengths of $s = 20$, $50$, $100$, and $200$. We report attack success rate (ASR), Inception Score (IS) \citep{Salimans2016}, and Fréchet Inception Distance (FID) \citep{Frechet1957} across ResNet-50 \citep{He2015}, Inception-v3 \citep{Inceptionv3_Szegedy2016}, and ViT-H \citep{Dosovitskiy2021} classifiers using an adversarial boundary guidance strength of $\mu = 0.2$. We take FID with respect to both ImageNet-Val \citep{Deng2009} and ImageNet-A \citep{Hendrycks2021} to assess how closely samples resemble natural images and known natural adversarial examples, respectively. Additionally, we examine the average classifier gradient over the diffusion sampling path. At each sampling time $t$, we compute the average gradient as $\frac{1}{n} \sum_{j=1}^n \left( \nabla_{\boldsymbol{x}_t} \log(p(y|\boldsymbol{x}_t)) \right)_j$. The exact parameter settings for each experiment are provided in Table~\ref{tab:ImageNet experiment parameters}.

From Table~\ref{tab:NatADiff ablation - adversarial classifier guidance}, we observe that ASR increases with adversarial classifier guidance strength. Gradient stability also depends on the guidance strength, as illustrated in Figures~\ref{fig:ResNet50 time-grad plots}, \ref{fig:Inception time-grad plots}, and \ref{fig:ViT time-grad plots}. When the guidance strength is too low, NatADiff's ASR drops, and the gradient does not smoothly converge to $0$ as $t \to 0$. This is likely because the guidance is insufficient to push the sample into regions of the adversarial class on the classifier manifold—evidenced by low ASR scores coinciding with weak guidance. Conversely, sufficiently large guidance strengths yield substantially higher ASR and gradients that smoothly converge to $0$ as the sample enters the desired region of the manifold.

We find that FID scores degrade as guidance strength increases, while IS remains relatively unaffected in all cases except the ViT-H classifier. This suggests that increasing guidance strength reduces sample diversity, consistent with the conventional understanding of classifier-guided diffusion. However, excessive classifier guidance can also ``push'' samples off the image manifold, as reflected by the IS degradation observed for ViT-H. This aligns with findings from the main paper, where targeted attacks against ViT-H yielded reduced image quality (see Appendix~\ref{apdx:Low-Quality Targeted ViT Samples}).

Overall, these results indicate that adversarial classifier guidance operates according to the same underlying mechanisms as standard classifier-guided diffusion. Following \cite{ClassifierGuidance_Dhariwal2021} and \cite{Dai2024}, for each classifier we manually tuned $s$ by incrementally increasing it until NatADiff successfully generated adversarial samples. We additionally implemented an adaptive scaling scheme that increased the guidance strength if a sample failed to fool the surrogate classifier.

\begin{table}
    \centering
    \caption{\textbf{Attack success rate} (ASR) and image quality of adversarial samples generated by NatADiff under targeted attack settings with varying adversarial classifier guidance strength, $s$. We use a ResNet-50 \citep{He2015} surrogate model. \textcolor{red}{\textbf{Bold}} values highlight the best score. White-box ASR (same surrogate and victim model) is denoted with an $^*$.}
    \label{tab:NatADiff ablation - adversarial classifier guidance}
    \resizebox{1\textwidth}{!}{
        \begin{tabular}{@{}cc@{}c@{\hspace{1em}}ccccc@{}c@{\hspace{1em}}cccc@{}ccccc@{}}
            \specialrule{1.2pt}{0pt}{0pt}
            \noalign{\vskip 0.5ex}
            \multicolumn{1}{@{}c}{\multirow{3}{*}{\shortstack{Surrogate\\Model}}} & \multicolumn{1}{c@{}}{\multirow{3}{*}{Attack}} && \multicolumn{10}{c}{Victim Model ASR (\%)} && \multicolumn{1}{c}{\multirow{3}{*}{\shortstack{Average\\ASR}}} & \multicolumn{1}{c}{\multirow{3}{*}{\shortstack{IS\\($\boldsymbol{\uparrow}$)}}} & \multicolumn{1}{c}{\multirow{3}{*}{\shortstack{FID-\\Val ($\boldsymbol{\downarrow}$)}}} & \multicolumn{1}{c}{\multirow{3}{*}{\shortstack{FID-\\A ($\boldsymbol{\downarrow}$)}}} \\
            &&& \multicolumn{5}{c}{CNNs} && \multicolumn{4}{c}{Transformers} &&&& \\
            \cline{4-8} \cline{10-13} 
            \noalign{\vskip 0.5ex}
            &&& RN-50 & Inc-v3 & RN-152 & AdvRes & AdvInc && ViT-H & Max-ViT & Swin-B & DeIT &&&& \\
            \specialrule{1.2pt}{0pt}{0pt}
            \noalign{\vskip 0.2ex}
            \multicolumn{1}{@{}c}{\multirow{4}{*}{RN-50}} & NatADiff ($s=20$) & & $38.0^*$ & $36.0$ & $34.0$ & $38.0$ & $36.0$ && $32.0$ & $32.0$ & $30.0$ & $26.0$ && $33.6$ & $4.4$ & $\textcolor{red}{\mathbf{233.5}}$ & $\textcolor{red}{\mathbf{280.5}}$ \\
            & NatADiff ($s=50$) & & $62.0^*$ & $42.0$ & $46.0$ & $38.0$ & $46.0$ && $36.0$ & $38.0$ & $36.0$ & $36.0$ && $42.2$ & $4.3$ & $242.4$ & $282.7$ \\
            & NatADiff ($s=100$) & & $86.0^*$ & $50.0$ & $54.0$ & $50.0$ & $46.0$ && $36.0$ & $44.0$ & $48.0$ & $50.0$ && $51.6$ & $4.5$ & $266.5$ & $293.3$ \\
            & NatADiff ($s=200$) & & $\textcolor{red}{\mathbf{96.0^*}}$ & $\textcolor{red}{\mathbf{62.0}}$ & $\textcolor{red}{\mathbf{76.0}}$ & $\textcolor{red}{\mathbf{68.0}}$ & $\textcolor{red}{\mathbf{60.0}}$ && $\textcolor{red}{\mathbf{46.0}}$ & $\textcolor{red}{\mathbf{62.0}}$ & $\textcolor{red}{\mathbf{66.0}}$ & $\textcolor{red}{\mathbf{66.0}}$ && $\textcolor{red}{\mathbf{66.9}}$ & $\textcolor{red}{\mathbf{4.7}}$ & $319.0$ & $326.2$ \\
            \specialrule{1.2pt}{0pt}{0pt}
            \noalign{\vskip 0.2ex}
            \multicolumn{1}{c}{\multirow{4}{*}{Inc-v3}} & NatADiff ($s=20$) & & $30.0^*$ & $46.0$ & $28.0$ & $36.0$ & $36.0$ && $26.0$ & $26.0$ & $26.0$ & $34.0$ && $32.0$ & $3.9$ & $\textcolor{red}{\mathbf{223.4}}$ & $\textcolor{red}{\mathbf{273.6}}$ \\
            & NatADiff ($s=50$) & & $44.0^*$ & $70.0$ & $44.0$ & $40.0$ & $48.0$ && $34.0$ & $34.0$ & $42.0$ & $42.0$ && $44.2$ & $\textcolor{red}{\mathbf{4.2}}$ & $244.5$ & $284.4$ \\
            & NatADiff ($s=100$) & & $50.0^*$ & $94.0$ & $54.0$ & $58.0$ & $54.0$ && $38.0$ & $44.0$ & $46.0$ & $46.0$ && $53.8$ & $\textcolor{red}{\mathbf{4.2}}$ & $269.2$ & $297.2$ \\
            & NatADiff ($s=200$) & & $\textcolor{red}{\mathbf{72.0^*}}$ & $\textcolor{red}{\mathbf{96.0}}$ & $\textcolor{red}{\mathbf{66.0}}$ & $\textcolor{red}{\mathbf{68.0}}$ & $\textcolor{red}{\mathbf{80.0}}$ && $\textcolor{red}{\mathbf{46.0}}$ & $\textcolor{red}{\mathbf{54.0}}$ & $\textcolor{red}{\mathbf{60.0}}$ & $\textcolor{red}{\mathbf{56.0}}$ && $\textcolor{red}{\mathbf{66.4}}$ & $4.1$ & $319.4$ & $320.8$ \\
            \specialrule{1.2pt}{0pt}{0pt}
            \noalign{\vskip 0.2ex}
            \multicolumn{1}{c}{\multirow{4}{*}{ViT-H}} & NatADiff ($s=20$) & & $28.0^*$ & $34.0$ & $26.0$ & $30.0$ & $28.0$ && $30.0$ & $24.0$ & $30.0$ & $28.0$ && $28.7$ & $3.9$ & $\textcolor{red}{\mathbf{224.7}}$ & $\textcolor{red}{\mathbf{279.0}}$ \\
            & NatADiff ($s=50$) & & $32.0^*$ & $46.0$ & $38.0$ & $40.0$ & $32.0$ && $38.0$ & $36.0$ & $34.0$ & $36.0$ && $36.9$ & $\textcolor{red}{\mathbf{4.0}}$ & $240.0$ & $286.4$ \\
            & NatADiff ($s=100$) & & $46.0^*$ & $46.0$ & $46.0$ & $40.0$ & $44.0$ && $52.0$ & $42.0$ & $48.0$ & $48.0$ && $45.8$ & $3.9$ & $242.2$ & $279.5$ \\
            & NatADiff ($s=200$) & & $\textcolor{red}{\mathbf{78.0^*}}$ & $\textcolor{red}{\mathbf{78.0}}$ & $\textcolor{red}{\mathbf{80.0}}$ & $\textcolor{red}{\mathbf{78.0}}$ & $\textcolor{red}{\mathbf{78.0}}$ && $\textcolor{red}{\mathbf{86.0}}$ & $\textcolor{red}{\mathbf{70.0}}$ & $\textcolor{red}{\mathbf{72.0}}$ & $\textcolor{red}{\mathbf{82.0}}$ && $\textcolor{red}{\mathbf{78.0}}$ & $2.9$ & $281.3$ & $294.5$ \\
            \specialrule{1.2pt}{0pt}{0pt}
            \noalign{\vskip 0.2ex}
        \end{tabular}
    }
\end{table}

\begin{figure*}
    \centering
    {
    \resizebox{\linewidth}{!}{%
    \scriptsize
    \begin{tabular}{@{}c@{\hspace{0.5em}}c@{}}
        \includegraphics[width=0.48\textwidth]{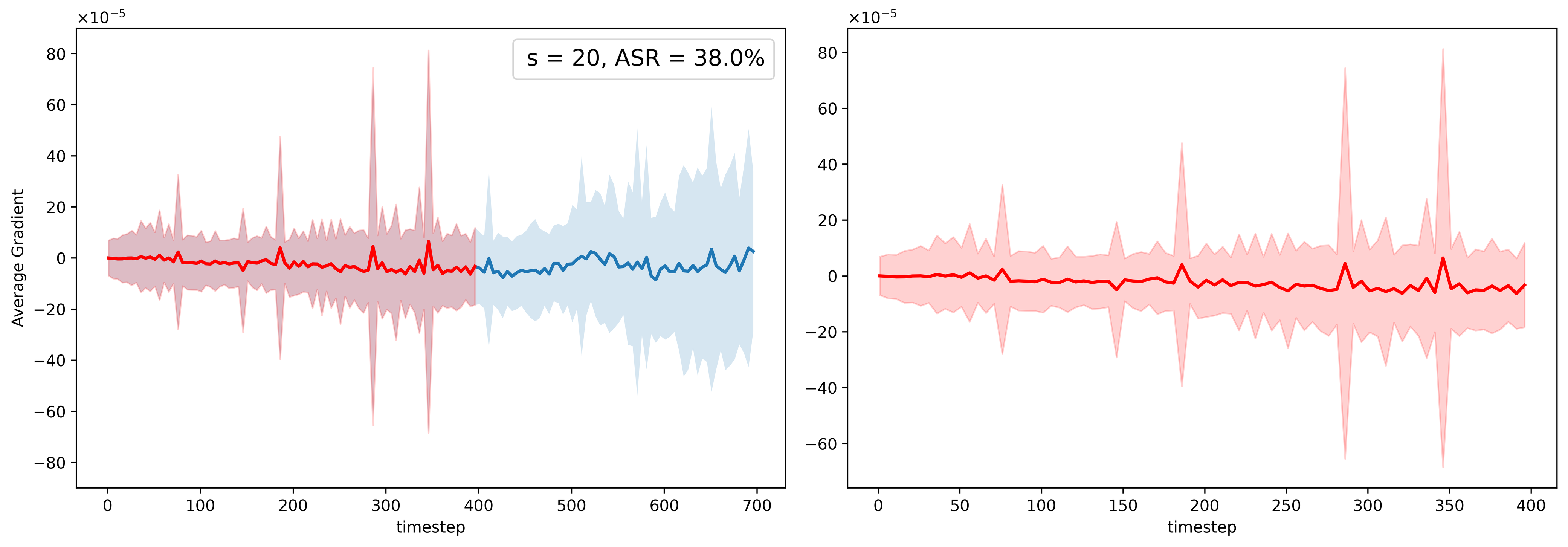} & \includegraphics[width=0.48\textwidth]{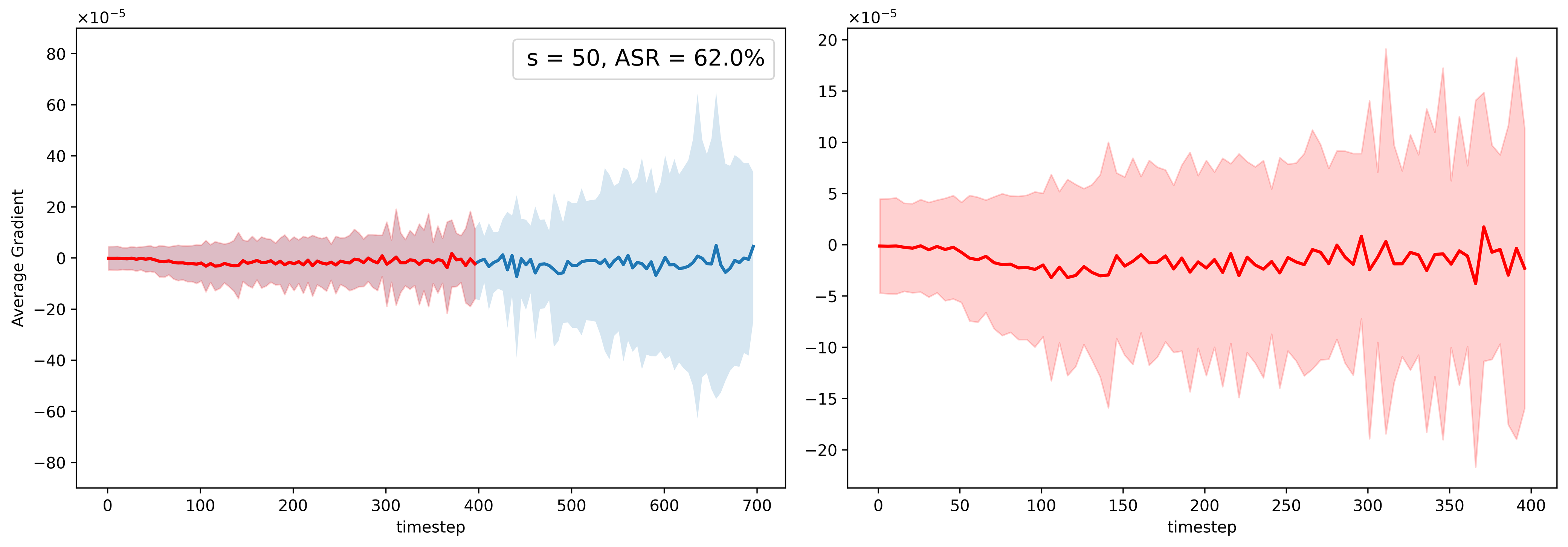} \\
        (a) & (b) \\

        \includegraphics[width=0.48\textwidth]{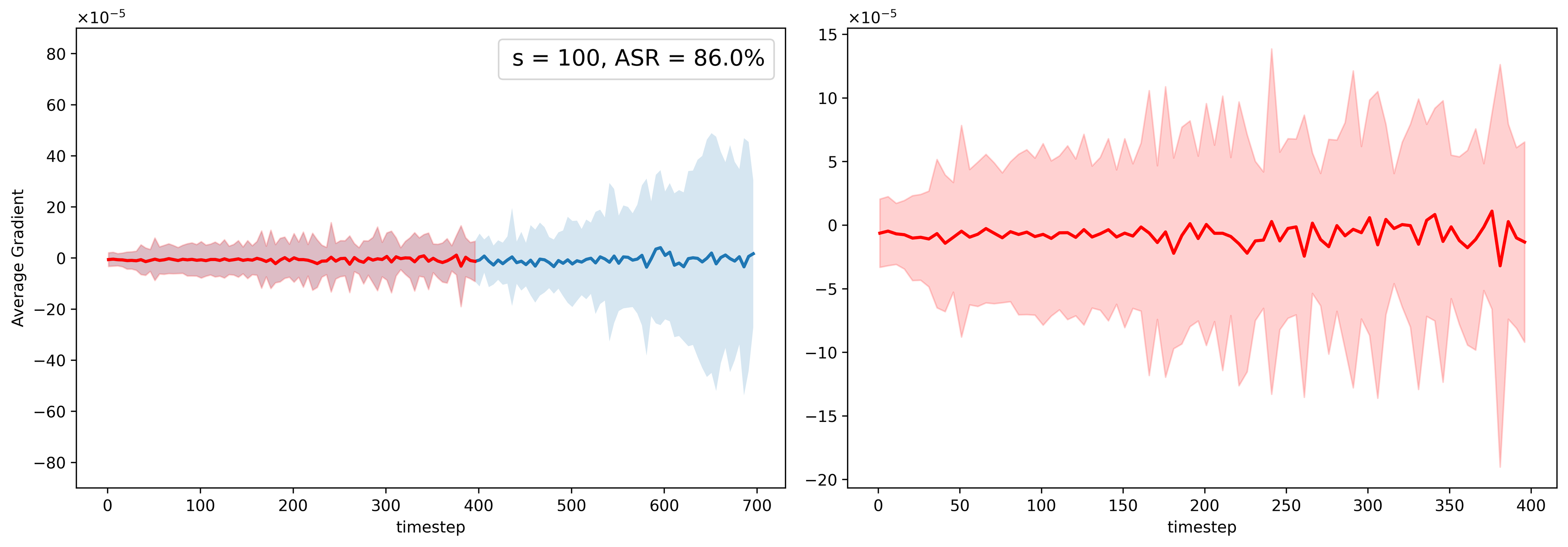} & \includegraphics[width=0.48\textwidth]{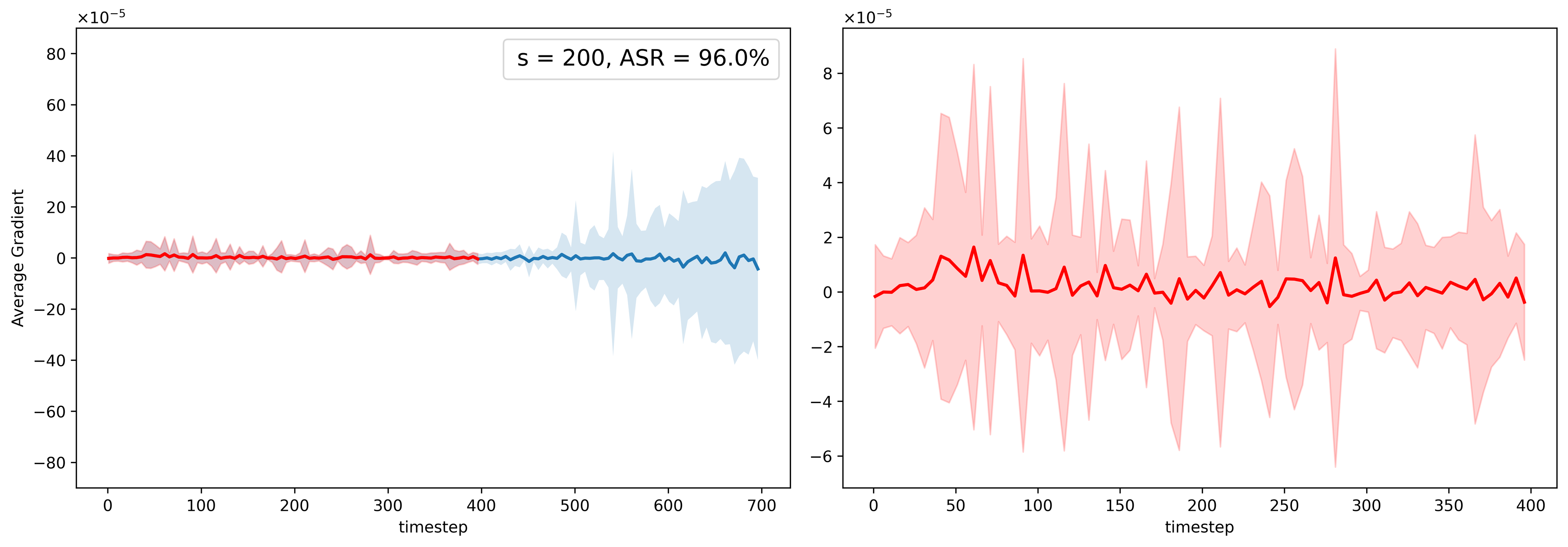} \\
        (c) & (d) \\
    \end{tabular}}
    }
    \vspace{-2mm}
    \caption{ResNet-50 \citep{He2015} classifier gradient with respect to $\boldsymbol{x}_t$ plotted against sampling time. The solid line denotes the average across the experimental run, while the shaded region is plus/minus 1 standard deviation. The final 400 sampling steps are enhanced and displayed in the second panel of each subfigure. For each subfigure, the classifier guidance strength ($s$) and white-box ASR are: (a) $s = 20$, $\text{ASR} = 38.0\%$; (b) $s = 50$, $\text{ASR} = 62.0\%$; (c) $s = 100$, $\text{ASR} = 86.0\%$; (d) $s = 200$, $\text{ASR} = 96.0\%$. Please zoom in for improved visibility.}
    \label{fig:ResNet50 time-grad plots}
\end{figure*}

\begin{figure*}
    \centering
    {
    \resizebox{\linewidth}{!}{%
    \scriptsize
    \begin{tabular}{@{}c@{\hspace{0.5em}}c@{}}
        \includegraphics[width=0.48\textwidth]{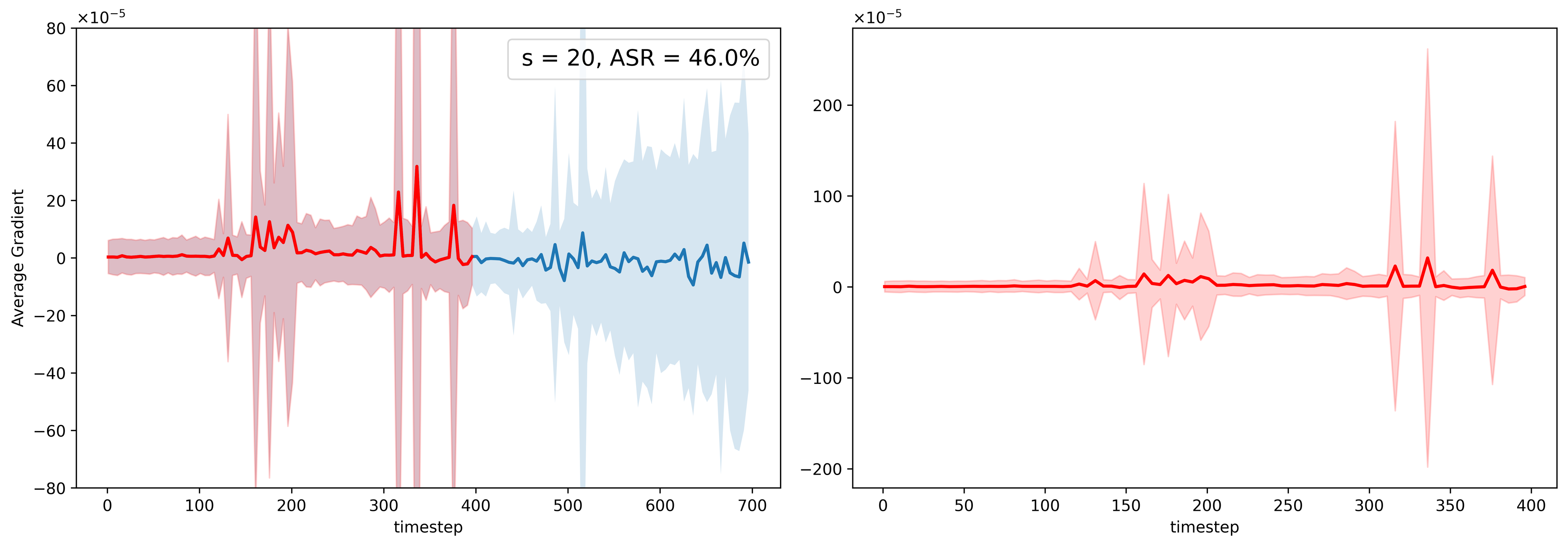} & \includegraphics[width=0.48\textwidth]{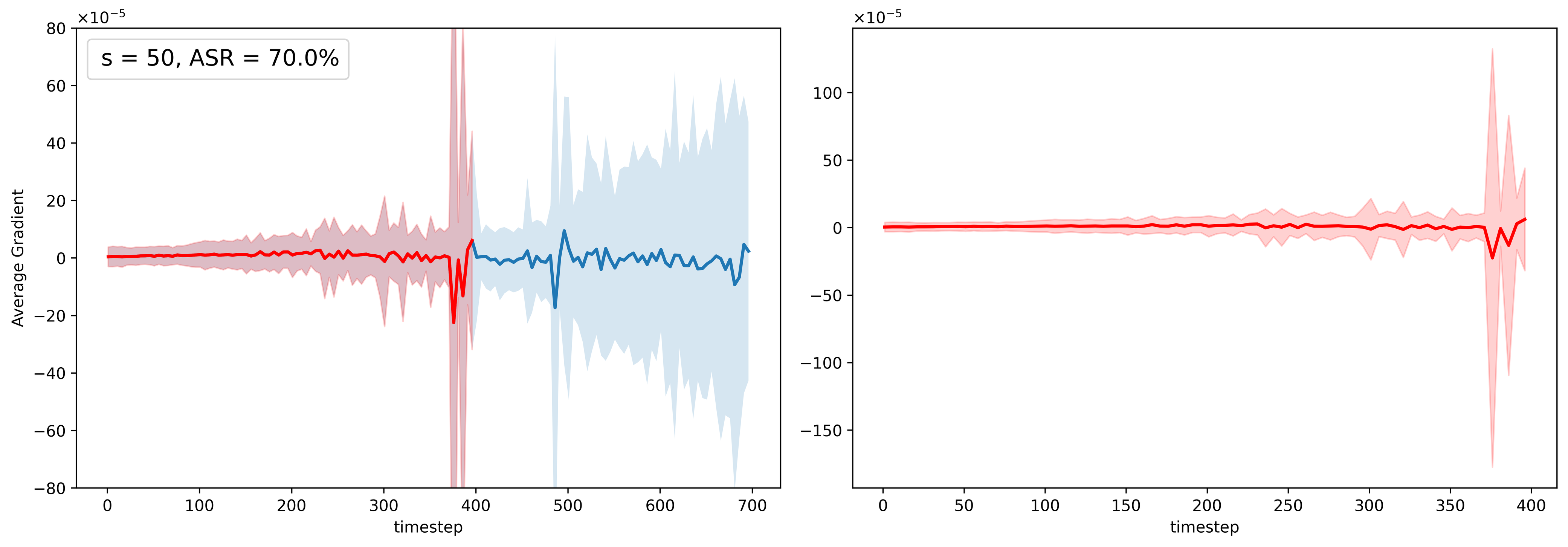} \\
        (a) & (b) \\

        \includegraphics[width=0.48\textwidth]{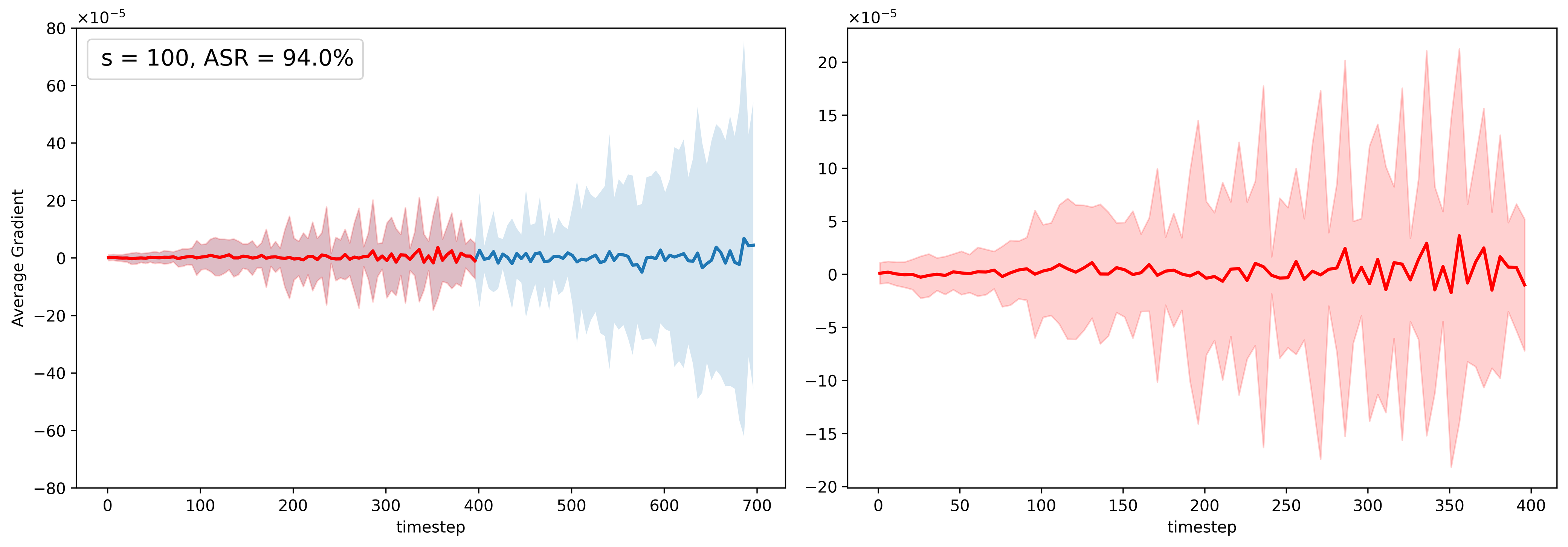} & \includegraphics[width=0.48\textwidth]{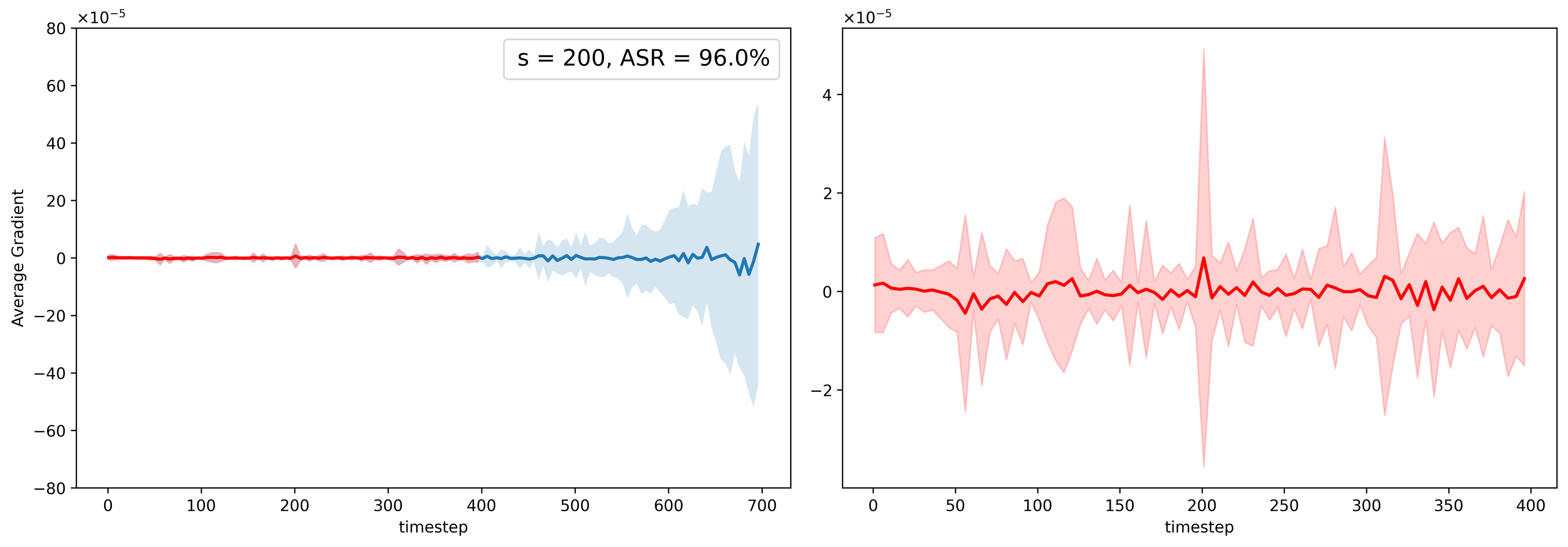} \\
        (c) & (d) \\
    \end{tabular}}
    }
    \vspace{-2mm}
    \caption{Inception-v3 \citep{Inceptionv3_Szegedy2016} classifier gradient with respect to $\boldsymbol{x}_t$ plotted against sampling time. The solid line denotes the average across all experiment runs, while the shaded region is plus/minus 1 standard deviation. The final 400 sampling steps are enhanced and displayed in the second panel of each subfigure. For each subfigure, the classifier guidance strength ($s$) and white-box ASR are: (a) $s = 20$, $\text{ASR} = 46.0\%$; (b) $s = 50$, $\text{ASR} = 70.0\%$; (c) $s = 100$, $\text{ASR} = 94.0\%$; (d) $s = 200$, $\text{ASR} = 96.0\%$. Please zoom in for improved visibility.}
    \label{fig:Inception time-grad plots}
\end{figure*}

\begin{figure*}
    \centering
    {
    \resizebox{\linewidth}{!}{%
    \scriptsize
    \begin{tabular}{@{}c@{\hspace{0.5em}}c@{}}
        \includegraphics[width=0.48\textwidth]{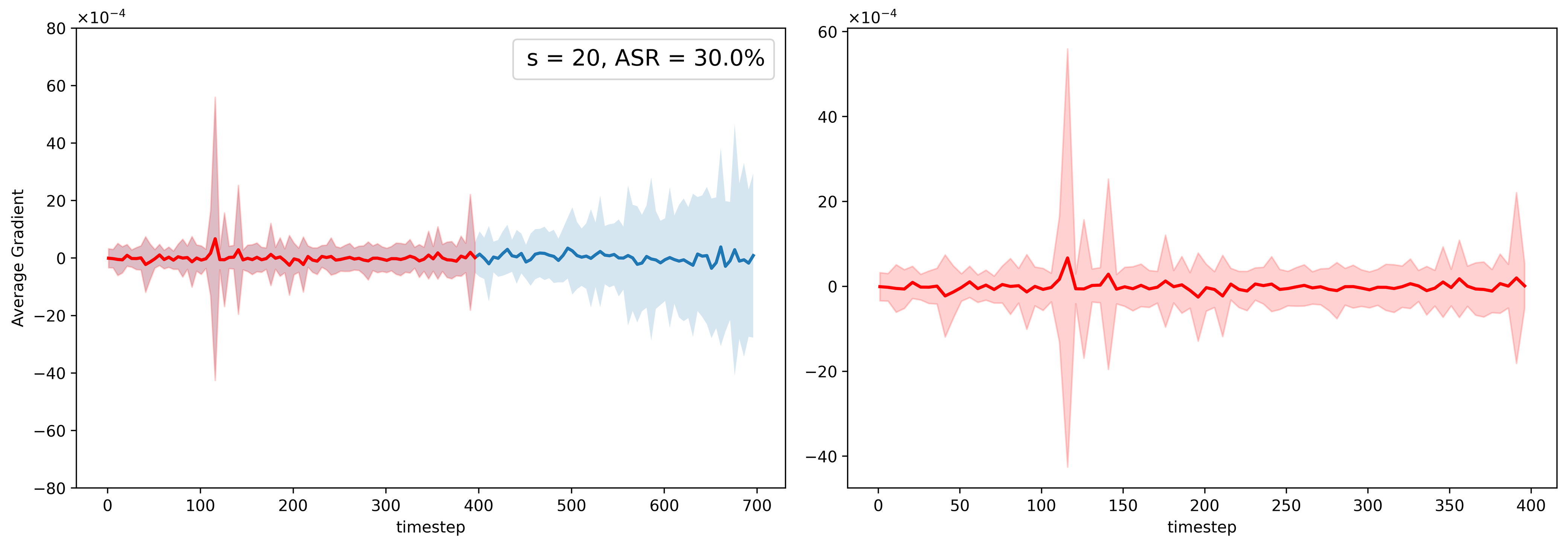} & \includegraphics[width=0.48\textwidth]{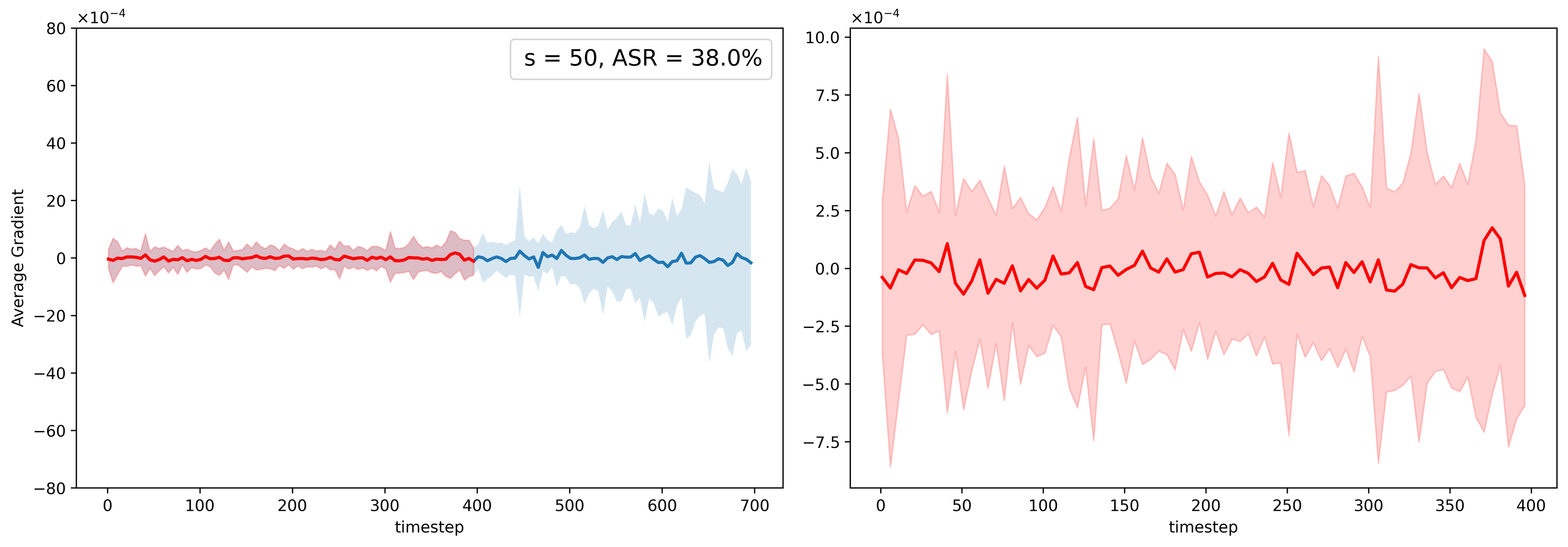} \\
        (a) & (b) \\

        \includegraphics[width=0.48\textwidth]{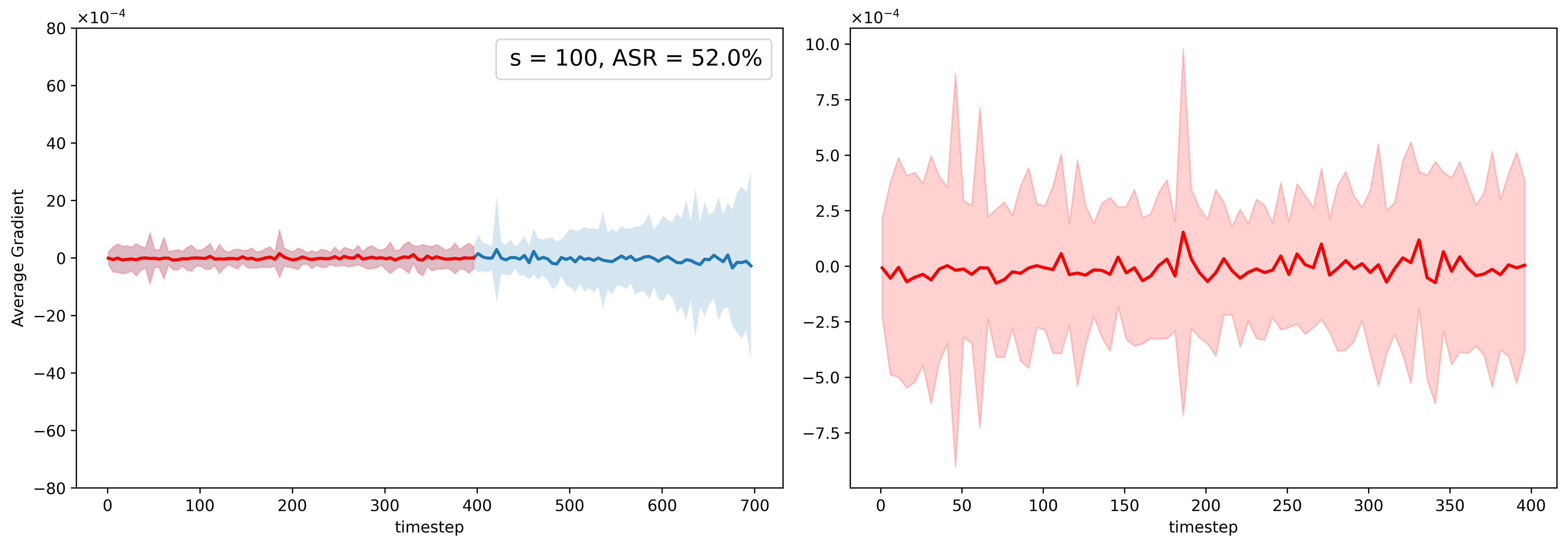} & \includegraphics[width=0.48\textwidth]{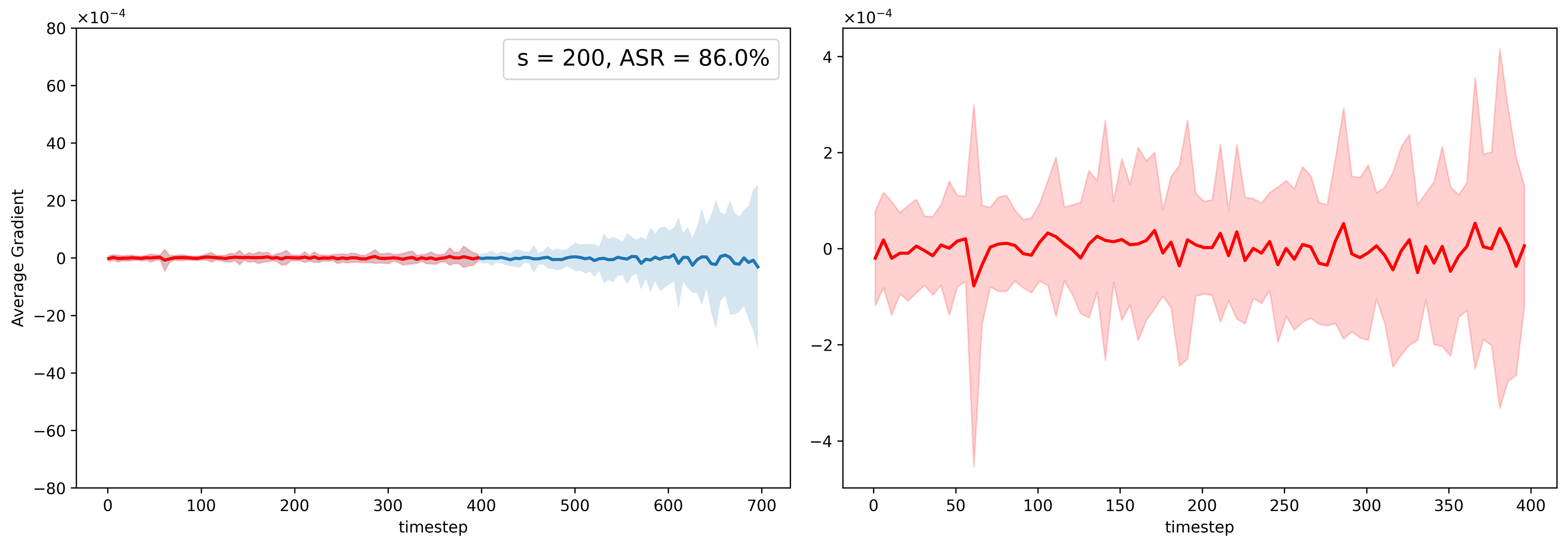} \\
        (c) & (d) \\
    \end{tabular}}
    }
    \vspace{-2mm}
    \caption{ViT-H \citep{Dosovitskiy2021} classifier gradient with respect to $\boldsymbol{x}_t$ plotted against sampling time. The solid line denotes the average across the experimental run, while the shaded region is plus/minus 1 standard deviation. The final 400 sampling steps are enhanced and displayed in the second panel of each subfigure. For each subfigure, the classifier guidance strength ($s$) and white-box ASR are: (a) $s = 20$, $\text{ASR} = 30.0\%$ (b) $s = 50$, $\text{ASR} = 38.0\%$, (c) $s = 100$, $\text{ASR} = 52.0\%$, (d) $s = 200$, $\text{ASR} = 86.0\%$. Please zoom in for improved visibility.}
    \label{fig:ViT time-grad plots}
\end{figure*}

\section{Manifold plots} \label{sec:Manifold plots}
In the main paper we claim that adversarial boundary guidance directs the diffusion sampling trajectory towards class intersections. To further support the validity of this claim, we visualise the learned classifier manifolds using both UMAP \citep{UMAP_McInnes2018} and t-SNE \citep{tSNE_Maaten2008}, which are dimensionality-reduction techniques that allow the image manifolds to be viewed on a 2D plane. UMAP provides a topology-preserving embedding that tends to maintain global structure, while t-SNE preserves local neighbourhood relationships at the cost of deforming global structure. We use both methods to provide complementary views that illustrate how NatADiff samples relate to the true-class and adversarial-class distributions.

To construct the image manifolds, we use NatADiff to generate adversarial samples for 12 true--adversarial class pairs. All adversarial samples are generated using a ResNet-50 \citep{He2015} surrogate model, and we pass UMAP and t-SNE the penultimate-layer feature embeddings from this classifier. The NatADiff samples are plotted alongside the feature embeddings of clean images from the true and adversarial classes taken from the ImageNet dataset \citep{Deng2009}. The resulting visualisations are provided in Figures~\ref{fig:NatADiff UMAP plots} and \ref{fig:NatADiff tSNE plots}.

Across both embeddings, NatADiff samples consistently form one or more distinct clusters that fall between the manifolds of the true and adversarial target classes. This positioning directly supports the main paper's claim: NatADiff guides samples toward class intersections, generating samples that naturally fall near class boundaries. Furthermore, this aligns with the established understanding of natural adversarial samples (test-time errors), which are believed to arise when classifiers rely on spurious contextual cues rather than truly discriminative features. As a sample moves closer to a classifier's decision boundary, the prevalence of contextual cues from the adversarial class increases, which provides a natural explanation for NatADiff's strong attack performance.

\begin{figure*}
    \centering
    {
    \resizebox{\linewidth}{!}{%
    \scriptsize
    \begin{tabular}{@{}c@{}c@{}c@{}}
        \includegraphics[width=0.33\textwidth]{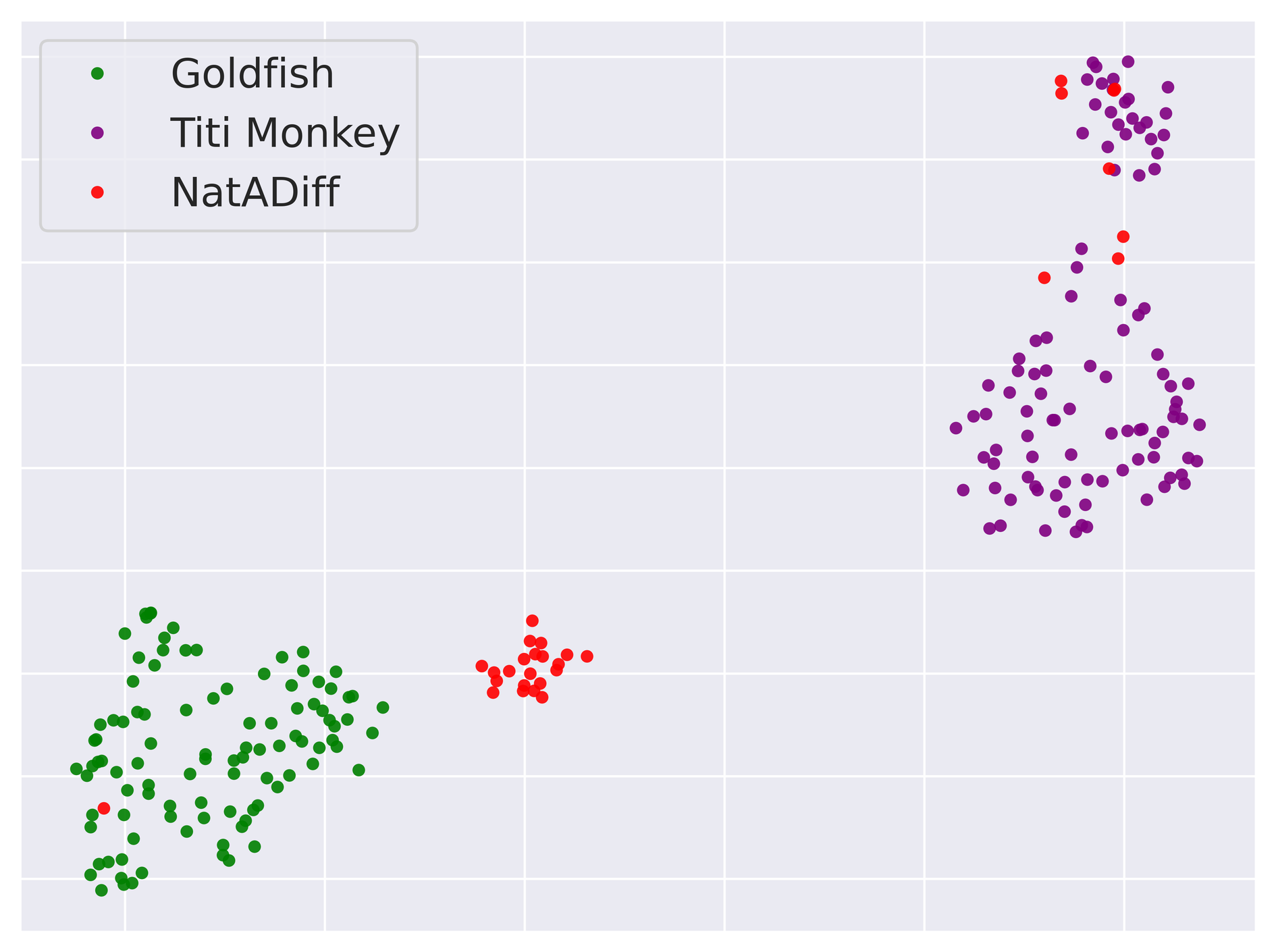} & \includegraphics[width=0.33\textwidth]{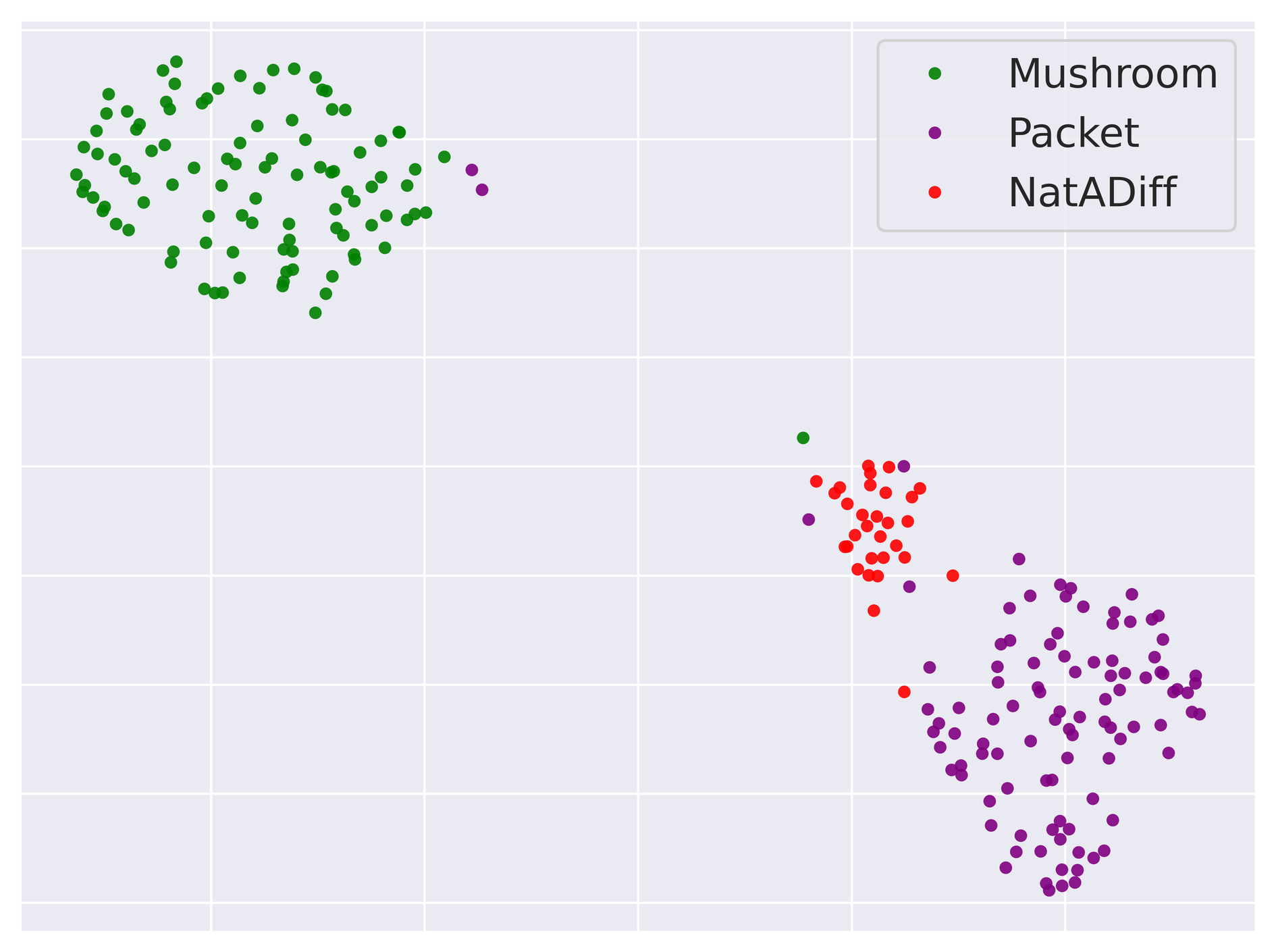} & \includegraphics[width=0.33\textwidth]{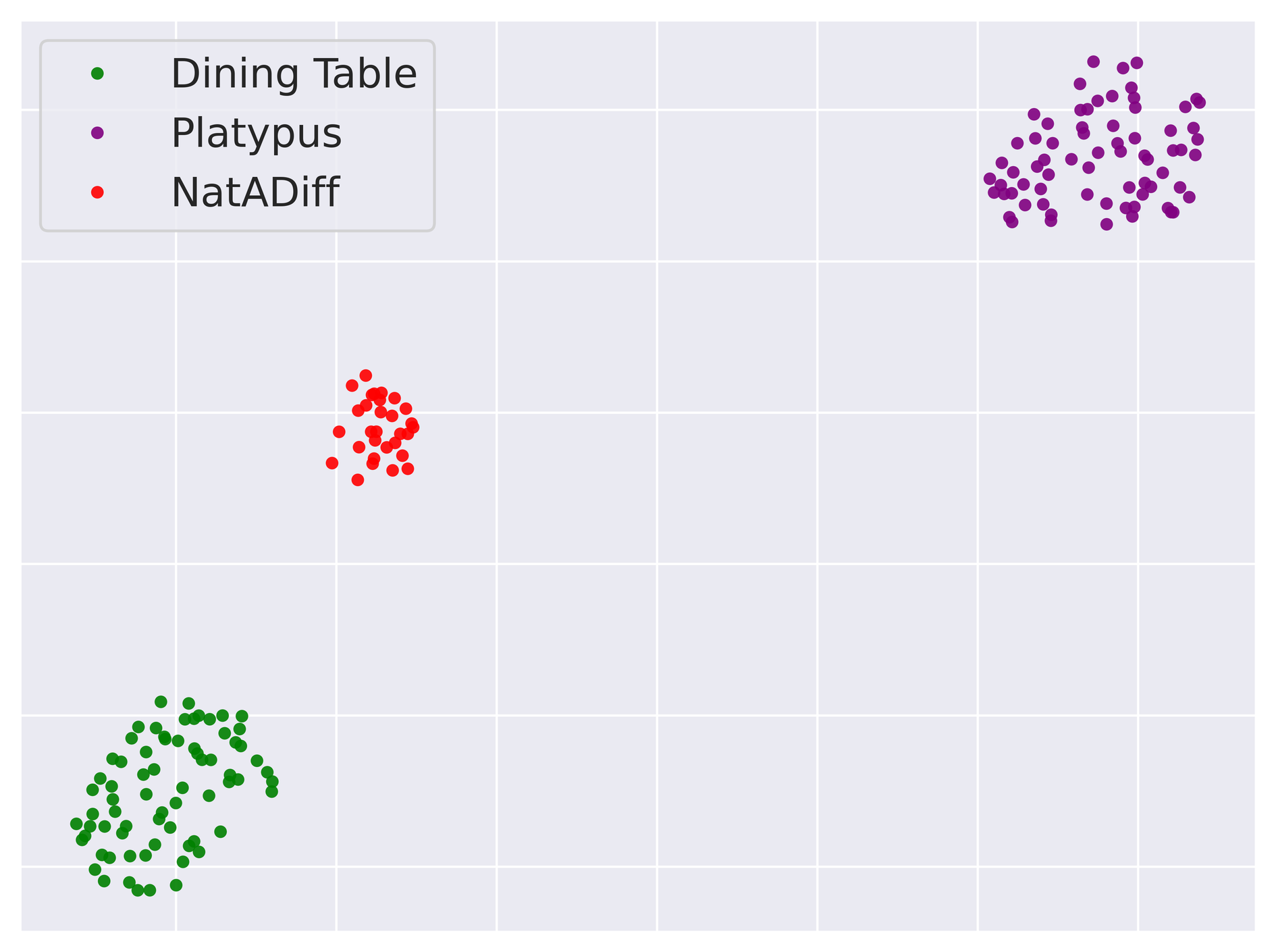} \\

        \includegraphics[width=0.33\textwidth]{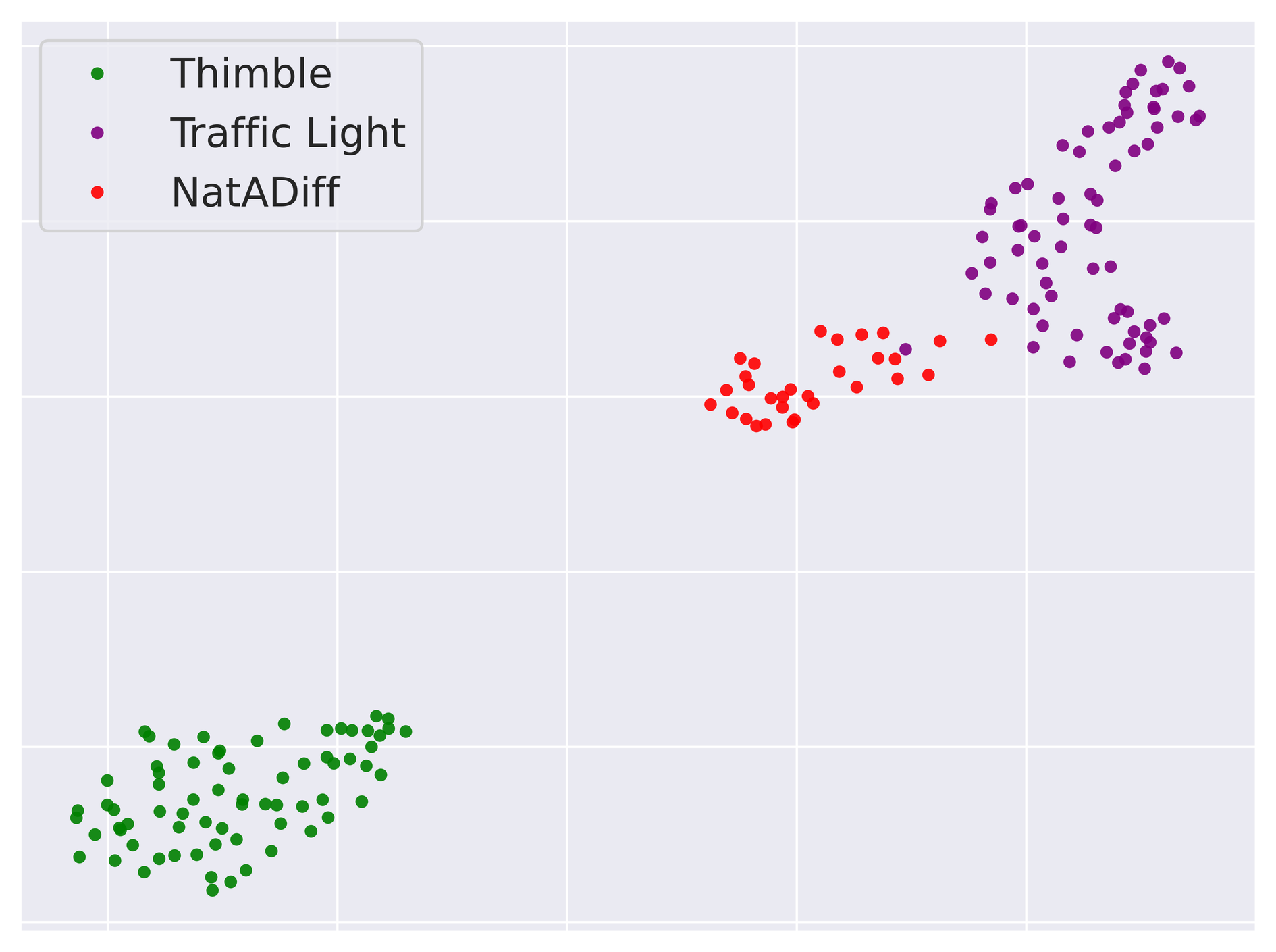} & \includegraphics[width=0.33\textwidth]{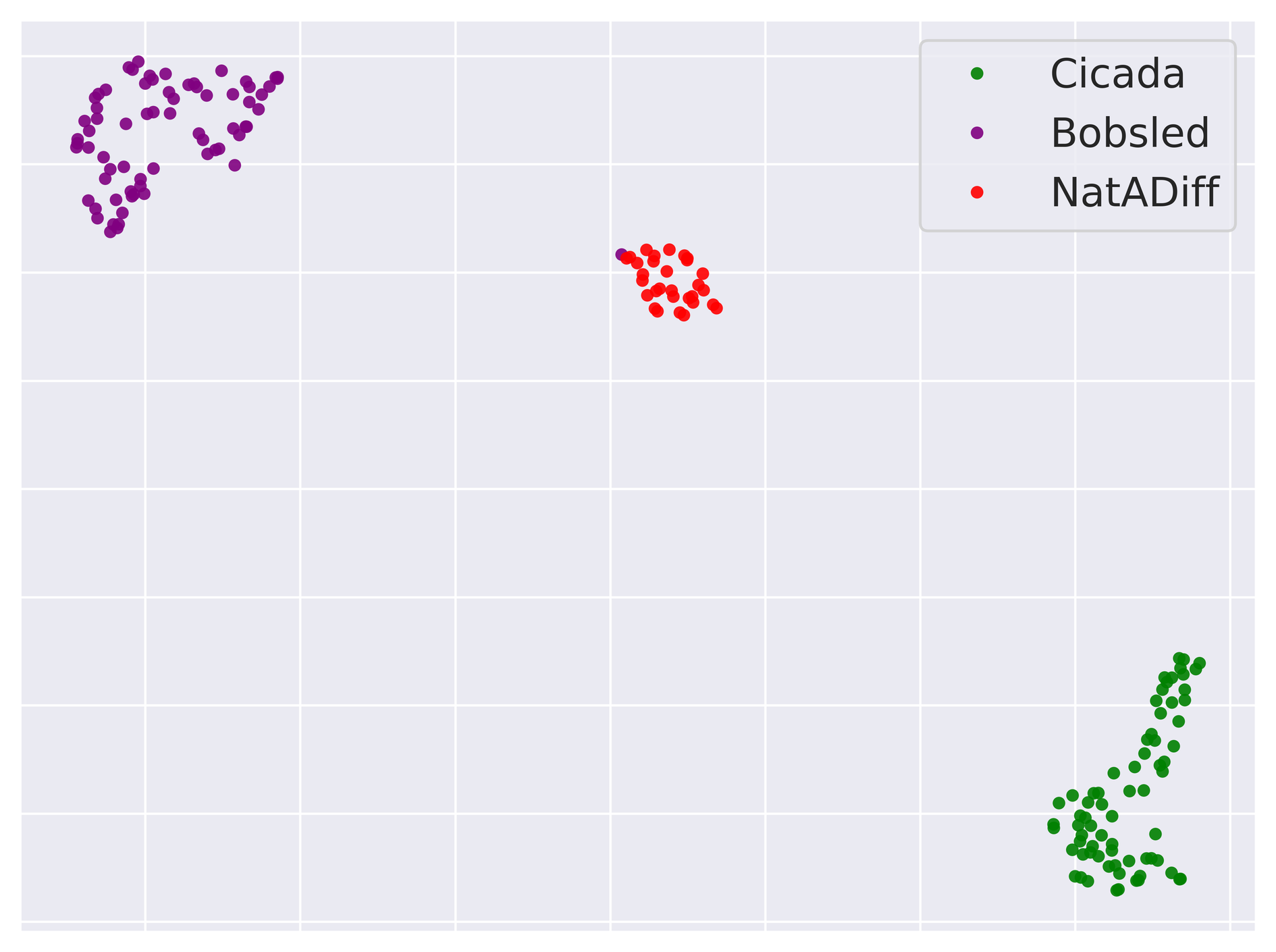} & \includegraphics[width=0.33\textwidth]{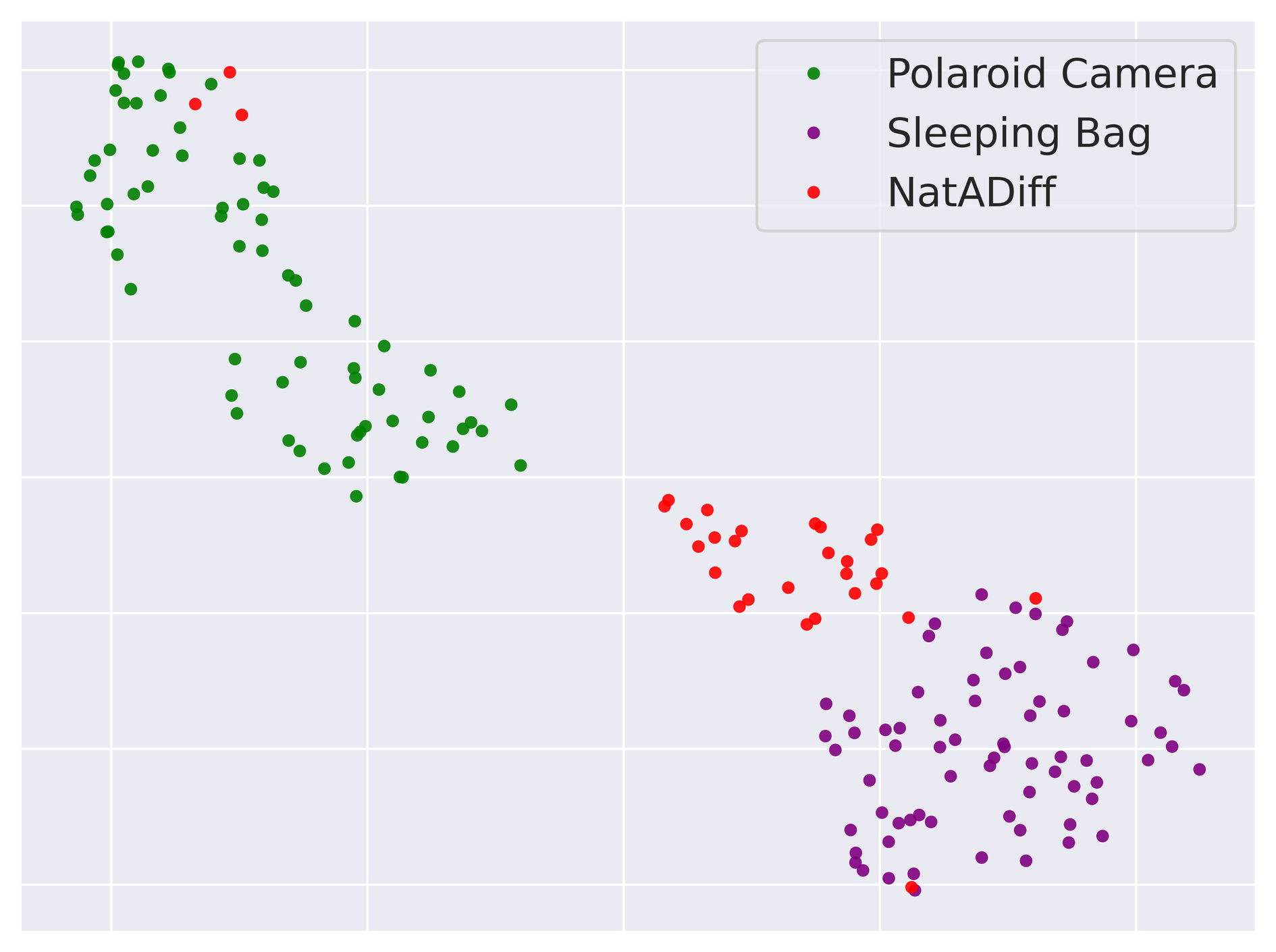} \\

        \includegraphics[width=0.33\textwidth]{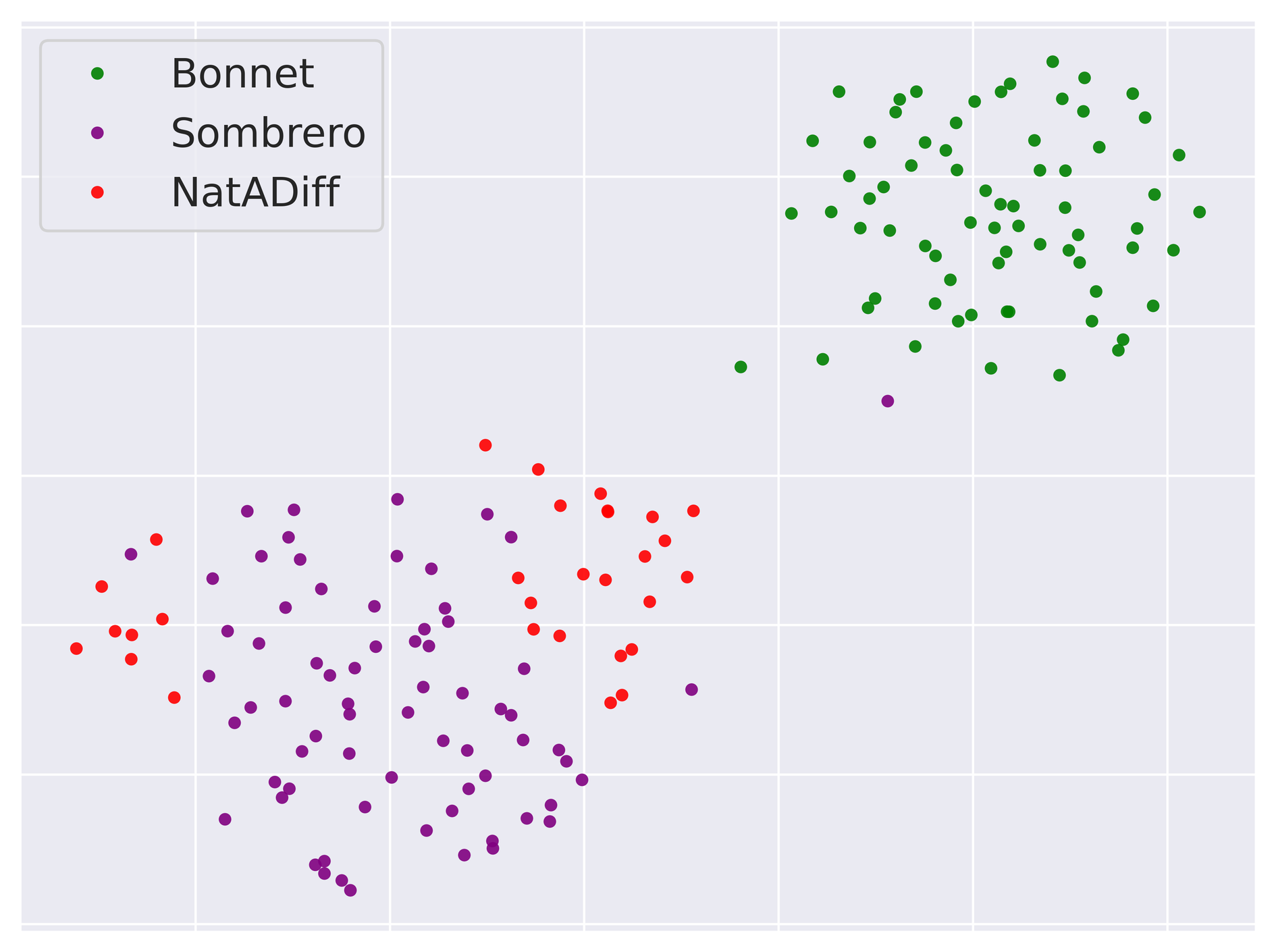} & \includegraphics[width=0.33\textwidth]{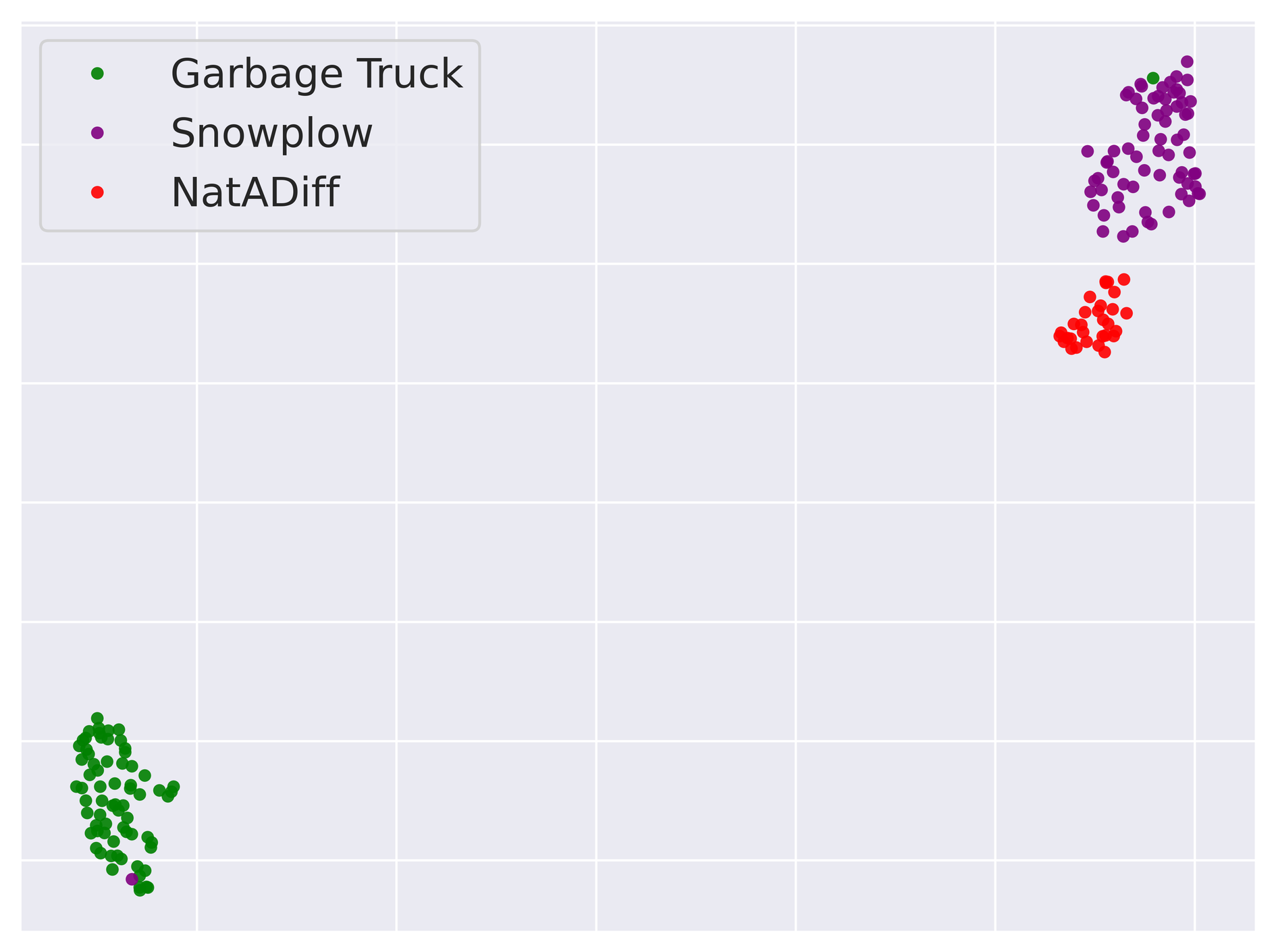} & \includegraphics[width=0.33\textwidth]{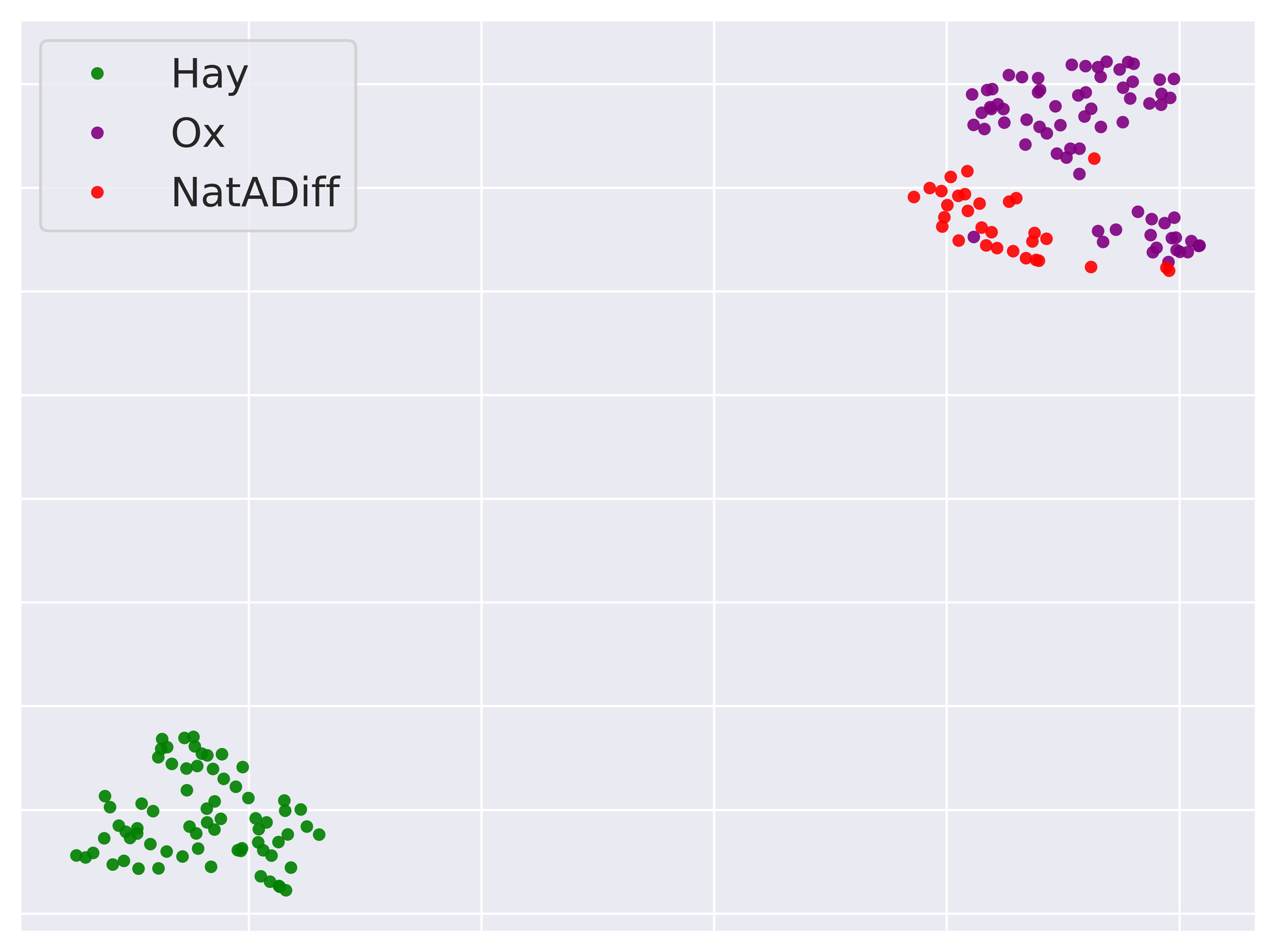} \\

        \includegraphics[width=0.33\textwidth]{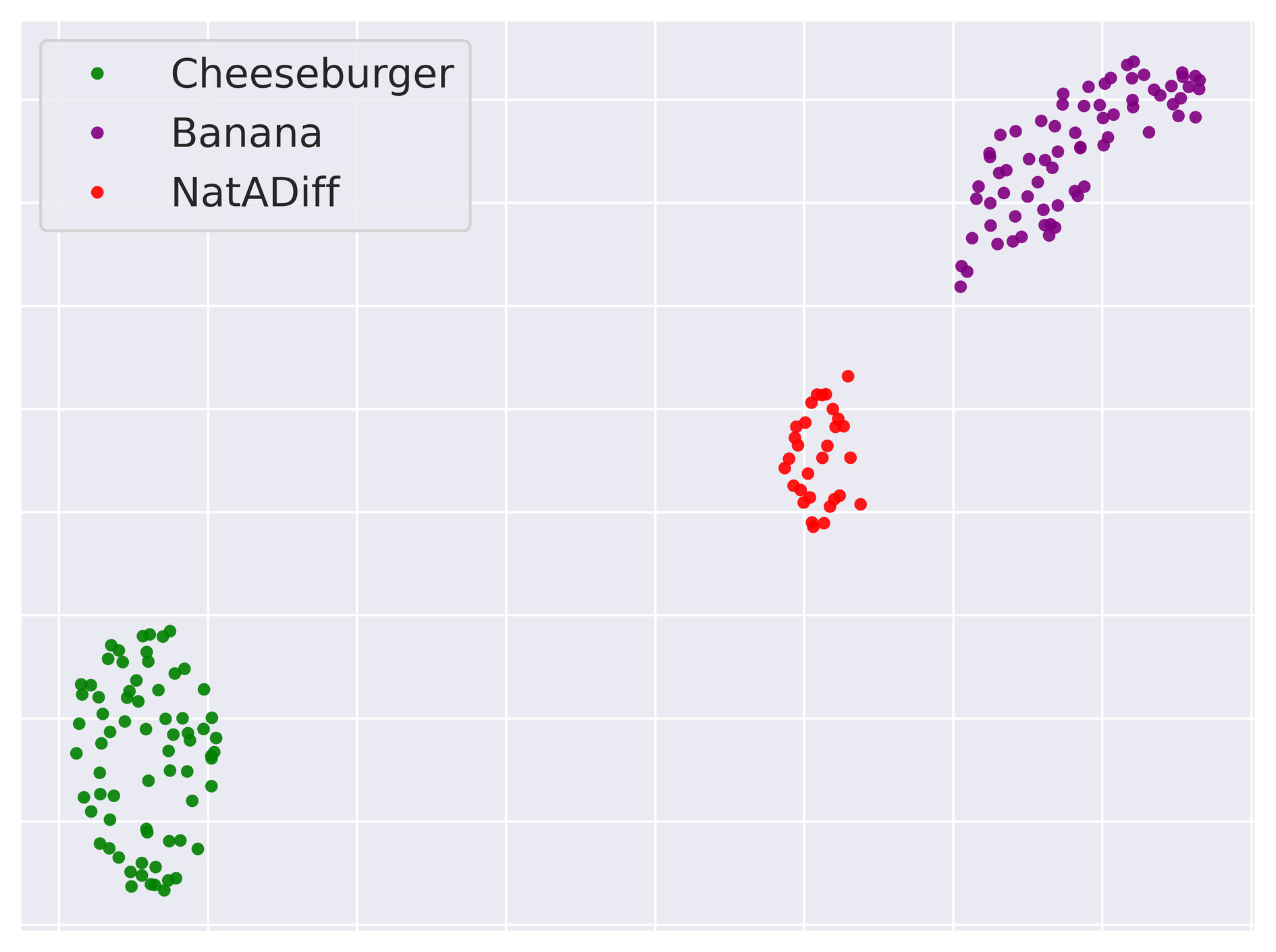} & \includegraphics[width=0.33\textwidth]{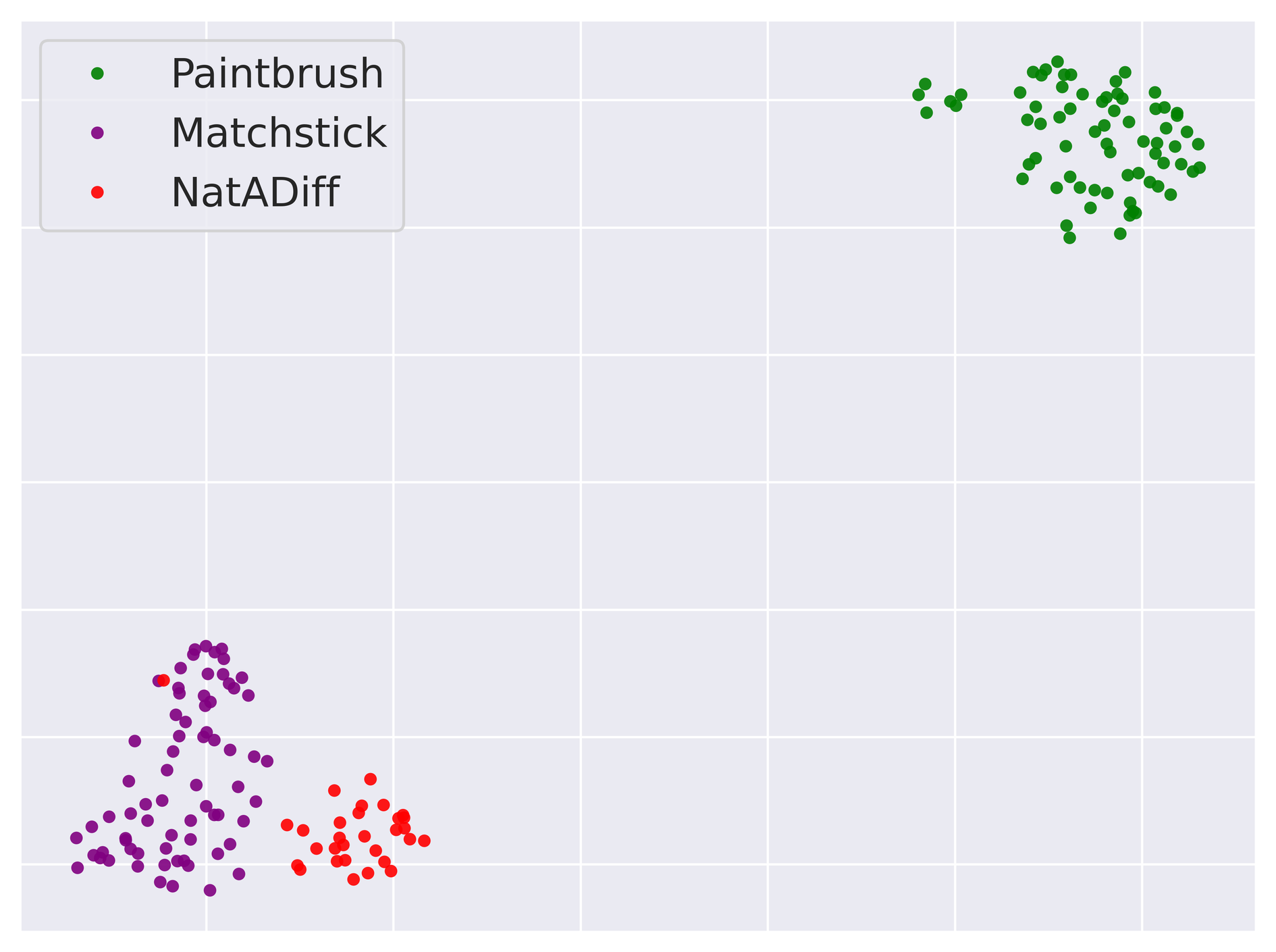} & \includegraphics[width=0.33\textwidth]{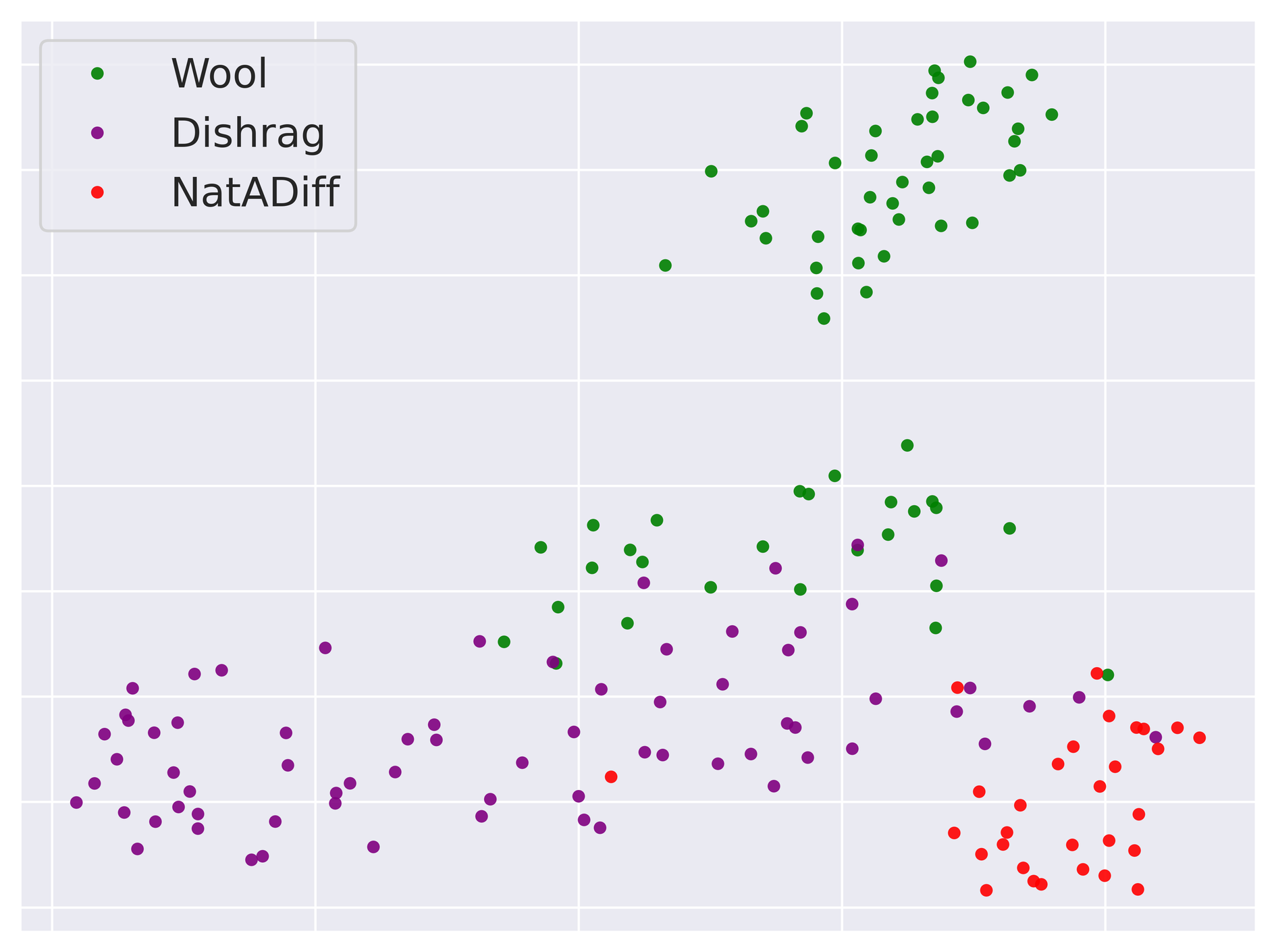}
    \end{tabular}}
    }
    \vspace{-2mm}
    \caption{UMAP \citep{UMAP_McInnes2018} plots of NatADiff samples alongside the ground truth manifolds. \textcolor{ForestGreen}{\textbf{Green}} and \textcolor{Plum}{\textbf{purple}} samples are real images of the true and adversarial classes taken from the ImageNet dataset \citep{Deng2009}. \textcolor{red}{\textbf{Red}} samples are generated by NatADiff using a ResNet-50 \citep{He2015} surrogate model. UMAP is applied to embeddings taken from the penultimate layer of a ResNet-50 classifier. Please zoom in for a better visual.}
    \label{fig:NatADiff UMAP plots}
\end{figure*}

\begin{figure*}
    \centering
    {
    \resizebox{\linewidth}{!}{%
    \scriptsize
    \begin{tabular}{@{}c@{}c@{}c@{}}
        \includegraphics[width=0.33\textwidth]{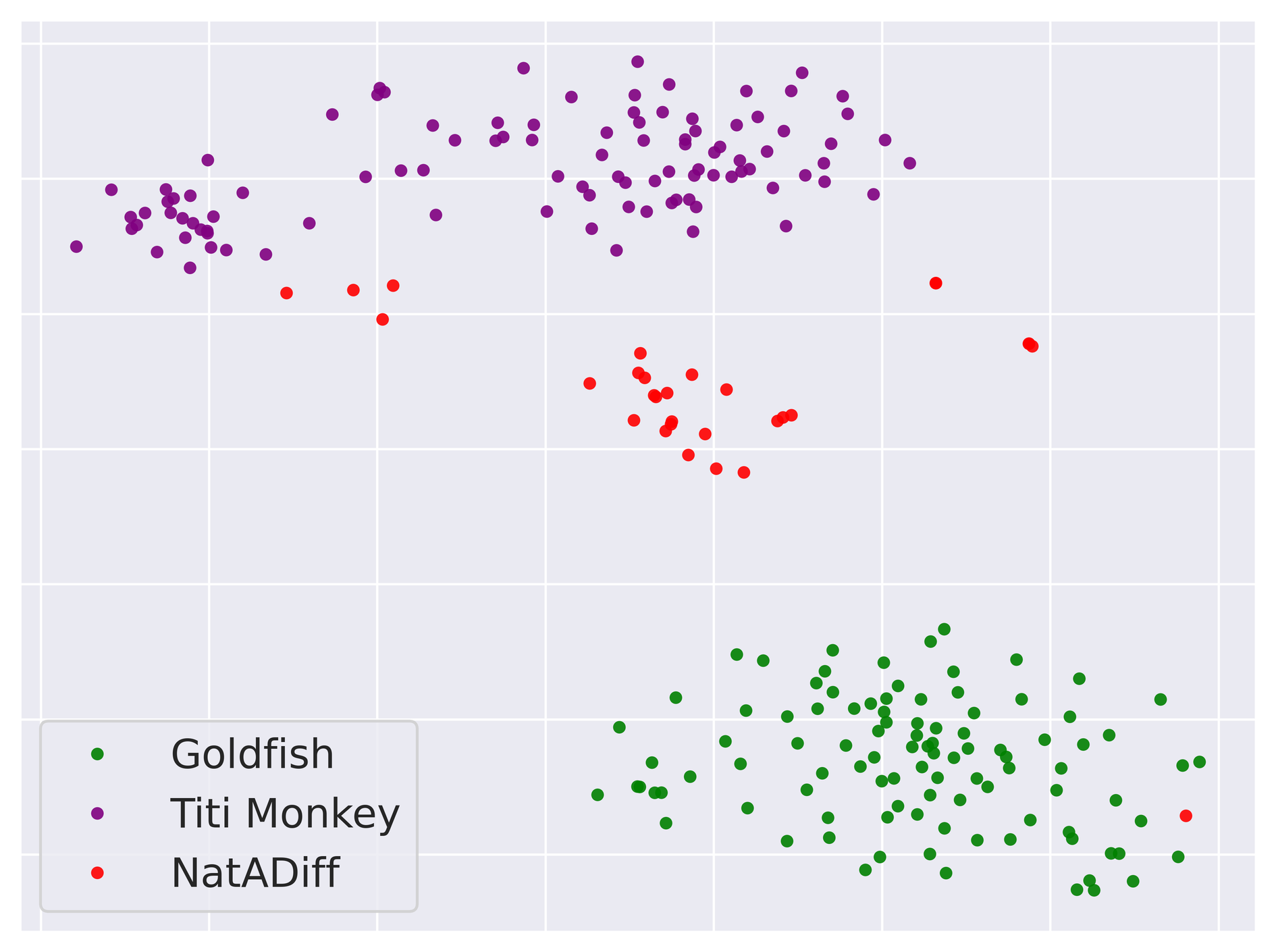} & \includegraphics[width=0.33\textwidth]{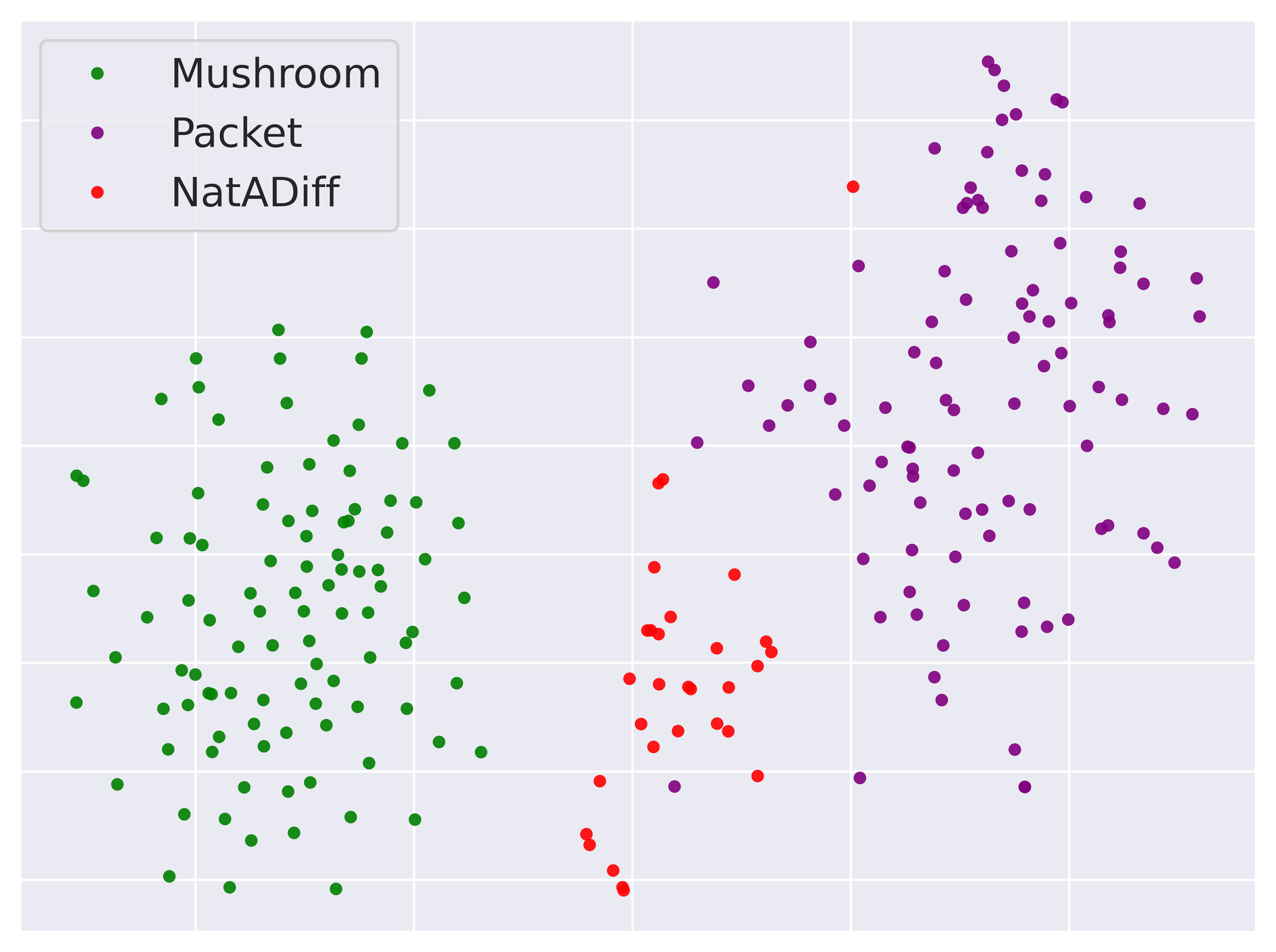} & \includegraphics[width=0.33\textwidth]{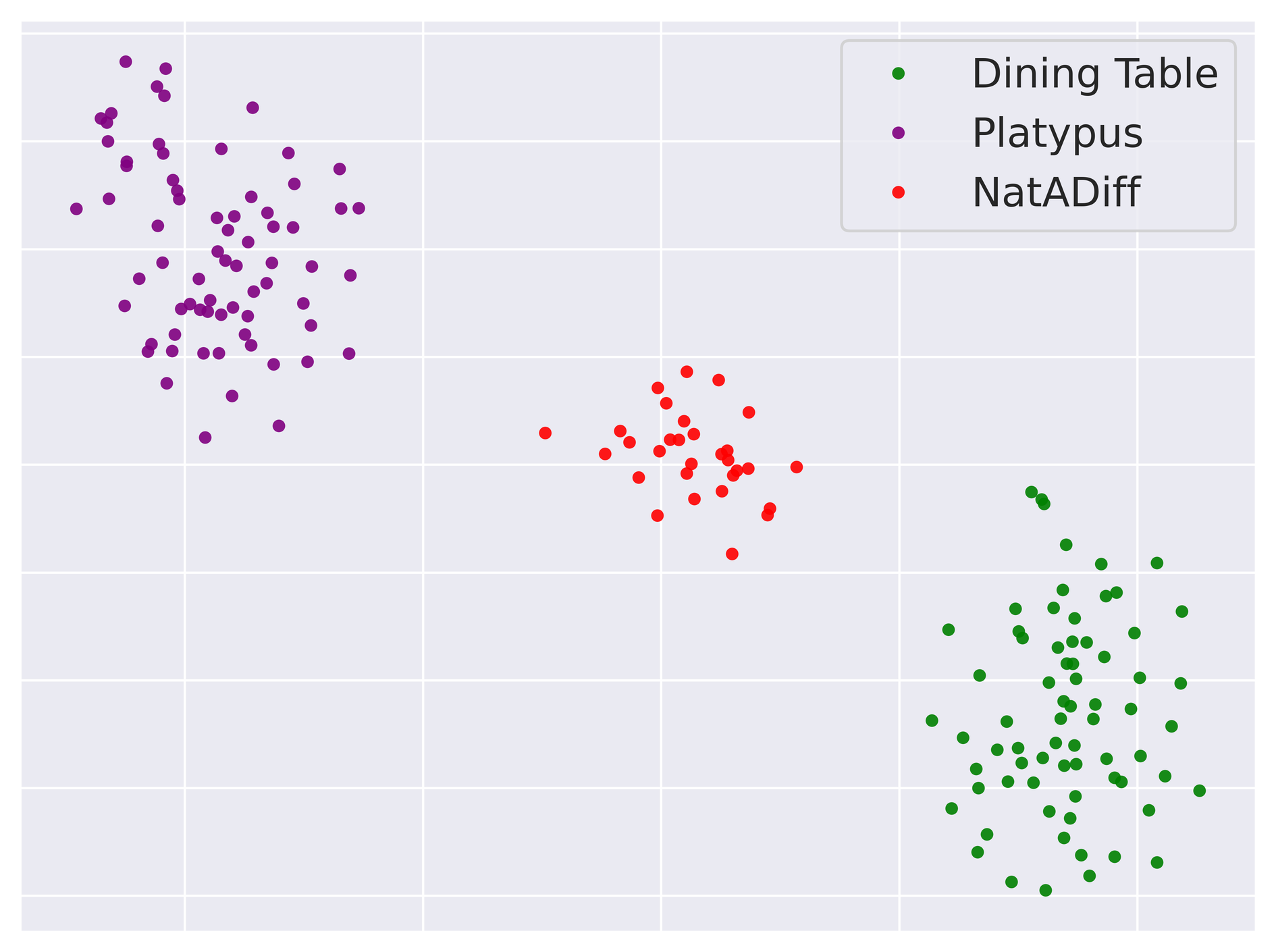} \\

        \includegraphics[width=0.33\textwidth]{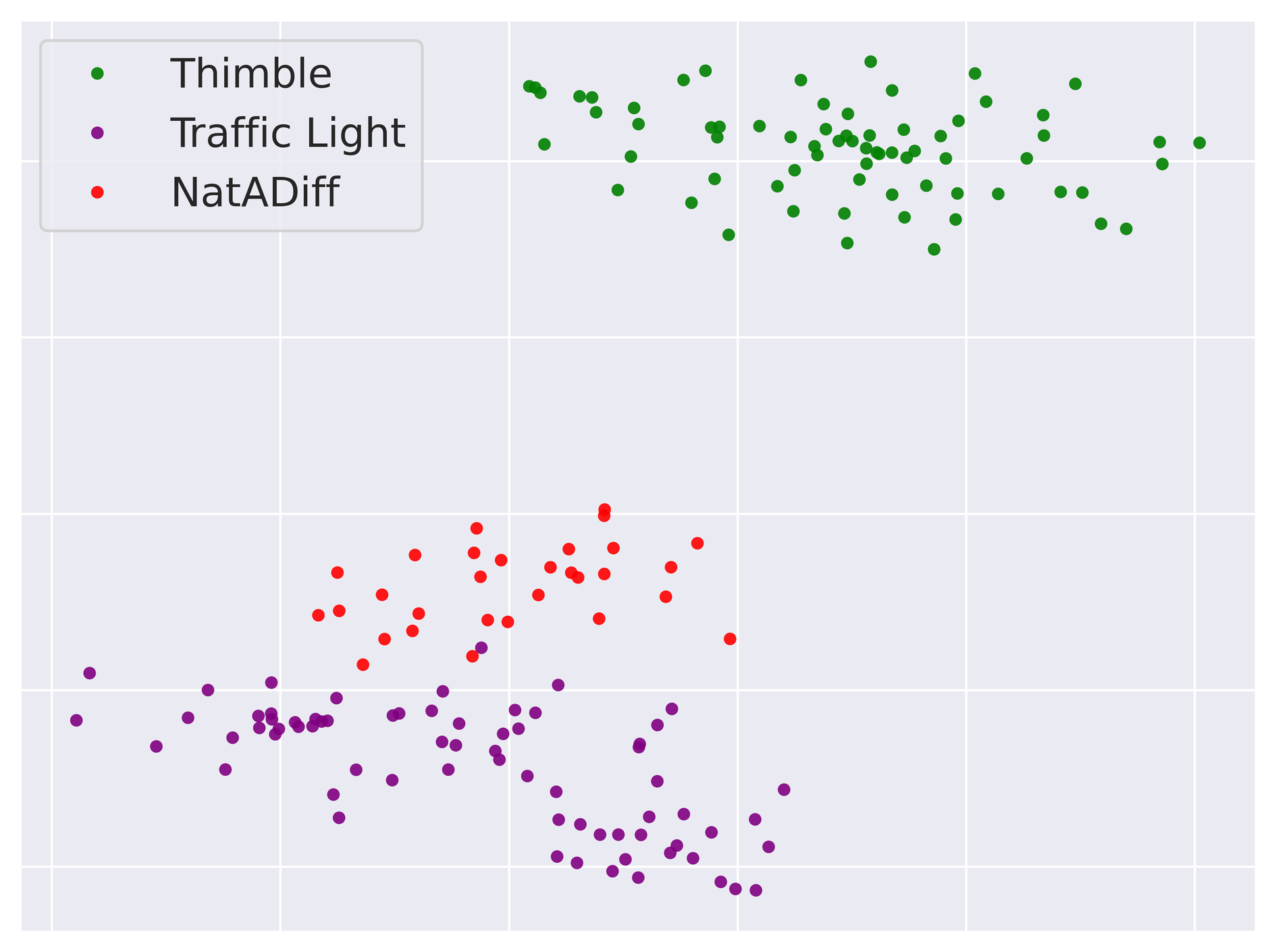} & \includegraphics[width=0.33\textwidth]{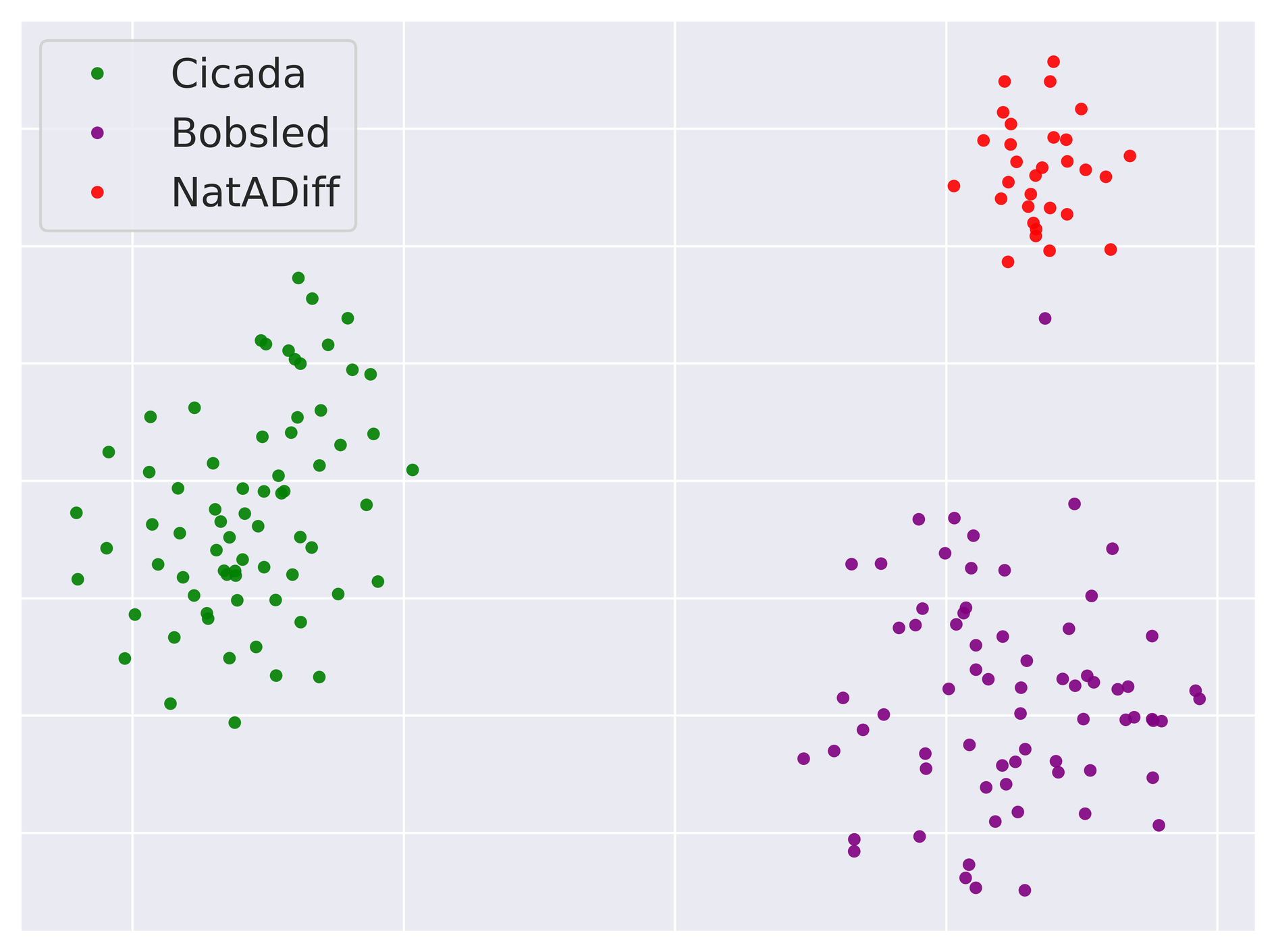} & \includegraphics[width=0.33\textwidth]{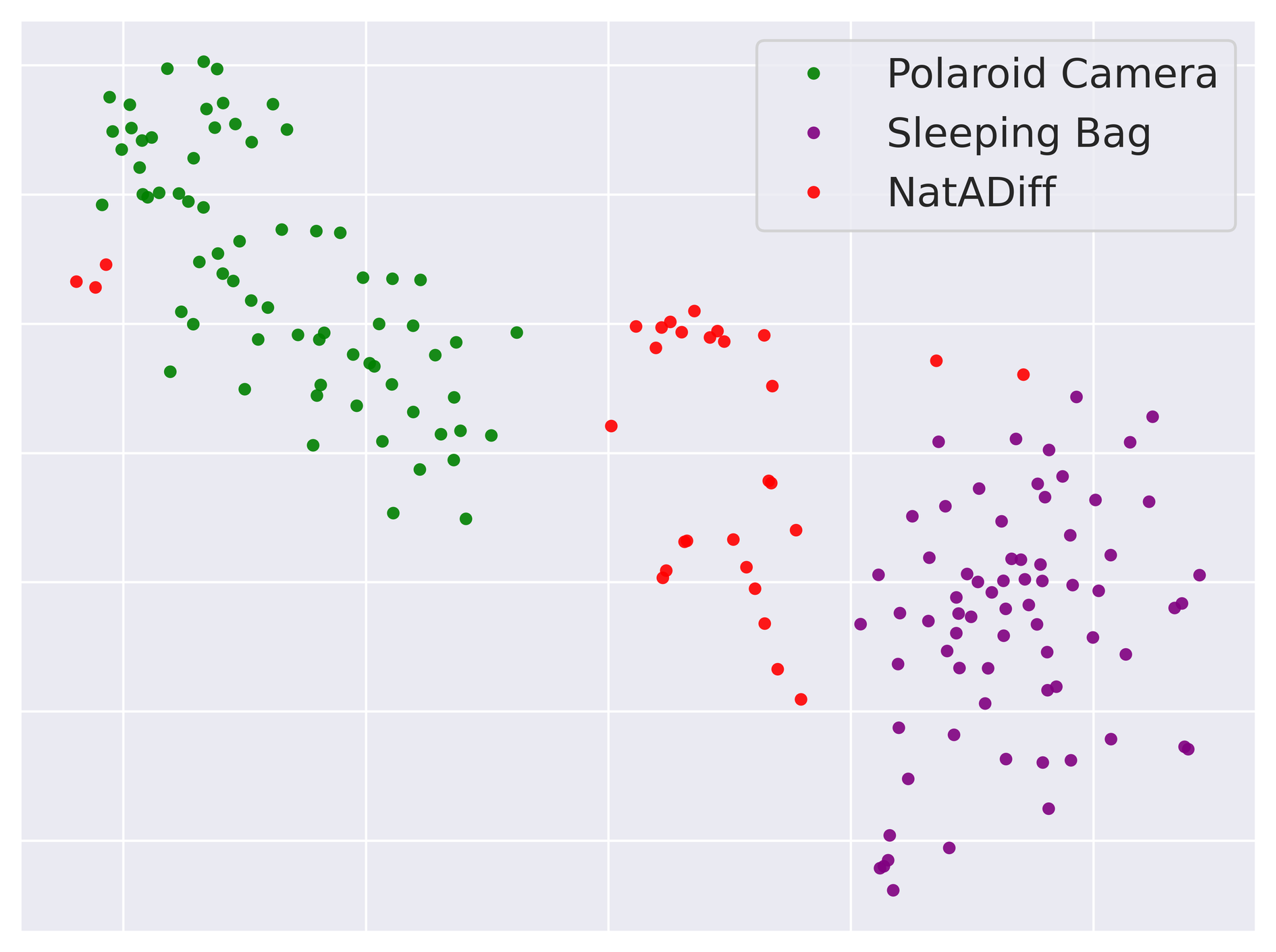} \\

        \includegraphics[width=0.33\textwidth]{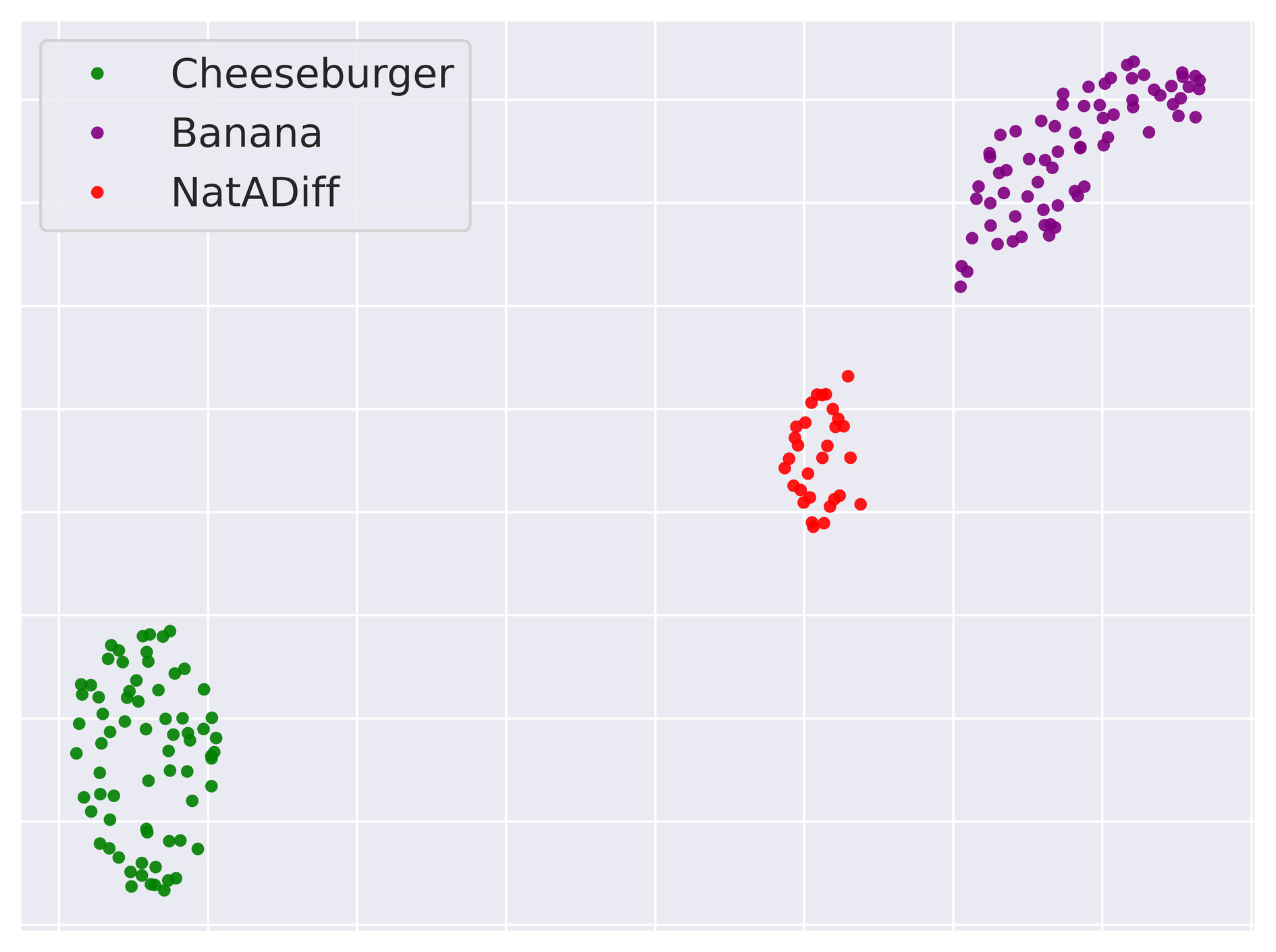} & \includegraphics[width=0.33\textwidth]{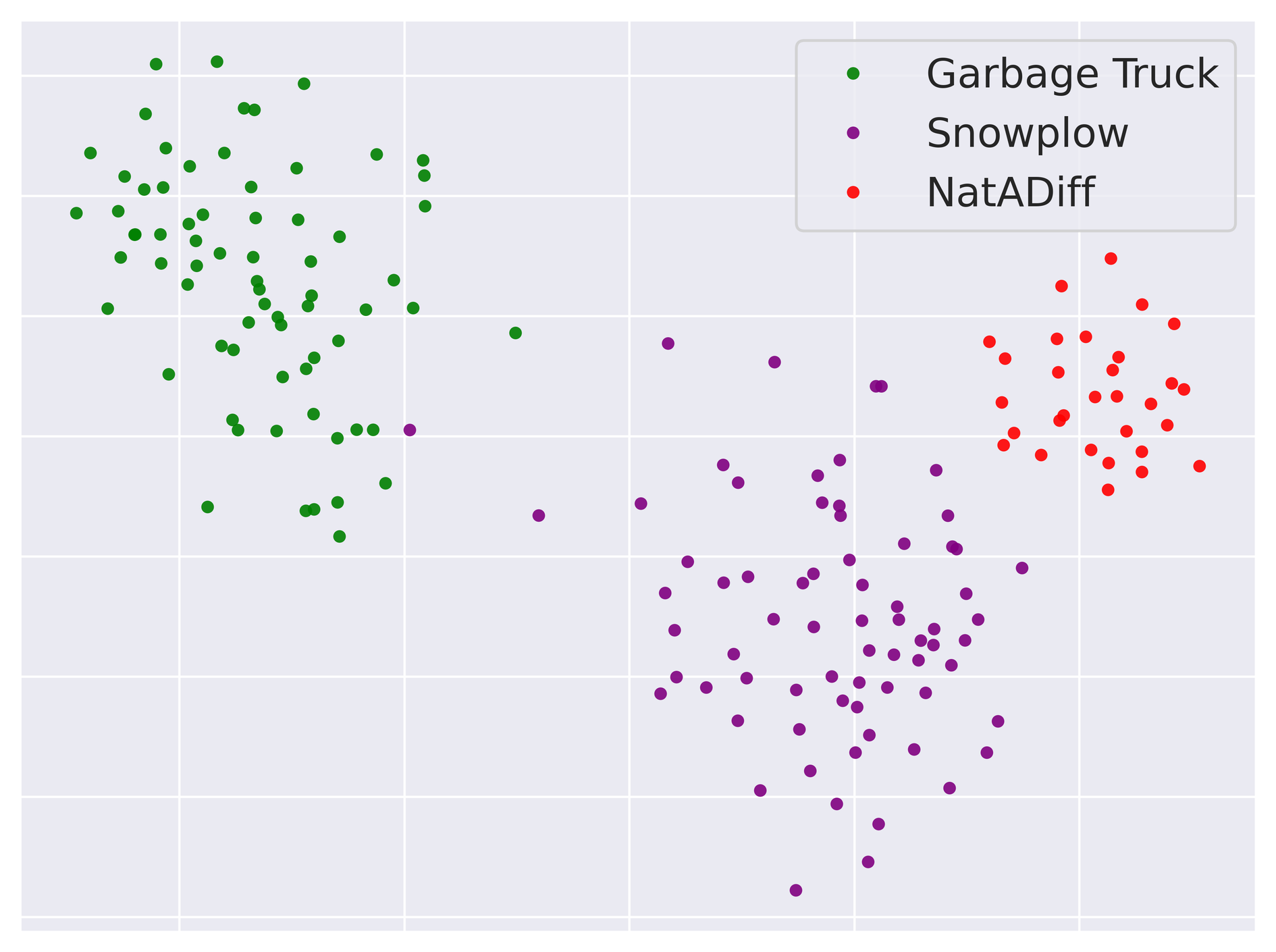} & \includegraphics[width=0.33\textwidth]{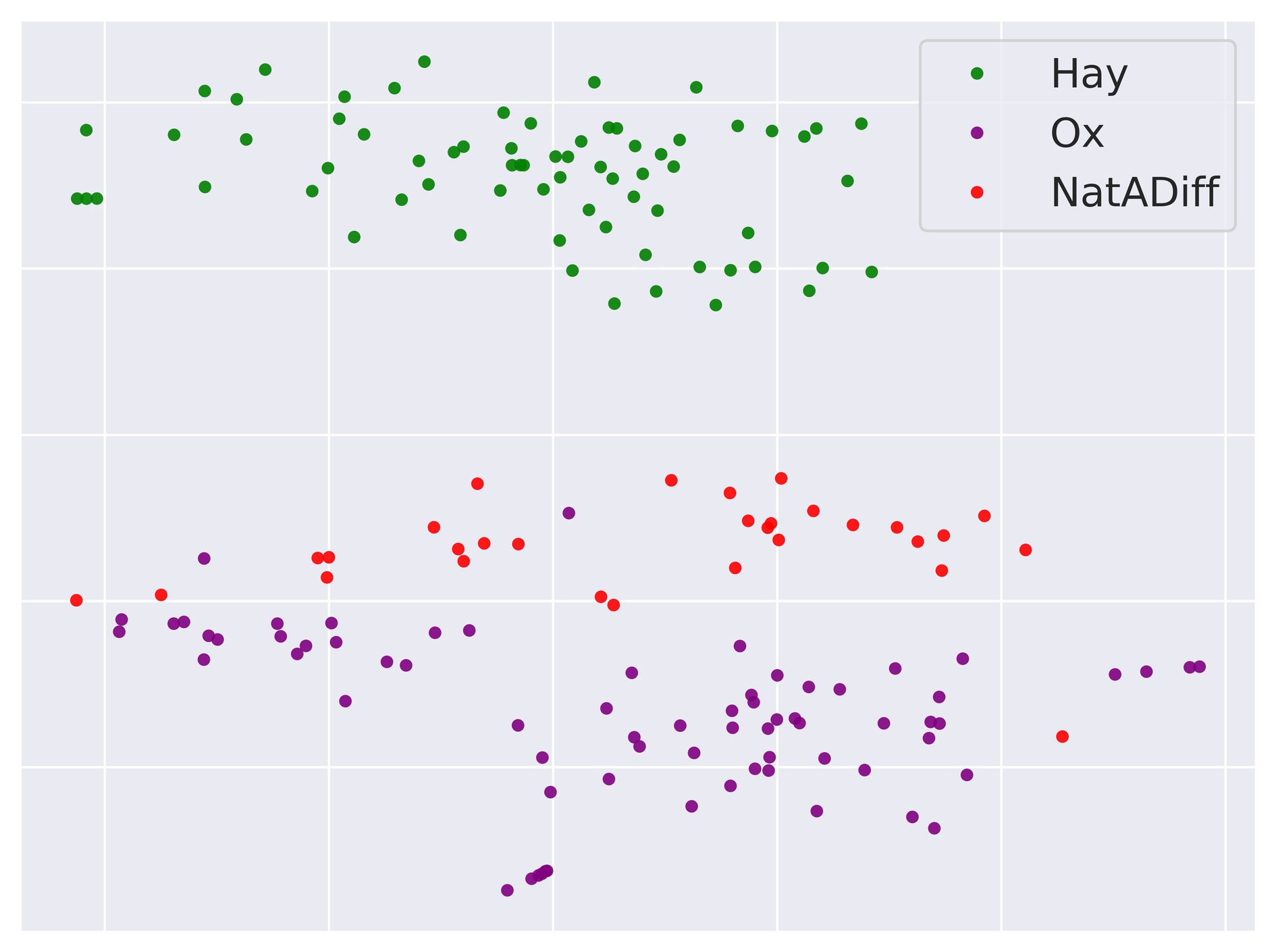} \\

        \includegraphics[width=0.33\textwidth]{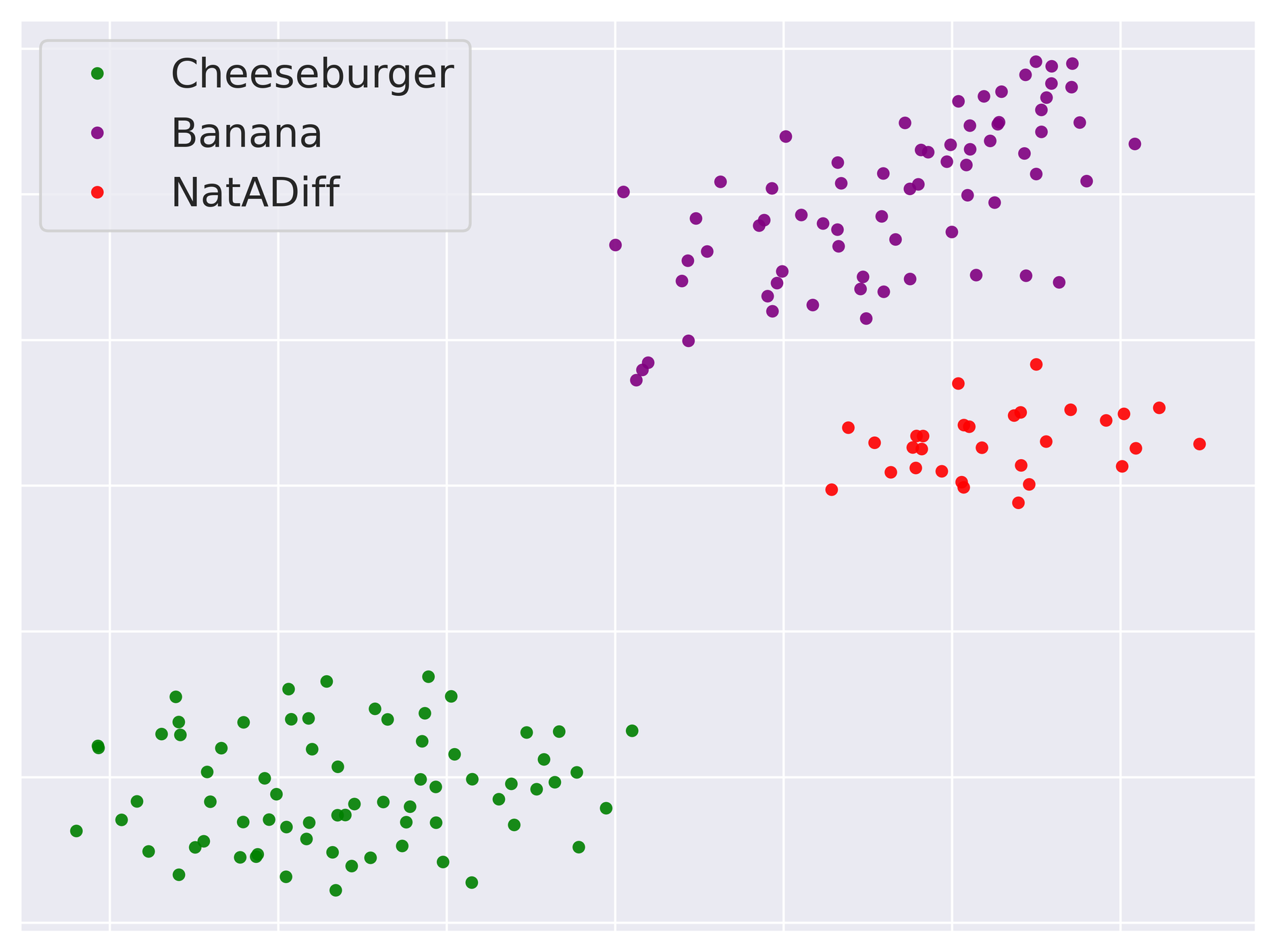} & \includegraphics[width=0.33\textwidth]{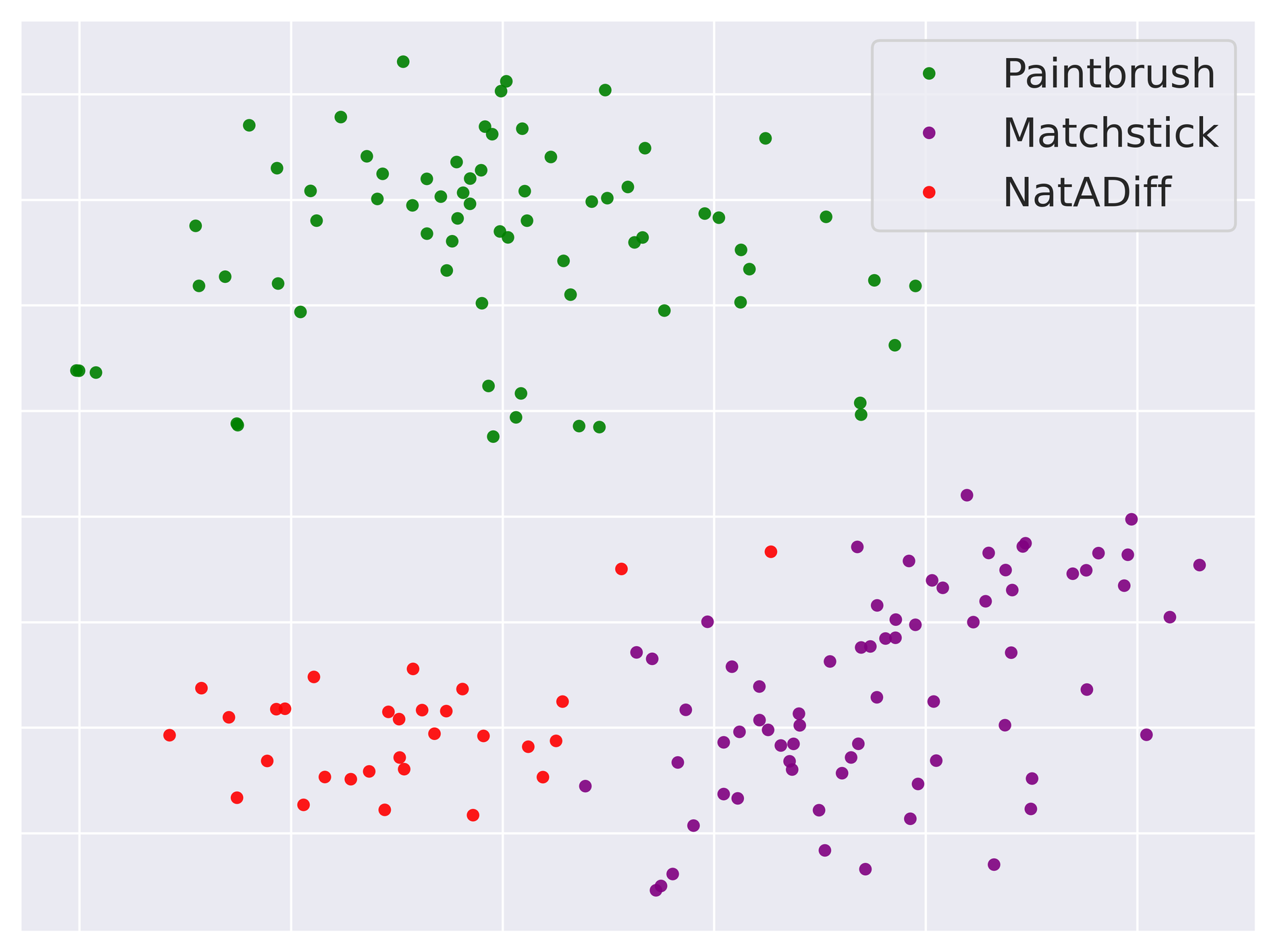} & \includegraphics[width=0.33\textwidth]{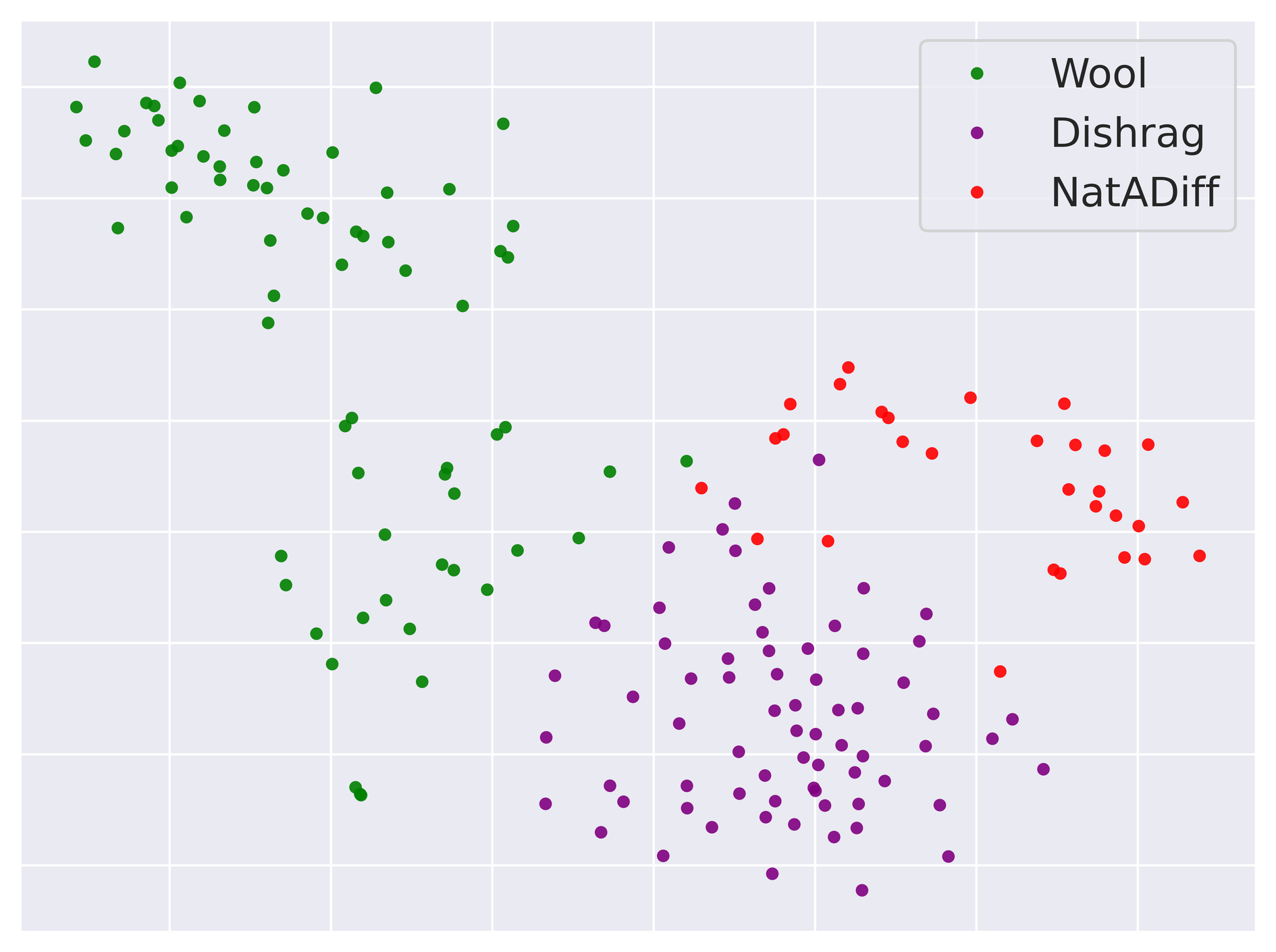}
    \end{tabular}}
    }
    \vspace{-2mm}
    \caption{t-SNE \citep{tSNE_Maaten2008} plots of NatADiff samples alongside the ground truth manifolds. \textcolor{ForestGreen}{\textbf{Green}} and \textcolor{Plum}{\textbf{purple}} samples are real images of the true and adversarial classes taken from the ImageNet dataset \citep{Deng2009}. \textcolor{red}{\textbf{Red}} samples are generated by NatADiff using a ResNet-50 \citep{He2015} surrogate model. t-SNE is applied to embeddings taken from the penultimate layer of a ResNet-50 classifier. Please zoom in for a better visual.}
    \label{fig:NatADiff tSNE plots}
\end{figure*}

\clearpage
\section{NatADiff on Oxford Pets} \label{sec:NatADiff on Oxford Pets}
In the main paper, we present results for NatADiff when targeting ImageNet classifiers exclusively. To strengthen the claims made in the main paper, we extend our evaluation to the Oxford Pets dataset \citep{OxPets_Parkhi2012}, which is more fine-grained and presents a unique challenge as opposed to ImageNet. In this setting, the adversarial structural elements introduced by NatADiff will need to be more subtle.

We use NatADiff to attack ResNet-50 \citep{He2015}, ResNet-152 \citep{He2015}, and ViT \citep{Dosovitskiy2021} classifiers with weights sourced from HuggingFace \citep{HuggingFace_Wolf}. For each class in Oxford Pets, we generate 5 adversarial samples using both targeted (randomly selected) and untargeted (similarity-based) variants of NatADiff. We compare NatADiff against PGD \citep{Madry2019}, AutoAttack \citep{Croce2020}, and adversarial classifier guidance (AdvClass) \citep{Dai2024}. We exclude comparisons with NCF \citep{NCF_Yuan2022}, DiffAttack \citep{DiffAttack_Chen2025}, and ACA \citep{ContentDiffusionAttack_Chen2023}, as these methods do not provide variants targeting Oxford Pets classifiers. We report attack success rate (ASR) and the image quality metrics BRISQUE \citep{BRISQUE_Mittal2012} and TReS \citep{TReS_Alireza2022}.

From Table~\ref{tab:Oxford Pets classifier ASR and image quality}, we observe that NatADiff behaves similarly to the ImageNet study, outperforming other methods in terms of attack transferability while maintaining competitive white-box ASR. NatADiff also typically achieves better BRISQUE and is close to AdvClass in terms of TReS, suggesting that NatADiff produces high-fidelity samples, as shown in Figure~\ref{fig:NatADiff Oxford Pets samples}. Overall, NatADiff generalizes well to fine-grained datasets, with adversarial boundary guidance extending to minute structural changes. These results are consistent with the strong performance of similarity-targeted attacks in the ImageNet study, indicating that adversarial boundary guidance is particularly effective when classes share similar structural elements.

\begin{table}[h]
    \centering
    \caption{\textbf{Attack success rate} (\%) and \textbf{image quality} of adversarial samples generated for the Oxford Pets dataset \citep{OxPets_Parkhi2012}. Superscripts T and U denote random and similarity targeted attacks, respectively. \textcolor{red}{\textbf{Bold}} and \textcolor{BrickRed}{\underline{underlined}} values highlight the best and second best scores. White-box ASR (same surrogate and victim model) is denoted with an $^*$. Note we do not report image quality for constrained perturbation-based attacks (these attacks make imperceptible image alterations).}
    \label{tab:Oxford Pets classifier ASR and image quality}
    \resizebox{1\textwidth}{!}{
        \begin{tabular}{@{}cc@{}c@{\hspace{1em}}ccc@{}cccc@{}}
            \specialrule{1.2pt}{0pt}{0pt}
            \noalign{\vskip 0.5ex}
            \multicolumn{1}{@{}c}{\multirow{2}{*}{\shortstack{Surrogate\\Model}}} & \multicolumn{1}{c@{}}{\multirow{2}{*}{Attack}} && \multicolumn{3}{c}{Victim Model ASR (\%)} && \multicolumn{1}{c}{\multirow{2}{*}{\shortstack{Average\\ASR}}} & \multicolumn{1}{c}{\multirow{2}{*}{\shortstack{BRISQUE ($\boldsymbol{\downarrow}$)}}} & \multicolumn{1}{c}{\multirow{2}{*}{\shortstack{TReS ($\boldsymbol{\uparrow}$)}}} \\
            \cline{4-6}
            \noalign{\vskip 0.5ex}
            &&& ResNet-50 & ResNet-152 & ViT &&&& \\
            \specialrule{1.2pt}{0pt}{0pt}
            \noalign{\vskip 0.5ex}
            & Clean & & $6.0$ & $8.6$ & $3.8$ && $6.1$ & - & - \\
            \specialrule{1.2pt}{0pt}{0pt}
            \noalign{\vskip 0.2ex}
            \multicolumn{1}{@{}c}{\multirow{6}{*}{ResNet-50}} & PGD & & $\textcolor{red}{\mathbf{100^*}}$ & $10.3$ & $3.8$ && $38.0$ & - & - \\
            & AA & & $\textcolor{red}{\mathbf{100^*}}$ & $8.1$ & $4.3$ && $37.5$ & - & - \\
            & AdvClass\textsuperscript{T} & & $\textcolor{BrickRed}{\underline{99.5^*}}$ & $17.3$ & $19.5$ && $45.4$ & $5.3$ & $\textcolor{BrickRed}{\underline{80.9}}$ \\
            & AdvClass\textsuperscript{U} & & $98.9^*$ & $19.5$ & $16.8$ && $45.0$ & $\textcolor{BrickRed}{\underline{5.0}}$ & $\textcolor{red}{\mathbf{81.3}}$ \\
            \cline{2-10}
            \noalign{\vskip 0.4ex}
            & NatADiff\textsuperscript{T} & & $97.8^*$ & $\textcolor{BrickRed}{\underline{30.8}}$ & $\textcolor{BrickRed}{\underline{27.6}}$ && $\textcolor{BrickRed}{\underline{52.1}}$ & $\textcolor{red}{\mathbf{4.2}}$ & $78.9$ \\
            & NatADiff\textsuperscript{U} & & $96.8^*$ & $\textcolor{red}{\mathbf{37.3}}$ & $\textcolor{red}{\mathbf{35.7}}$ && $\textcolor{red}{\mathbf{56.6}}$ & $5.5$ & $80.2$ \\
            \specialrule{1.2pt}{0pt}{0pt}
            \noalign{\vskip 0.2ex}
            \multicolumn{1}{@{}c}{\multirow{6}{*}{ResNet-152}} & PGD & & $14.6$ & $\textcolor{BrickRed}{\underline{99.5^*}}$ & $9.2$ && $41.1$ & - & - \\
            & AA & & $12.4$ & $\textcolor{red}{\mathbf{100^*}}$ & $6.5$ && $39.6$ & - & - \\
            & AdvClass\textsuperscript{T} & & $34.6$ & $98.4^*$ & $22.2$ && $51.7$ & $7.2$ & $\textcolor{BrickRed}{\underline{80.3}}$ \\
            & AdvClass\textsuperscript{U} & & $37.3$ & $98.9^*$ & $26.5$ && $54.2$ & $6.9$ & $\textcolor{red}{\mathbf{80.9}}$ \\
            \cline{2-10}
            \noalign{\vskip 0.4ex}
            & NatADiff\textsuperscript{T} & & $\textcolor{BrickRed}{\underline{51.9}}$ & $95.7^*$ & $\textcolor{BrickRed}{\underline{37.3}}$ && $\textcolor{BrickRed}{\underline{61.6}}$ & $\textcolor{BrickRed}{\underline{5.7}}$ & $79.4$ \\
            & NatADiff\textsuperscript{U} & & $\textcolor{red}{\mathbf{57.8}}$ & $98.9^*$ & $\textcolor{red}{\mathbf{48.6}}$ && $\textcolor{red}{\mathbf{68.5}}$ & $\textcolor{red}{\mathbf{5.2}}$ & $79.6$ \\
            \specialrule{1.2pt}{0pt}{0pt}
            \noalign{\vskip 0.2ex}
            \multicolumn{1}{@{}c}{\multirow{6}{*}{ViT}} & PGD & & $13.0$ & $10.8$ & $\textcolor{red}{\mathbf{100^*}}$ && & - & - \\
            & AA & & $7.6$ & $8.6$ & $\textcolor{BrickRed}{\underline{97.3^*}}$ && $37.8$ & - & - \\
            & AdvClass\textsuperscript{T} & & $28.1$ & $22.2$ & $\textcolor{red}{\mathbf{100^*}}$ && $50.1$ & $6.7$ & $\textcolor{BrickRed}{\underline{80.2}}$ \\
            & AdvClass\textsuperscript{U} & & $24.3$ & $24.3$ & $\textcolor{red}{\mathbf{100^*}}$ && $49.5$ & $6.7$ & $80.1$ \\
            \cline{2-10}
            \noalign{\vskip 0.4ex}
            & NatADiff\textsuperscript{T} & & $\textcolor{red}{\mathbf{56.2}}$ & $\textcolor{BrickRed}{\underline{29.2}}$ & $\textcolor{red}{\mathbf{100^*}}$ && $\textcolor{BrickRed}{\underline{61.8}}$ & $\textcolor{red}{\mathbf{4.9}}$ & $\textcolor{red}{\mathbf{80.3}}$ \\
            & NatADiff\textsuperscript{U} & & $\textcolor{BrickRed}{\underline{54.6}}$ & $\textcolor{red}{\mathbf{46.5}}$ & $\textcolor{red}{\mathbf{100^*}}$ && $\textcolor{red}{\mathbf{67.0}}$ & $\textcolor{BrickRed}{\underline{5.3}}$ & $80.1$ \\
            \specialrule{1.2pt}{0pt}{0pt}
            \noalign{\vskip 0.2ex}
        \end{tabular}
    }
    \vspace{-6mm}
\end{table}

\begin{figure*}
    \centering
    {
    \resizebox{\linewidth}{!}{%
    \scriptsize
    \begin{tabular}{@{}l@{}l@{\hspace{0.5em}}l@{}l@{\hspace{0.5em}}l@{}l@{}}
        \multicolumn{2}{c}{\normalsize\textbf{ResNet-50}} & \multicolumn{2}{c}{\normalsize\textbf{ResNet-152}} & \multicolumn{2}{c}{\normalsize\textbf{ViT}} \\
        \includegraphics[width=0.2\textwidth]{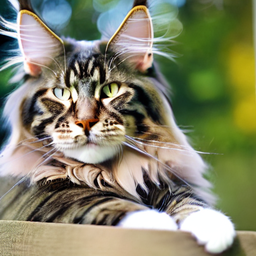} & \includegraphics[width=0.2\textwidth]{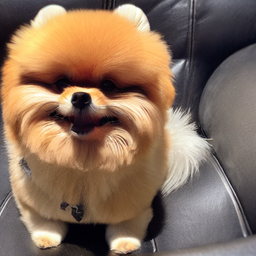} & \includegraphics[width=0.2\textwidth]{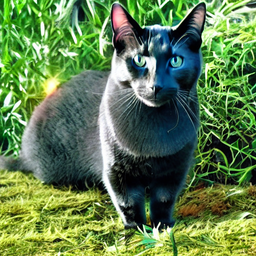} & \includegraphics[width=0.2\textwidth]{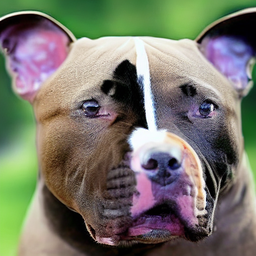} & \includegraphics[width=0.2\textwidth]{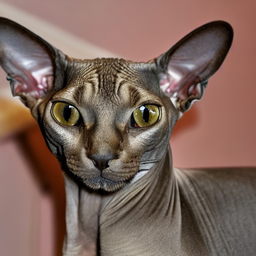} & \includegraphics[width=0.2\textwidth]{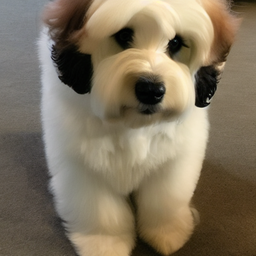} \\

        \textcolor{Green}{\textbf{True:}} Maine Coon & \textcolor{Green}{\textbf{True:}} Pomeranian & \textcolor{Green}{\textbf{True:}} Bombay & \textcolor{Green}{\textbf{True:}} Pit Bull Terrier & \textcolor{Green}{\textbf{True:}} Sphynx & \textcolor{Green}{\textbf{True:}} Havanese \\
        \textcolor{red}{\textbf{Adv\textsuperscript{T}:}} Bengal & \textcolor{red}{\textbf{Adv\textsuperscript{T}:}} Havanese & \textcolor{red}{\textbf{Adv\textsuperscript{T}:}} Russian Blue & \textcolor{red}{\textbf{Adv\textsuperscript{T}:}} Pomeranian & \textcolor{red}{\textbf{Adv\textsuperscript{T}:}} Abyssinian & \textcolor{red}{\textbf{Adv\textsuperscript{T}:}} Newfoundland \\
        
        \includegraphics[width=0.2\textwidth]{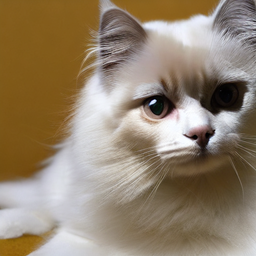} & \includegraphics[width=0.2\textwidth]{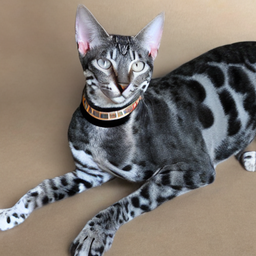} & \includegraphics[width=0.2\textwidth]{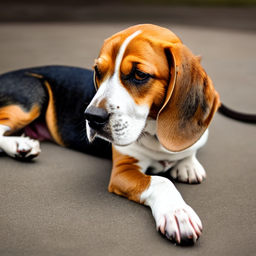} & \includegraphics[width=0.2\textwidth]{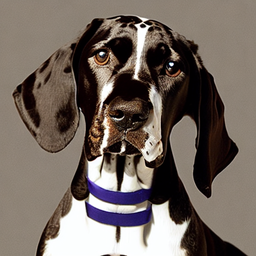} & \includegraphics[width=0.2\textwidth]{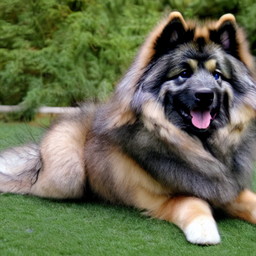} & \includegraphics[width=0.2\textwidth]{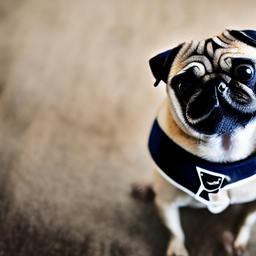} \\
        
        \textcolor{Green}{\textbf{True:}} Birman & \textcolor{Green}{\textbf{True:}} Egyptian Mau & \textcolor{Green}{\textbf{True:}} Beagle & \textcolor{Green}{\textbf{True:}} German Shorthaired & \textcolor{Green}{\textbf{True:}} Keeshound & \textcolor{Green}{\textbf{True:}} Pug \\
        \textcolor{red}{\textbf{Adv\textsuperscript{U}:}} Ragdoll & \textcolor{red}{\textbf{Adv\textsuperscript{U}:}} German Shorthaired & \textcolor{red}{\textbf{Adv\textsuperscript{U}:}} Basset Hound & \textcolor{red}{\textbf{Adv\textsuperscript{U}:}} Basset Hound & \textcolor{red}{\textbf{Adv\textsuperscript{U}:}} Leonberger & \textcolor{red}{\textbf{Adv\textsuperscript{U}:}} Chihuahua \\
    \end{tabular}}
    }
    \vspace{-2mm}
    \caption{Adversarial samples generated by NatADiff for the Oxford Pets dataset \citep{OxPets_Parkhi2012} using ResNet-50 \citep{He2015}, ResNet-152 \citep{He2015}, and ViT \citep{Dosovitskiy2021} surrogate models (see column labels). We report the true class and adversarial target for each image. Superscripts T and U denote random and similarity targeted attacks, respectively.}
    \label{fig:NatADiff Oxford Pets samples}
\end{figure*}

\clearpage
\section{NatADiff experiment parameter settings} \label{apdx:Experiment parameter settings}
Here we provide the NatADiff parameter values used in our experiments (see Tables~\ref{tab:ImageNet experiment parameters} and \ref{tab:Oxford Pets experiment parameters}).

\begin{table}[h]
    \centering
    \caption{\textbf{Experiment parameters} used in diffusion-based adversarial sampling experiments on the ImageNet dataset \citep{Deng2009}. Parameters refer to those defined in Algorithms~\ref{alg:NatADiff} and \ref{alg:NatADiff Similarity}. Experiments were conducted with ResNet-50 \citep{He2015}, Inception-v3 \citep{Inceptionv3_Szegedy2016}, and ViT-H \citep{Dosovitskiy2021} surrogate models.}
    \label{tab:ImageNet experiment parameters}
    \resizebox{1\textwidth}{!}{
        \begin{tabular}{@{}cc@{\hspace{1em}}c@{}cccccccccccc@{}}
            \specialrule{1.2pt}{0pt}{0pt}
            \noalign{\vskip 0.5ex}
            \multicolumn{1}{@{}c}{\multirow{2}{*}{\shortstack{Surrogate\\Model}}} & \multicolumn{1}{c@{\hspace{1em}}}{\multirow{2}{*}{Attack}} && \multicolumn{12}{c}{NatADiff Parameters} \\
            \noalign{\vskip 0.5ex}
            \cline{3-15}
            \noalign{\vskip 0.5ex}
            &&& $\omega$ & $\rho$ & $\mu$ & $s$ & $R$ & $r_l$ & $r_u$ & $c_l$ & $c_u$ & $S$ & $\delta_\mu$ & $\delta_s$ \\
            \specialrule{1.2pt}{0pt}{0pt}
            \noalign{\vskip 0.5ex}
            \multicolumn{1}{@{}c}{\multirow{15}{*}{RN-50}} & NatADiff\textsuperscript{T} (Aug) && $7.5$ & $7.5$ & $0.2$ & $50$ & $5$ & $500$ & $800$ & $0$ & $700$ & $5$ & $0$ & $15$ \\
            & NatADiff\textsuperscript{T} ($\mu = 0.0$) && $7.5$ & $7.5$ & $0.0$ & $50$ & $5$ & $500$ & $800$ & $0$ & $700$ & $5$ & $0$ & $15$ \\
            & NatADiff\textsuperscript{T} ($\mu = 0.1$) && $7.5$ & $7.5$ & $0.1$ & $50$ & $5$ & $500$ & $800$ & $0$ & $700$ & $5$ & $0$ & $15$ \\
            & NatADiff\textsuperscript{T} ($\mu = 0.2$) && $7.5$ & $7.5$ & $0.2$ & $50$ & $5$ & $500$ & $800$ & $0$ & $700$ & $5$ & $0$ & $15$ \\
            & NatADiff\textsuperscript{T} ($\mu = 0.3$) && $7.5$ & $7.5$ & $0.3$ & $50$ & $5$ & $500$ & $800$ & $0$ & $700$ & $5$ & $0$ & $15$ \\
            & NatADiff\textsuperscript{T} ($\mu = 0.4$) && $7.5$ & $7.5$ & $0.4$ & $50$ & $5$ & $500$ & $800$ & $0$ & $700$ & $5$ & $0$ & $15$ \\
            & NatADiff\textsuperscript{T} ($\mu = 0.5$) && $7.5$ & $7.5$ & $0.5$ & $50$ & $5$ & $500$ & $800$ & $0$ & $700$ & $5$ & $0$ & $15$ \\
            & NatADiff\textsuperscript{T} (Prompt) && $7.5$ & $7.5$ & $0.2$ & $50$ & $5$ & $500$ & $800$ & $0$ & $700$ & $5$ & $0$ & $15$ \\            
            & NatADiff\textsuperscript{T} ($s = 20$) && $7.5$ & $7.5$ & $0.2$ & $20$ & $5$ & $500$ & $800$ & $0$ & $700$ & $5$ & $0$ & $15$ \\
            & NatADiff\textsuperscript{T} ($s = 50$) && $7.5$ & $7.5$ & $0.2$ & $50$ & $5$ & $500$ & $800$ & $0$ & $700$ & $5$ & $0$ & $15$ \\
            & NatADiff\textsuperscript{T} ($s = 100$) && $7.5$ & $7.5$ & $0.2$ & $100$ & $5$ & $500$ & $800$ & $0$ & $700$ & $5$ & $0$ & $15$ \\
            & NatADiff\textsuperscript{T} ($s = 200$) && $7.5$ & $7.5$ & $0.2$ & $200$ & $5$ & $500$ & $800$ & $0$ & $700$ & $5$ & $0$ & $15$ \\
            & AdvClass\textsuperscript{T} && $7.5$ & $0.0$ & $0.0$ & $500$ & $0$ & $0$ & $0$ & $0$ & $200$ & $5$ & $0$ & $250$ \\
            & AdvClass\textsuperscript{U} && $7.5$ & $0.0$ & $0.0$ & $500$ & $0$ & $0$ & $0$ & $0$ & $200$ & $5$ & $0$ & $250$ \\ 
            & NatADiff\textsuperscript{T} && $7.5$ & $7.5$ & $0.2$ & $50$ & $5$ & $500$ & $800$ & $0$ & $700$ & $5$ & $0$ & $15$ \\
            & NatADiff\textsuperscript{U} && $7.5$ & $7.5$ & $0.2$ & $50$ & $5$ & $500$ & $800$ & $0$ & $700$ & $5$ & $0$ & $25$ \\
            \specialrule{1.2pt}{0pt}{0pt}
            \noalign{\vskip 0.2ex}
            \multicolumn{1}{@{}c}{\multirow{8}{*}{Inc-v3}} & NatADiff\textsuperscript{T} ($s = 20$) && $7.5$ & $7.5$ & $0.2$ & $20$ & $5$ & $500$ & $800$ & $0$ & $700$ & $5$ & $0$ & $20$ \\
            & NatADiff\textsuperscript{T} ($s = 50$) && $7.5$ & $7.5$ & $0.2$ & $50$ & $5$ & $500$ & $800$ & $0$ & $700$ & $5$ & $0$ & $20$ \\
            & NatADiff\textsuperscript{T} ($s = 100$) && $7.5$ & $7.5$ & $0.2$ & $100$ & $5$ & $500$ & $800$ & $0$ & $700$ & $5$ & $0$ & $20$ \\
            & NatADiff\textsuperscript{T} ($s = 200$) && $7.5$ & $7.5$ & $0.2$ & $200$ & $5$ & $500$ & $800$ & $0$ & $700$ & $5$ & $0$ & $20$ \\
            & AdvClass\textsuperscript{T} && $7.5$ & $0.0$ & $0.0$ & $500$ & $0$ & $0$ & $0$ & $0$ & $200$ & $5$ & $0$ & $250$ \\
            & AdvClass\textsuperscript{U} && $7.5$ & $0.0$ & $0.0$ & $500$ & $0$ & $0$ & $0$ & $0$ & $200$ & $5$ & $0$ & $250$ \\
            & NatADiff\textsuperscript{T} && $7.5$ & $7.5$ & $0.2$ & $50$ & $5$ & $500$ & $800$ & $0$ & $700$ & $5$ & $0$ & $20$ \\
            & NatADiff\textsuperscript{U} && $7.5$ & $7.5$ & $0.2$ & $50$ & $5$ & $500$ & $800$ & $0$ & $700$ & $5$ & $0$ & $20$ \\
            \specialrule{1.2pt}{0pt}{0pt}
            \noalign{\vskip 0.2ex}
            \multicolumn{1}{@{}c}{\multirow{8}{*}{ViT-H}} & NatADiff\textsuperscript{T} ($s = 20$) && $7.5$ & $7.5$ & $0.2$ & $20$ & $5$ & $500$ & $800$ & $0$ & $700$ & $5$ & $0$ & $50$ \\
            & NatADiff\textsuperscript{T} ($s = 50$) && $7.5$ & $7.5$ & $0.2$ & $50$ & $5$ & $500$ & $800$ & $0$ & $700$ & $5$ & $0$ & $50$ \\
            & NatADiff\textsuperscript{T} ($s = 100$) && $7.5$ & $7.5$ & $0.2$ & $100$ & $5$ & $500$ & $800$ & $0$ & $700$ & $5$ & $0$ & $50$ \\
            & NatADiff\textsuperscript{T} ($s = 200$) && $7.5$ & $7.5$ & $0.2$ & $200$ & $5$ & $500$ & $800$ & $0$ & $700$ & $5$ & $0$ & $50$ \\
            & AdvClass\textsuperscript{T} && $7.5$ & $0.0$ & $0.0$ & $500$ & $0$ & $0$ & $0$ & $0$ & $200$ & $5$ & $0$ & $250$ \\
            & AdvClass\textsuperscript{U} && $7.5$ & $0.0$ & $0.0$ & $500$ & $0$ & $0$ & $0$ & $0$ & $200$ & $5$ & $0$ & $250$ \\
            & NatADiff\textsuperscript{T} && $7.5$ & $7.5$ & $0.2$ & $100$ & $5$ & $500$ & $800$ & $0$ & $700$ & $5$ & $0$ & $50$ \\
            & NatADiff\textsuperscript{U} && $7.5$ & $7.5$ & $0.2$ & $100$ & $5$ & $500$ & $800$ & $0$ & $700$ & $5$ & $0$ & $50$ \\
            \specialrule{1.2pt}{0pt}{0pt}
            \noalign{\vskip 0.2ex}
        \end{tabular}
    }
\end{table}

\begin{table}[h]
    \centering
    \caption{\textbf{Experiment parameters} used in diffusion-based adversarial sampling experiments on the Oxford Pets dataset \citep{OxPets_Parkhi2012}. Parameters refer to those defined in Algorithms~\ref{alg:NatADiff} and \ref{alg:NatADiff Similarity}. Experiments were conducted with ResNet-50 \citep{He2015}, ResNet-152 \citep{He2015}, and ViT \citep{Dosovitskiy2021} surrogate models.}
    \label{tab:Oxford Pets experiment parameters}
    \resizebox{1\textwidth}{!}{
        \begin{tabular}{@{}cc@{\hspace{1em}}c@{}cccccccccccc@{}}
            \specialrule{1.2pt}{0pt}{0pt}
            \noalign{\vskip 0.5ex}
            \multicolumn{1}{@{}c}{\multirow{2}{*}{\shortstack{Surrogate\\Model}}} & \multicolumn{1}{c@{\hspace{1em}}}{\multirow{2}{*}{Attack}} && \multicolumn{12}{c}{NatADiff Parameters} \\
            \noalign{\vskip 0.5ex}
            \cline{3-15}
            \noalign{\vskip 0.5ex}
            &&& $\omega$ & $\rho$ & $\mu$ & $s$ & $R$ & $r_l$ & $r_u$ & $c_l$ & $c_u$ & $S$ & $\delta_\mu$ & $\delta_s$ \\
            \specialrule{1.2pt}{0pt}{0pt}
            \noalign{\vskip 0.5ex}
            \multicolumn{1}{@{}c}{\multirow{4}{*}{RN-50}} & AdvClass\textsuperscript{T} && $7.5$ & $0.0$ & $0.0$ & $500$ & $0$ & $0$ & $0$ & $0$ & $200$ & $5$ & $0$ & $250$ \\
            & AdvClass\textsuperscript{U} && $7.5$ & $0.0$ & $0.0$ & $500$ & $0$ & $0$ & $0$ & $0$ & $200$ & $5$ & $0$ & $250$ \\ 
            & NatADiff\textsuperscript{T} && $7.5$ & $7.5$ & $0.2$ & $50$ & $5$ & $500$ & $800$ & $0$ & $600$ & $5$ & $0$ & $25$ \\
            & NatADiff\textsuperscript{U} && $7.5$ & $7.5$ & $0.2$ & $50$ & $5$ & $500$ & $800$ & $0$ & $600$ & $5$ & $0$ & $25$ \\
            \specialrule{1.2pt}{0pt}{0pt}
            \noalign{\vskip 0.2ex}
            \multicolumn{1}{@{}c}{\multirow{4}{*}{RN-152}} & AdvClass\textsuperscript{T} && $7.5$ & $0.0$ & $0.0$ & $500$ & $0$ & $0$ & $0$ & $0$ & $200$ & $5$ & $0$ & $250$ \\
            & AdvClass\textsuperscript{U} && $7.5$ & $0.0$ & $0.0$ & $500$ & $0$ & $0$ & $0$ & $0$ & $200$ & $5$ & $0$ & $250$ \\
            & NatADiff\textsuperscript{T} && $7.5$ & $7.5$ & $0.2$ & $50$ & $5$ & $500$ & $800$ & $0$ & $600$ & $5$ & $0$ & $25$ \\
            & NatADiff\textsuperscript{U} && $7.5$ & $7.5$ & $0.2$ & $50$ & $5$ & $500$ & $800$ & $0$ & $600$ & $5$ & $0$ & $25$ \\
            \specialrule{1.2pt}{0pt}{0pt}
            \noalign{\vskip 0.2ex}
            \multicolumn{1}{@{}c}{\multirow{4}{*}{ViT}} & AdvClass\textsuperscript{T} && $7.5$ & $0.0$ & $0.0$ & $500$ & $0$ & $0$ & $0$ & $0$ & $200$ & $5$ & $0$ & $250$ \\
            & AdvClass\textsuperscript{U} && $7.5$ & $0.0$ & $0.0$ & $500$ & $0$ & $0$ & $0$ & $0$ & $200$ & $5$ & $0$ & $250$ \\
            & NatADiff\textsuperscript{T} && $7.5$ & $7.5$ & $0.2$ & $50$ & $5$ & $500$ & $800$ & $0$ & $600$ & $5$ & $0$ & $25$ \\
            & NatADiff\textsuperscript{U} && $7.5$ & $7.5$ & $0.2$ & $50$ & $5$ & $500$ & $800$ & $0$ & $600$ & $5$ & $0$ & $25$ \\
            \specialrule{1.2pt}{0pt}{0pt}
            \noalign{\vskip 0.2ex}
        \end{tabular}
    }
\end{table}

\clearpage
\section{Additional NatADiff samples}
We provide NatADiff samples alongside the classification scores of ResNet-50 \citep{He2015}, Inception-v3 \citep{Inceptionv3_Szegedy2016}, ViT-H \citep{Dosovitskiy2021}, and adversarially trained ResNet-50 and Inception victim models \citep{Kurakin2018}. Samples were generated using ResNet-50 \citep{He2015}, Inception-v3 \citep{Inceptionv3_Szegedy2016}, and ViT-H \citep{Dosovitskiy2021} surrogate models.

\subsection{Mixed Samples}
\begin{figure*}[h]
    \centering
    {
    \resizebox{\linewidth}{!}{%
    \tiny
    \begin{tabular}{@{}l@{}l@{\hspace{0.5em}}l@{}l@{\hspace{0.5em}}l@{}l@{}}
        \multicolumn{2}{c}{\large\textbf{ResNet-50}} & \multicolumn{2}{c}{\large\textbf{Inception-v3}} & \multicolumn{2}{c}{\large\textbf{ViT-H}} \\
        \includegraphics[width=0.2\textwidth]{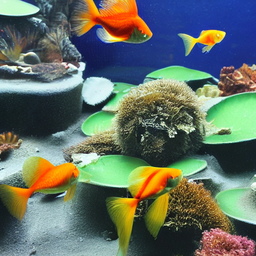} & \includegraphics[width=0.2\textwidth]{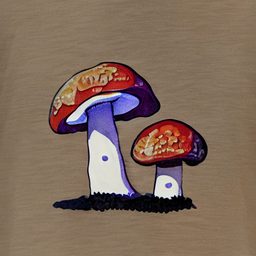} & \includegraphics[width=0.2\textwidth]{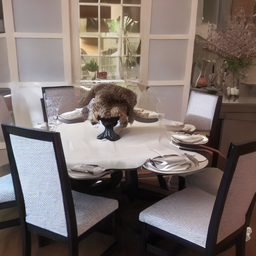} & \includegraphics[width=0.2\textwidth]{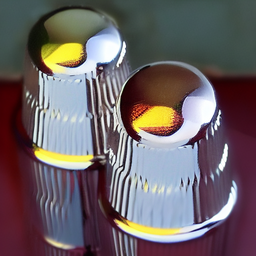} & \includegraphics[width=0.2\textwidth]{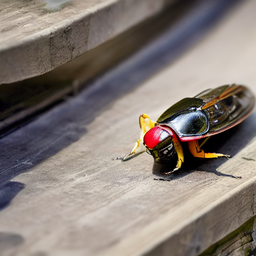} & \includegraphics[width=0.2\textwidth]{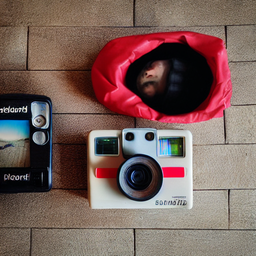} \\

        \textcolor{Green}{\textbf{True:}} Goldfish & \textcolor{Green}{\textbf{True:}} Mushroom & \textcolor{Green}{\textbf{True:}} Dining Table & \textcolor{Green}{\textbf{True:}} Thimble & \textcolor{Green}{\textbf{True:}} Cicada & \textcolor{Green}{\textbf{True:}} Polaroid Camera \\
        \textcolor{red}{\textbf{Adv\textsuperscript{T}:}} Titi Monkey & \textcolor{red}{\textbf{Adv\textsuperscript{T}:}} Packet & \textcolor{red}{\textbf{Adv\textsuperscript{T}:}} Platypus & \textcolor{red}{\textbf{Adv\textsuperscript{T}:}} Traffic Light & \textcolor{red}{\textbf{Adv\textsuperscript{T}:}} Bobsled & \textcolor{red}{\textbf{Adv\textsuperscript{T}:}} Sleeping Bag \\
        \textbf{Res:} \textcolor{red}{Titi Monkey $92 \%$} & \textbf{Res:} \textcolor{red}{Packet $80 \%$} & \textbf{Res:} \textcolor{red}{Platypus $40 \%$} & \textbf{Res:} \textcolor{red}{Traffic Light $39 \%$} & \textbf{Res:} \textcolor{red}{Bobsled $16 \%$} & \textbf{Res:} \textcolor{Green}{Polaroid Camera $22 \%$} \\
        \textbf{Inc:} \textcolor{red}{Titi Monkey $72 \%$} & \textbf{Inc:} \textcolor{BurntOrange}{T-Shirt $36 \%$} & \textbf{Inc:} \textcolor{red}{Platypus $100 \%$} & \textbf{Inc:} \textcolor{red}{Traffic Light $100 \%$} & \textbf{Inc:} \textcolor{BurntOrange}{Leaf Beetle $18 \%$} & \textbf{Inc:} \textcolor{Green}{Polaroid Camera $76 \%$} \\
        \textbf{ViT:} \textcolor{Green}{Goldfish $48 \%$} & \textbf{ViT:} \textcolor{BurntOrange}{T-Shirt $87 \%$} & \textbf{ViT:} \textcolor{Green}{Dining Table $100 \%$} & \textbf{ViT:} \textcolor{BurntOrange}{Saltshaker $69 \%$} & \textbf{ViT:} \textcolor{red}{Bobsled $100 \%$} & \textbf{ViT:} \textcolor{red}{Sleeping Bag $100 \%$} \\
        \textbf{AdvRes:} \textcolor{Green}{Goldfish $85 \%$} & \textbf{AdvRes:} \textcolor{BurntOrange}{Bolete $20 \%$} & \textbf{AdvRes:} \textcolor{red}{Platypus $82 \%$} & \textbf{AdvRes:} \textcolor{red}{Traffic Light $83 \%$} & \textbf{AdvRes:} \textcolor{Green}{Cicada $75 \%$} & \textbf{AdvRes:} \textcolor{red}{Sleeping Bag $97 \%$} \\
        \textbf{AdvInc:} \textcolor{red}{Titi Monkey $50 \%$} & \textbf{AdvInc:} \textcolor{red}{Packet $24 \%$} & \textbf{AdvInc:} \textcolor{red}{Platypus $86 \%$} & \textbf{AdvInc:} \textcolor{BurntOrange}{Tennis Ball $46 \%$} & \textbf{AdvInc:} \textcolor{BurntOrange}{Fly $26 \%$} & \textbf{AdvInc:} \textcolor{Green}{Polaroid Camera $70 \%$} \\
        
        \includegraphics[width=0.2\textwidth]{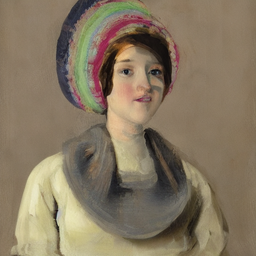} & \includegraphics[width=0.2\textwidth]{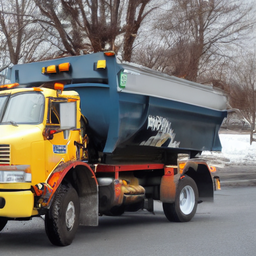} & \includegraphics[width=0.2\textwidth]{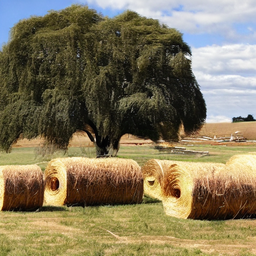} & \includegraphics[width=0.2\textwidth]{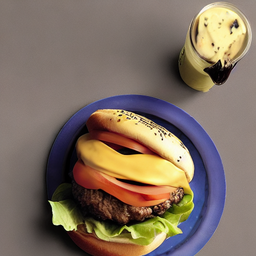} & \includegraphics[width=0.2\textwidth]{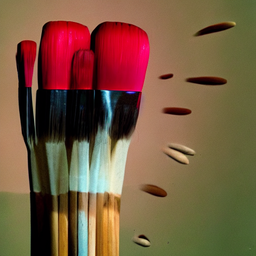} & \includegraphics[width=0.2\textwidth]{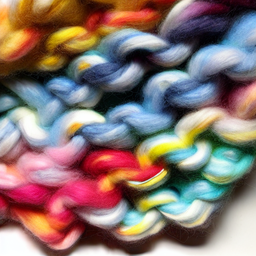} \\
        
        \textcolor{Green}{\textbf{True:}} Bonnet & \textcolor{Green}{\textbf{True:}} Garbage Truck & \textcolor{Green}{\textbf{True:}} Hay & \textcolor{Green}{\textbf{True:}} Cheeseburger & \textcolor{Green}{\textbf{True:}} Paintbrush & \textcolor{Green}{\textbf{True:}} Wool \\
        \textcolor{red}{\textbf{Adv\textsuperscript{U}:}} Sombrero & \textcolor{red}{\textbf{Adv\textsuperscript{U}:}} Snowplow & \textcolor{red}{\textbf{Adv\textsuperscript{U}:}} Ox & \textcolor{red}{\textbf{Adv\textsuperscript{U}:}} Banana & \textcolor{red}{\textbf{Adv\textsuperscript{U}:}} Matchstick & \textcolor{red}{\textbf{Adv\textsuperscript{U}:}} Dishrag \\
        \textbf{Res:} \textcolor{red}{Sombrero $97 \%$} & \textbf{Res:} \textcolor{red}{Snowplow $91 \%$} & \textbf{Res:} \textcolor{Green}{Hay $64 \%$} & \textbf{Res:} \textcolor{Green}{Cheeseburger $14 \%$} & \textbf{Res:} \textcolor{Green}{Paintbrush $26 \%$} & \textbf{Res:} \textcolor{Green}{Wool $15 \%$} \\
        \textbf{Inc:} \textcolor{red}{Sombrero $97 \%$} & \textbf{Inc:} \textcolor{red}{Snowplow $57\%$} & \textbf{Inc:} \textcolor{red}{Ox $100 \%$} & \textbf{Inc:} \textcolor{red}{Banana $100 \%$} & \textbf{Inc:} \textcolor{red}{Matchstick $71 \%$} & \textbf{Inc:} \textcolor{BurntOrange}{Chain Armor $26 \%$} \\
        \textbf{ViT:} \textcolor{Green}{Bonnet $92 \%$} & \textbf{ViT:} \textcolor{red}{Snowplow $98 \%$} & \textbf{ViT:} \textcolor{Green}{Hay $100 \%$} & \textbf{ViT:} \textcolor{Green}{Cheeseburger $93 \%$} & \textbf{ViT:} \textcolor{red}{Matchstick $100 \%$} & \textbf{ViT:} \textcolor{red}{Dishrag $100 \%$} \\
        \textbf{AdvRes:} \textcolor{red}{Sombrero $86 \%$} & \textbf{AdvRes:} \textcolor{red}{Snowplow $70 \%$} & \textbf{AdvRes:} \textcolor{red}{Ox $84 \%$} & \textbf{AdvRes:} \textcolor{Green}{Cheeseburger $64 \%$} & \textbf{AdvRes:} \textcolor{Green}{Paintbrush $88 \%$} & \textbf{AdvRes:} \textcolor{red}{Dishrag $65 \%$} \\
        \textbf{AdvInc:} \textcolor{red}{Sombrero $97 \%$} & \textbf{AdvInc:} \textcolor{Green}{Garbage Truck $52 \%$} & \textbf{AdvInc:} \textcolor{red}{Ox $36 \%$} & \textbf{AdvInc:} \textcolor{red}{Banana $100 \%$} & \textbf{AdvInc:} \textcolor{Green}{Paintbrush $94 \%$} & \textbf{AdvInc:} \textcolor{Green}{Wool $13 \%$} \\
    \end{tabular}}
    }
    \caption{Adversarial samples generated using NatADiff with ResNet-50 \citep{He2015}, Inception-v3 \citep{Inceptionv3_Szegedy2016}, and ViT-H \citep{Dosovitskiy2021} surrogate models (see column labels). We report the true class, adversarial target, and classification scores of the surrogate and adversarially trained ResNet-50 and Inception victim models \citep{Kurakin2018}. Superscripts T and U denote random and similarity targeted attacks, respectively.}
    \label{fig:NatADiff mixed extra samples}
\end{figure*}

\clearpage
\newpage
\subsection{ResNet-50 samples}
\begin{figure*}[h]
    \centering
    {
    \resizebox{\linewidth}{!}{%
    \tiny
    \begin{tabular}{@{}l@{}l@{}l@{}l@{}l@{}l@{}}
        \includegraphics[width=0.2\textwidth]{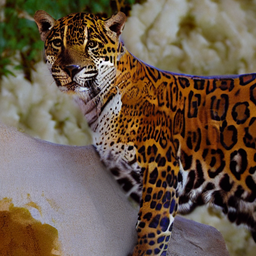} & \includegraphics[width=0.2\textwidth]{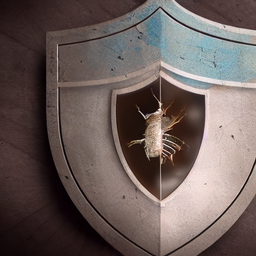} & \includegraphics[width=0.2\textwidth]{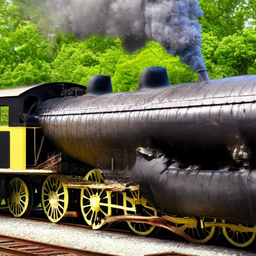} & \includegraphics[width=0.2\textwidth]{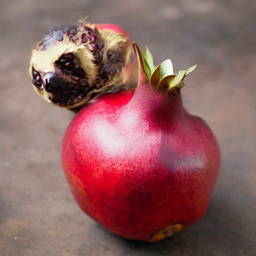} & \includegraphics[width=0.2\textwidth]{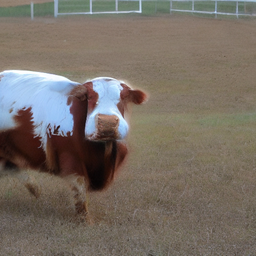} & \includegraphics[width=0.2\textwidth]{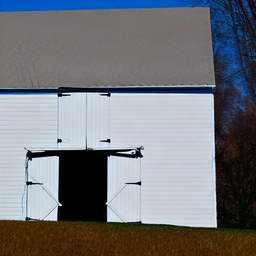}  \\

        \textcolor{Green}{\textbf{True:}} Jaguar & \textcolor{Green}{\textbf{True:}} Shield & \textcolor{Green}{\textbf{True:}} Steam Train & \textcolor{Green}{\textbf{True:}} Pomegranate & \textcolor{Green}{\textbf{True:}} Ox & \textcolor{Green}{\textbf{True:}} Barn  \\
        \textcolor{red}{\textbf{Adv\textsuperscript{T}:}} Mashed Potato & \textcolor{red}{\textbf{Adv\textsuperscript{T}:}} Isopod & \textcolor{red}{\textbf{Adv\textsuperscript{T}:}} Submarine & \textcolor{red}{\textbf{Adv\textsuperscript{T}:}} Three-Toed Sloth & \textcolor{red}{\textbf{Adv\textsuperscript{T}:}} Brittany Spaniel & \textcolor{red}{\textbf{Adv\textsuperscript{T}:}} Ring-Binder  \\
        \textbf{Res:} \textcolor{red}{Mashed Potato $96 \%$} & \textbf{Res:} \textcolor{red}{Isopod $54 \%$} & \textbf{Res:} \textcolor{red}{Submarine $59 \%$} & \textbf{Res:} \textcolor{red}{Three-Toed Sloth $51 \%$} & \textbf{Res:} \textcolor{red}{Brittany Spaniel $57 \%$} & \textbf{Res:} \textcolor{red}{Ring-Binder $79 \%$}  \\
        \textbf{Inc:} \textcolor{Green}{Jaguar $61 \%$} & \textbf{Inc:} \textcolor{red}{Isopod $85 \%$} & \textbf{Inc:} \textcolor{red}{Submarine $98 \%$} & \textbf{Inc:} \textcolor{Green}{Pomegranate $86 \%$} & \textbf{Inc:} \textcolor{red}{Brittany Spaniel $77 \%$} & \textbf{Inc:} \textcolor{red}{Ring-Binder $38 \%$}  \\
        \textbf{ViT:} \textcolor{Green}{Jaguar $97 \%$} & \textbf{ViT:} \textcolor{BurntOrange}{Cockroach $74 \%$} & \textbf{ViT:} \textcolor{Green}{Steam Train $98 \%$} & \textbf{ViT:} \textcolor{red}{Three-Toed Sloth $88 \%$} & \textbf{ViT:} \textcolor{Green}{Ox $93 \%$} & \textbf{ViT:} \textcolor{Green}{Barn $100 \%$}  \\
        \textbf{AdvRes:} \textcolor{red}{Mashed Potato $81 \%$} & \textbf{AdvRes:} \textcolor{Green}{Shield $57 \%$} & \textbf{AdvRes:} \textcolor{Green}{Steam Train $88 \%$} & \textbf{AdvRes:} \textcolor{Green}{Pomegranate $60 \%$} & \textbf{AdvRes:} \textcolor{red}{Brittany Spaniel $90 \%$} & \textbf{AdvRes:} \textcolor{Green}{Barn $24 \%$}  \\
        \textbf{AdvInc:} \textcolor{Green}{Jaguar $76 \%$} & \textbf{AdvInc:} \textcolor{BurntOrange}{Cricket $38 \%$} & \textbf{AdvInc:} \textcolor{red}{Submarine $62 \%$} & \textbf{AdvInc:} \textcolor{Green}{Pomegranate $68 \%$} & \textbf{AdvInc:} \textcolor{Green}{Ox $36 \%$} & \textbf{AdvInc:} \textcolor{Green}{Barn $57 \%$}  \\
        
        \includegraphics[width=0.2\textwidth]{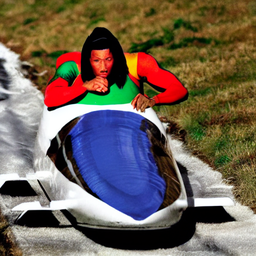} & \includegraphics[width=0.2\textwidth]{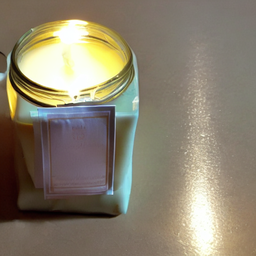} & \includegraphics[width=0.2\textwidth]{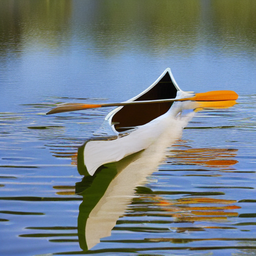} & \includegraphics[width=0.2\textwidth]{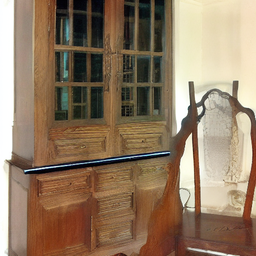} & \includegraphics[width=0.2\textwidth]{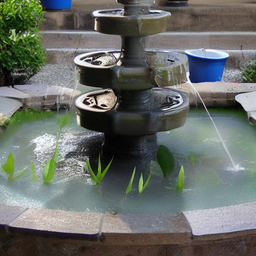} & \includegraphics[width=0.2\textwidth]{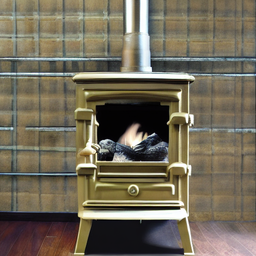}  \\

        \textcolor{Green}{\textbf{True:}} Bobsled & \textcolor{Green}{\textbf{True:}} Candle & \textcolor{Green}{\textbf{True:}} Canoe & \textcolor{Green}{\textbf{True:}} China Cabinet & \textcolor{Green}{\textbf{True:}} Fountain & \textcolor{Green}{\textbf{True:}} Stove  \\
        \textcolor{red}{\textbf{Adv\textsuperscript{T}:}} Russian Wolfhound & \textcolor{red}{\textbf{Adv\textsuperscript{T}:}} Perfume & \textcolor{red}{\textbf{Adv\textsuperscript{T}:}} White Heron & \textcolor{red}{\textbf{Adv\textsuperscript{T}:}} Rifle & \textcolor{red}{\textbf{Adv\textsuperscript{T}:}} Typewriter & \textcolor{red}{\textbf{Adv\textsuperscript{T}:}} Ape  \\
        \textbf{Res:} \textcolor{red}{Russian Wolfhound $99 \%$} & \textbf{Res:} \textcolor{red}{Perfume $96 \%$} & \textbf{Res:} \textcolor{red}{White Heron $44 \%$} & \textbf{Res:} \textcolor{red}{Rifle $98 \%$} & \textbf{Res:} \textcolor{red}{Typewriter $55 \%$} & \textbf{Res:} \textcolor{red}{Ape $73 \%$}  \\
        \textbf{Inc:} \textcolor{red}{Russian Wolfhound $14 \%$} & \textbf{Inc:} \textcolor{red}{Perfume $88 \%$} & \textbf{Inc:} \textcolor{red}{White Heron $15 \%$} & \textbf{Inc:} \textcolor{red}{Rifle $82 \%$} & \textbf{Inc:} \textcolor{BurntOrange}{Dutch Oven $30 \%$} & \textbf{Inc:} \textcolor{BurntOrange}{Hourglass $9 \%$}  \\
        \textbf{ViT:} \textcolor{Green}{Bobsled $100 \%$} & \textbf{ViT:} \textcolor{Green}{Candle $95 \%$} & \textbf{ViT:} \textcolor{Green}{Canoe $84 \%$} & \textbf{ViT:} \textcolor{Green}{China Cabinet $26 \%$} & \textbf{ViT:} \textcolor{Green}{Fountain $99 \%$} & \textbf{ViT:} \textcolor{Green}{Stove $95 \%$}  \\
        \textbf{AdvRes:} \textcolor{Green}{Bobsled $59 \%$} & \textbf{AdvRes:} \textcolor{Green}{Candle $76 \%$} & \textbf{AdvRes:} \textcolor{Green}{Canoe $44 \%$} & \textbf{AdvRes:} \textcolor{Green}{China Cabinet $73 \%$} & \textbf{AdvRes:} \textcolor{BurntOrange}{Dutch Oven $83 \%$} & \textbf{AdvRes:} \textcolor{Green}{Stove $68 \%$}  \\
        \textbf{AdvInc:} \textcolor{BurntOrange}{Sleeping Bag $44 \%$} & \textbf{AdvInc:} \textcolor{red}{Perfume $99 \%$} & \textbf{AdvInc:} \textcolor{red}{White Heron $37 \%$} & \textbf{AdvInc:} \textcolor{red}{Rifle $90 \%$} & \textbf{AdvInc:} \textcolor{BurntOrange}{Dutch Oven $29 \%$} & \textbf{AdvInc:} \textcolor{BurntOrange}{Guillotine $39 \%$}  \\
        
        \includegraphics[width=0.2\textwidth]{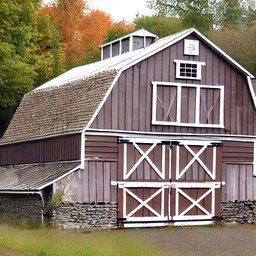} & \includegraphics[width=0.2\textwidth]{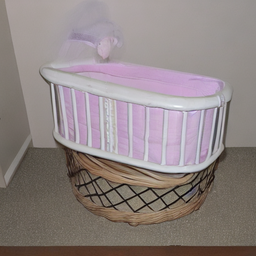} & \includegraphics[width=0.2\textwidth]{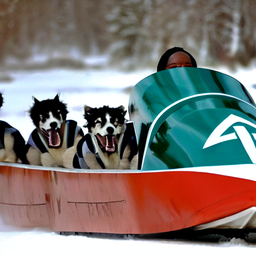} & \includegraphics[width=0.2\textwidth]{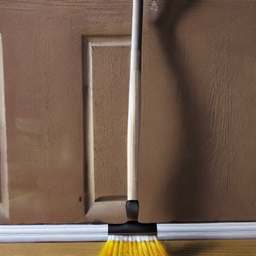} & \includegraphics[width=0.2\textwidth]{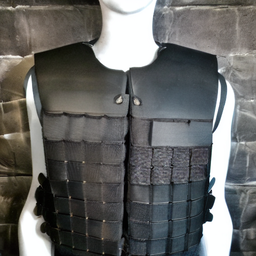} & \includegraphics[width=0.2\textwidth]{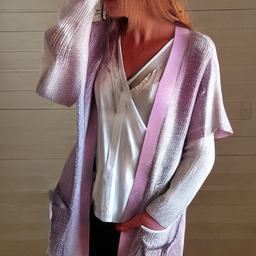}  \\

        \textcolor{Green}{\textbf{True:}} Barn & \textcolor{Green}{\textbf{True:}} Bassinet & \textcolor{Green}{\textbf{True:}} Bobsled & \textcolor{Green}{\textbf{True:}} Broom & \textcolor{Green}{\textbf{True:}} Bulletproof Vest & \textcolor{Green}{\textbf{True:}} Cardigan  \\
        \textcolor{red}{\textbf{Adv\textsuperscript{U}:}} Boathouse & \textcolor{red}{\textbf{Adv\textsuperscript{U}:}} Hamper & \textcolor{red}{\textbf{Adv\textsuperscript{U}:}} Dogsled & \textcolor{red}{\textbf{Adv\textsuperscript{U}:}} Crutch & \textcolor{red}{\textbf{Adv\textsuperscript{U}:}} Cuirass & \textcolor{red}{\textbf{Adv\textsuperscript{U}:}} Kimono  \\
        \textbf{Res:} \textcolor{red}{Boathouse $70 \%$} & \textbf{Res:} \textcolor{red}{Hamper $99 \%$} & \textbf{Res:} \textcolor{red}{Dogsled $59 \%$} & \textbf{Res:} \textcolor{red}{Crutch $34 \%$} & \textbf{Res:} \textcolor{red}{Cuirass $86 \%$} & \textbf{Res:} \textcolor{red}{Kimono $89 \%$}  \\
        \textbf{Inc:} \textcolor{Green}{Barn $90 \%$} & \textbf{Inc:} \textcolor{Green}{Bassinet $38 \%$} & \textbf{Inc:} \textcolor{red}{Dogsled $75 \%$} & \textbf{Inc:} \textcolor{Green}{Broom $39 \%$} & \textbf{Inc:} \textcolor{red}{Cuirass $67 \%$} & \textbf{Inc:} \textcolor{BurntOrange}{Stole $44 \%$}  \\
        \textbf{ViT:} \textcolor{Green}{Barn $96 \%$} & \textbf{ViT:} \textcolor{red}{Hamper $48 \%$} & \textbf{ViT:} \textcolor{red}{Dogsled $78 \%$} & \textbf{ViT:} \textcolor{Green}{Broom $99 \%$} & \textbf{ViT:} \textcolor{Green}{Bulletproof Vest $96 \%$} & \textbf{ViT:} \textcolor{Green}{Cardigan $83 \%$}  \\
        \textbf{AdvRes:} \textcolor{red}{Boathouse $80 \%$} & \textbf{AdvRes:} \textcolor{red}{Hamper $84 \%$} & \textbf{AdvRes:} \textcolor{red}{Dogsled $92 \%$} & \textbf{AdvRes:} \textcolor{red}{Crutch $77 \%$} & \textbf{AdvRes:} \textcolor{red}{Cuirass $57 \%$} & \textbf{AdvRes:} \textcolor{Green}{Cardigan $74 \%$}  \\
        \textbf{AdvInc:} \textcolor{Green}{Barn $65 \%$} & \textbf{AdvInc:} \textcolor{red}{Hamper $84 \%$} & \textbf{AdvInc:} \textcolor{Green}{Bobsled $38 \%$} & \textbf{AdvInc:} \textcolor{red}{Crutch $70 \%$} & \textbf{AdvInc:} \textcolor{red}{Cuirass $44 \%$} & \textbf{AdvInc:} \textcolor{red}{Kimono $44 \%$}  \\

        \includegraphics[width=0.2\textwidth]{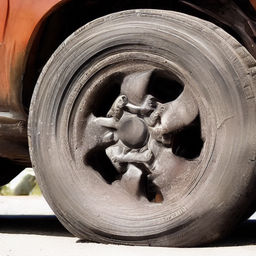} & \includegraphics[width=0.2\textwidth]{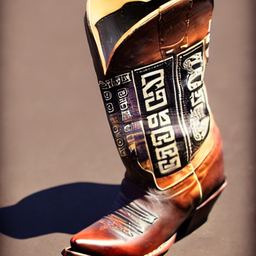} & \includegraphics[width=0.2\textwidth]{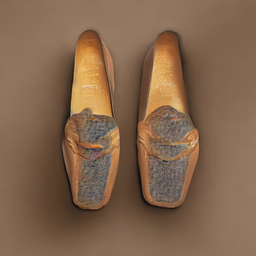} & \includegraphics[width=0.2\textwidth]{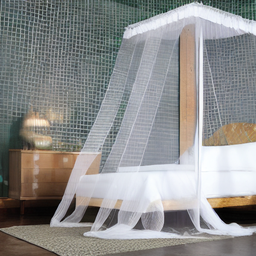} & \includegraphics[width=0.2\textwidth]{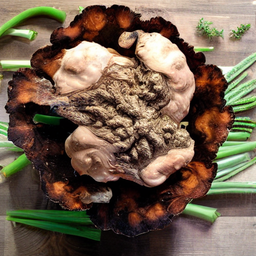} & \includegraphics[width=0.2\textwidth]{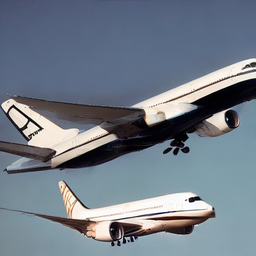}  \\

        \textcolor{Green}{\textbf{True:}} Car Wheel & \textcolor{Green}{\textbf{True:}} Cowboy Boot & \textcolor{Green}{\textbf{True:}} Loafer & \textcolor{Green}{\textbf{True:}} Mosquito Net & \textcolor{Green}{\textbf{True:}} Hen-Of-The-Woods & \textcolor{Green}{\textbf{True:}} Airliner  \\
        \textcolor{red}{\textbf{Adv\textsuperscript{U}:}} Potter'S Wheel & \textcolor{red}{\textbf{Adv\textsuperscript{U}:}} Beer Bottle & \textcolor{red}{\textbf{Adv\textsuperscript{U}:}} French Loaf & \textcolor{red}{\textbf{Adv\textsuperscript{U}:}} Window Screen & \textcolor{red}{\textbf{Adv\textsuperscript{U}:}} Gyromitra & \textcolor{red}{\textbf{Adv\textsuperscript{U}:}} Space Shuttle  \\
        \textbf{Res:} \textcolor{red}{Potter's Wheel $76 \%$} & \textbf{Res:} \textcolor{red}{Beer Bottle $62 \%$} & \textbf{Res:} \textcolor{red}{French Loaf $91 \%$} & \textbf{Res:} \textcolor{red}{Window Screen $66 \%$} & \textbf{Res:} \textcolor{red}{Gyromitra $98 \%$} & \textbf{Res:} \textcolor{red}{Space Shuttle $85 \%$}  \\
        \textbf{Inc:} \textcolor{red}{Potter's Wheel $64 \%$} & \textbf{Inc:} \textcolor{red}{Beer Bottle $44 \%$} & \textbf{Inc:} \textcolor{red}{French Loaf $22 \%$} & \textbf{Inc:} \textcolor{red}{Window Screen $33 \%$} & \textbf{Inc:} \textcolor{BurntOrange}{Ice Cream $64 \%$} & \textbf{Inc:} \textcolor{Green}{Airliner $70 \%$}  \\
        \textbf{ViT:} \textcolor{Green}{Car Wheel $99 \%$} & \textbf{ViT:} \textcolor{Green}{Cowboy Boot $99 \%$} & \textbf{ViT:} \textcolor{Green}{Loafer $81 \%$} & \textbf{ViT:} \textcolor{Green}{Mosquito Net $98 \%$} & \textbf{ViT:} \textcolor{BurntOrange}{Acorn Squash $49 \%$} & \textbf{ViT:} \textcolor{Green}{Airliner $82 \%$}  \\
        \textbf{AdvRes:} \textcolor{Green}{Car Wheel $81 \%$} & \textbf{AdvRes:} \textcolor{Green}{Cowboy Boot $42 \%$} & \textbf{AdvRes:} \textcolor{BurntOrange}{Clog $92 \%$} & \textbf{AdvRes:} \textcolor{Green}{Mosquito Net $81 \%$} & \textbf{AdvRes:} \textcolor{red}{Gyromitra $85 \%$} & \textbf{AdvRes:} \textcolor{red}{Space Shuttle $66 \%$}  \\
        \textbf{AdvInc:} \textcolor{BurntOrange}{Disk Brake $25 \%$} & \textbf{AdvInc:} \textcolor{red}{Beer Bottle $73 \%$} & \textbf{AdvInc:} \textcolor{BurntOrange}{Cowboy Boot $22 \%$} & \textbf{AdvInc:} \textcolor{red}{Window Screen $64 \%$} & \textbf{AdvInc:} \textcolor{BurntOrange}{Ice Cream $46 \%$} & \textbf{AdvInc:} \textcolor{red}{Space Shuttle $39 \%$}  \\
    \end{tabular}}
    }
\caption{Adversarial samples generated by NatADiff with a ResNet-50 \citep{He2015} surrogate model. We report the true class, adversarial target, and classification scores of ResNet-50 \citep{He2015}, Inception-v3 \citep{Inceptionv3_Szegedy2016}, ViT-H \citep{Dosovitskiy2021}, and adversarially trained ResNet-50 and Inception victim models \citep{Kurakin2018}. Superscripts T and U indicate targeted and untargeted (similarity-based) attacks, respectively.}
\label{fig:NatADiff ResNet-50 extra samples}
\end{figure*}

\clearpage
\newpage
\subsection{Inception-v3 samples}
\begin{figure*}[h]
    \centering
    {
    \resizebox{\linewidth}{!}{%
    \tiny
    \begin{tabular}{@{}l@{}l@{}l@{}l@{}l@{}l@{}}
        \includegraphics[width=0.2\textwidth]{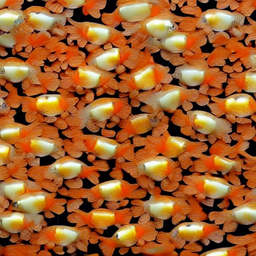} & \includegraphics[width=0.2\textwidth]{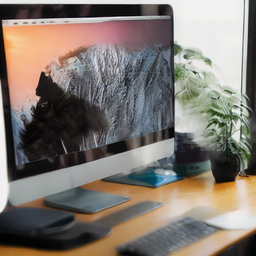} & \includegraphics[width=0.2\textwidth]{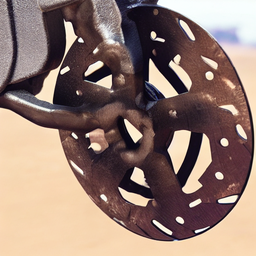} & \includegraphics[width=0.2\textwidth]{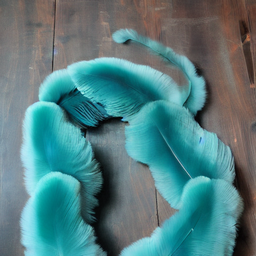} & \includegraphics[width=0.2\textwidth]{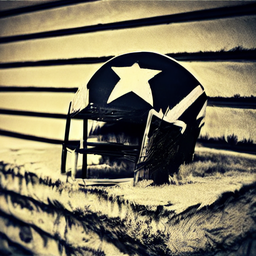} & \includegraphics[width=0.2\textwidth]{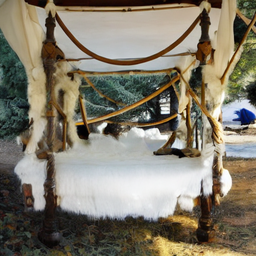}  \\

        \textcolor{Green}{\textbf{True:}} Goldfish & \textcolor{Green}{\textbf{True:}} Desktop Computer & \textcolor{Green}{\textbf{True:}} Disk Brake & \textcolor{Green}{\textbf{True:}} Feather Boa & \textcolor{Green}{\textbf{True:}} Football Helmet & \textcolor{Green}{\textbf{True:}} Four-Poster Bed  \\
        \textcolor{red}{\textbf{Adv\textsuperscript{T}:}} Corn & \textcolor{red}{\textbf{Adv\textsuperscript{T}:}} Persian Cat & \textcolor{red}{\textbf{Adv\textsuperscript{T}:}} Pretzel & \textcolor{red}{\textbf{Adv\textsuperscript{T}:}} Hen-Of-The-Woods & \textcolor{red}{\textbf{Adv\textsuperscript{T}:}} Barn & \textcolor{red}{\textbf{Adv\textsuperscript{T}:}} Dogsled  \\
        \textbf{Res:} \textcolor{red}{Corn $20 \%$} & \textbf{Res:} \textcolor{Green}{Desktop Computer $22 \%$} & \textbf{Res:} \textcolor{red}{Pretzel $48 \%$} & \textbf{Res:} \textcolor{BurntOrange}{Mushroom $12 \%$} & \textbf{Res:} \textcolor{red}{Barn $3 \%$} & \textbf{Res:} \textcolor{red}{Dogsled $42 \%$}  \\
        \textbf{Inc:} \textcolor{red}{Corn $100 \%$} & \textbf{Inc:} \textcolor{red}{Persian Cat $100 \%$} & \textbf{Inc:} \textcolor{red}{Pretzel $100 \%$} & \textbf{Inc:} \textcolor{red}{Hen-Of-The-Woods $80 \%$} & \textbf{Inc:} \textcolor{red}{Barn $100 \%$} & \textbf{Inc:} \textcolor{red}{Dogsled $100 \%$}  \\
        \textbf{ViT:} \textcolor{Green}{Goldfish $97 \%$} & \textbf{ViT:} \textcolor{Green}{Desktop Computer $18 \%$} & \textbf{ViT:} \textcolor{BurntOrange}{Padlock $36 \%$} & \textbf{ViT:} \textcolor{Green}{Feather Boa $95 \%$} & \textbf{ViT:} \textcolor{Green}{Football Helmet $67 \%$} & \textbf{ViT:} \textcolor{Green}{Four-Poster Bed $93 \%$}  \\
        \textbf{AdvRes:} \textcolor{red}{Corn $90 \%$} & \textbf{AdvRes:} \textcolor{red}{Persian Cat $66 \%$} & \textbf{AdvRes:} \textcolor{red}{Pretzel $93 \%$} & \textbf{AdvRes:} \textcolor{BurntOrange}{Jellyfish $13 \%$} & \textbf{AdvRes:} \textcolor{red}{Barn $91 \%$} & \textbf{AdvRes:} \textcolor{red}{Dogsled $97 \%$}  \\
        \textbf{AdvInc:} \textcolor{red}{Corn $70 \%$} & \textbf{AdvInc:} \textcolor{red}{Persian Cat $21 \%$} & \textbf{AdvInc:} \textcolor{red}{Pretzel $94 \%$} & \textbf{AdvInc:} \textcolor{BurntOrange}{Mushroom $45 \%$} & \textbf{AdvInc:} \textcolor{Green}{Football Helmet $25 \%$} & \textbf{AdvInc:} \textcolor{red}{Dogsled $88 \%$}  \\

        \includegraphics[width=0.2\textwidth]{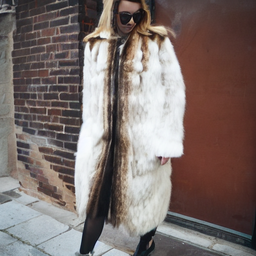} & \includegraphics[width=0.2\textwidth]{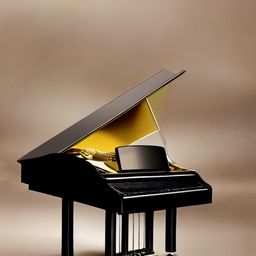} & \includegraphics[width=0.2\textwidth]{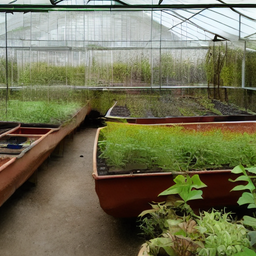} & \includegraphics[width=0.2\textwidth]{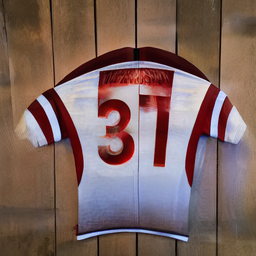} & \includegraphics[width=0.2\textwidth]{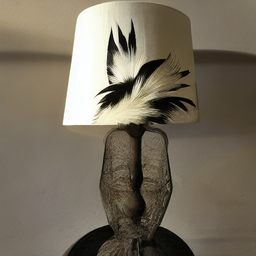} & \includegraphics[width=0.2\textwidth]{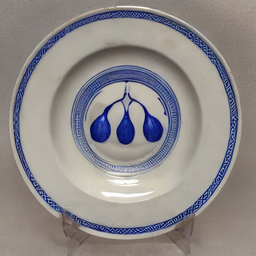}  \\

        \textcolor{Green}{\textbf{True:}} Fur Coat & \textcolor{Green}{\textbf{True:}} Grand Piano & \textcolor{Green}{\textbf{True:}} Greenhouse & \textcolor{Green}{\textbf{True:}} Jersey, T-Shirt, Tee Shirt & \textcolor{Green}{\textbf{True:}} Lampshade & \textcolor{Green}{\textbf{True:}} Plate  \\
        \textcolor{red}{\textbf{Adv\textsuperscript{T}:}} Barracouta & \textcolor{red}{\textbf{Adv\textsuperscript{T}:}} Notebook & \textcolor{red}{\textbf{Adv\textsuperscript{T}:}} Canoe & \textcolor{red}{\textbf{Adv\textsuperscript{T}:}} Planetarium & \textcolor{red}{\textbf{Adv\textsuperscript{T}:}} Skunk & \textcolor{red}{\textbf{Adv\textsuperscript{T}:}} Water Jug  \\
        \textbf{Res:} \textcolor{Green}{Fur Coat $45 \%$} & \textbf{Res:} \textcolor{Green}{Grand Piano $22 \%$} & \textbf{Res:} \textcolor{red}{Canoe $31 \%$} & \textbf{Res:} \textcolor{BurntOrange}{Vestment $9 \%$} & \textbf{Res:} \textcolor{Green}{Lampshade $16 \%$} & \textbf{Res:} \textcolor{BurntOrange}{Soup Bowl $20 \%$}  \\
        \textbf{Inc:} \textcolor{red}{Barracouta $100 \%$} & \textbf{Inc:} \textcolor{red}{Notebook $99 \%$} & \textbf{Inc:} \textcolor{red}{Canoe $100 \%$} & \textbf{Inc:} \textcolor{red}{Planetarium $100 \%$} & \textbf{Inc:} \textcolor{red}{Skunk $99 \%$} & \textbf{Inc:} \textcolor{red}{Water Jug $82 \%$}  \\
        \textbf{ViT:} \textcolor{Green}{Fur Coat $100 \%$} & \textbf{ViT:} \textcolor{Green}{Grand Piano $86 \%$} & \textbf{ViT:} \textcolor{Green}{Greenhouse $99 \%$} & \textbf{ViT:} \textcolor{BurntOrange}{Football Helmet $81 \%$} & \textbf{ViT:} \textcolor{Green}{Lampshade $88 \%$} & \textbf{ViT:} \textcolor{BurntOrange}{Fig $62 \%$}  \\
        \textbf{AdvRes:} \textcolor{red}{Barracouta $89 \%$} & \textbf{AdvRes:} \textcolor{red}{Notebook $83 \%$} & \textbf{AdvRes:} \textcolor{Green}{Greenhouse $79 \%$} & \textbf{AdvRes:} \textcolor{BurntOrange}{Cinema $54 \%$} & \textbf{AdvRes:} \textcolor{red}{Skunk $86 \%$} & \textbf{AdvRes:} \textcolor{Green}{Plate $37 \%$}  \\
        \textbf{AdvInc:} \textcolor{red}{Barracouta $74 \%$} & \textbf{AdvInc:} \textcolor{red}{Notebook $80 \%$} & \textbf{AdvInc:} \textcolor{red}{Canoe $95 \%$} & \textbf{AdvInc:} \textcolor{BurntOrange}{Mailbox $23 \%$} & \textbf{AdvInc:} \textcolor{red}{Skunk $33 \%$} & \textbf{AdvInc:} \textcolor{BurntOrange}{Potter'S Wheel $16 \%$}  \\

        \includegraphics[width=0.2\textwidth]{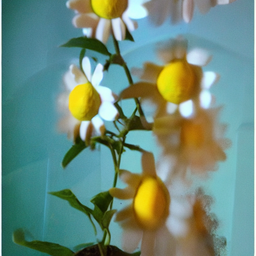} & \includegraphics[width=0.2\textwidth]{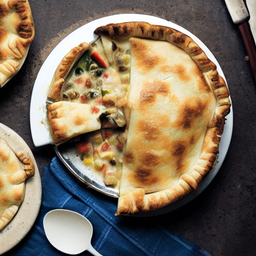} & \includegraphics[width=0.2\textwidth]{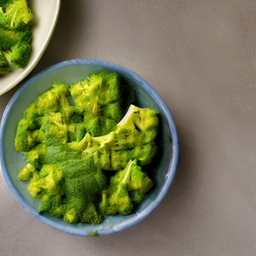} & \includegraphics[width=0.2\textwidth]{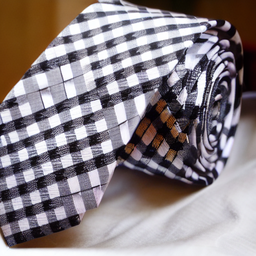} & \includegraphics[width=0.2\textwidth]{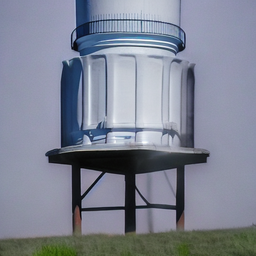} & \includegraphics[width=0.2\textwidth]{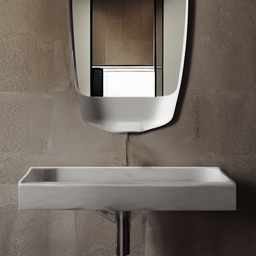}  \\

        \textcolor{Green}{\textbf{True:}} Daisy & \textcolor{Green}{\textbf{True:}} Potpie & \textcolor{Green}{\textbf{True:}} Broccoli & \textcolor{Green}{\textbf{True:}} Windsor Tie & \textcolor{Green}{\textbf{True:}} Water Tower & \textcolor{Green}{\textbf{True:}} Washbasin  \\
        \textcolor{red}{\textbf{Adv\textsuperscript{U}:}} Lemon & \textcolor{red}{\textbf{Adv\textsuperscript{U}:}} Pizza & \textcolor{red}{\textbf{Adv\textsuperscript{U}:}} Guacamole & \textcolor{red}{\textbf{Adv\textsuperscript{U}:}} Bow Tie & \textcolor{red}{\textbf{Adv\textsuperscript{U}:}} Water Jug & \textcolor{red}{\textbf{Adv\textsuperscript{U}:}} Soap Dispenser  \\
        \textbf{Res:} \textcolor{red}{Lemon $32 \%$} & \textbf{Res:} \textcolor{red}{Pizza $57 \%$} & \textbf{Res:} \textcolor{Green}{Broccoli $58 \%$} & \textbf{Res:} \textcolor{red}{Bow Tie $42 \%$} & \textbf{Res:} \textcolor{red}{Water Jug $8 \%$} & \textbf{Res:} \textcolor{red}{Soap Dispenser $75 \%$}  \\
        \textbf{Inc:} \textcolor{red}{Lemon $100 \%$} & \textbf{Inc:} \textcolor{red}{Pizza $100 \%$} & \textbf{Inc:} \textcolor{red}{Guacamole $100 \%$} & \textbf{Inc:} \textcolor{red}{Bow Tie $100 \%$} & \textbf{Inc:} \textcolor{red}{Water Jug $93 \%$} & \textbf{Inc:} \textcolor{red}{Soap Dispenser $100 \%$}  \\
        \textbf{ViT:} \textcolor{Green}{Daisy $59 \%$} & \textbf{ViT:} \textcolor{Green}{Potpie $100 \%$} & \textbf{ViT:} \textcolor{Green}{Broccoli $80 \%$} & \textbf{ViT:} \textcolor{Green}{Windsor Tie $99 \%$} & \textbf{ViT:} \textcolor{Green}{Water Tower $100 \%$} & \textbf{ViT:} \textcolor{Green}{Washbasin $97 \%$}  \\
        \textbf{AdvRes:} \textcolor{red}{Lemon $91 \%$} & \textbf{AdvRes:} \textcolor{red}{Pizza $86 \%$} & \textbf{AdvRes:} \textcolor{red}{Guacamole $88 \%$} & \textbf{AdvRes:} \textcolor{red}{Bow Tie $96 \%$} & \textbf{AdvRes:} \textcolor{red}{Water Jug $40 \%$} & \textbf{AdvRes:} \textcolor{red}{Soap Dispenser $100 \%$}  \\
        \textbf{AdvInc:} \textcolor{red}{Lemon $80 \%$} & \textbf{AdvInc:} \textcolor{Green}{Potpie $73 \%$} & \textbf{AdvInc:} \textcolor{red}{Guacamole $98 \%$} & \textbf{AdvInc:} \textcolor{red}{Bow Tie $92 \%$} & \textbf{AdvInc:} \textcolor{red}{Water Jug $19 \%$} & \textbf{AdvInc:} \textcolor{red}{Soap Dispenser $100 \%$}  \\

        \includegraphics[width=0.2\textwidth]{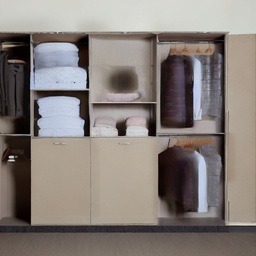} & \includegraphics[width=0.2\textwidth]{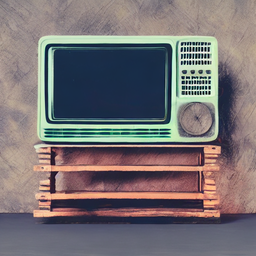} & \includegraphics[width=0.2\textwidth]{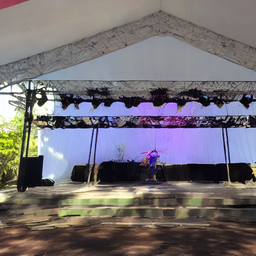} & \includegraphics[width=0.2\textwidth]{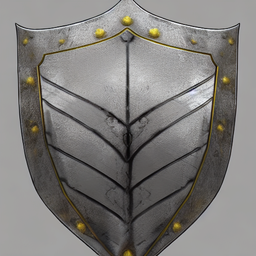} & \includegraphics[width=0.2\textwidth]{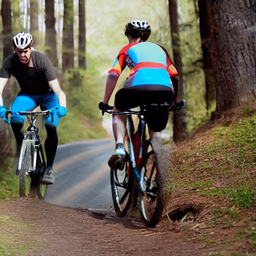} & \includegraphics[width=0.2\textwidth]{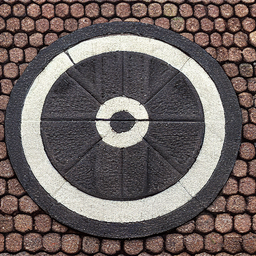}  \\

        \textcolor{Green}{\textbf{True:}} Wardrobe & \textcolor{Green}{\textbf{True:}} Television & \textcolor{Green}{\textbf{True:}} Stage & \textcolor{Green}{\textbf{True:}} Shield & \textcolor{Green}{\textbf{True:}} Mountain Bike & \textcolor{Green}{\textbf{True:}} Manhole Cover  \\
        \textcolor{red}{\textbf{Adv\textsuperscript{U}:}} Medicine Chest & \textcolor{red}{\textbf{Adv\textsuperscript{U}:}} Radio, Wireless & \textcolor{red}{\textbf{Adv\textsuperscript{U}:}} Swing & \textcolor{red}{\textbf{Adv\textsuperscript{U}:}} Breastplate & \textcolor{red}{\textbf{Adv\textsuperscript{U}:}} Tandem & \textcolor{red}{\textbf{Adv\textsuperscript{U}:}} Doormat  \\
        \textbf{Res:} \textcolor{red}{Medicine Chest $40 \%$} & \textbf{Res:} \textcolor{red}{Radio, Wireless $21 \%$} & \textbf{Res:} \textcolor{red}{Swing $59 \%$} & \textbf{Res:} \textcolor{Green}{Shield $29 \%$} & \textbf{Res:} \textcolor{Green}{Mountain Bike $19 \%$} & \textbf{Res:} \textcolor{BurntOrange}{Labyrinth $28 \%$}  \\
        \textbf{Inc:} \textcolor{red}{Medicine Chest $100 \%$} & \textbf{Inc:} \textcolor{red}{Radio, Wireless $100 \%$} & \textbf{Inc:} \textcolor{red}{Swing $100 \%$} & \textbf{Inc:} \textcolor{red}{Breastplate $100 \%$} & \textbf{Inc:} \textcolor{red}{Tandem $98 \%$} & \textbf{Inc:} \textcolor{red}{Doormat $100 \%$}  \\
        \textbf{ViT:} \textcolor{Green}{Wardrobe $86 \%$} & \textbf{ViT:} \textcolor{Green}{Television $84 \%$} & \textbf{ViT:} \textcolor{Green}{Stage $54 \%$} & \textbf{ViT:} \textcolor{red}{Breastplate $56 \%$} & \textbf{ViT:} \textcolor{red}{Tandem $49 \%$} & \textbf{ViT:} \textcolor{BurntOrange}{Sundial $33 \%$}  \\
        \textbf{AdvRes:} \textcolor{red}{Medicine Chest $89 \%$} & \textbf{AdvRes:} \textcolor{red}{Radio, Wireless $95 \%$} & \textbf{AdvRes:} \textcolor{red}{Swing $99 \%$} & \textbf{AdvRes:} \textcolor{Green}{Shield $32 \%$} & \textbf{AdvRes:} \textcolor{red}{Tandem $67 \%$} & \textbf{AdvRes:} \textcolor{red}{Doormat $84 \%$}  \\
        \textbf{AdvInc:} \textcolor{red}{Medicine Chest $87 \%$} & \textbf{AdvInc:} \textcolor{red}{Radio, Wireless $98 \%$} & \textbf{AdvInc:} \textcolor{Green}{Stage $18 \%$} & \textbf{AdvInc:} \textcolor{red}{Breastplate $80 \%$} & \textbf{AdvInc:} \textcolor{red}{Tandem $50 \%$} & \textbf{AdvInc:} \textcolor{Green}{Manhole Cover $31 \%$}  \\
    \end{tabular}}
    }
\caption{Adversarial samples generated by NatADiff with an Inception-v3 \citep{Inceptionv3_Szegedy2016} surrogate model. We report the true class, adversarial target, and classification scores of ResNet-50 \citep{He2015}, Inception-v3 \citep{Inceptionv3_Szegedy2016}, ViT-H \citep{Dosovitskiy2021}, and adversarially trained ResNet-50 and Inception victim models \citep{Kurakin2018}. Superscripts T and U indicate targeted and untargeted (similarity-based) attacks, respectively.}
\label{fig:NatADiff Inception-v3 extra samples}
\end{figure*}

\clearpage
\newpage
\subsection{ViT samples}
\begin{figure*}[h]
    \centering
    {
    \resizebox{\linewidth}{!}{%
    \tiny
    \begin{tabular}{@{}l@{}l@{}l@{}l@{}l@{}l@{}}
        \includegraphics[width=0.2\textwidth]{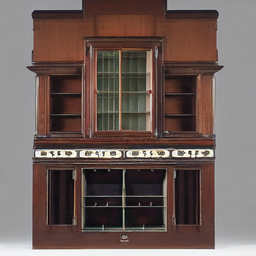} & \includegraphics[width=0.2\textwidth]{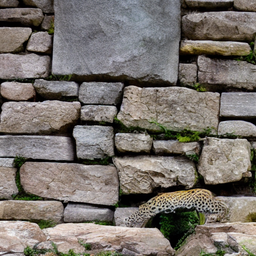} & \includegraphics[width=0.2\textwidth]{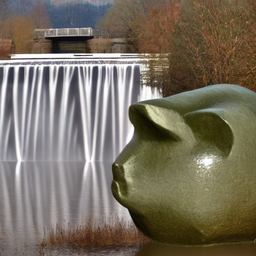} & \includegraphics[width=0.2\textwidth]{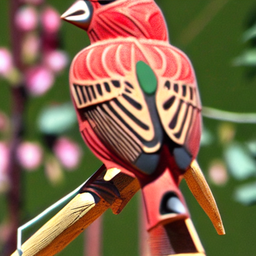} & \includegraphics[width=0.2\textwidth]{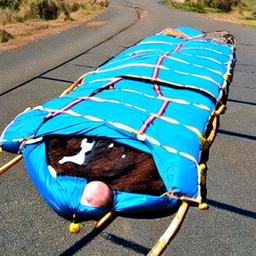} & \includegraphics[width=0.2\textwidth]{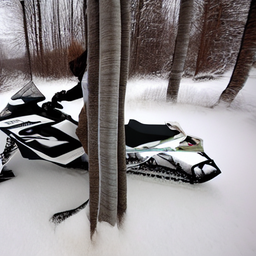}  \\

        \textcolor{Green}{\textbf{True:}} Chiffonier & \textcolor{Green}{\textbf{True:}} Stone Wall & \textcolor{Green}{\textbf{True:}} Dam & \textcolor{Green}{\textbf{True:}} House Finch & \textcolor{Green}{\textbf{True:}} Sleeping Bag & \textcolor{Green}{\textbf{True:}} Snowmobile  \\
        \textcolor{red}{\textbf{Adv\textsuperscript{T}:}} Tobacco Shop & \textcolor{red}{\textbf{Adv\textsuperscript{T}:}} Jaguar & \textcolor{red}{\textbf{Adv\textsuperscript{T}:}} Piggy Bank & \textcolor{red}{\textbf{Adv\textsuperscript{T}:}} Totem Pole & \textcolor{red}{\textbf{Adv\textsuperscript{T}:}} Oxcart & \textcolor{red}{\textbf{Adv\textsuperscript{T}:}} Shower Curtain  \\
        \textbf{Res:} \textcolor{red}{Tobacco Shop $15 \%$} & \textbf{Res:} \textcolor{Green}{Stone Wall $18 \%$} & \textbf{Res:} \textcolor{red}{Piggy Bank $18 \%$} & \textbf{Res:} \textcolor{red}{Totem Pole $4 \%$} & \textbf{Res:} \textcolor{BurntOrange}{Stretcher $29 \%$} & \textbf{Res:} \textcolor{BurntOrange}{Four-Poster $8 \%$}  \\
        \textbf{Inc:} \textcolor{BurntOrange}{Window Shade $6 \%$} & \textbf{Inc:} \textcolor{BurntOrange}{Leopard $46 \%$} & \textbf{Inc:} \textcolor{BurntOrange}{Fountain $68 \%$} & \textbf{Inc:} \textcolor{BurntOrange}{Maraca $14 \%$} & \textbf{Inc:} \textcolor{BurntOrange}{Park Bench $21 \%$} & \textbf{Inc:} \textcolor{BurntOrange}{Ski $54 \%$}  \\
        \textbf{ViT:} \textcolor{red}{Tobacco Shop $99 \%$} & \textbf{ViT:} \textcolor{red}{Jaguar $99 \%$} & \textbf{ViT:} \textcolor{red}{Piggy Bank $97 \%$} & \textbf{ViT:} \textcolor{red}{Totem Pole $99 \%$} & \textbf{ViT:} \textcolor{red}{Oxcart $98 \%$} & \textbf{ViT:} \textcolor{red}{Shower Curtain $89 \%$}  \\
        \textbf{AdvRes:} \textcolor{BurntOrange}{Library $47 \%$} & \textbf{AdvRes:} \textcolor{BurntOrange}{Leopard $69 \%$} & \textbf{AdvRes:} \textcolor{red}{Piggy Bank $77 \%$} & \textbf{AdvRes:} \textcolor{BurntOrange}{Toyshop $24 \%$} & \textbf{AdvRes:} \textcolor{BurntOrange}{Stretcher $37 \%$} & \textbf{AdvRes:} \textcolor{BurntOrange}{Dam $23 \%$}  \\
        \textbf{AdvInc:} \textcolor{BurntOrange}{Bookshop $20 \%$} & \textbf{AdvInc:} \textcolor{Green}{Stone Wall $54 \%$} & \textbf{AdvInc:} \textcolor{red}{Piggy Bank $84 \%$} & \textbf{AdvInc:} \textcolor{BurntOrange}{Jay $6 \%$} & \textbf{AdvInc:} \textcolor{Green}{Sleeping Bag $26 \%$} & \textbf{AdvInc:} \textcolor{Green}{Snowmobile $26 \%$}  \\

        \includegraphics[width=0.2\textwidth]{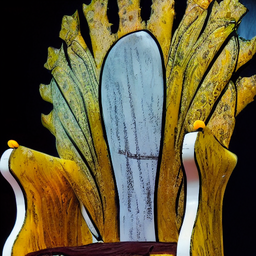} & \includegraphics[width=0.2\textwidth]{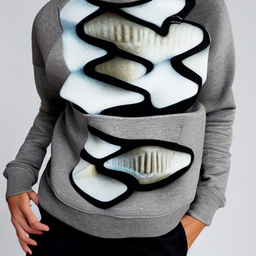} & \includegraphics[width=0.2\textwidth]{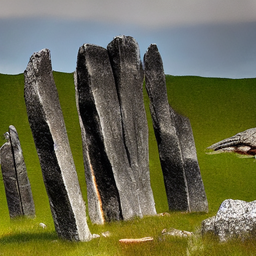} & \includegraphics[width=0.2\textwidth]{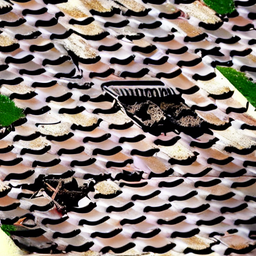} & \includegraphics[width=0.2\textwidth]{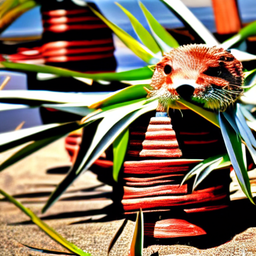} & \includegraphics[width=0.2\textwidth]{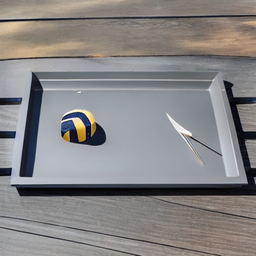}  \\

        \textcolor{Green}{\textbf{True:}} Throne & \textcolor{Green}{\textbf{True:}} Sweatshirt & \textcolor{Green}{\textbf{True:}} Megalith & \textcolor{Green}{\textbf{True:}} Tile Roof & \textcolor{Green}{\textbf{True:}} Otter & \textcolor{Green}{\textbf{True:}} Tray  \\
        \textcolor{red}{\textbf{Adv\textsuperscript{T}:}} Flatworm & \textcolor{red}{\textbf{Adv\textsuperscript{T}:}} Neck Brace & \textcolor{red}{\textbf{Adv\textsuperscript{T}:}} Quail & \textcolor{red}{\textbf{Adv\textsuperscript{T}:}} Agaric & \textcolor{red}{\textbf{Adv\textsuperscript{T}:}} Pineapple & \textcolor{red}{\textbf{Adv\textsuperscript{T}:}} Volleyball  \\
        \textbf{Res:} \textcolor{BurntOrange}{Boat Paddle $10 \%$} & \textbf{Res:} \textcolor{BurntOrange}{Ski Mask $32 \%$} & \textbf{Res:} \textcolor{Green}{Megalith $9 \%$} & \textbf{Res:} \textcolor{BurntOrange}{Apron $29 \%$} & \textbf{Res:} \textcolor{red}{Pineapple $2 \%$} & \textbf{Res:} \textcolor{Green}{Tray $17 \%$}  \\
        \textbf{Inc:} \textcolor{BurntOrange}{Clog $35 \%$} & \textbf{Inc:} \textcolor{BurntOrange}{Spatula $25 \%$} & \textbf{Inc:} \textcolor{BurntOrange}{Worm Fence $13 \%$} & \textbf{Inc:} \textcolor{BurntOrange}{Apron $37 \%$} & \textbf{Inc:} \textcolor{BurntOrange}{Pinwheel $12 \%$} & \textbf{Inc:} \textcolor{red}{Volleyball $85 \%$}  \\
        \textbf{ViT:} \textcolor{red}{Flatworm $78 \%$} & \textbf{ViT:} \textcolor{red}{Neck Brace $100 \%$} & \textbf{ViT:} \textcolor{red}{Quail $75 \%$} & \textbf{ViT:} \textcolor{red}{Agaric $60 \%$} & \textbf{ViT:} \textcolor{red}{Pineapple $100 \%$} & \textbf{ViT:} \textcolor{red}{Volleyball $85 \%$}  \\
        \textbf{AdvRes:} \textcolor{BurntOrange}{Wooden Spoon $27 \%$} & \textbf{AdvRes:} \textcolor{BurntOrange}{Ski Mask $38 \%$} & \textbf{AdvRes:} \textcolor{BurntOrange}{Chickadee $31 \%$} & \textbf{AdvRes:} \textcolor{Green}{Tile Roof $93 \%$} & \textbf{AdvRes:} \textcolor{Green}{Otter $57 \%$} & \textbf{AdvRes:} \textcolor{red}{Volleyball $96 \%$}  \\
        \textbf{AdvInc:} \textcolor{BurntOrange}{Wooden Spoon $71 \%$} & \textbf{AdvInc:} \textcolor{BurntOrange}{Wool $16 \%$} & \textbf{AdvInc:} \textcolor{BurntOrange}{Worm Fence $31 \%$} & \textbf{AdvInc:} \textcolor{Green}{Tile Roof $43 \%$} & \textbf{AdvInc:} \textcolor{Green}{Otter $6 \%$} & \textbf{AdvInc:} \textcolor{Green}{Tray $36 \%$}  \\

        \includegraphics[width=0.2\textwidth]{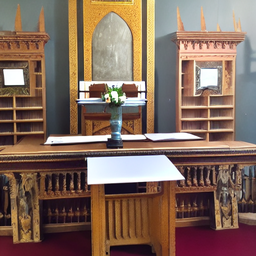} & \includegraphics[width=0.2\textwidth]{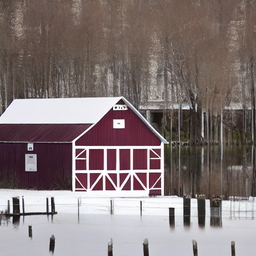} & \includegraphics[width=0.2\textwidth]{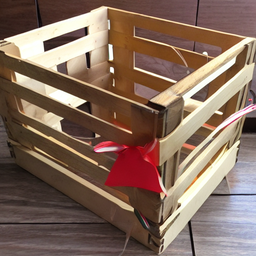} & \includegraphics[width=0.2\textwidth]{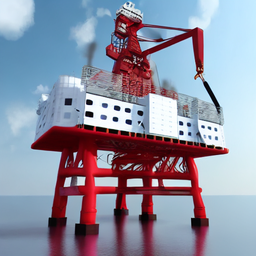} & \includegraphics[width=0.2\textwidth]{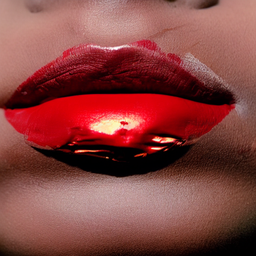} & \includegraphics[width=0.2\textwidth]{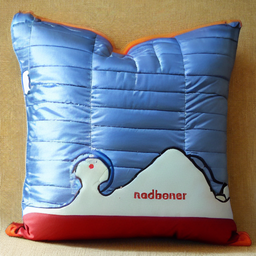}  \\

        \textcolor{Green}{\textbf{True:}} Altar & \textcolor{Green}{\textbf{True:}} Barn & \textcolor{Green}{\textbf{True:}} Crate & \textcolor{Green}{\textbf{True:}} Drilling Platform & \textcolor{Green}{\textbf{True:}} Lipstick & \textcolor{Green}{\textbf{True:}} Pillow  \\
        \textcolor{red}{\textbf{Adv\textsuperscript{U}:}} Desk & \textcolor{red}{\textbf{Adv\textsuperscript{U}:}} Boathouse & \textcolor{red}{\textbf{Adv\textsuperscript{U}:}} Hamper & \textcolor{red}{\textbf{Adv\textsuperscript{U}:}} Container Ship & \textcolor{red}{\textbf{Adv\textsuperscript{U}:}} Red Wine & \textcolor{red}{\textbf{Adv\textsuperscript{U}:}} Sleeping Bag  \\
        \textbf{Res:} \textcolor{red}{Desk $8 \%$} & \textbf{Res:} \textcolor{red}{Boathouse $40 \%$} & \textbf{Res:} \textcolor{Green}{Crate $48 \%$} & \textbf{Res:} \textcolor{BurntOrange}{Crane $20 \%$} & \textbf{Res:} \textcolor{Green}{Lipstick $69 \%$} & \textbf{Res:} \textcolor{red}{Sleeping Bag $23 \%$}  \\
        \textbf{Inc:} \textcolor{BurntOrange}{Throne $77 \%$} & \textbf{Inc:} \textcolor{Green}{Barn $82 \%$} & \textbf{Inc:} \textcolor{Green}{Crate $71 \%$} & \textbf{Inc:} \textcolor{Green}{Drilling Platform $30 \%$} & \textbf{Inc:} \textcolor{BurntOrange}{Mask $33 \%$} & \textbf{Inc:} \textcolor{red}{Sleeping Bag $30 \%$}  \\
        \textbf{ViT:} \textcolor{red}{Desk $98 \%$} & \textbf{ViT:} \textcolor{red}{Boathouse $100 \%$} & \textbf{ViT:} \textcolor{red}{Hamper $91 \%$} & \textbf{ViT:} \textcolor{red}{Container Ship $99 \%$} & \textbf{ViT:} \textcolor{red}{Red Wine $93 \%$} & \textbf{ViT:} \textcolor{red}{Sleeping Bag $100 \%$}  \\
        \textbf{AdvRes:} \textcolor{BurntOrange}{Throne $84 \%$} & \textbf{AdvRes:} \textcolor{Green}{Barn $71 \%$} & \textbf{AdvRes:} \textcolor{BurntOrange}{Cradle $35 \%$} & \textbf{AdvRes:} \textcolor{BurntOrange}{Crane $65 \%$} & \textbf{AdvRes:} \textcolor{BurntOrange}{Conch $25 \%$} & \textbf{AdvRes:} \textcolor{Green}{Pillow $46 \%$}  \\
        \textbf{AdvInc:} \textcolor{Green}{Altar $36 \%$} & \textbf{AdvInc:} \textcolor{Green}{Barn $83 \%$} & \textbf{AdvInc:} \textcolor{BurntOrange}{Cradle $7 \%$} & \textbf{AdvInc:} \textcolor{Green}{Drilling Platform $54 \%$} & \textbf{AdvInc:} \textcolor{red}{Red Wine $20 \%$} & \textbf{AdvInc:} \textcolor{red}{Sleeping Bag $32 \%$}  \\

        \includegraphics[width=0.2\textwidth]{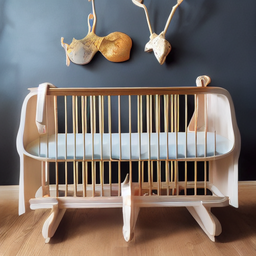} & \includegraphics[width=0.2\textwidth]{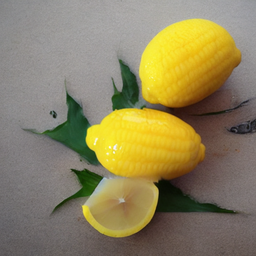} & \includegraphics[width=0.2\textwidth]{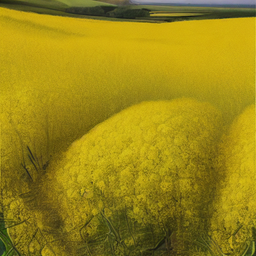} & \includegraphics[width=0.2\textwidth]{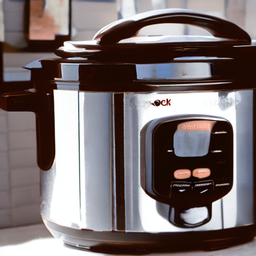} & \includegraphics[width=0.2\textwidth]{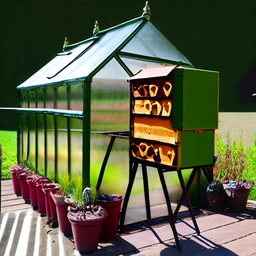} & \includegraphics[width=0.2\textwidth]{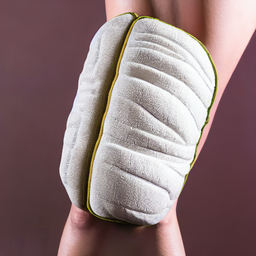}  \\

        \textcolor{Green}{\textbf{True:}} Crib & \textcolor{Green}{\textbf{True:}} Lemon & \textcolor{Green}{\textbf{True:}} Rapeseed & \textcolor{Green}{\textbf{True:}} Crock Pot & \textcolor{Green}{\textbf{True:}} Greenhouse & \textcolor{Green}{\textbf{True:}} Knee Pad  \\
        \textcolor{red}{\textbf{Adv\textsuperscript{U}:}} Cradle & \textcolor{red}{\textbf{Adv\textsuperscript{U}:}} Corn & \textcolor{red}{\textbf{Adv\textsuperscript{U}:}} Lemon & \textcolor{red}{\textbf{Adv\textsuperscript{U}:}} Coffeepot & \textcolor{red}{\textbf{Adv\textsuperscript{U}:}} Apiary & \textcolor{red}{\textbf{Adv\textsuperscript{U}:}} Pillow  \\
        \textbf{Res:} \textcolor{BurntOrange}{Plate Rack $24 \%$} & \textbf{Res:} \textcolor{Green}{Lemon $43 \%$} & \textbf{Res:} \textcolor{Green}{Rapeseed $43 \%$} & \textbf{Res:} \textcolor{red}{Coffeepot $21 \%$} & \textbf{Res:} \textcolor{BurntOrange}{Mailbox $10 \%$} & \textbf{Res:} \textcolor{Green}{Knee Pad $19 \%$}  \\
        \textbf{Inc:} \textcolor{BurntOrange}{Plate Rack $31 \%$} & \textbf{Inc:} \textcolor{Green}{Lemon $27 \%$} & \textbf{Inc:} \textcolor{red}{Lemon $25 \%$} & \textbf{Inc:} \textcolor{Green}{Crock Pot $55 \%$} & \textbf{Inc:} \textcolor{BurntOrange}{Guillotine $13 \%$} & \textbf{Inc:} \textcolor{BurntOrange}{Bath Towel $39 \%$}  \\
        \textbf{ViT:} \textcolor{red}{Cradle $99 \%$} & \textbf{ViT:} \textcolor{red}{Corn $100 \%$} & \textbf{ViT:} \textcolor{red}{Lemon $97 \%$} & \textbf{ViT:} \textcolor{red}{Coffeepot $62 \%$} & \textbf{ViT:} \textcolor{red}{Apiary $100 \%$} & \textbf{ViT:} \textcolor{red}{Pillow $98 \%$}  \\
        \textbf{AdvRes:} \textcolor{Green}{Crib $37 \%$} & \textbf{AdvRes:} \textcolor{Green}{Lemon $84 \%$} & \textbf{AdvRes:} \textcolor{Green}{Rapeseed $99 \%$} & \textbf{AdvRes:} \textcolor{Green}{Crock Pot $93 \%$} & \textbf{AdvRes:} \textcolor{Green}{Greenhouse $38 \%$} & \textbf{AdvRes:} \textcolor{BurntOrange}{Wool $57 \%$}  \\
        \textbf{AdvInc:} \textcolor{BurntOrange}{Plate Rack $46 \%$} & \textbf{AdvInc:} \textcolor{red}{Corn $46 \%$} & \textbf{AdvInc:} \textcolor{BurntOrange}{Bath Towel $4 \%$} & \textbf{AdvInc:} \textcolor{Green}{Crock Pot $84 \%$} & \textbf{AdvInc:} \textcolor{Green}{Greenhouse $27 \%$} & \textbf{AdvInc:} \textcolor{BurntOrange}{Bath Towel $78 \%$}  \\
    \end{tabular}}
    }
\caption{Adversarial samples generated by NatADiff with a ViT-H \citep{Dosovitskiy2021} surrogate model. We report the true class, adversarial target, and classification scores of ResNet-50 \citep{He2015}, Inception-v3 \citep{Inceptionv3_Szegedy2016}, ViT-H \citep{Dosovitskiy2021}, and adversarially trained ResNet-50 and Inception victim models \citep{Kurakin2018}. Superscripts T and U indicate targeted and untargeted (similarity-based) attacks, respectively.}
\label{fig:NatADiff ViT extra samples}
\end{figure*}

\clearpage
\newpage
\subsection{Low-quality targeted ViT samples} \label{apdx:Low-Quality Targeted ViT Samples}
Targeted ViT-H samples show degraded image quality, as seen in Figure~\ref{fig:Low-Quality Targeted ViT Samples} and supported by the IS and FID-Val scores in Table~\ref{tab:Classifier ASR and image quality}. Targeted attacks typically blend features from disparate classes, which places greater demands on the diffusion model to locate a feasible point on the image manifold. ViT-H is a strong classifier, and its decision boundaries between unrelated classes appear to be more accurate than those of the other classifiers we examined. As a result, the diffusion model struggles to generate feasible targeted adversarial samples for ViT-H, and artifacts are introduced. These artifacts substantially degrade image quality and artificially inflate the attack success rate of the targeted ViT-H samples (as seen in Table \ref{tab:Classifier ASR and image quality}). Importantly, these artifacts occur only in the ViT-H targeted setting, which is directly observable from the image-quality metrics in Table~\ref{tab:Classifier ASR and image quality}.

\begin{figure*}[h]
    \centering
    {
    \resizebox{\linewidth}{!}{%
    \tiny
    \begin{tabular}{@{}l@{}l@{}l@{}l@{}l@{}l@{}}
        \includegraphics[width=0.2\textwidth]{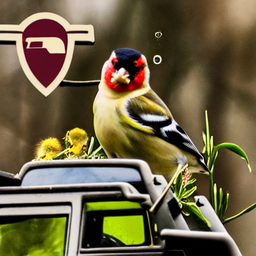} & \includegraphics[width=0.2\textwidth]{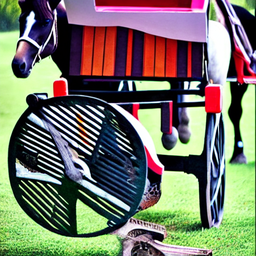} & \includegraphics[width=0.2\textwidth]{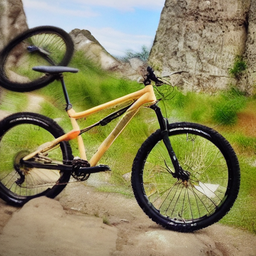} & \includegraphics[width=0.2\textwidth]{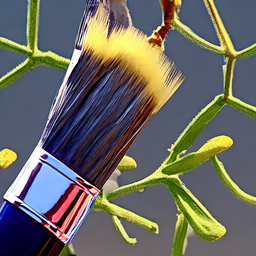} & \includegraphics[width=0.2\textwidth]{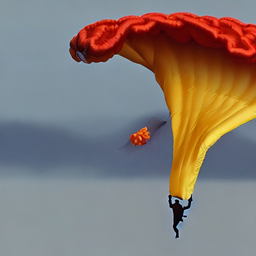} & \includegraphics[width=0.2\textwidth]{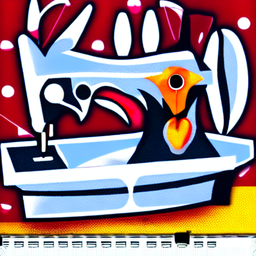}  \\

        \textcolor{Green}{\textbf{True:}} Goldfinch & \textcolor{Green}{\textbf{True:}} Horse Cart & \textcolor{Green}{\textbf{True:}} Mountain Bike & \textcolor{Green}{\textbf{True:}} Paintbrush & \textcolor{Green}{\textbf{True:}} Parachute & \textcolor{Green}{\textbf{True:}} Sewing Machine  \\
        \textcolor{red}{\textbf{Adv\textsuperscript{T}:}} Jeep & \textcolor{red}{\textbf{Adv\textsuperscript{T}:}} Dowitcher (Bird) & \textcolor{red}{\textbf{Adv\textsuperscript{T}:}} Matchstick & \textcolor{red}{\textbf{Adv\textsuperscript{T}:}} Apiary & \textcolor{red}{\textbf{Adv\textsuperscript{T}:}} Coral Fungus & \textcolor{red}{\textbf{Adv\textsuperscript{T}:}} Rooster  \\

        \includegraphics[width=0.2\textwidth]{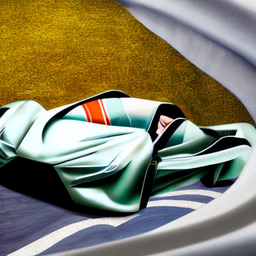} & \includegraphics[width=0.2\textwidth]{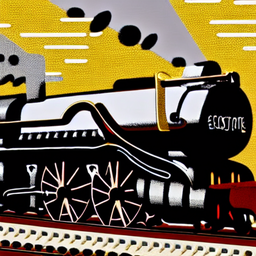} & \includegraphics[width=0.2\textwidth]{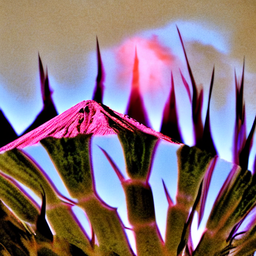} & \includegraphics[width=0.2\textwidth]{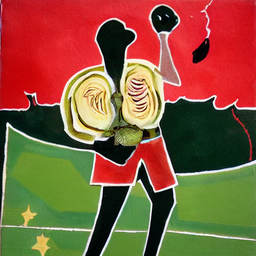} & \includegraphics[width=0.2\textwidth]{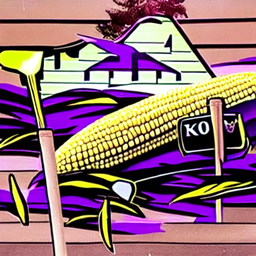} & \includegraphics[width=0.2\textwidth]{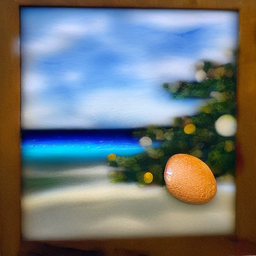}  \\

        \textcolor{Green}{\textbf{True:}} Sports Car & \textcolor{Green}{\textbf{True:}} Steam Locomotive & \textcolor{Green}{\textbf{True:}} Volcano & \textcolor{Green}{\textbf{True:}} Baseball Player & \textcolor{Green}{\textbf{True:}} Corn & \textcolor{Green}{\textbf{True:}} Seashore \\
        \textcolor{red}{\textbf{Adv\textsuperscript{T}:}} Kimono & \textcolor{red}{\textbf{Adv\textsuperscript{T}:}} Typewriter & \textcolor{red}{\textbf{Adv\textsuperscript{T}:}} Cardoon (Plant) & \textcolor{red}{\textbf{Adv\textsuperscript{T}:}} Artichoke & \textcolor{red}{\textbf{Adv\textsuperscript{T}:}} Llama & \textcolor{red}{\textbf{Adv\textsuperscript{T}:}} Eggnog  \\
    \end{tabular}}
    }
\caption{Low-quality adversarial samples generated by NatADiff with a ViT-H \citep{Dosovitskiy2021} surrogate model under targeted attack settings. We report the true class and adversarial target for each image. Superscripts T and U indicate targeted and untargeted (similarity-based) attacks, respectively.}
\label{fig:Low-Quality Targeted ViT Samples}
\end{figure*}

\section{Additional image quality metrics}
We use NIQE \citep{NIQE_Mittal2013}, BRISQUE \citep{BRISQUE_Mittal2012}, and TReS \citep{TReS_Alireza2022} to provide additional no-reference image quality evaluations of adversarial sampling methods. We assess the image quality of NCF \citep{NCF_Yuan2022}, DiffAttack \citep{DiffAttack_Chen2025}, ACA \citep{ContentDiffusionAttack_Chen2023}, adversarial classifier guidance (AdvClass) \citep{Dai2024}, and NatADiff across ResNet-50 \citep{He2015}, Inception-v3 \citep{Inceptionv3_Szegedy2016}, and ViT-H \citep{Dosovitskiy2021} surrogate models. These metrics more closely align with human perception of image quality, but they do not address adherence to a target data distribution, i.e., how well samples “fit into” the ImageNet dataset.

Adversarial classifier guidance and DiffAttack frequently outperformed other methods on NIQE, BRISQUE, and TReS (see Table~\ref{tab:Full Image Quality}); however, as discussed in Section~\ref{sec:Results}, this is likely because these methods apply constrained perturbations to source images and clean stable diffusion outputs, respectively. When considering only methods that make structural image alterations, we see that NatADiff outperforms NCF across all image metrics, and outperforms ACA on NIQE and BRISQUE, with similar TReS scores that slightly favour ACA. This supports the findings from the main paper that NatADiff is able to construct visually high-quality adversarial samples.

\begin{table}[h]
    \centering
    \caption{\textbf{Image quality} of adversarial samples generated using ACA \citep{ContentDiffusionAttack_Chen2023}, DiffAttack \citep{DiffAttack_Chen2025}, adversarial classifier guidance \citep{Dai2024}, and NatADiff. \textcolor{red}{\textbf{Bold}} and \textcolor{BrickRed}{\underline{underlined}} values highlight the best and second best scores for each surrogate model. Superscripts T and U denote targeted and untargeted attacks, respectively. Note that we report FID with respect to ImageNet-Val (FID-Val) and ImageNet-A (FID-A).}
    \label{tab:Full Image Quality}
    \resizebox{1\textwidth}{!}{
        \begin{tabular}{@{}cc@{\hspace{1em}}c@{}cccccc@{}}
            \specialrule{1.2pt}{0pt}{0pt}
            \noalign{\vskip 0.5ex}
            \multicolumn{1}{@{}c}{\multirow{2}{*}{\shortstack{Surrogate\\Model}}} & \multicolumn{1}{c@{\hspace{1em}}}{\multirow{2}{*}{Attack}} && \multicolumn{6}{c}{Image Quality Metrics} \\
            \noalign{\vskip 0.5ex}
            \cline{3-9}
            \noalign{\vskip 0.5ex}
            &&& IS ($\boldsymbol{\uparrow}$) & FID-Val ($\boldsymbol{\downarrow}$) & FID-A ($\boldsymbol{\downarrow}$) & NIQE ($\boldsymbol{\downarrow}$) & BRISQUE ($\boldsymbol{\downarrow}$) & TReS ($\boldsymbol{\uparrow}$) \\
            \specialrule{1.2pt}{0pt}{0pt}
            \noalign{\vskip 0.5ex}
            & Clean & & $55.0$ & $58.0$ & $94.7$ & $4.6$ & $18.1$ & $88.6$ \\
            \specialrule{1.2pt}{0pt}{0pt}
            \noalign{\vskip 0.5ex}
            \multicolumn{1}{@{}c}{\multirow{7}{*}{RN-50}} & NCF & & $30.4$ & $69.7$ & $85.5$ & $5.5$ & $19.8$ & $68.9$ \\
            & DiffAttack & & $26.8$ & $64.1$ & $\textcolor{red}{\mathbf{76.8}}$ & $5.7$ & $17.7$ & $\textcolor{red}{\mathbf{81.8}}$ \\
            & ACA & & $23.9$ & $65.0$ & $77.9$ & $6.8$ & $24.4$ & $80.8$ \\
            & AdvClass\textsuperscript{T} & & $38.3$ & $\textcolor{red}{\mathbf{48.9}}$ & $92.4$ & $\textcolor{red}{\mathbf{4.3}}$ & $\textcolor{red}{\mathbf{11.9}}$ & $\textcolor{red}{\mathbf{81.8}}$ \\
            & AdvClass\textsuperscript{U} & & $\textcolor{BrickRed}{\underline{38.5}}$ & $\textcolor{BrickRed}{\underline{50.2}}$ & $92.7$ & $\textcolor{BrickRed}{\underline{4.6}}$ & $12.2$ & $\textcolor{BrickRed}{\underline{81.3}}$ \\
            \cline{2-9}
            \noalign{\vskip 0.5ex}
            & NatADiff\textsuperscript{T} & & $26.0$ & $66.5$ & $\textcolor{BrickRed}{\underline{77.3}}$ & $4.8$ & $12.3$ & $76.0$ \\
            & NatADiff\textsuperscript{U} & & $\textcolor{red}{\mathbf{43.2}}$ & $51.4$ & $95.9$ & $4.8$ & $\textcolor{BrickRed}{\underline{12.0}}$ & $77.9$ \\
            \specialrule{1.2pt}{0pt}{0pt}
            \noalign{\vskip 0.2ex}
            \multicolumn{1}{@{}c}{\multirow{7}{*}{Inc-v3}} & NCF & & $31.7$ & $69.1$ & $83.0$ & $\textcolor{BrickRed}{\underline{4.7}}$ & $19.0$ & $76.3$ \\
            & DiffAttack & & $33.2$ & $63.7$ & $\textcolor{red}{\mathbf{78.2}}$ & $5.8$ & $18.0$ & $\textcolor{BrickRed}{\underline{81.2}}$ \\
            & ACA & & $23.1$ & $68.0$ & $\textcolor{BrickRed}{\underline{78.8}}$ & $7.8$ & $28.4$ & $78.7$ \\
            & AdvClass\textsuperscript{T} & & $33.7$ & $51.0$ & $89.2$ & $\textcolor{red}{\mathbf{4.5}}$ & $\textcolor{BrickRed}{\underline{12.3}}$ & $\textcolor{red}{\mathbf{81.4}}$ \\
            & AdvClass\textsuperscript{U} & & $\textcolor{BrickRed}{\underline{39.7}}$ & $\textcolor{red}{\mathbf{49.4}}$ & $93.3$ & $\textcolor{red}{\mathbf{4.5}}$ & $\textcolor{BrickRed}{\underline{12.3}}$ & $\textcolor{BrickRed}{\underline{81.2}}$ \\
            \cline{2-9}
            \noalign{\vskip 0.5ex}
            & NatADiff\textsuperscript{T} & & $27.7$ & $66.6$ & $\textcolor{red}{\mathbf{78.2}}$ & $4.8$ & $12.4$ & $76.6$ \\
            & NatADiff\textsuperscript{U} & & $\textcolor{red}{\mathbf{47.0}}$ & $\textcolor{BrickRed}{\underline{50.5}}$ & $98.9$ & $\textcolor{BrickRed}{\underline{4.7}}$ & $\textcolor{red}{\mathbf{11.7}}$ & $78.9$ \\
            \specialrule{1.2pt}{0pt}{0pt}
            \noalign{\vskip 0.2ex}
            \multicolumn{1}{@{}c}{\multirow{7}{*}{ViT-H}} & NCF & & $\textcolor{red}{\mathbf{39.8}}$ & $63.1$ & $86.4$ & $5.6$ & $20.0$ & $70.1$ \\
            & DiffAttack & & $35.2$ & $63.4$ & $\textcolor{red}{\mathbf{80.0}}$ & $6.0$ & $18.4$ & $\textcolor{red}{\mathbf{81.3}}$ \\
            & ACA & & $25.5$ & $64.2$ & $\textcolor{BrickRed}{\underline{80.9}}$ & $7.4$ & $25.5$ & $79.2$ \\
            & AdvClass\textsuperscript{T} & & $38.9$ & $\textcolor{red}{\mathbf{48.5}}$ & $95.2$ & $\textcolor{red}{\mathbf{4.2}}$ & $\textcolor{BrickRed}{\underline{14.8}}$ & $77.7$ \\
            & AdvClass\textsuperscript{U} & & $\textcolor{BrickRed}{\underline{39.2}}$ & $\textcolor{red}{\mathbf{48.5}}$ & $98.8$ & $\textcolor{BrickRed}{\underline{4.3}}$ & $\textcolor{red}{\mathbf{13.5}}$ & $\textcolor{BrickRed}{\underline{79.8}}$ \\
            \cline{2-9}
            \noalign{\vskip 0.5ex}
            & NatADiff\textsuperscript{T} & & $15.3$ & $88.0$ & $93.5$ & $4.9$ & $20.9$ & $74.7$ \\
            & NatADiff\textsuperscript{U} & & $31.9$ & $\textcolor{BrickRed}{\underline{53.9}}$ & $96.2$ & $4.6$ & $\textcolor{BrickRed}{\underline{14.8}}$ & $78.8$ \\
            \specialrule{1.2pt}{0pt}{0pt}
            \noalign{\vskip 0.2ex}
        \end{tabular}
    }
\end{table}

\section{Resistance to adversarial defences} \label{apdx:Resistance to adversarial defences}
It has previously been shown that perturbation-based adversarial attacks are sensitive to image transformations--such as rotations, crops, and translations--which can substantially reduce attack success rates \citep{Guo2018}. In addition to such transformation-based defences, purification approaches aim to remove adversarial noise prior to classification. One such method is \textit{DiffPure} \citep{Nie2022}, which leverages a denoising diffusion model to project adversarial samples back onto the natural image manifold. Given an adversarial image, $\tilde{\boldsymbol{x}}_0$, the forward diffusion process, $p(\boldsymbol{x}_t | \boldsymbol{x}_0 = \tilde{\boldsymbol{x}}_0)$, is applied for $t \ll T$ (recall that $T$ is the termination time of the forward process), introducing Gaussian noise without fully destroying the original signal. The sample is then passed through the reverse-time diffusion process \citep{Anderson1982} or flow ODE \citep{Song2021} to recover a purified image. \cite{Nie2022} empirically demonstrated and theoretically proved that this procedure effectively removes perturbation-based adversarial noise, enabling the reverse process to reconstruct a ``clean'' version of the image. Intuitively, the noise injected during forward diffusion overwhelms the adversarial signal, allowing the diffusion model to project the corrupted sample back onto the natural image manifold.

We evaluate the robustness of standard image transformations and DiffPure against adversarial samples generated by PGD \citep{Madry2019}, AutoAttack \citep{Croce2020}, NCF \citep{NCF_Yuan2022}, DiffAttack \citep{DiffAttack_Chen2025}, ACA \citep{ContentDiffusionAttack_Chen2023}, adversarial classifier guidance (AdvClass) \citep{Dai2024}, and NatADiff. We use a ResNet-50 \citep{He2015} surrogate model and test NatADiff and adversarial classifier under both targeted and untargeted modes. Defences are applied as a pre-processing step to classification; for image transformations, we average classification probabilities over augmented views obtained via cropping, rotation, and grayscale conversion. To quantify defence effectiveness, we report attack success rate (ASR).

Our results show that transform-purification did not meaningfully reduce the efficacy of NatADiff, though it successfully defended against PGD attacks (see Table~\ref{tab:Transform and diffpure attack success rate}). In contrast, DiffPure provided a much stronger defence to most attacks, reducing NatADiff's average ASR by $7.9\%$. However, DiffPure occasionally degraded overall classifier accuracy, leading to increased ASR for non-surrogate classifiers. This likely occurred when the reverse diffusion process failed to recover the original image, rendering classification unreliable. Compared to other attacks, NatADiff still exhibited superior white-box performance and achieved either best or second-best transferability across all victim classifiers. Interestingly, NCF showed a significant increase in transferability under the DiffPure defence. We hypothesize that this stems from NCF’s color-based attacks pushing samples into low-probability regions of the manifold during the forward diffusion process. This increases the likelihood that DiffPure fails to recover the original image in the reverse process, thereby degrading classifier accuracy.

Consistent with findings in the main paper, NatADiff achieved best or near-best performance under both transformation and DiffPure defences. Given that natural adversarial samples are known to bypass perturbation-based defences and image transformations \citep{DetectingNaturalvsArtificialAdvSamples_Agarwal2022}, these results further support our claim that NatADiff generates adversarial examples that are more semantically aligned with naturally occurring test-time errors.

\begin{table}[H]
    \centering
    \caption{\textbf{Attack success rate} of \textbf{transform} and \textbf{DiffPure-purified} adversarial samples generated by PGD \citep{Madry2019}, AutoAttack \citep{Croce2020}, NCF \citep{NCF_Yuan2022}, DiffAttack \citep{DiffAttack_Chen2025}, ACA \citep{ContentDiffusionAttack_Chen2023}, adversarial classifier guidance (AdvClass) \citep{Dai2024}, and NatADiff. Samples are generated using a ResNet-50 \citep{He2015} surrogate model. \textcolor{red}{\textbf{Bold}} and \textcolor{BrickRed}{\underline{underlined}} values highlight the best and second best scores for each purification method. Superscripts T and U denote targeted and untargeted attacks, respectively. White-box ASR (same surrogate and victim model) is denoted with an $^*$.}
    \label{tab:Transform and diffpure attack success rate}
\resizebox{1\textwidth}{!}{
        \begin{tabular}{@{}cc@{}c@{\hspace{1em}}ccccc@{}c@{\hspace{1em}}cccc@{}cc@{}}
            \specialrule{1.2pt}{0pt}{0pt}
            \noalign{\vskip 0.5ex}
            \multicolumn{1}{@{}c}{\multirow{3}{*}{\shortstack{Purification\\Method}}} & \multicolumn{1}{c@{}}{\multirow{3}{*}{Attack}} && \multicolumn{10}{c}{Victim Model ASR (\%)} && \multicolumn{1}{c}{\multirow{3}{*}{\shortstack{Average\\ASR}}} \\
            &&& \multicolumn{5}{c}{CNNs} && \multicolumn{4}{c}{Transformers} && \\
            \cline{4-8} \cline{10-13} 
            \noalign{\vskip 0.5ex}
            &&& RN-50 & Inc-v3 & RN-152 & AdvRes & AdvInc && ViT-H & Max-ViT & Swin-B & DeIT && \\
            \specialrule{1.2pt}{0pt}{0pt}
            \noalign{\vskip 0.2ex}
            & Clean & & $5.3$ & $7.6$ & $2.9$ & $3.0$ & $5.8$ && $10.9$ & $3.8$ & $4.5$ & $7.4$ && $5.7$ \\
            \specialrule{1.2pt}{0pt}{0pt}
            \noalign{\vskip 0.2ex}
            \multicolumn{1}{@{}c}{\multirow{9}{*}{None}} & PGD & & $99.4^*$ & $11.8$ & $5.2$ & $4.9$ & $8.1$ && $10.5$ & $4.4$ & $5.5$ & $8.2$ && $17.6$ \\
            & AA & & $\textcolor{red}{\mathbf{100^*}}$ & $13.3$ & $10.0$ & $3.9$ & $8.8$ && $10.5$ & $5.4$ & $5.6$ & $8.0$ && $18.4$ \\
            & NCF & & $74.8^*$ & $33.4$ & $37.3$ & $28.2$ & $31.2$ && $17.2$ & $24.0$ & $31.7$ & $37.2$ && $35.0$ \\
            & DiffAttack & & $92.5^*$ & $47.1$ & $52.5$ & $35.3$ & $43.3$ && $28.4$ & $44.6$ & $42.4$ & $38.9$ && $47.2$ \\
            & ACA & & $78.8^*$ & $53.3$ & $52.7$ & $49.8$ & $53.1$ && $\textcolor{BrickRed}{\underline{41.8}}$ & $\textcolor{BrickRed}{\underline{46.4}}$ & $\textcolor{BrickRed}{\underline{49.3}}$ & $50.6$ && $52.9$ \\
            & AdvClass\textsuperscript{T} & & $99.6^*$ & $35.0$ & $32.1$ & $31.4$ & $33.5$ && $25.8$ & $30.0$ & $30.8$ & $32.8$ && $39.0$ \\
            & AdvClass\textsuperscript{U} & & $\textcolor{BrickRed}{\underline{99.9^*}}$ & $42.5$ & $44.3$ & $38.7$ & $41.1$ && $29.7$ & $37.6$ & $38.4$ & $39.1$ && $45.7$ \\
            \cline{2-15}
            \noalign{\vskip 0.5ex}
            & NatADiff\textsuperscript{T} & & $96.9^*$ & $\textcolor{BrickRed}{\underline{60.1}}$ & $\textcolor{BrickRed}{\underline{56.5}}$ & $\textcolor{BrickRed}{\underline{55.3}}$ & $\textcolor{BrickRed}{\underline{58.9}}$ && $36.8$ & $45.3$ & $49.0$ & $\textcolor{BrickRed}{\underline{52.3}}$ && $\textcolor{BrickRed}{\underline{56.8}}$ \\
            & NatADiff\textsuperscript{U} & & $99.3^*$ & $\textcolor{red}{\mathbf{68.3}}$ & $\textcolor{red}{\mathbf{72.1}}$ & $\textcolor{red}{\mathbf{65.3}}$ & $\textcolor{red}{\mathbf{66.8}}$ && $\textcolor{red}{\mathbf{45.3}}$ & $\textcolor{red}{\mathbf{64.1}}$ & $\textcolor{red}{\mathbf{65.2}}$ & $\textcolor{red}{\mathbf{67.0}}$ && $\textcolor{red}{\mathbf{68.2}}$ \\
            \specialrule{1.2pt}{0pt}{0pt}
            \noalign{\vskip 0.2ex}
            \multicolumn{1}{@{}c}{\multirow{9}{*}{Transform}} & PGD & & $14.5^*$ & $12.1$ & $4.8$ & $4.7$ & $7.4$ && $9.2$ & $3.5$ & $5.2$ & $7.3$ && $7.6$ \\
            & AA & & $79.4^*$ & $12.4$ & $7.9$ & $3.3$ & $9.1$ && $10.2$ & $3.4$ & $5.8$ & $7.3$ && $15.4$ \\
            & NCF & & $60.2^*$ & $35.1$ & $39.6$ & $28.4$ & $31.2$ && $16.7$ & $27.1$ & $33.7$ & $37.7$ && $34.4$ \\
            & DiffAttack & & $73.9^*$ & $48.6$ & $50.1$ & $39.9$ & $45.8$ && $28.6$ & $43.4$ & $46.7$ & $39.0$ && $46.2$ \\
            & ACA &  & $64.2^*$ & $54.8$ & $52.2$ & $50.4$ & $56.4$ && $\textcolor{BrickRed}{\underline{40.8}}$ & $\textcolor{BrickRed}{\underline{47.1}}$ & $\textcolor{BrickRed}{\underline{51.9}}$ & $51.0$ && $52.1$ \\
            & AdvClass\textsuperscript{T} & & $35.2^*$ & $33.4$ & $30.4$ & $29.7$ & $31.5$ && $25.5$ & $28.9$ & $31.4$ & $32.1$ && $30.9$ \\
            & AdvClass\textsuperscript{U} & & $65.8^*$ & $39.8$ & $40.8$ & $38.3$ & $40.0$ && $29.1$ & $36.3$ & $37.0$ & $38.4$ && $40.6$ \\
            \cline{2-15}
            \noalign{\vskip 0.5ex}
            & NatADiff\textsuperscript{T} & & $\textcolor{BrickRed}{\underline{85.7^*}}$ & $\textcolor{BrickRed}{\underline{59.8}}$ & $\textcolor{BrickRed}{\underline{55.6}}$ & $\textcolor{BrickRed}{\underline{55.0}}$ & $\textcolor{BrickRed}{\underline{56.8}}$ && $36.1$ & $46.7$ & $48.6$ & $\textcolor{BrickRed}{\underline{52.4}}$ && $\textcolor{BrickRed}{\underline{55.2}}$ \\
            & NatADiff\textsuperscript{U} & & $\textcolor{red}{\mathbf{96.6^*}}$ & $\textcolor{red}{\mathbf{68.7}}$ & $\textcolor{red}{\mathbf{73.3}}$ & $\textcolor{red}{\mathbf{67.3}}$ & $\textcolor{red}{\mathbf{68.3}}$ && $\textcolor{red}{\mathbf{45.0}}$ & $\textcolor{red}{\mathbf{65.4}}$ & $\textcolor{red}{\mathbf{66.8}}$ & $\textcolor{red}{\mathbf{70.3}}$ && $\textcolor{red}{\mathbf{69.1}}$ \\
            \specialrule{1.2pt}{0pt}{0pt}
            \noalign{\vskip 0.2ex}
            \multicolumn{1}{@{}c}{\multirow{9}{*}{DiffPure}} & PGD & & $21.9^*$ & $30.8$ & $20.3$ & $23.3$ & $26.5$ && $21.3$ & $16.7$ & $19.4$ & $20.6$ && $22.3$ \\
            & AA & & $23.5^*$ & $32.5$ & $19.8$ & $22.7$ & $28.4$ && $21.7$ & $18.9$ & $21.2$ & $23.6$ && $23.6$ \\
            & NCF & & $\textcolor{BrickRed}{\underline{68.8^*}}$ & $59.9$ & $\textcolor{BrickRed}{\underline{60.4}}$ & $53.9$ & $55.3$ && $46.1$ & $\textcolor{BrickRed}{\underline{57.5}}$ & $\textcolor{red}{\mathbf{62.3}}$ & $\textcolor{BrickRed}{\underline{59.7}}$ && $\textcolor{BrickRed}{\underline{58.2}}$ \\
            & DiffAttack & & $45.9^*$ & $44.0$ & $39.3$ & $37.3$ & $43.5$ && $38.4$ & $37.4$ & $41.8$ & $38.7$ && $40.7$ \\
            & ACA & & $60.8^*$ & $\textcolor{red}{\mathbf{63.1}}$ & $55.3$ & $\textcolor{BrickRed}{\underline{57.3}}$ & $\textcolor{BrickRed}{\underline{60.1}}$ && $\textcolor{red}{\mathbf{50.4}}$ & $55.1$ & $56.7$ & $56.5$ && $57.3$ \\
            & AdvClass\textsuperscript{T} & & $35.0^*$ & $37.8$ & $34.2$ & $35.2$ & $37.2$ && $30.3$ & $34.3$ & $34.9$ & $36.7$ && $35.1$ \\
            & AdvClass\textsuperscript{U} & & $42.4^*$ & $42.6$ & $39.8$ & $39.8$ & $41.9$ && $34.3$ & $38.8$ & $40.3$ & $41.3$ && $40.1$ \\
            \cline{2-15}
            \noalign{\vskip 0.5ex}
            & NatADiff\textsuperscript{T} & & $56.3^*$ & $56.6$ & $52.2$ & $53.1$ & $54.8$ && $42.3$ & $49.1$ & $51.0$ & $53.4$ && $52.1$ \\
            & NatADiff\textsuperscript{U} & & $\textcolor{red}{\mathbf{71.3^*}}$ & $\textcolor{BrickRed}{\underline{62.2}}$ & $\textcolor{red}{\mathbf{61.2}}$ & $\textcolor{red}{\mathbf{61.0}}$ & $\textcolor{red}{\mathbf{61.3}}$ && $\textcolor{BrickRed}{\underline{47.8}}$ & $\textcolor{red}{\mathbf{58.5}}$ & $\textcolor{BrickRed}{\underline{60.2}}$ & $\textcolor{red}{\mathbf{61.2}}$ && $\textcolor{red}{\mathbf{60.5}}$ \\
            \specialrule{1.2pt}{0pt}{0pt}
            \noalign{\vskip 0.2ex}
        \end{tabular}
    }
\end{table}

\section{Runtime comparison} \label{apdx:Runtime comparison}
We provide a runtime comparison of PGD \citep{Madry2019}, AutoAttack (AA) \citep{Croce2020}, NCF \citep{NCF_Yuan2022}, DiffAttack \citep{DiffAttack_Chen2025}, ACA \citep{ContentDiffusionAttack_Chen2023}, adversarial classifier guidance (AdvClass) \citep{Dai2024} and NatADiff. It is clear that the generative approach of NatADiff requires substantially greater runtime; however, this cost yields stronger adversarial samples that transfer more effectively across classifiers (see Table~\ref{tab:Classifier ASR and image quality}), and that are more resistant to adversarial purification (see Appendix~\ref{apdx:Resistance to adversarial defences}). We argue that the trade-off between runtime and state-of-the-art adversarial strength makes NatADiff a compelling attack strategy despite its slower speed.

\begin{table}[h]
    \centering
    \caption{\textbf{Time comparison} of adversarial attack methods.}
    \label{tab:Attack time comparison}
    \resizebox{0.35\textwidth}{!}{
        \begin{tabular}{@{}cc@{}}
            \specialrule{1.2pt}{0pt}{0pt}
            \noalign{\vskip 0.5ex}
            \multicolumn{1}{@{}c}{\multirow{2}{*}{Attack}} & \multicolumn{1}{c@{}}{\multirow{2}{*}{\shortstack{Average Runtime per\\Sample (seconds)}}} \\
            & \\
            \specialrule{1.2pt}{0pt}{0pt}
            \noalign{\vskip 0.5ex}
            PGD & $0.3$ \\
            AutoAttack & $0.7$ \\
            NCF & $6.9$ \\
            DiffAttack & $14.2$ \\
            ACA & $96.8$ \\
            AdvClass & $13.5$ \\
            NatADiff & $103.1$ \\
            \specialrule{1.2pt}{0pt}{0pt}
            \noalign{\vskip 0.2ex}
        \end{tabular}
    }
\end{table}

\section{User study}
NatADiff is a generative attack method designed to produce adversarial samples that lie near the decision boundary between the true and adversarial classes. However, as discussed in Appendix~\ref{apdx:Selection of mu}, the generated image may occasionally fully manifest the adversarial class, resulting in the loss of the intended ``true'' class. To mitigate this, we conservatively set the adversarial boundary guidance strength to $\mu = 0.2$. We conduct a human study to verify that this choice preserves the human-perceived class.

Accurate classification of ImageNet classes typically requires trained human annotators \citep{ImageNetBenchmark_Russakovsky2015, EvaluatingImageNet_Shankar2020}. Previous work has found that untrained annotators often exhibit significant class unawareness, particularly for fine-grained categories such as dog breeds \citep{ImageNetBenchmark_Russakovsky2015, EvaluatingImageNet_Shankar2020}. To address these limitations, we restrict our study to a curated subset of 148 easily recognisable ImageNet classes\footnote{Selected classes: 1, 9, 22, 47, 105, 130, 151, 152, 153, 154, 155, 156, 158, 159, 160, 161, 162, 163, 285, 286, 287, 288, 289, 290, 291, 292, 407, 425, 440, 448, 449, 462, 465, 468, 469, 470, 471, 472, 486, 493, 495, 496, 497, 502, 503, 505, 510, 514, 515, 516, 517, 520, 525, 526, 527, 532, 537, 555, 580, 607, 609, 663, 670, 671, 672, 701, 719, 734, 780, 787, 820, 909, 937, 963, 975, 999, 972, 973, 569, 933, 452, 958, 985, 900, 991, 400, 635, 920, 103, 482, 392, 843, 833, 259, 614, 757, 490, 715, 834, 222, 509, 374, 815, 927, 545, 531, 21, 303, 282, 132, 536, 913, 741, 624, 269, 547, 313, 640, 182, 849, 954, 524, 109, 657, 889, 884, 961, 94, 797, 817, 653, 134, 978, 803, 668, 436, 581, 519, 554, 942, 806, 410, 144, 693, 523, 26, 403, and 846.}.

We present 22 human participants with 60 randomly selected NatADiff samples generated using a ResNet-50 surrogate classifier under both targeted and untargeted attack settings. Each participant is asked whether the image contains the true class, the adversarial class, or neither. Participants reported seeing the true class in $91 \%$ of cases, the adversarial class in $7 \%$, and neither class in $2 \%$. NatADiff's label-flip rate of $9 \%$ is comparable to the $10 \%$ flip rate reported by \cite{Dai2024} for adversarial classifier guidance on the MNIST dataset \citep{Deng2012}. These results indicate that NatADiff reliably produces images that appear to belong to the true class for human observers, even while fooling classifiers with high probability. Moreover, in practical attack settings, an adversary could manually inspect generated images prior to deployment, further mitigating the impact of occasional class-flipped samples.

\clearpage
\section{Useful denoising diffusion results} \label{apdx:DiffusionMath}
This section outlines results necessary for working with denoising diffusion models. Citations of original authors are provided where applicable.

\subsection{Conditional forward distribution}
The following theorem describes the conditional forward distribution of the denoising diffusion model when conditioned on an arbitrary time $\tau$.

\begin{theorem}[Conditional Forward Distribution for Denosing Diffusion] \label{Thm:Conditional Forward}
    Let $\boldsymbol{x}_t \in \mathbb{R}^m$, $f(t) : \mathbb{R} \rightarrow \mathbb{R}$ and $g(t) : \mathbb{R} \rightarrow \mathbb{R}$ be continuous functions of $t$, and $\mathop{\cdot} d\boldsymbol{B}_t$ denote an It\^{o} integral with respect to the standard multi-dimensional Brownian motion process. Then the denoising diffusion model with forward process, \begin{equation}
        \label{eq:proof conditional forward diffusion}
        d\boldsymbol{x}_t = f(t)\boldsymbol{x}_t dt + g(t) \cdot d\boldsymbol{B}_t,
    \end{equation} admits a conditional forward distribution of \begin{equation}
        X_t | X_\tau \sim \mathcal{N} \left( \alpha(\tau, t) x_{\tau}, \ \beta(\tau, t)^2 I^{(m \times m)} \right) \ \forall \ t > \tau, \nonumber
    \end{equation}
    where $\alpha(\tau, t) = \exp \left( \int_{\tau}^t f(u) du \right)$, $\beta(\tau, t)^2 = \alpha(\tau, t)^2 \int_{\tau}^{t}\frac{g(u)^2}{\alpha(\tau, u)^2}du$, and $I^{(m \times m)}$ is the $m$-dimensional identity matrix. Additionally, it is understood that with slight abuse of notation $\alpha(t) = \alpha(0, t)$ and $\beta(t) = \alpha(0, t)$
\end{theorem}
\begin{proof}
    The diffusion in (\ref{eq:proof conditional forward diffusion}) has Stratonovich representation (see \citep{Pavliotis2014} for a treatment of It\^{o} and Stratonovich SDE formulations), \begin{equation}
        d\boldsymbol{x}_t = f(t)\boldsymbol{x}_t dt + g(t) \circ d\boldsymbol{B}_t, \nonumber
    \end{equation} where $\mathop{\circ} d\boldsymbol{B}_t$ denotes a Stratonovich integral with respect to the standard multi-dimensional Brownian motion process. Thus, the SDE can be solved in the usual manner:
    \begin{align*}
        d\boldsymbol{x}_t &= f(t)\boldsymbol{x}_t dt + g(t) \circ d\boldsymbol{B}_t \\
        \frac{d\boldsymbol{x}_t}{dt} &= f(t)\boldsymbol{x}_t + g(t) \circ \frac{d\boldsymbol{B}_t}{dt} \\
        \implies \left( e^{-\int_{\tau}^t f(u) du} \right) \frac{d\boldsymbol{x}_t}{dt} &= \left( e^{-\int_{\tau}^t f(u) du} \right) f(t)\boldsymbol{x}_t \\
        &\quad + \left( e^{-\int_{\tau}^t f(u) du} \right) g(t) \circ \frac{d\boldsymbol{B}_t}{dt} \quad \forall \ t > \tau \\
        \implies \left( e^{-\int_{\tau}^t f(u) du} \right) g(t) \circ \frac{d\boldsymbol{B}_t}{dt} &= \left( e^{-\int_{\tau}^t f(u) du} \right) \frac{d\boldsymbol{x}_t}{dt} - \left( e^{-\int_{\tau}^t f(u) du} \right) f(t)\boldsymbol{x}_t \\
        \implies \int_{\tau}^t \left( e^{-\int_{\tau}^v f(u) du} \right) g(v) \circ \frac{d\boldsymbol{B}_v}{dv} dv &= \int_{\tau}^t \left( e^{-\int_{\tau}^v f(u) du} \right) \frac{dx_v}{dv} - \left( e^{-\int_{\tau}^v f(u) du} \right) f(v)x_v dv \\
        \int_{\tau}^t \left( e^{-\int_{\tau}^v f(u) du} \right) g(v) \circ d\boldsymbol{B}_v &= \left . \left [ x_v \left( e^{-\int_{\tau}^v f(u) du} \right) \right ] \right \rvert_{v=\tau}^{v=t} \\
        \int_{\tau}^t \left( e^{-\int_{\tau}^v f(u) du} \right) g(v) \circ d\boldsymbol{B}_v &= \boldsymbol{x}_t \left( e^{-\int_{\tau}^t f(u) du} \right) - x_{\tau} \\
        \int_{\tau}^t \frac{g(v)}{\alpha(\tau, v)} \circ d\boldsymbol{B}_v &= \frac{\boldsymbol{x}_t}{\alpha(\tau, t)} - x_{\tau} \\
        \implies \boldsymbol{x}_t &= \alpha(\tau, t) x_{\tau} + \alpha(\tau, t) \int_{\tau}^t \frac{g(v)}{\alpha(\tau, v)} \circ d\boldsymbol{B}_v. \numberthis \label{eq:final diffusion strat sde}
    \end{align*}
    By rewriting (\ref{eq:final diffusion strat sde}) in its It\^{o} representation we have
    \begin{align*}
        \boldsymbol{x}_t &= \alpha(\tau, t) x_{\tau} + \alpha(\tau, t) \int_{\tau}^t \frac{g(v)}{\alpha(\tau, v)} \cdot d\boldsymbol{B}_v,
        \intertext{and as $\int_{\tau}^t \frac{g(u)}{\alpha(\tau, u)} \cdot d\boldsymbol{B}_u \sim \mathcal{N} \left(0, \ \int_{\tau}^{t}\frac{g(u)^2}{\alpha(\tau, u)^2}du \cdot I^{(m \times m)} \right)$, it follows that}
        \boldsymbol{x}_t | \boldsymbol{x}_\tau &\sim \mathcal{N} \left( \alpha(\tau, t) x_{\tau}, \ \alpha(\tau, t)^2 \int_{\tau}^{t}\frac{g(u)^2}{\alpha(\tau, u)^2}du \cdot I^{(m \times m)} \right).
    \end{align*}
\end{proof}

\subsection{Conditional forward alternate parameterisation}
When implementing time-travel sampling \citep{Lugmayr2022} we require access to the conditional forward distribution, $p(\boldsymbol{x}_t | \boldsymbol{x}_{\tau})$. However, it is frequently the case that diffusion schemes are formulated with respect to the full forward distribution, $p(\boldsymbol{x}_t | \boldsymbol{x}_0)$, and some proposed method of sampling the reverse-time diffusion \citep{Anderson1982}, or solving the flow ODE \citep{Song2021}. Thus, we provide a simple result to derive the conditional forward distribution, $p(\boldsymbol{x}_t | \boldsymbol{x}_{\tau})$, from the parameterisation of the full forward, $p(\boldsymbol{x}_t | \boldsymbol{x}_0)$.

\begin{lemma}[Conditional Forward Alternate Parameterisation] \label{Lma:Conditional Forward Alternate Parameterisation}
    Given the diffusion formulation in Theorem \ref{Thm:Conditional Forward}, then the conditional forward distribution can alternately be expressed as \begin{equation}
        X_t | X_\tau \sim \mathcal{N} \left( a x_{\tau}, \ b^2 I^{(m \times m)} \right) \ \forall \ t > \tau, \nonumber
    \end{equation}
    where \begin{equation}
        X_t | X_0 \sim \mathcal{N} \left( \alpha(0, t) x_{0}, \ \beta(0, t)^2 I^{(m \times m)} \right) \ \forall \ t > 0, \nonumber
    \end{equation} $a = \frac{\alpha(0, t)}{\alpha(0, \tau)}$, $b^2 = \beta(0, t)^2 - \left( a \beta(0, \tau) \right)^2$, , and $I^{(m \times m)}$ is the $m$-dimensional identity matrix. Additionally, it is understood that with slight abuse of notation $\alpha(t) = \alpha(0, t)$ and $\beta(t) = \alpha(0, t)$
\end{lemma}
\begin{proof}
    We need to show that $a = \alpha(\tau, t) = \exp \left( \int_{\tau}^t f(u) du \right)$ and $b^2 = \beta(\tau, t)^2 = \alpha(\tau, t)^2 \int_{\tau}^{t}\frac{g(u)^2}{\alpha(\tau, u)^2}du$ as per Theorem \ref{Thm:Conditional Forward}. It follows that
    \begin{align*}
        a &= \frac{\alpha(0, t)}{\alpha(0, \tau)} \\
        &= \frac{\exp \left( \int_0^t f(u) du \right)}{\exp \left( \int_0^\tau f(u) du \right)} \\
        &= \exp \left( \int_0^t f(u) du - \int_0^\tau f(u) du \right) \\
        &= \exp \left( \int_{\tau}^t f(u) du \right) \\
        &= \alpha(\tau, t), \numberthis \label{eq:conditional forward alternate parameterisation a}
        \intertext{and}
        b^2 &= \beta(0, t)^2 - \left( a \beta(0, \tau) \right)^2 \\
        &= \beta(0, t)^2 - a^2 \beta(0, \tau)^2 \\
        &= \alpha(0, t)^2 \int_{0}^{t} \frac{g(u)^2}{\alpha(0, u)^2}du - \frac{\alpha(0, t)^2}{\alpha(0, \tau)^2} \alpha(0, \tau)^2 \int_{0}^{\tau}\frac{g(u)^2}{\alpha(0, u)^2} du \\ 
        &= \alpha(0, t)^2 \left[ \int_{0}^{t} \frac{g(u)^2}{\alpha(0, u)^2}du - \int_{0}^{\tau}\frac{g(u)^2}{\alpha(0, u)^2} du \right] \\
        &= \alpha(0, t)^2 \int_{\tau}^{t} \frac{g(u)^2}{\alpha(0, u)^2}du \\
        &= \alpha(0, t)^2 \int_{\tau}^{t} \frac{g(u)^2}{\alpha(0, \tau)^2 \alpha(\tau, u)^2}du \quad \text{as $\alpha(0, u) = \alpha(0, \tau) \alpha(\tau, u)$ by (\ref{eq:conditional forward alternate parameterisation a})} \\
        &= \frac{\alpha(0, t)^2}{\alpha(0, \tau)^2} \int_{\tau}^{t} \frac{g(u)^2}{\alpha(\tau, u)^2}du \\
        &= \alpha(\tau, t)^2 \int_{\tau}^{t} \frac{g(u)^2}{\alpha(\tau, u)^2}du \\
        &=\beta(\tau, t)^2. 
    \end{align*}
\end{proof}

\subsection{Score-model link}
The following Score-Model Link theorem is based on \citeauthor{Karras2022}'s \citep{Karras2022} argument. However, we provide a minor extension by conditioning on measurable sets taken from the sigma-algebra of an auxiliary random variable, $Y$. For additional treatments see \citep{Karras2022, Hyvarinen2005, Vincent2011}.
\begin{theorem}[Score-Model Link] \label{Thm:Score-Model Link}
    Let $\boldsymbol{x}_t \in \mathbb{R}^m$, $f(t) : \mathbb{R} \rightarrow \mathbb{R}$ and $g(t) : \mathbb{R} \rightarrow \mathbb{R}$ be continuous functions of $t$, and $\mathop{\cdot} d\boldsymbol{B}_t$ denote an It\^{o} integral with respect to the standard multi-dimensional Brownian motion process. Suppose that $\boldsymbol{x}_t$ evolves according to the diffusion, \begin{equation}
        \label{eq:proof score-model conditional forward diffusion}
        d\boldsymbol{x}_t = f(t)\boldsymbol{x}_t dt + g(t) \cdot d\boldsymbol{B}_t,
    \end{equation} with observed initial data distribution, $p_{\text{data}}(\boldsymbol{x}_0 | y \in \xi)$, where $\xi$ is taken to be an arbitrary element of the sigma-algebra, $\mathcal{Y}$, associated with the random variable\footnote{Note that this is a slight abuse of notation. We are assuming that $(\Omega_Y, \mathcal{Y}, P)$ is a probability space and $Y: \Omega_Y \rightarrow \Omega_Y$ a random variable such that $Y(\omega) = \omega \ \forall \ \omega \in \Omega_Y$. That is to say, we do not need to take the pre-image when crafting probability statements.} $Y$, i.e., $\xi \in \mathcal{Y}$. 
    
    Define \begin{equation} 
        \hat{\boldsymbol{x}}_0(\boldsymbol{x}_t, t, \xi) = \frac{\boldsymbol{x}_t - \beta(t) \boldsymbol{\epsilon}_{\theta}(\boldsymbol{x}_t, t, \xi)}{\alpha(t)},
    \end{equation}
    where $\alpha(t) = \exp \left( \int_0^t f(u) du \right)$, $\beta(t)^2 = \alpha(t)^2 \int_0^t \frac{g(u)^2}{\alpha(u)^2}du$, and $\boldsymbol{\epsilon}_{\theta} : \mathbb{R}^m \times \mathbb{R} \times \mathcal{Y} \rightarrow \mathbb{R}^m$ is a model parameterised by $\theta \in \Theta$ with sufficient capacity such that the Universal Approximation Theorem \citep{Cybenko1989, Hornik1991} holds for all $\xi \in \mathcal{Y}$. Then if $\beta(t)^2 > 0 \ \forall \ t \in (0, T]$, the following statements are true:
    \begin{gather}
        \intertext{1.}
        \nabla_{\boldsymbol{x}_t} \log(p(\boldsymbol{x}_t | y \in \xi)) = - \frac{1}{\beta(t)} \boldsymbol{\epsilon}_{\theta^\star}(\boldsymbol{x}_t, t, \xi) \quad \forall \ t \in (0, T], \ \boldsymbol{x}_t \in \mathbb{R}^m, \ \xi \in \mathcal{Y}; \numberthis \label{eq:score-model equation 1}
        \intertext{2.}
        \boldsymbol{\epsilon}_{\theta^{\star}}(\boldsymbol{x}_t, t, \xi) = \frac{{\E}_{\boldsymbol{x}_0 \sim p_{\text{data}}(\boldsymbol{x}_0 | y \in \xi)} \left[ \frac{1}{\beta(t)} \bigl(\boldsymbol{x}_t - \alpha(t) \boldsymbol{x}_0 \bigr) p(\boldsymbol{x}_t | \boldsymbol{x}_0) \right]}{p(\boldsymbol{x}_t | y \in \xi)} \quad \forall \ t \in (0, T], \numberthis \label{eq:score-model equation 2} \\
        \intertext{$\boldsymbol{x}_t \in \mathbb{R}^m, \ \xi \in \mathcal{Y}$; where}
        \theta^{\star} \triangleq \underset{\theta}{\textup{arg} \min} \ {\E}_{\boldsymbol{x}_0, \boldsymbol{x}_t, t \sim p(\boldsymbol{x}_0, \boldsymbol{x}_t, t | y \in \xi)} \left[ \left \lVert \boldsymbol{x}_0 - \hat{\boldsymbol{x}}_0(\boldsymbol{x}_t, t, \xi) \right \rVert_2^2 \right], \numberthis \label{eq:score-model optimisation} \\
        \intertext{$p(\boldsymbol{x}_0$, $\boldsymbol{x}_t, t | y \in \xi) = p(\boldsymbol{x}_t | \boldsymbol{x}_0) p_{\text{data}}(\boldsymbol{x}_0 | y \in \xi) p(t)$, and $p(t) \triangleq \frac{1}{T}$.}
    \end{gather}
\end{theorem}
\begin{proof}
    Let $\{ \boldsymbol{x}_0^{(1)}, \boldsymbol{x}_0^{(2)}, \dots, \boldsymbol{x}_0^{(N)} \}$ and $\{ y^{(1)}, y^{(2)}, \dots, y^{(N)} \}$ denote observed values of $\boldsymbol{x}_0$ and $Y$. Then the data density function is given by:
    \begin{equation}
        \label{eq:proof score-model initial data distribution}
        p_{\text{data}}(\boldsymbol{x}_0 | y \in \xi) = \frac{ \sum_{i=1}^N \delta(\boldsymbol{x}_0 - \boldsymbol{x}_0^{(i)}) \mathds{1}_{\{ y^{(i)} \in \xi \}}}{\sum_{i=1}^N \mathds{1}_{\{ y^{(i)} \in \xi \}}},
    \end{equation}
    where $\delta(\cdot)$ and $\mathds{1}_{\{ \cdot \}}$ denote the Dirac delta and indicator functions, respectively. The forward diffusion process in (\ref{eq:proof score-model conditional forward diffusion}) is independent of $Y$, and thus,
    \begin{align*}
        p(\boldsymbol{x}_t | y \in \xi) &= \int_{\Omega} p(\boldsymbol{x}_t | \boldsymbol{x}_0) p_{\text{data}}(\boldsymbol{x}_0 | y \in \xi) d\boldsymbol{x}_0 \\
        &= \frac{ \sum_{i=1}^N p(\boldsymbol{x}_t | \boldsymbol{x}_0^{(i)}) \mathds{1}_{\{ y^{(i)} \in \xi \}}}{\sum_{i=1}^N \mathds{1}_{\{ y^{(i)} \in \xi \}}}. \numberthis \label{eq:proof score-model data distribution}
        \intertext{Substituting (\ref{eq:proof score-model data distribution}) into the LHS of (\ref{eq:score-model equation 1}),}
        \nabla_{\boldsymbol{x}_t} \log(p(\boldsymbol{x}_t | y \in \xi)) &= \nabla_{\boldsymbol{x}_t} \log \left( \frac{ \sum_{i=1}^N p(\boldsymbol{x}_t | \boldsymbol{x}_0^{(i)}) \mathds{1}_{\{ y^{(i)} \in \xi \}}}{\sum_{i=1}^N \mathds{1}_{\{ y^{(i)} \in \xi \}}} \right) \\
        &= \nabla_{\boldsymbol{x}_t} \log \left( \sum_{i=1}^N p(\boldsymbol{x}_t | \boldsymbol{x}_0^{(i)}) \mathds{1}_{\{ y^{(i)} \in \xi \}} \right)  \\
        &= \frac{\sum_{i=1}^N \nabla_{\boldsymbol{x}_t} p(\boldsymbol{x}_t | \boldsymbol{x}_0^{(i)}) \mathds{1}_{\{ y^{(i)} \in \xi \}}}{\sum_{i=1}^N p(\boldsymbol{x}_t | \boldsymbol{x}_0^{(i)}) \mathds{1}_{\{ y^{(i)} \in \xi \}}}.
        \intertext{By Theorem \ref{Thm:Conditional Forward}, $p(\boldsymbol{x}_t|\boldsymbol{x}_0) = \mathcal{N} \left( \alpha(t) \boldsymbol{x}_0, \ \beta(t)^2 I^{(m \times m)} \right) \implies \nabla_{\boldsymbol{x}_t} p(\boldsymbol{x}_t | \boldsymbol{x}_0) = -\frac{1}{\beta(t)^2}(\boldsymbol{x}_t - \alpha(t) \boldsymbol{x}_0)p(\boldsymbol{x}_t | \boldsymbol{x}_0)$, thus,}
        \nabla_{\boldsymbol{x}_t} \log(p(\boldsymbol{x}_t | y \in \xi)) &= -\frac{1}{\beta(t)^2} \frac{\sum_{i=1}^N (\boldsymbol{x}_t - \alpha(t) \boldsymbol{x}_0^{(i)})p(\boldsymbol{x}_t | \boldsymbol{x}_0^{(i)}) \mathds{1}_{\{ y^{(i)} \in \xi \}}}{\sum_{i=1}^N p(\boldsymbol{x}_t | \boldsymbol{x}_0^{(i)}) \mathds{1}_{\{ y^{(i)} \in \xi \}}}. \numberthis \label{eq:proof log statement}
    \end{align*}
    Now we consider the optimisation problem in (\ref{eq:score-model optimisation}).
    \begin{align*}
        &\mathcal{L}(\xi; \theta) = {\E}_{\boldsymbol{x}_0, \boldsymbol{x}_t, t \sim p(\boldsymbol{x}_0, \boldsymbol{x}_t, t | y \in \xi)} \left[ \left \lVert \boldsymbol{x}_0 - \hat{\boldsymbol{x}}_0(\boldsymbol{x}_t, t, \xi) \right \rVert_2^2 \right] \\
        &= {\E}_{t \sim p(t)}\left[{\E}_{\boldsymbol{x}_0 \sim p(\boldsymbol{x}_0 | y \in \xi)}\left[{\E}_{\boldsymbol{x}_t \sim p(\boldsymbol{x}_t | \boldsymbol{x}_0)}\left[ \left \lVert \boldsymbol{x}_0 - \hat{\boldsymbol{x}}_0(\boldsymbol{x}_t, t, \xi) \right \rVert_2^2 \right]\right]\right] \\
        &= {\E}_{t \sim p(t)}\left[{\E}_{\boldsymbol{x}_0 \sim p(\boldsymbol{x}_0 | y \in \xi)}\left[ \int_{\Omega_{\boldsymbol{x}_t}} \left \lVert \boldsymbol{x}_0 - \hat{\boldsymbol{x}}_0(\boldsymbol{x}_t, t, \xi) \right \rVert_2^2 \ p(\boldsymbol{x}_t | \boldsymbol{x}_0)d\boldsymbol{x}_t \right]\right] \\
        &= {\E}_{t \sim p(t)}\left[ \int_{\Omega_{\boldsymbol{x}_0}} \int_{\Omega_{\boldsymbol{x}_t}} \left \lVert \boldsymbol{x}_0 - \hat{\boldsymbol{x}}_0(\boldsymbol{x}_t, t, \xi) \right \rVert_2^2 \ p(\boldsymbol{x}_t | \boldsymbol{x}_0)d\boldsymbol{x}_t \ p(\boldsymbol{x}_0 | y \in \xi) d\boldsymbol{x}_0 \right] \\
        &= {\E}_{t \sim p(t)}\left[ \int_{\Omega_{\boldsymbol{x}_0}} \int_{\Omega_{\boldsymbol{x}_t}} \left \lVert \boldsymbol{x}_0 - \hat{\boldsymbol{x}}_0(\boldsymbol{x}_t, t, \xi) \right \rVert_2^2 \ p(\boldsymbol{x}_t | \boldsymbol{x}_0)d\boldsymbol{x}_t \ \left( \frac{ \sum_{i=1}^N \delta(\boldsymbol{x}_0 - \boldsymbol{x}_0^{(i)}) \mathds{1}_{\{ y^{(i)} \in \xi \}}}{\sum_{i=1}^N \mathds{1}_{\{ y^{(i)} \in \xi \}}} \right) d\boldsymbol{x}_0 \right] \\
        &= {\E}_{t \sim p(t)}\left[ \frac{1}{\sum_{i=1}^N \mathds{1}_{\{ y^{(i)} \in \xi \}}} \sum_{i=1}^N \int_{\Omega_{\boldsymbol{x}_t}} \left \lVert \boldsymbol{x}_0^{(i)} - \hat{\boldsymbol{x}}_0(\boldsymbol{x}_t, t, \xi) \right \rVert_2^2 \ p(\boldsymbol{x}_t | \boldsymbol{x}_0^{(i)})d\boldsymbol{x}_t \ \mathds{1}_{\{ y^{(i)} \in \xi \}} \right] \\
        &= {\E}_{t \sim p(t)}\left[ \frac{1}{\sum_{i=1}^N \mathds{1}_{\{ y^{(i)} \in \xi \}}} \int_{\Omega_{\boldsymbol{x}_t}} \sum_{i=1}^N p(\boldsymbol{x}_t | \boldsymbol{x}_0^{(i)}) \mathds{1}_{\{ y^{(i)} \in \xi \}} \left \lVert \boldsymbol{x}_0^{(i)} - \hat{\boldsymbol{x}}_0(\boldsymbol{x}_t, t, \xi) \right \rVert_2^2 d\boldsymbol{x}_t \right] \\
        &= \frac{1}{T} \frac{1}{\sum_{i=1}^N \mathds{1}_{\{ y^{(i)} \in \xi \}}} \int_0^T \int_{\Omega_{\boldsymbol{x}_t}} \sum_{i=1}^N p(\boldsymbol{x}_t | \boldsymbol{x}_0^{(i)}) \mathds{1}_{\{ y^{(i)} \in \xi \}} \left \lVert \boldsymbol{x}_0^{(i)} - \hat{\boldsymbol{x}}_0(\boldsymbol{x}_t, t, \xi) \right \rVert_2^2 d\boldsymbol{x}_t dt \\
        &= \frac{1}{T} \frac{1}{\sum_{i=1}^N \mathds{1}_{\{ y^{(i)} \in \xi \}}} \int_0^T \int_{\Omega_{\boldsymbol{x}_t}} \underbrace{\sum_{i=1}^N p(\boldsymbol{x}_t | \boldsymbol{x}_0^{(i)}) \mathds{1}_{\{ y^{(i)} \in \xi \}} \left \lVert \boldsymbol{x}_0^{(i)} - \frac{\boldsymbol{x}_t - \beta(t) \boldsymbol{\epsilon}_{\theta}(\boldsymbol{x}_t, t, \xi)}{\alpha(t)} \right \rVert_2^2}_{\ell(\boldsymbol{x}_t, t, \xi; \theta)} d\boldsymbol{x}_t dt \numberthis \label{eq:score-model loss function}
    \end{align*}
    To minimise (\ref{eq:score-model loss function}) with respect to $\theta$, it suffices to find $\theta$ such that $\ell(\boldsymbol{x}_t, t, \xi; \theta)$ is minimised for each combination of $\boldsymbol{x}_t$ and $t$. That is to say, we find the optimal value of $\boldsymbol{\epsilon}_{\theta}(\boldsymbol{x}_t, t, \xi)$ for each combination of $\boldsymbol{x}_t$ and $t$. Furthermore, $\ell(\boldsymbol{x}_t, t, \xi; \theta)$ constitutes a convex optimisation problem with respect to $\boldsymbol{\epsilon}_{\theta}(\boldsymbol{x}_t, t, \xi)$. Thus,
    \begin{align*}
        \frac{\partial \ell}{\partial \boldsymbol{\epsilon}_{\theta}} = 0 &= \sum_{i=1}^N \frac{2\beta(t)}{\alpha(t)} p(\boldsymbol{x}_t | \boldsymbol{x}_0^{(i)}) \mathds{1}_{\{ y^{(i)} \in \xi \}} \left ( \boldsymbol{x}_0^{(i)} - \frac{\boldsymbol{x}_t - \beta(t) \boldsymbol{\epsilon}_{\theta}(\boldsymbol{x}_t, t, \xi)}{\alpha(t)} \right ) \\
        &= \sum_{i=1}^N p(\boldsymbol{x}_t | \boldsymbol{x}_0^{(i)}) \mathds{1}_{\{ y^{(i)} \in \xi \}} \left ( \alpha(t) \boldsymbol{x}_0^{(i)} - \boldsymbol{x}_t + \beta(t) \boldsymbol{\epsilon}_{\theta}(\boldsymbol{x}_t, t, \xi) \right ) \\
        \implies \sum_{i=1}^N (\boldsymbol{x}_t - \alpha(t) \boldsymbol{x}_0^{(i)}) &p(\boldsymbol{x}_t | \boldsymbol{x}_0^{(i)}) \mathds{1}_{\{ y^{(i)} \in \xi \}} = \sum_{i=1}^N \beta(t) \boldsymbol{\epsilon}_{\theta}(\boldsymbol{x}_t, t, \xi) p(\boldsymbol{x}_t | \boldsymbol{x}_0^{(i)}) \mathds{1}_{\{ y^{(i)} \in \xi \}} \\
        \implies \boldsymbol{\epsilon}_{\theta}^{\star}(\boldsymbol{x}_t, t, \xi) &= \frac{1}{\beta(t)} \frac{\sum_{i=1}^N (\boldsymbol{x}_t - \alpha(t) \boldsymbol{x}_0^{(i)}) p(\boldsymbol{x}_t | \boldsymbol{x}_0^{(i)}) \mathds{1}_{\{ y^{(i)} \in \xi \}}}{\sum_{i=1}^N p(\boldsymbol{x}_t | \boldsymbol{x}_0^{(i)}) \mathds{1}_{\{ y^{(i)} \in \xi \}}} \numberthis \label{eq:proof optimal epsilon} \\
        &= \frac{1}{\beta(t)} \frac{ \left( \frac{\sum_{i=1}^N (\boldsymbol{x}_t - \alpha(t) \boldsymbol{x}_0^{(i)}) p(\boldsymbol{x}_t | \boldsymbol{x}_0^{(i)}) \mathds{1}_{\{ y^{(i)} \in \xi \}}}{\sum_{i=1}^N \mathds{1}_{\{ y^{(i)} \in \xi \}}} \right) }{ \left( \frac{\sum_{i=1}^N p(\boldsymbol{x}_t | \boldsymbol{x}_0^{(i)}) \mathds{1}_{\{ y^{(i)} \in \xi \}}}{\sum_{i=1}^N \mathds{1}_{\{ y^{(i)} \in \xi \}}} \right) }\\
        &= \frac{{\E}_{\boldsymbol{x}_0 \sim p_{\text{data}}(\boldsymbol{x}_0 | y \in \xi)} \left[ \frac{1}{\beta(t)} \bigl(\boldsymbol{x}_t - \alpha(t) \boldsymbol{x}_0 \bigr) p(\boldsymbol{x}_t | \boldsymbol{x}_0) \right]}{p(\boldsymbol{x}_t | y \in \xi)}, \numberthis \label{eq:proof expected epsilon}
    \end{align*}
    $\forall \ \boldsymbol{x}_t \in \mathbb{R}^m, \ t \in (0, T], \ \xi \in \mathcal{Y}$. As $f(t)$ and $g(t)$ are both continuous functions then $\alpha(t)$ and $\beta(t)$ are also continuous. It follows that (\ref{eq:proof optimal epsilon}) is continuous with respect to $\boldsymbol{x}_t$ and $t$, as it is the sum of continuous functions and $\beta(t)^2 > 0 \ \forall \ t \in (0, T]$. Thus, the Universal Approximation Theorem \citep{Cybenko1989, Hornik1991} holds and there exists a $\theta^{\star} \in \Theta$ such that $\boldsymbol{\epsilon}_{\theta^{\star}}(\boldsymbol{x}_t, t, \xi) = \boldsymbol{\epsilon}_{\theta}^{\star}(\boldsymbol{x}_t, t, \xi) \ \forall \ \boldsymbol{x}_t \in \mathbb{R}^m, \ t \in (0, T], \ \xi \in \mathcal{Y}$. Finally, by comparing (\ref{eq:proof log statement}) and (\ref{eq:proof optimal epsilon}) we observe that,
    \[
        \nabla_{\boldsymbol{x}_t} \log(p(\boldsymbol{x}_t | y \in \xi)) = -\frac{1}{\beta(t)} \boldsymbol{\epsilon}_{\theta^{\star}}(\boldsymbol{x}_t, t, \xi),
    \]
    which proves (\ref{eq:score-model equation 1}), and (\ref{eq:score-model equation 2}) follows from (\ref{eq:proof expected epsilon}).
\end{proof}
It is worth noting that Theorem \ref{Thm:Score-Model Link} unifies the score-model link between conditional and unconditional models. That is to say,
\begin{align*}
    \nabla_{\boldsymbol{x}_t} \log(p(\boldsymbol{x}_t)) &= \nabla_{\boldsymbol{x}_t} \log(p(\boldsymbol{x}_t | y \in \Omega_y)) \\
    &= -\frac{1}{\beta(t)} \boldsymbol{\epsilon}_{\theta^{\star}}(\boldsymbol{x}_t, t, \Omega_y).
\end{align*}

\end{document}